\documentclass[11pt]{article}
\usepackage{fullpage}
\usepackage{latexsym}
\usepackage{amsmath}
\usepackage{amssymb}
\usepackage{amsthm}
\usepackage{epsfig}
\usepackage{xcolor}
\usepackage{graphicx}
\usepackage{bm}
\usepackage{enumitem}

\usepackage{color}
\usepackage{amsmath,amssymb}

\newtheorem{theorem}{Theorem}
\newtheorem{lemma}{Lemma}
\newtheorem{proposition}{Proposition}

\newtheorem{corollary}{Corollary}
\newtheorem{remark}{Remark}[section]

\usepackage{chngcntr}
\usepackage{apptools}

\usepackage{hyperref}
\hypersetup{
    colorlinks,
    linkcolor={blue!80!black},
    citecolor={green!50!black},
}
\colorlet{linkequation}{blue}

\usepackage{mathtools}
\renewcommand{\eqref}[1]{(\ref{#1})}

\usepackage[mathscr]{euscript}

\DeclareSymbolFont{rsfs}{U}{rsfs}{m}{n}
\DeclareSymbolFontAlphabet{\mathscrsfs}{rsfs}

\newcommand{\bea}{\begin{eqnarray}}
\newcommand{\eea}{\end{eqnarray}}
\newcommand{\<}{\langle}
\renewcommand{\>}{\rangle}
\newcommand{\E}{{\mathbb E}}

\def\eps{{\varepsilon}}

\def\id{{\boldsymbol{I}}}

\def\supp{{\rm supp}}

\def\cuP{\mathscrsfs{P}}

\def\hrho{\hat{\rho}}

\def\btheta{{\boldsymbol{\theta}}}

\def\bxi{{\boldsymbol{\xi}}}
\def\bfe{{\boldsymbol{e}}}

\def\oR{\overline{R}}

\def\bZ{{\boldsymbol{Z}}}
\def\bSigma{{\boldsymbol{\Sigma}}}
\def\bP{{\boldsymbol{P}}}
\def\bPp{{\boldsymbol{P}}_{\perp}}

\def\bQ{{\boldsymbol{Q}}}

\def\bmu{{\boldsymbol{\mu}}}
\def\bg{{\boldsymbol{g}}}

\def\bzero{{\mathbf 0}}

\def\cF{{\mathcal F}}

\def\cC{{\mathcal C}}

\def\hy{\hat{y}}
\def\cO{{\mathcal O}}

\def\sUB{\mbox{\tiny\rm UB}}
\def\op{\mbox{\tiny\rm op}}
\def\sLip{\mbox{\tiny\rm Lip}}
\def\good{\mbox{\tiny\rm good}}
\def\relu{\mbox{\tiny\rm ReLU}}
\def\unif{\mbox{\tiny\rm unif}}

\def\naturals{{\mathbb N}}
\def\reals{{\mathbb R}}

\def\normal{{\sf N}}

\def\sT{{\sf T}}

\def\bv{{\boldsymbol{v}}}
\def\bz{{\boldsymbol{z}}}
\def\bx{{\boldsymbol{x}}}
\def\ba{{\boldsymbol{a}}}
\def\bb{{\boldsymbol{b}}}
\def\bA{\boldsymbol{A}}
\def\bB{\boldsymbol{B}}
\def\bJ{\boldsymbol{J}}
\def\bH{\boldsymbol{H}}

\def\de{{\rm d}}
\def\bX{\boldsymbol{X}}
\def\bY{\boldsymbol{Y}}
\def\bW{\boldsymbol{W}}
\def\prob{{\mathbb P}}
\def\E{{\mathbb E}}

\def\<{\langle}
\def\>{\rangle}
\def\Tr{{\sf Tr}}

\def\Ball{{\sf B}}
\def\argmax{{\arg\!\max}}

\def\ed{\stackrel{{\rm d}}{=}}

\def\cV{{\cal V}}
\def\cL{{\cal L}}

\def\by{{\boldsymbol{y}}}
\def\bw{{\boldsymbol{w}}}

\def\cP{{\mathcal P}}
\def\cE{{\mathcal E}}

\def\txi{\tilde{\xi}}

\def\blambda{{\boldsymbol{\lambda}}}

\def\bphi{{\boldsymbol{\varphi}}}
\def\bu{{\boldsymbol{u}}}

\def\Var{{\rm Var}}

\def\enter{\mbox{\tiny\rm enter}}
\def\exit{\mbox{\tiny\rm exit}}
\def\return{\mbox{\tiny\rm return}}
\def\sBL{\mbox{\tiny\rm BL}}
\def\rad{\overline{\rho}}

\def\bfone{{\boldsymbol 1}}
\def\bF{{\boldsymbol F}}
\def\bG{{\boldsymbol G}}
\def\bR{{\boldsymbol R}}

\def\onu{\overline{\nu}}
\def\orh{\overline{\rho}}

\def\obtheta{\overline{\boldsymbol \theta}}
\def\tbtheta{\tilde{\boldsymbol \theta}}

\def\rP{{\rm P}}

\def\obX{\overline{\boldsymbol X}}
\def\bdelta{{\boldsymbol \delta}}
\def\ocD{{\overline {\mathscrsfs{P}}}}
\def\cD{{{\mathcal D}}}
\def\tcL{{\tilde {\mathcal L}}}
\def\ent{{\rm Ent}}

\def\trho{\tilde{\rho}}

\def\tU{\tilde{U}}
\def\tV{\tilde{V}}
\def\cA{{\mathcal A}}
\def\cI{{\mathcal I}}
\def\cS{{\mathcal S}}
\def\bn{{\boldsymbol n}}
\def\cP{{\mathcal P}}
\def\cB{{\mathcal B}}
\def\br{{\boldsymbol r}}
\def\argmin{{\rm arg\,min}}
\def\db{{\bar d}}
\def\cK{{\mathcal K}}
\def\ess{{\rm ess}}
\renewcommand{\l}{\vert}

\def\balpha{{\boldsymbol \alpha}}
\def\setE{{E}}
\def\setS{{S}}

\def\err{{\sf err}}
\def\lambdaknot{\lambda_0}
\def\barR{\overline R}
\def\star{*}

\begin{document}

\title{A Mean Field View of the Landscape of Two-Layer \\Neural Networks}

\author{Song Mei\thanks{Institute for Computational and Mathematical Engineering, Stanford University}\and  Andrea Montanari\thanks{Department of Electrical Engineering and
  Department of Statistics,    Stanford University}
\and Phan-Minh Nguyen\thanks{Department of Electrical Engineering,    Stanford University}}


\maketitle

\begin{abstract}
Multi-layer neural networks are among the most powerful models in machine learning,
yet the fundamental reasons for this success defy mathematical understanding.
Learning a neural network requires to optimize a  non-convex  high-dimensional 
objective (risk function), a problem which is usually attacked using stochastic gradient descent (SGD). 
Does SGD converge to a global optimum of the risk or only to a local optimum? In the first case,
does this happen because local minima are absent, or because SGD somehow avoids them? In the
second, why do local minima reached by SGD have good generalization properties?

In this paper we consider a simple case, namely  two-layers neural networks,
and prove that --in a suitable scaling limit--  SGD dynamics is captured by a certain non-linear 
partial differential equation (PDE) that we call \emph{distributional
  dynamics} (DD). We then consider several specific examples, and show
how  DD can be used to prove convergence of SGD to networks with 
nearly-ideal  generalization error. This description allows to `average-out' some of the complexities
of the landscape of neural networks, and can be used to prove a general convergence result for noisy SGD.
\end{abstract}

{
\hypersetup{linkcolor=black}
  \tableofcontents
}

\section{Introduction}

Multi-layer neural networks are one of the oldest approaches to statistical machine learning, dating back
at least to the 1960's \cite{rosenblatt1962principles}. Over the last ten years, under the impulse of increasing computer power and 
larger data availability,  they have emerged as a powerful tool for a wide variety of learning tasks \cite{krizhevsky2012imagenet,goodfellow2016deep}.

In this paper we focus on the classical setting of supervised learning, whereby we are given data
points $(\bx_i,y_i)\in\reals^d\times \reals$, indexed by $i\in\naturals$, which are assumed to be independent and identically
distributed from an unknown distribution $\prob$ on $\reals^d\times \reals$. Here $\bx_i\in\reals^d$ is a feature vector
(e.g. a set of descriptors of an image), and $y_i\in\reals$ is a label (e.g. labeling the object in the image).
Our objective is to model the dependence of the label $y_i$ on the feature vector $\bx_i$ in order to assign labels to previously 
unlabeled examples.
In a two-layers neural network, this dependence is modeled as 
\begin{equation}
\hy(\bx;\btheta) = \frac{1}{N}\sum_{i=1}^{N}\sigma_*(\bx;\btheta_i)\, .
\end{equation}
Here $N$ is the number of hidden units (neurons),  $\sigma_*:\reals^d\times\reals^D\to\reals$ is
an activation function, and  $\btheta_i\in \reals^D$ are parameters, which we collectively denote by $\btheta=(\btheta_1,\dots,\btheta_N)$.
The factor $(1/N)$ is introduced for convenience and can be eliminated by redefining the activation.
Often $\btheta_i=(a_i,b_i,\bw_i)$ 
and
\begin{equation}
\sigma_*(\bx;\btheta_i) = a_i\, \sigma(\<\bw_i,\bx\>+b_i)\, ,\label{eq:OutputShift}
\end{equation}
for some $\sigma:\reals\to\reals$.
Ideally, the parameters $\btheta=(\btheta_i)_{i\le N}$ should be chosen as to minimize the risk (generalization error) 
$R_{N}(\btheta) =\E \{\ell(y,\hy(\bx;\btheta))\}$ where $\ell:\reals\times\reals\to\reals$ is a certain loss function.
For the sake of simplicity, we will focus on the square loss $\ell(y,\hy) = (y-\hy)^2$ but more general choices can be treated along the same lines.

In practice, the parameters of neural networks are learned by stochastic gradient descent \cite{robbins1951stochastic} (SGD) or its variants. 
In the present case, this amounts to the iteration 
\begin{equation}
\btheta_i^{k+1} = \btheta_i^{k} + 2s_k\,\big(y_k-\hy(\bx_k;\btheta^k)\big)\, \nabla_{\btheta_i}\sigma_*(\bx_k;\btheta^k_i)\, . \label{eq:First_SGD}
\end{equation}
Here $\btheta^k=(\btheta^k_i)_{i\le N}$ denotes the parameters after $k$ iterations, $s_k$ is a step size, and $(\bx_k,y_k)$ is the $k$-th example.
Throughout the paper, we make the following assumption:

\vspace{0.1cm}

\noindent\emph{One-pass assumption.} Training examples are never revisited. Equivalently, $\{(\bx_k,y_k)\}_{k\ge 1}$ are
i.i.d. $(\bx_k,y_k)\sim\prob$.

\vspace{0.1cm}

In large scale applications, this is not far from truth: the data is so large that each example
is visited at most a few times \cite{bottou2010large}. Further, theoretical guarantees suggest
that there is limited advantage to be gained from multiple passes \cite{shalev2014understanding}. For recent work deriving scaling limits under such assumption (in different problems) see \cite{wang2017scaling}.

Understanding the optimization landscape of two-layers neural 
networks is largely an open problem even when we have access to an infinite number of examples,
i.e. to the population risk $R_{N}(\btheta)$. Several studies have focused on special choices of the activation function
$\sigma_*$ and of the data distribution $\prob$, proving that the population
risk has no bad local minima \cite{soltanolkotabi2017theoretical,ge2017learning,brutzkus2017globally}. This type
of analysis requires delicate calculations that are somewhat sensitive to the specific choice of the model.
Another line of work proposes new algorithms with theoretical guarantees
\cite{arora2014provable,sedghi2015provable, janzamin2015beating,zhang2016l1,tian2017symmetry,zhong2017recovery},
which  use initializations  based on  tensor factorization.

In this paper, we prove that --in a suitable scaling limit-- the SGD dynamics admits an asymptotic description in terms of
a certain non-linear partial differential equation (PDE). This PDE has a remarkable mathematical structure, in that it corresponds to a gradient flow
in the metric space $(\cuP(\reals^D),W_2)$: the space of probability measures on $\reals^D$, endowed with the Wasserstein metric. This gradient flow minimizes
an asymptotic version of the population risk which is defined for $\rho\in\cuP(\reals^D)$ and will be denoted by $R(\rho)$. 
This description simplifies  the analysis of the landscape of 
two-layers neural networks, for instance by exploiting underlying symmetries. We illustrate this by obtaining new results on several concrete examples,
as well as a general convergence result for `noisy SGD.'
In the next section, we provide an informal outline, focusing on basic intuitions rather than on formal results.
We then present the consequences of these ideas on a few specific examples, and subsequently
state our general results.

\subsection{An informal overview}
\label{sec:Overview}

A good starting point is to rewrite the population risk $R_{N}(\btheta) = \E\{[y-\hy(\bx;\btheta)]^2\}$ as
\begin{equation}
R_{N}(\btheta) = R_{\#}+\frac{2}{N}\sum_{i=1}^NV(\btheta_i) +\frac{1}{N^2}\sum_{i,j=1}^NU(\btheta_i,\btheta_j)\, ,\label{eq:R0N}
\end{equation}
where we defined the potentials $V(\btheta) = - \E\big\{y\,\sigma_*(\bx;\btheta)\big\}$, 
$U(\btheta_1,\btheta_2) = \E\big\{\sigma_*(\bx;\btheta_1) \sigma_*(\bx;\btheta_2)\big\}$.
In particular  $U(\,\cdot\,,\,\cdot\,)$ is a symmetric positive semidefinite kernel.
The constant $R_{\#}=\E\{y^2\}$ is the risk of the trivial predictor $\hy=0$.

Notice that $R_{N}(\btheta)$ only depends on $\btheta_1,\dots,\btheta_N$ through their empirical
distribution $\hrho^{(N)}=N^{-1}\sum_{i=1}^N \delta_{\btheta_i}$. This  suggests to consider 
a risk function defined for  $\rho\in\cuP(\reals^D)$ (we denote by $\cuP(\Omega)$ the space of probability distributions on $\Omega$):
\begin{equation}
R(\rho) \! = \! R_{\#}\!+\!2\!\int \!V(\btheta)\; \rho(\de\btheta) \!+\!\int \!U(\btheta_1,\btheta_2)\; \rho(\de\btheta_1)\, \rho(\de\btheta_2)\, . \label{eq:DensityRisk}
\end{equation}
Formal relationships can be established between $R_{N}(\btheta)$ and $R(\rho)$.
For instance, under mild assumptions,  $\inf_{\btheta}R_{N}(\btheta)=\inf_{\rho}R(\rho)+O(1/N)$. 
We refer to the next sections for mathematical statements of this type.

Roughly speaking, $R(\rho)$ corresponds to the population risk when the number of hidden units 
goes to infinity, and the empirical distribution of parameters $\hrho^{(N)}$ converges to $\rho$. Since $U(\,\cdot\,,\,\cdot\,)$ is positive semidefinite,
we obtain that the risk becomes convex in this limit. The fact that learning can be viewed
as convex optimization in an infinite-dimensional space was indeed pointed out  in the past \cite{lee1996efficient,bengio2006convex}. 
Does this mean that the landscape of the population risk simplifies for large $N$ 
and descent algorithms will converge to a unique (or nearly unique) global optimum?

The answer to the last question is generally negative, and a physics analogy can explain why.
Think of $\btheta_1,\dots,\btheta_N$ as the positions of $N$ particles in a $D$-dimensional space. 
When $N$ is large, the behavior of such a `gas' of particles is effectively described by a density $\rho_t(\btheta)$ (with $t$ indexing time). However, not all `small' 
changes of this density profile can be realized in the actual physical dynamics:
the dynamics conserves mass locally because particles cannot move discontinuously.
For instance, if $\supp(\rho_t) = S_1\cup S_2$ for two disjoint compact sets $S_1,S_2\subseteq\reals^D$, and all $t\in [t_1,t_2]$,
then the total mass in each of these regions cannot change over time, i.e. $\rho_t(S_1) = 1-\rho_t(S_2)$ does not depend on $t\in [t_1,t_2]$. 

We will prove that stochastic gradient descent is well approximated (in a precise quantitative sense described below) 
 by a continuum dynamics that enforces this local mass conservation principle. 
Namely, assume that the step size in SGD  given by $s_k= \eps\, \xi(k\eps)$, for $\xi:\reals_{\ge 0}\to\reals_{\ge 0}$ a sufficiently 
regular function. Denoting by $\hrho_k^{(N)}=N^{-1}\sum_{i=1}^N \delta_{\btheta^k_i}$ the empirical distribution of parameters
after $k$ SGD steps, we prove that
\begin{align}
\hrho^{(N)}_{t/\eps}\Rightarrow \rho_t
\end{align}
 when $N\to\infty$, $\eps\to 0$ (here $\Rightarrow$ denotes weak convergence).
The asymptotic dynamics of  $\rho_t$ is defined by the following PDE, which we shall refer to as \emph{distributional dynamics} (DD)
\begin{align}
\partial_t \rho_t & =2\xi(t)\, \nabla_{\btheta}\cdot \Big(\rho_t \nabla_{\btheta}\Psi(\btheta;\rho_t)\Big)\, ,\label{eq:GeneralPDE}\\
\Psi(\btheta;\rho) & \equiv V(\btheta)+\int U(\btheta,\btheta')\, \rho(\de\btheta')\, .
\end{align}
(Here $\nabla_{\btheta}\cdot \bv(\btheta)$ denotes the divergence of the vector field $\bv(\btheta)$.)
This should  be interpreted as an evolution equation in $\cuP(\reals^D)$.
While we described the convergence to this dynamics in asymptotic terms, the results  in the next sections
provide explicit non-asymptotic bounds. In particular, $\rho_t$ is a good approximation of $\hrho_k^{(N)}$,
$k=t/\eps$, as soon as $\eps\ll 1/D$ and $N\gg D$.

Using these results, analyzing learning in two-layer neural networks reduces to analyzing the PDE \eqref{eq:GeneralPDE}.
While this is far from being an easy task,  the  PDE formulation leads to several simplifications and insights.
First of all, it factors out the invariance of the risk \eqref{eq:R0N}  (and of the SGD dynamics \eqref{eq:First_SGD}),
with respect to permutations of the units $\{1,\dots,N\}$. 

Second, it allows to exploit symmetries in the data distribution $\prob$. If $\prob$ is left invariant under a group of transformations (e.g. rotations), we can look for
a solution  $\rho_t$ of the DD \eqref{eq:GeneralPDE} that enjoys the same symmetry, hence reducing the dimensionality of the problem. This is impossible 
for the finite-$N$ dynamics \eqref{eq:First_SGD}, since no arrangement of the points $\{\btheta_1,\dots, \btheta_N\}\subseteq \reals^D$
 is left invariant --say-- under rotations. 
We will provide examples of this approach in the next sections.

Third, there is rich mathematical literature on the PDE \eqref{eq:GeneralPDE} which was
motivated by the study of interacting particle systems in mathematical physics. As mentioned above, a key structure 
exploited in this line of work is that \eqref{eq:GeneralPDE} can be viewed as a gradient flow for the cost function $R(\rho)$ in the space 
$(\cuP(\reals^D),W_2)$, of probability measures on $\reals^D$
endowed with the Wasserstein metric \cite{jordan1998variational,ambrosio2008gradient,carrillo2003kinetic}. Roughly speaking,
this means that the trajectory $t\mapsto \rho_t$ attempts to minimize the risk $R(\rho)$ while maintaining the `local mass conservation' constraint. 
Recall that Wasserstein distance is defined as
\begin{align}
W_2(\rho_1,\rho_2) & = \Big(\inf_{\gamma\in \cC(\rho_1,\rho_2)}  \int \|\btheta_1-\btheta_2 \|_2^2 \gamma(\de \btheta_1,\de\btheta_2) \,\Big)^{1/2} , \label{eq:Wasserstein_dual}
\end{align}
where the infimum is taken over all couplings of $\rho_1$ and $\rho_2$. 
Informally, the fact that $\rho_t$ is a gradient flow means that \eqref{eq:GeneralPDE} is equivalent, for small $\tau$,
to
\begin{align}
\rho_{t+\tau} \approx \arg\min_{\rho\in \cuP(\reals^D)} \Big\{ R(\rho) +\frac{1}{2\xi(t)\tau}W_2(\rho,\rho_t)^2\Big\}\, .
\end{align}
Powerful tools from the mathematical
literature on gradient flows in measure spaces \cite{ambrosio2008gradient} can be exploited to study the behavior of \eqref{eq:GeneralPDE}. 

Most importantly, the scaling limit elucidates the dependence of 
the landscape of two-layer neural networks on the number of hidden units $N$. 

A remarkable feature of neural networks is the observation that, while they might be dramatically over parametrized,
this does not lead to performance degradation. In the case of bounded activation functions, this phenomenon  was clarified in the nineties
for empirical risk minimization algorithms, see e.g. \cite{bartlett1998sample}. The present work provides analogous insight for the 
SGD dynamics: roughly speaking, our results imply that the landscape 
remains essentially unchanged as $N$ grows, provided $N\gg D$.
In particular, assume that the PDE \eqref{eq:GeneralPDE} converges close to an optimum in time $t_*(D)$. This might depend on $D$, but does 
not depend on the number of hidden units 
$N$ (which does not appear in the DD PDE \eqref{eq:GeneralPDE}). If $t_*(D) =O_D(1)$, we can then take $N$ arbitrarily  (as long as $N\gg D$) and will achieve a population
risk which is independent of $N$ (and corresponds to the optimum), using $k=t_*/\eps = O(D)$ samples.

Our analysis can accommodate some important variants of SGD, a particularly interesting one being \emph{noisy SGD}:
\begin{equation}
\btheta_i^{k+1} = (1-2\lambda s_k) \btheta_i^{k}+ 2s_k\,\big(y_k-\hy_k\big)\, \nabla_{\btheta_i}\sigma_*(\bx_k;\btheta^k_i) +\sqrt{2s_k/\beta}\, \bg^k_i\, ,
\label{eq:Noisy_SGD}
\end{equation}
where $\bg^k_i\sim\normal(0,\id_D)$ and $\hy_k=\hy(\bx_k;\btheta^k)$. 
(The term $-2\lambda s_k \btheta_i^{k}$ corresponds to an $\ell_2$ regularization and will be useful for our analysis below.) 
The resulting scaling limit differ from \eqref{eq:GeneralPDE}  by the addition of a diffusion term:
\begin{align}
\partial_t \rho_t & =2\xi(t)\, \nabla_{\btheta}\cdot \Big(\rho_t \nabla_{\btheta}\Psi_{\lambda}(\btheta;\rho_t)\Big) +2\xi(t)\beta^{-1} \Delta_{\btheta}\rho_t\, ,\label{eq:GeneralPDE_Temp}
\end{align}
where $\Psi_{\lambda}(\btheta;\rho)=\Psi(\btheta;\rho)+(\lambda/2)\|\btheta\|_2^2$, and $\Delta_{\btheta}f(\btheta) = \sum_{i=1}^d\partial_{\theta_i}^2 f(\btheta)$
denotes the usual Laplacian.
This can be viewed as a gradient flow for the free energy $F_{\beta,\lambda}(\rho) = (1/2)R(\rho)+(\lambda/2)\int \|\btheta\|_2^2\rho(\de\btheta)-\beta^{-1} \ent(\rho)$, where 
$\ent(\rho) = -\int \rho(\btheta)\log \rho(\btheta)\, \de\btheta$ is the entropy of $\rho$ (by definition $\ent(\rho)=-\infty$ if
$\rho$ is singular).
$F_{\beta,\lambda}(\rho)$ is an entropy-regularized  risk, which penalizes strongly non-uniform $\rho$.

We will prove below that, for $\beta<\infty$, the evolution \eqref{eq:GeneralPDE_Temp} generically converges to the minimizer of
$F_{\beta, \lambda}(\rho)$, hence implying global convergence of noisy SGD in a number of steps \emph{independent of $N$}.

\section{Examples}
\label{sec:Examples}

In this section, we discuss some simple applications of the general approach outlined 
above. Let us emphasize that  these examples are not realistic. First, the data distribution $\prob$ is extremely simple: we made 
this choice in order to be able to carry out explicit calculations. Second, the activation function $\sigma_*(\bx;\btheta)$
is not necessarily optimal: we made this choice in order to illustrate some interesting phenomena.

\begin{figure}[ht!]
\centering
\includegraphics[width=0.9\linewidth]{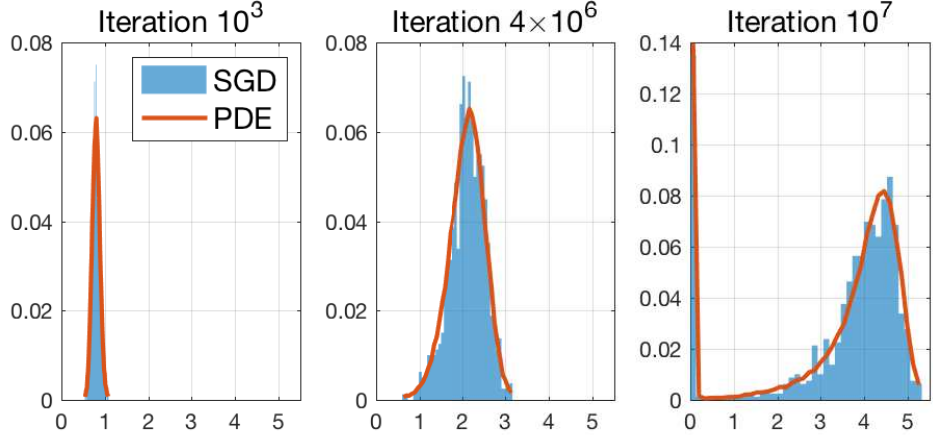}
\caption{Evolution of the radial distribution $\rad_t$ for the  isotropic Gaussian model, with $\Delta=0.8$.
Histograms are obtained from SGD experiments with $d=40$, $N=800$, initial weights distribution $\rho_0 =  \normal(\bzero,0.8^2/d\cdot\id_d)$, step size $\epsilon=10^{-6}$ and $\xi(t) = 1$. Continuous lines correspond to a numerical solution of the  DD \eqref{eq:ReducedPDE}.
\label{fig:SGD_Spherically}}
\end{figure}

\subsection{Centered isotropic Gaussians}

One-neuron neural networks perform well with (nearly)  linearly separable data.
The simplest classification problem which requires multilayer networks is --arguably-- the one of
distinguishing two Gaussians with the same mean.
Assume the joint law $\prob$ of $(y,\bx)$ to be as follows:
\begin{itemize}
\item[] With probability $1/2$: $y=+1$, $\bx\sim\normal(0,(1+\Delta)^2\id_d)$
\item[] With probability $1/2$: $y=-1$,
  $\bx\sim\normal(0,(1-\Delta)^2\id_d)$.
\end{itemize}
(This example will be generalized later.) 
Of course, optimal classification in this model becomes entirely trivial
if we compute the feature $h(\bx)= \|\bx\|_2$. However, it is non-trivial that a SGD-trained neural 
network will succeed.

We  choose an activation function without offset or output weights, namely $\sigma_*(\bx;\btheta_i) = \sigma(\<\bw_i,\bx\>)$. 
While qualitatively similar results are obtained for other choices of $\sigma$, 
we will use a simple piecewise linear function as a running example:
$\sigma(t) = s_1$ if $t\le t_1$, $\sigma(t) = s_2$ if $t\ge t_2$, and $\sigma(t)$ interpolated linearly for $t\in (t_1,t_2)$.
In simulations we use $t_1 = 0.5$, $t_2=1.5$, $s_1=-2.5$, $s_2=7.5$.

\begin{figure}[t!]
\begin{center}
\includegraphics[width=0.9\linewidth]{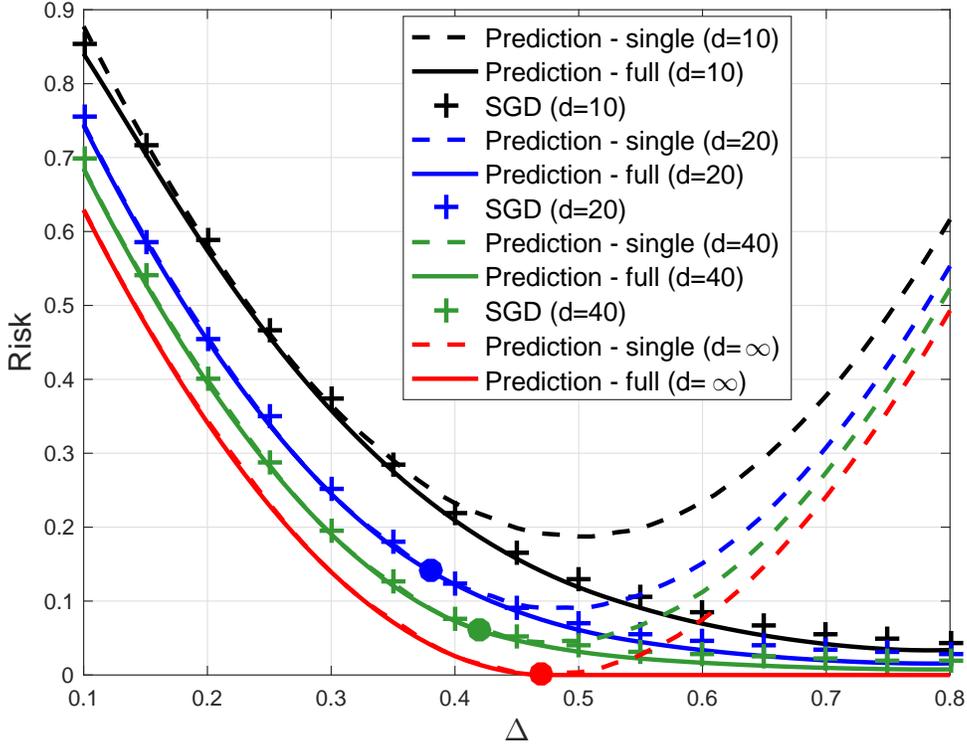}
\end{center}
\caption{Population risk in the problem of separating two isotropic Gaussians, as a function of the separation parameter $\Delta$. We use a  a two-layers network with piecewise linear activation, no offset, and output weights equal to one. Empirical results obtained by SGD (a single run per data point) are marked "$+$". Continuous lines are theoretical predictions obtained by numerically minimizing $R(\rho)$ (see SI for details). Dashed lines are theoretical predictions from the single-delta ansatz of \eqref{eq:Isotropic}. Notice that this ansatz is incorrect for $\Delta>\Delta_d^{\rm h}$, which is marked as a solid round dot. Here $N=800$.\label{fig:R_r_isotropic}}
\end{figure}

We run SGD  with initial weights $(\bw^0_i)_{i\le N}\sim_{iid} \rho_0$, where $\rho_0$ is spherically symmetric. 
Figure \ref{fig:SGD_Spherically} reports the result of such an experiment.
Due to the symmetry of the distribution $\prob$, the distribution $\rho_t$ remains spherically symmetric for all $t$,
and hence is completely determined by the distribution $\rad_t$ of the norm $r = \|\bw\|_2$. This distribution satisfies
a one-dimensional reduced DD:
\begin{align}
\partial_t\rad_t = 2\xi(t)\, \partial_r\big(\rad_t\partial_r \psi(r;\rad_t)\big)\,,\label{eq:ReducedPDE}
\end{align}
where the form of $\psi(r;\rho)$ can be derived from $\Psi(\btheta;\rho)$.
This reduced PDE can be efficiently solved numerically, see Supplementary Information (SI) for technical details.
As illustrated by Fig. \ref{fig:SGD_Spherically}, the empirical results match closely the predictions produced by this PDE.

In Fig. \ref{fig:R_r_isotropic}, we compare the asymptotic risk achieved by SGD with the prediction obtained by minimizing
$R(\rho)$, cf. \eqref{eq:DensityRisk} over spherically symmetric distributions. It turns out that, for certain values of  $\Delta$, the minimum 
is achieved by the uniform distribution over a sphere of radius $\|\bw\|_2 = r_*$, to be denoted by $\rho^{\unif}_{r_*}$. The value of $r_*$ is computed by minimizing 
\begin{align}
\oR^{(1)}_d(r) &= 1+2 v(r) \,+ u_d(r,r)\, ,\label{eq:Isotropic}
\end{align}
where  expressions for $v(r)$, $u_d(r_1,r_2)$ can be readily derived from $V(\bw)$, $U(\bw_1,\bw_2)$ 
and are given in the SI.
\begin{lemma}\label{lemma:OneDeltaCondition}
Let $r_*$ be a global minimizer of $r\mapsto R_d^{(1)}(r)$. Then $\rho^{\unif}_{r_*}$ is a global minimizer of $\rho\mapsto R(\rho)$ if and only if  $v(r)+u_d(r,r_*)\ge v(r_*)+u_d(r_*,r_*)$
for all $r\ge 0$.
\end{lemma}
Checking numerically this condition yields that $\rho^{\unif}_{r_*}$ is a global minimizer
for $\Delta$ in an interval $[\Delta^{\rm l}_d,\Delta^{\rm h}_d]$, where 
$\lim_{d\to\infty}\Delta^{\rm l}_d=0$ and $\lim_{d\to\infty}\Delta^{\rm h}_d=\Delta_\infty \approx 0.47$. 
%

Figure \ref{fig:R_r_isotropic} shows good quantitative  agreement between empirical results and theoretical predictions,
and suggests that SGD  achieves a value of the risk which is close to optimum. 
Can we prove that this is indeed the case, and that the SGD dynamics does not get stuck in local minima?
It turns out that we can use our general  theory (see next section) to prove that this is the case for large $d$.
In order to state this result, we need to introduce a class of good
uninformative initializations $\cuP_{\good}\subseteq\cuP(\reals_{\ge 0})$ for which convergence to the optimum takes place.
For $\rad\in\cuP(\reals_{\ge 0})$, we let $\oR_d(\rad) \equiv R(\rad \times {\rm Unif}(\mathbb S^{d - 1}))$. This risk has a well defined limit as $d\to\infty$.
We say that $\rad\in\cuP_{\good}$ if: $(i)$ $\rad$ is absolutely continuous with respect to Lebesgue measure, with bounded density; $(ii)$ $\oR_{\infty}(\rad) < 1$.
\begin{theorem}\label{thm:ConvergenceIsotropic}
For any $\eta, \Delta, \delta > 0$, and $\rad_0 \in \cuP_{\good}$, there exists $d_0 = d_0(\eta, \rad_0, \Delta)$, $T = T(\eta, \rad_0, \Delta)$, and
$C_0 = C_0(\eta, \rad_0, \Delta, \delta)$, such that the following holds for the problem of classifying isotropic Gaussians. For any dimension $d \ge d_0$, number of neurons 
$N \ge C_0 d$, consider SGD initialized  with $(\bw_i^0)_{i \le N}
\sim_{iid} \rad_0 \times {\rm Unif}(\mathbb S^{d - 1})$ and step size
$\eps \in [1/N^{10}, 1/(C_0 d)]$. Then we have 
\begin{align}
R_{ N}(\btheta^{k}) \le \inf_{\btheta\in\reals^{N\times d}} R(\btheta)
+ \eta
\end{align}
 for any $k \in [T/\eps, 10 T/\eps]$ with probability at least $1 - \delta$.
\end{theorem}
In particular, if we set $\eps = 1/(C_0 d)$, then the number of SGD steps is $k\in [(C_0T)\, d, (10C_0T)\, d]$:
the number of samples used by SGD does not depend on the number of hidden units $N$, and 
is only linear in the dimension. Unfortunately the proof does not provide the dependence of $T$ on $\eta$, but
Theorem \ref{thm:StabilityDelta} below suggests exponential local convergence. 

While we stated Theorem \ref{thm:ConvergenceIsotropic} for the
piecewise linear sigmoids, the SI presents technical conditions under which it holds for a general monotone function $\sigma: \reals\to\reals$.

\subsection{Centered anisotropic Gaussians}

We can generalize the previous result to a problem in which the network needs to
 select a subset of relevant nonlinear features out of many \emph{a priori} equivalent ones. 
We assume the joint law of $(y,\bx)$ to be as follows:
\begin{itemize}
\item[] With probability $1/2$: $y=+1$, $\bx\sim\normal(0,\bSigma_+)$, and 
\item[] With probability $1/2$: $y=-1$, $\bx\sim\normal(0,\bSigma_-)$.
\end{itemize}
Given a linear subspace $\cV\subseteq \reals^d$ of dimension $s_0\le d$, we assume that $\bSigma_+$, $\bSigma_-$
differ uniquely along $\cV$: $\bSigma_{\pm} = \id_d+(\tau_\pm^2-1)\bP_{\cV}$, where $\tau_{\pm} = (1\pm\Delta)$
and $\bP_{\cV}$ is the orthogonal projector onto $\cV$. In other words,
the projection of $\bx$ on the subspace $\cV$ is distributed according to a isotropic Gaussian with variance $\tau_+^2$ (if
$y=+1$) or $\tau_-^2$ (if $y=-1$). The projection orthogonal to $\cV$ has instead the same variance in the two classes.
A successful classifier must be able to learn the relevant subspace $\cV$.
We assume the same class of activations $\sigma_*(\bx;\btheta) = \sigma(\<\bw,\bx\>)$ as for the isotropic case. 

The distribution $\prob$ is invariant under a reduced symmetry group $\cO(s_0)\times \cO(d-s_0)$. As a consequence, letting $r_1= \|\bP_{\cV}\bw\|_2$ and
$r_2\equiv \|(\id_d-\bP_{\cV})\bw\|_2$, it is sufficient to consider distributions $\rho$ that are uniform, conditional on the values of $r_1$ and
$r_2$. If we initialize $\rho_0$ to be uniform conditional on $(r_1,r_2)$, this property is preserved 
by the evolution \eqref{eq:GeneralPDE}.
As in the isotropic case, we can use our general theory to prove convergence to a near-optimum if $d$ is large enough. 
\begin{theorem}\label{thm:ConvergenceAnisotropic}
For any $\eta, \Delta, \delta > 0$, and $\rad_0 \in \cuP_{\good}$, there exists $d_0 = d_0(\eta, \rad_0, \Delta,\gamma)$, $T = T(\eta, \rad_0, \Delta,\gamma)$, and $C_0 = C_0(\eta, \rad_0, \Delta, \delta,\gamma)$, such that the following holds for the problem of 
classifying anisotropic Gaussians with $s_0=\gamma d$, $\gamma\in (0,1)$ fixed.
For any dimension parameters  $s_0 = \gamma d \ge d_0$, number of neurons $N \ge C_0 d $, consider SGD initialized with initialization $(\bw_i^0)_{i \le N} \sim_{iid}  \rad_0 \times {\rm Unif}(\mathbb S^{d - 1})$ and step size $\eps \in [1/N^{10}, 1/(C_0 d)]$. Then we have 
$R_{N}(\btheta^{k}) \le \inf_{\btheta\in\reals^{N\times d}} R_{N}(\btheta) + \eta$ for any $k \in [T/\eps, 10 T/\eps]$ with probability at least $1 - \delta$.
\end{theorem}
Even with a reduced degree of symmetry, SGD converges to a network with nearly-optimal  risk, after using a number of samples 
$k = O(d)$, which is independent of the number of hidden units $N$.
\begin{figure}[t!]
\centering
\includegraphics[width=0.9\linewidth]{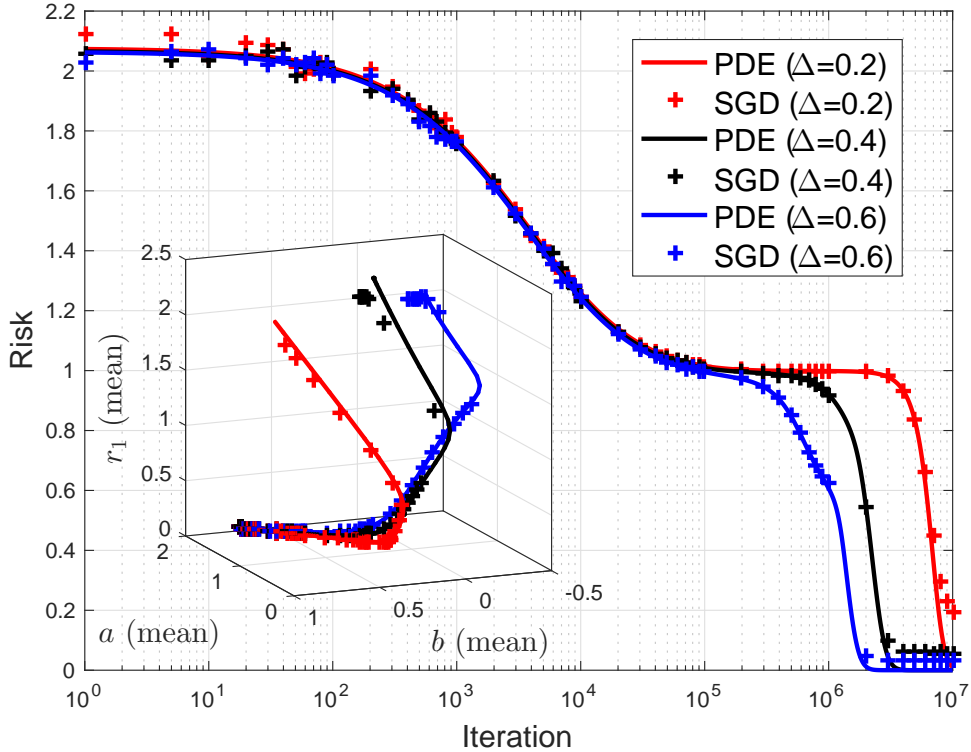}
\caption{Evolution of the population risk for the variable selection problem using a two-layers neural network with ReLU activations. Here $d=320$, $s_0 =60$,  $N=800$, and
we used $\xi(t) = t^{-1/4}$ and $\eps=2\times 10^{-4}$ to set the step size. Numerical simulations using SGD (one run per data point) are marked "$+$",
and curves are solutions of the reduced PDE with $d=\infty$. 
Inset: evolution of three parameters of the reduced distribution
$\rad_t$ (average output weights $a$, average offsets $b$ and average $\ell_2$ norm in the relevant subspace $r_1$) for the same setting.\label{fig:RiskAnisotropic}}
\end{figure}

\subsection{A better activation function}

Our previous examples use activation functions $\sigma_*(\bx;\btheta) = \sigma(\<\bw,\bx\>)$ without output weights or offset, in
order to simplify the analysis and illustrate some interesting phenomena. Here we consider instead a standard rectified linear unit (ReLU) activation,
and fit both the output weight and the offset:
$\sigma_*(\bx;\btheta)= a\, \sigma_{\relu}(\<\bw,\bx\>+b)$ where $\sigma_{\relu}(x) = \max(x,0)$.
Hence $\btheta= (\bw,a,b)\in\reals^{d+2}$. 

We consider the same data distribution introduced in the last section (anisotropic Gaussians). Figure \ref{fig:RiskAnisotropic}
reports the evolution of the risk $R_{N}(\btheta^k)$ for three experiments with $d=320$, $s_0 =60$ and different values of $\Delta$.
SGD is initialized by setting $a_i=1$, $b_i=1$ and $\bw_i^0 \sim_{iid} \normal(\bzero,0.8^2/d\cdot \id_d)$ for $i \le N$.
We observe that SGD converges to a network with very small risk, but this convergence has a nontrivial structure and presents long flat regions. 

The empirical results are well captured by our predictions based on the continuum limit. In this case
we obtain a reduced PDE for the joint distribution of the four quantities $\br=(a,b,r_1=\|\bP_{\cV}\bw\|_2,r_2=\|\bP^{\perp}_{\cV}\bw\|_2)$,
denoted by $\rad_t$. The reduced PDE is analogous to \eqref{eq:ReducedPDE} albeit in $4$ rather than $1$ dimensions.
In Figure \ref{fig:RiskAnisotropic} we consider the evolution of the risk, alongside  three properties of the distribution $\rad_t$ --the means of the output
weight $a$, of the offset $b$, and of $r_1$. 

\begin{figure}[t!]
\begin{center}
\includegraphics[width=0.9\linewidth]{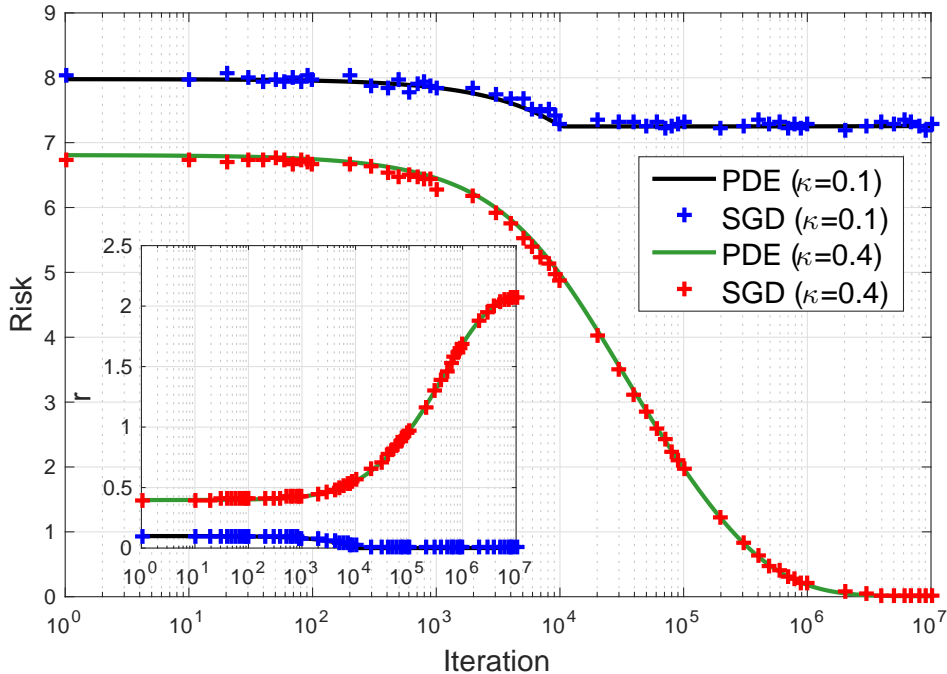}
\end{center}
\caption{Separating two isotropic Gaussians, with a non-monotone activation function (see text for details). Here $N=800$, $d=320$, $\Delta=0.5$.
The main frame presents the evolution of the population risk along the SGD trajectory, starting from two different initializations of $(\bw_i^0)_{i \le N} \sim_{iid}\mathsf{N}(\bzero,\kappa^2/d\cdot \mathbf{I}_d)$ for either $\kappa=0.1$ or $\kappa=0.4$. In the inset,
we plot the evolution of the average of $\|\bw\|_2$ for the same conditions. Symbols are empirical results. Continuous lines are prediction obtained with the reduced PDE \eqref{eq:ReducedPDE}.
\label{fig:Failure}}
\end{figure}

\subsection{Predicting failure}

SGD does not always converge to a near global optimum. 
Our analysis allows to construct examples in which SGD fails. For instance,
Figure \ref{fig:Failure} reports results for isotropic Gaussians problem. We violate the assumptions of Theorem \ref{thm:ConvergenceIsotropic}
by using non monotone activation function. Namely, we use $\sigma_*(\bx;\btheta) = \sigma(\<\bw,\bx\>)$,
where $\sigma(t) = -2.5$ for $t\leq 0$, $\sigma(t) = 7.5$ for $t\geq 1.5$, and $\sigma(t)$ linearly interpolates from $(0,-2.5)$ to $(0.5,-4)$, and from $(0.5,-4)$ to $(1.5,7.5)$.

Depending on the initialization, SGD converges to two different limits, one with a small risk, and the second with high risk.
Again this behavior is well tracked by solving a one-dimensional PDE for the distribution $\rad_t$ of $r=\|\bw\|_2$.

\section{General results}
\label{sec:General}

In this section we return to the general supervised learning problem described in the introduction and describe our
general results. Proofs are deferred to the SI.

First, we note that the minimum of the asymptotic risk $R(\rho)$ of \eqref{eq:DensityRisk} provides a good approximation of 
the minimum of the finite-$N$ risk $R_{N}(\btheta)$. 
\begin{proposition}\label{thm:NtoInfty}
Assume that either one of the following conditions hold: $(a)$ $\inf_{\rho}R(\rho)$ is achieved by a distribution $\rho_*$ such that 
$\int U(\btheta, \btheta)\, \rho_*(\de\btheta)\le K$; $(b)$ There exists $\eps_0>0$ such that, for  any $\rho\in \cuP(\reals^D)$ such that 
$R(\rho)\le \inf_{\rho}R(\rho)+\eps_0$ we have  $\int U(\btheta, \btheta)\, \rho(\de\btheta)\le K$.
Then
\begin{align}
\big|\inf_{\btheta}R_{N}(\btheta)-\inf_{\rho}R(\rho)\big|\le K/N\, . \label{eq:SimpleBound}
\end{align}
Further, assume that $\btheta\mapsto V(\btheta)$ and $(\btheta_1,\btheta_2)\mapsto U(\btheta_1, \btheta_2)$ are continuous,
with $U$ bounded below.
A probability measure $\rho_*$ is a global minimum of $R$ if $\inf_{\btheta\in\reals^D}\Psi(\btheta;\rho_*)>-\infty$ and
\begin{align}
\supp(\rho_*)\subseteq \arg\min_{\btheta\in\reals^D}\Psi(\btheta;\rho_*)\, . \label{eq:GeneralMinCondition}
\end{align}
\end{proposition}

We next consider the distributional dynamics \eqref{eq:GeneralPDE} and \eqref{eq:GeneralPDE_Temp}. These should be interpreted to hold in weak
sense, cf. SI.
%
%
%
In order to establish that these PDEs indeed describe the limit of the SGD dynamics, we make the following assumptions.
\begin{itemize}
\item[{\sf A1.}] $t\mapsto \xi(t)$ is bounded Lipschitz: $\|\xi\|_{\infty}, \|\xi\|_{\sLip}\le K_1$, with $\int_0^{\infty}\xi(t) \de t=\infty$.
\item[{\sf A2.}] The activation function $(\bx,\btheta)\mapsto \sigma_*(\bx;\btheta)$ is bounded, with sub-Gaussian gradient: $\|\sigma_*\|_{\infty}\le K_2$, $\|\nabla_{\btheta}
\sigma_*(\bX;\btheta)\|_{\psi_2}\le K_2$. Labels are bounded $|y_k|\le K_2$.
\item[{\sf A3.}] The gradients $\btheta\mapsto \nabla V(\btheta)$, $(\btheta_1,\btheta_2)\mapsto \nabla_{\btheta_1}U(\btheta_1,\btheta_2)$  are bounded, Lipschitz continuous 
(namely $\| \nabla_{\btheta}V(\btheta) \|_2$, $\|\nabla_{\btheta_1}U(\btheta_1,\btheta_2)\|_2\le K_3$, $\|\nabla_{\btheta}V(\btheta)-\nabla_{\btheta}V(\btheta')\|_2\le K_3 \|\btheta-\btheta'\|_2$, $\|\nabla_{\btheta_1}U(\btheta_1,\btheta_2)-\nabla_{\btheta_1}U(\btheta'_1,\btheta_2')\|_2\le K_3 \|(\btheta_1,\btheta_2)-
(\btheta_1',\btheta_2')\|_2$).
\end{itemize}

We also introduce the following error term which quantifies in a non-asymptotic 
sense the accuracy of our PDE model 
\begin{align}
\err_{N,D}(z)\equiv \sqrt{1/N \vee \eps } \cdot \Big[\sqrt{D + \log (N  / \eps)} + z \Big]\, .
\end{align}
The convergence of the SGD process to the PDE model is an example of a phenomenon 
which is known in probability theory as \emph{propagation of chaos} \cite{sznitman1991topics}.
\begin{theorem}\label{thm:GeneralPDE}
Assume that conditions {\sf A1}, {\sf A2}, {\sf A3} hold.
For $\rho_0\in\cuP(\reals^D)$, consider SGD with initialization $(\btheta^0_i)_{i\le N}\sim_{iid}\rho_0$ and step size $s_k = \eps \xi(k \eps)$. For $t\ge 0$, let $\rho_t$ be the solution of PDE \eqref{eq:GeneralPDE}.
Then, for any fixed $t \ge 0$,  $\hrho^{(N)}_{\lfloor t / \eps \rfloor} \Rightarrow \rho_{t}$ almost surely along any sequence $(N,\eps = \eps_N)$ such that $N \to \infty$, $\eps_N \to 0$, $N/\log (N/\eps_N) \to\infty$ and $\eps_N \log(N / \eps_N)\to 0$.
Further, there exists a constant $C$  (depending uniquely on the parameters $K_i$ of conditions {\sf A1}-{\sf A3}) such that,
for any $f:\reals^D\times \reals\to\reals$, with $\|f\|_{\infty}, \|f\|_{\sLip}\le 1$, $\eps\le 1$,
\begin{align}
&\sup_{k \in  [0, T/\eps] \cap \naturals}\Big|\frac{1}{N}\sum_{i=1}^Nf(\btheta^{k}_i)- \int\!
  f(\btheta)  \rho_{k\eps}(\de \btheta) \Big| \le Ce^{CT}\,\err_{N,D}(z)\, ,\nonumber\\
&\sup_{k \in  [0, T/\eps] \cap \naturals}\big|R_{N}(\btheta^{k})-R(\rho_{k\eps})\big| \le Ce^{CT}\,\err_{N,D}(z)\, ,
\end{align} 
with probability $1- e^{-z^2}$.
The same statements hold for noisy SGD \eqref{eq:Noisy_SGD}, provided \eqref{eq:GeneralPDE} is replaced by \eqref{eq:GeneralPDE_Temp},
and if $\beta\ge 1$, $\lambda\le 1$, and $\rho_0$ is $K_0$ sub-Gaussian for some $K_0>0$.
\end{theorem}
Notice that dependence of the error terms in $N$ and $D$ is rather benign. On the other hand, the error grows exponentially with the  time horizon $T$,
which limits its applicability to cases in which the  DD converges rapidly to a good solution.
We do not expect this behavior to be improvable within the general setting of \ref{thm:GeneralPDE}, which \emph{a priori} includes cases in which the dynamics is unstable. 

We can regard $\bJ(\btheta;\rho_t) = \rho_t(\btheta) \nabla_{\btheta}\Psi(\btheta;\rho_t)$ as a current.
The fixed points of the continuum dynamics are densities that correspond to zero current, as stated below.
\begin{proposition}\label{thm:FixedPoints}
Assume $V(\,\cdot\,), U(\,\cdot\,,\,\cdot\,)$ to be differentiable with  bounded gradient.
If $\rho_t$ is a solution of the PDE \eqref{eq:GeneralPDE}, then $R(\rho_t)$ is non-increasing.
%
Further, probability distribution $\rho$ is a fixed point of the PDE \eqref{eq:GeneralPDE} if and only if
\begin{align}
\supp(\rho) \subseteq \Big\{\,\btheta\, :\;\;\;
\nabla_{\btheta}\Psi(\btheta;\rho) = \bzero\, \Big\}\, . \label{eq:FP_Condition}
\end{align}
\end{proposition}
Note that global optimizers of $R(\rho)$, defined by condition \eqref{eq:GeneralMinCondition}, are fixed points, but the set of fixed points is, in general, larger
than the set of optimizers.
Our next proposition provides an analogous characterization of the fixed points of diffusion DD \eqref{eq:GeneralPDE_Temp} (see \cite{carrillo2003kinetic} for  related results).
\begin{proposition}\label{thm:FixedPoints_Temp_finite}
Assume that conditions {\sf A1}-{\sf A3} hold and that $\rho_0$ is absolutely continuous with respect to Lebesgue measure, with $F_{\beta,\lambda}(\rho_0)<\infty$. 
If $(\rho_t)_{t\ge 0}$ is a solution of the diffusion PDE \eqref{eq:GeneralPDE_Temp}, then $\rho_t$ is absolutely continuous.
Further, there is at most one  fixed point $\rho_*=\rho_*^{\beta,\lambda}$ of \eqref{eq:GeneralPDE_Temp} satisfying $F_{\beta,\lambda}(\rho_*)<\infty$. This fixed point is absolutely continuous and its density satisfies
\begin{align}
\rho_*(\btheta) = \frac{1}{Z(\beta)} \exp\Big\{-\beta \Psi_{\lambda}(\btheta;\rho_*)\Big\}\, . \label{eq:FP_Condition_Temp}
\end{align}
\end{proposition}
In the next sections we state our results about convergence of the distributional dynamics
to its fixed point. In the case of noisy SGD (and for the diffusion PDE \eqref{eq:GeneralPDE_Temp}), a general convergence result can be established
(although at the cost of an additional regularization). For noiseless SGD (and the continuity equation \eqref{eq:GeneralPDE_Temp}), 
we do not have such general result. However, we obtain a stability condition for fixed point containing one point mass, which is
useful to characterize possible limiting points (and is used in treating the examples in the previous section).

\subsection{Convergence: noisy SGD}

Remarkably, the diffusion PDE \eqref{eq:GeneralPDE_Temp} generically
admits a unique fixed point, which is the global minimum of  $F_{\beta,\lambda}(\rho)$
and the evolution \eqref{eq:GeneralPDE_Temp} converges to it, if initialized so that $F_{\beta, \lambda}(\rho_0) <\infty$. 
This statement requires some qualifications. 
 First of all, we introduce sufficient regularity assumptions to
guarantee the existence of sufficiently smooth solutions of \eqref{eq:GeneralPDE_Temp}.
\begin{enumerate}
\item[{\sf A4}] $V \in C^4(\reals^D)$, $U \in C^4(\reals^D \times \reals^D)$, $\nabla^k_{\btheta_1} U(\btheta_1, \btheta_2)$ is uniformly bounded for $0 \le k \le 4$. 
\end{enumerate}
Next notice that the right-hand side of the fixed point equation \eqref{eq:FP_Condition_Temp}
is not necessarily normalizable (for instance, it is not when $V(\,\cdot\,)$, $U(\,\cdot\, ,\,\cdot\,)$ are bounded).
 In order to ensure the existence of a fixed point, we need $\lambda>0$.
\begin{theorem}\label{thm:GeneralPDE_Noisy_fixed_point}
Assume that conditions {\sf A1}-{\sf A4} hold, and $1/K_0\le \lambda\le K_0$ for some $K_0>0$
Then $F_{\beta,\lambda}(\rho)$ has a unique minimizer, denoted by $\rho_*^{\beta,\lambda}$, which satisfies 
\begin{align}
R(\rho_*^{\beta,\lambda}) \le \inf_{\btheta\in\reals^{N\times D}} R_{N}(\btheta) + C\, D/\beta\, ,
\end{align}
where $C$ is a constant depending on $K_0$,$K_1$,$K_2$,$K_3$.
Further, letting $\rho_t$ be a solution of the diffusion PDE \eqref{eq:GeneralPDE_Temp} with initialization satisfying $F_{\beta,\lambda}(\rho_0)<\infty$,
we have, as $t\to\infty$,
\begin{align}
\rho_t\Rightarrow \rho_*^{\beta,\lambda}\, .
\end{align}
\end{theorem}
The proof of this theorem is based on the following formula that describes the free energy decrease along the 
trajectories of the distributional dynamics \eqref{eq:GeneralPDE_Temp}:
\begin{align}\label{eqn:F_non_increasing_bis}
\frac{\de F_{\beta, \lambda}(\rho_t)}{\de t} =& - 2 \xi(t) \int_{\reals^D} \| \nabla_\btheta (\Psi_\lambda(\btheta; \rho_t) + 1/\beta \cdot \log \rho_t(\btheta))\|_2^2 \rho_t(\btheta) \de \btheta.
\end{align}
(A key technical hurdle is of course proving that this expression makes sense, which we do by showing the existence of strong solutions.)
It follows that the right-hand side must vanish as $t\to\infty$, from which we prove that (eventually taking subsequences) $\rho_t\Rightarrow\rho_*$ 
where $\rho_*$ must satisfy $\beta\Psi_\lambda(\btheta; \rho_*) + \log \rho_*(\btheta)= {\rm const}$. This in turns mean $\rho_*$ is a solution of the fixed point condition 
\ref{eq:FP_Condition_Temp} and is in fact a global minimum of $F_{\beta,\lambda}$ by convexity.

This result can be used in conjunction with Theorem \ref{thm:GeneralPDE}, in order to 
analyze the regularized noisy SGD algorithm \eqref{eq:Noisy_SGD}. 
\begin{theorem}\label{thm:GeneralPDE_Noisy}
Assume that conditions {\sf A1}-{\sf A4}  hold.
Let $\rho_0\in\cuP(\reals^D)$ be absolutely continuous with $F_{\beta,\lambda}(\rho_0)<\infty$ and $K_0$ sub-Gaussian. Consider 
regularized  noisy SGD, cf. \eqref{eq:Noisy_SGD}, at inverse temperature $\beta<\infty$, regularization $1/K_0\le \lambda\le K_0$
with initialization  $(\btheta^0_i)_{i\le N}\sim_{iid}\rho_0$.
Then for any $\eta>0$, there exists $K = K(\eta, \{ K_i \})$ and setting $\beta \ge K D$, there exists $T=T(\eta,V,U,\{K_i\},D,\beta)<\infty$ and $C_0 = C_0(\eta, \{K_i\},\delta)$ (independent of the dimension $D$ and temperature $\beta$)
such that the following happens for $N , (1/\eps)\ge C_0 e^{C_0T}D $, $\eps \ge 1/N^{10}$:
for any $k \in [T/\eps, 10 T/\eps]$ we have, with probability  $1 - \delta$,
\begin{align}
R_{N}(\btheta^{k}) \le \inf_{\rho \in \cuP(\reals^D)} R_{\lambda}(\rho) + \eta\, .
\end{align}
\end{theorem}
Let us emphasize that the convergence time $T$ in the last theorem can depend on the dimension $D$ and
on the data distribution $\prob$, but is independent of the number of hidden units $N$.
As illustrated by the examples in the previous section, understanding the dependence of $T$ on $D$ requires 
further analysis, but  examining the proof of this theorem suggests $T=e^{O(D)}$ quite generally (examples in which $T=O(1)$ or $T=e^{\Theta(D)}$ 
can be constructed).  We expect that our techniques could be pushed to investigate the dependence of $T$ on $\eta$ (see discussion in SI).
In highly structured cases, the dimension $D$ can be of constant order, and be much smaller than  $d$.

\subsection{Convergence: noiseless SGD}

The next theorems provide necessary and sufficient conditions for distributions containing a single point mass to be a stable fixed point of the evolution.
This result is useful in order to characterize the large time asymptotics of the dynamics \eqref{eq:GeneralPDE}. 
Here, we write $\nabla_1U(\btheta_1,\btheta_2)$ for the gradient of $U$ with respect to its first argument, 
and $\nabla^2_{1,1}U$ for the corresponding Hessian. Further, for a probability distribution $\rho_*$, we define
\begin{align}
\bH_0(\rho_*) & = \nabla^2V(\btheta_*) +\int \nabla_{1,1}^2U(\btheta_*,\btheta)\, \rho_*(\de\btheta)\, . \label{eq:H0Def}
\end{align}
Note that $\bH_0(\rho_*)$ is nothing but the Hessian of $\btheta\mapsto \Psi(\btheta;\rho_*)$ at $\btheta_*$.
\begin{theorem}\label{thm:StabilityDelta}
Assume $V, U$ to be twice differentiable with bounded gradient and bounded continuous Hessian.
Let $\btheta_*\in\reals^D$ be given. Then $\rho_*=\delta_{\btheta_*}$ is a fixed point of the evolution \eqref{eq:GeneralPDE} if and only if
$\nabla V(\btheta_*) +\nabla_1 U(\btheta_*,\btheta_*) = \bzero$.

Define  $\bH_0(\delta_{\btheta_*})\in\reals^{D\times D}$ as per \eqref{eq:H0Def}.
 If $\lambda_{\min}(\bH_0(\delta_{\btheta_*}))>0$, then there exists $r_0>0$ such that,
if $\supp(\rho_{t_0})\subseteq \Ball(\btheta_*;r_0)\equiv \{\btheta:\, \|\btheta-\btheta_*\|_2\le r_0\}$, then $\rho_t\Rightarrow \rho_*$
as $t\to\infty$. In fact, convergence is exponentially fast, namely $\int \|\btheta-\btheta_*\|_2^2\rho_t(\de\btheta)\le e^{-\lambda(t-t_0)}$
for some $\lambda>0$.
\end{theorem}
\begin{theorem}\label{thm:InstabilityDelta}
Under the same assumptions of Theorem \ref{thm:StabilityDelta}, let $\rho_* = p_*\delta_{\btheta_*}+(1-p_*)\trho_*\in \cuP(\reals^D)$ be a fixed point of dynamics \eqref{eq:GeneralPDE}, 
with $p_*\in (0,1]$ and $\nabla \Psi(\btheta_*;\rho_*) = \bzero$ (which, in particular, is implied by the fixed point condition \eqref{eq:FP_Condition}).
Define the level sets
$\cL(\eta) \equiv \{\btheta: \; \Psi(\btheta;\rho_*)\le \Psi(\btheta_*;\rho_*)-\eta\}$ and make the following assumptions:
{\sf B1.} The eigenvalues of $\bH_0=\bH_0(\rho_*)$ are all different from $0$, with $\lambda_{\min}(\bH_0)<0$;
{\sf B2.} $\trho_*(\cL(\eta))\uparrow 1$ as $\eta\downarrow 0$;
{\sf B3.} There exists 
$\eta_0>0$ such that the sets $\partial \cL(\eta)$ are compact for all $\eta\in(0,\eta_0)$.

 If $\rho_0$ has a bounded density
  with respect to Lebesgue measure, then it cannot be that $\rho_t$ converges weakly to $\rho_*$ as $t\to\infty$.
\end{theorem}

\section{Discussion and future directions}
\label{sec:Ext}

In this paper we developed a new approach to the analysis of two-layers neural networks.
Using a propagation-of-chaos argument, we proved that --if  the number of hidden units satisfies $N\gg D $--
 SGD dynamics is well approximated by the PDE in \eqref{eq:GeneralPDE}, while noisy SGD is well approximated by \eqref{eq:GeneralPDE_Temp}.
Both of these asymptotic descriptions correspond to Wasserstein gradient flows for certain energy (or free energy) functionals.
While empirical risk minimization is known to be insensitive to overparametrization \cite{bartlett1998sample},
the present work clarifies that \emph{the SGD behavior is also
independent of the number of hidden units, as soon as this is large enough}.

We illustrated our approach on several concrete examples, by proving convergence of SGD to a near-global
optimum.  This type of analysis provides a new mechanism for avoiding the perils of non-convexity.
We do not prove that the finite-$N$  risk $R_{N}(\btheta)$ has a unique local minimum, or that all local minima are 
close to each other. Such claims have often been the target of earlier work, but might be too strong for the
case of neural networks. 
We prove instead that the PDE \eqref{eq:GeneralPDE} converges to a near global optimum, when initialized with a bounded 
density. This effectively gets rid of some exceptional stationary points of $R_{N}(\btheta)$,
and  merges multiple finite $N$ stationary points that  result into similar distributions $\rho$. 

In the case of noisy SGD \eqref{eq:Noisy_SGD}, we prove that it converges generically to a near-global
minimum of the regularized risk, in time independent of the number of hidden units.

We emphasize that while we focused here on the case of square loss, our approach should be generalizable to other loss functions
as well, cf. SI.

The present work opens the way to several interesting research directions. We will mention two of them.
$(i)$~The PDE \eqref{eq:GeneralPDE}  corresponds to gradient flow in the Wasserstein metric for the risk $R(\rho)$,
see \cite{ambrosio2008gradient}. Building on this remark, tools from optimal transportation theory can be used to prove convergence.
$(ii)$~Multiple finite-$N$ local minima can correspond to the same  minimizer $\rho_*$ of $R(\rho)$ in the limit
$N\to\infty$. Ideas from glass theory \cite{mezard1999thermodynamics} might be useful to investigate this structure.

Let us finally mention that, after a first version of this paper appeared as a preprint, several other groups obtained results that are closely related to 
Theorem \ref{thm:GeneralPDE} \cite{rotskoff2018neural,sirignano2018mean,chizat2018global}.
%
%
\section*{Acknowledgements}
This work was partially supported by grants NSF DMS-1613091, NSF CCF-1714305 and NSF IIS-1741162. S. M. was partially supported by Office of Technology Licensing Stanford Graduate Fellowship. P.-M. N. was partially supported by William R. Hewlett Stanford Graduate Fellowship. The authors would like to thank Jiajun Tong for helpful discussions concerning strong solutions for parabolic PDEs.

\bibliographystyle{alpha}
\bibliography{mean_field_rev.bbl}

\newpage

\section*{\hfil Supplementary information \hfil }

\vskip1cm

We present here proofs and additional technical details
for our mathematical results, as well as additional information concerning the numerical experiments. 

\numberwithin{equation}{section}
\numberwithin{lemma}{section}
\numberwithin{figure}{section}


\section{Notations}
\label{sec:Notation}

We use lowercase bold for vectors (e.g. $\bu,\bv,\dots$), uppercase bold for matrices (e.g. $\bA, \bB,\dots$),
and lowercase plain for scalar ($x,y,\dots$). 
\begin{itemize}
\item Given a measurable space $\Omega$, we denote by $\cuP(\Omega)$ the set of probability measures on $\Omega$.
\item $\Ball^d(\bx;r)$ denotes the Euclidean ball with center $\bx$ and radius $r$ in $\reals^d$. We will drop the dimension superscript
whenever clear from the context.
 \item Given a measurable function $f$, and a measure $\mu$, we denote by $\<f,\mu\> = \<\mu,f\> =\int f\, \de\mu$ the corresponding integral.
\item For a univariate function $f:\reals\to\reals$, we denote by $f'(x)$ its derivative at $x$. If the argument is time, we
will also use $\dot{f}(t)$.
\item $\|f\|_{\sLip}\equiv \sup_{x\neq y}|f(\bx)-f(\by)|/\|\bx-\by\|_2$ denotes the Lipshitz constant of a function $f$.
\item $d_{\sBL}(\,\cdot\,,\,\cdot\,)$ is the bounded Lipschitz distance between probability measures
\begin{align}
d_{\sBL}(\mu,\nu) &= \sup\left\{\left|\int f (\bx)\, \mu(\de \bx) -\int f(\bx)\,\nu(\de \bx)\right|\;:\;\; \|f\|_{\infty}\le 1, \|f\|_{\sLip}\le 1 \right\} \label{eq:BL_primal}\\
& \le 2 \inf_{\gamma\in \cC(\mu,\nu)}  \int \big(\|\bx-\by\big\|_2\wedge 1 \big) \gamma(\de \bx,\de\by) \le 4 \, d_{\sBL}(\mu, \nu)\, . \label{eq:BL_dual}
\end{align}
Here  $\cC(\mu,\nu)$ is the set of couplings of $\mu$ and $\nu$.
\item $W_p(\,\cdot\,,\,\cdot\,)$ is the Wasserstein distance between probability measures
\begin{align}
W_p(\mu,\nu) & = \Big(\inf_{\gamma\in \cC(\mu,\nu)}  \int \|\bx-\by\big\|_2 ^p \gamma(\de \bx,\de\by) \,\Big)^{1/p} . \label{eq:Wasserstein_dual}
\end{align}
For $p = 1$, the Kantorovich-Rubinstein duality gives
\begin{align}
W_1(\mu,\nu) & = \sup\left\{\left|\int f (\bx)\, \mu(\de \bx) -\int f(\bx)\,\nu(\de \bx)\right|\;:\;\; \|f\|_{\sLip}\le 1 \right\}. \label{eq:Wasserstein_primal}
\end{align}

\item $K$ is a generic constant depending on $K_0, K_1, K_2, K_3$, where  $K_i$'s are constants which will be specified from the context. 


\item $\mathbb N = \{0, 1, 2, \ldots \}$ denote the set of natural numbers. 

\end{itemize}

\section{General results: Statics}

In this section, we discuss some properties of the population risk, $R_N(\btheta)$, and its continuum counterpart $R(\rho)$. For future reference, we copy the key definitions from the main
text:
\begin{align}
R_N(\btheta) & \equiv R_{\#}+\frac{2}{N}\sum_{i=1}^NV(\btheta_i) +\frac{1}{N^2}\sum_{i,j=1}^N U(\btheta_i, \btheta_j)\, ,\\
R(\rho) & \equiv R_{\#}+2\int V(\btheta)\; \rho(\de\btheta) +\int U(\btheta_1, \btheta_2)\; \rho(\de\btheta_1)\, \rho(\de\btheta_2)\, ,\label{eq:R-Rho-def}\\
R_{\#} & = \E\{y^2\}\, ,\;\;\;\;\;\; V(\btheta) = - \E\big\{y\,\sigma_*(\bx;\btheta)\big\}\, ,\\
U(\btheta_1, \btheta_2) & = \E\big\{\sigma_*(\bx;\btheta_1) \sigma_*(\bx;\btheta_2)\big\}\, .
\end{align}
We further recall the notation
\begin{align}
\Psi(\btheta;\rho) = V(\btheta)+ \int U(\btheta, \btheta')\; \rho(\de\btheta')\, .
\end{align}
We will always assume that the expectations defining $V(\btheta), U(\btheta_1,\btheta_2)$ exist  finite for all $\btheta,\btheta_1,\btheta_2\in\reals^D$. A necessary and sufficient condition for this
is that $\E\{\sigma_*(\bx;\btheta)^2\}<\infty$ for all $\btheta$. Since in most cases of interest $|\sigma_*(\bx;\btheta)| \le M(\btheta)\|\bx\|_2$,
for this to happen, it is sufficient that $\bx$ has a finite second moment.

Note that this $\rho\mapsto R(\rho)$ is a convex function on the set of probability measures on $\reals^D$. We will denote by $\cuP_{V,U}$ the subset
of probability measures $\rho$ such  that the expectations on the right-hand side are finite. We define $R(\rho) = \infty$ if $\rho\in \cuP(\reals^D)\setminus\cuP_{V,U}$.

\subsection{Proof of Proposition \ref{thm:NtoInfty}}
\label{sec:Proof NtoInfty}

The proof is divided in two parts:
\begin{enumerate}
\item We show that minimizing the population risk $R_N(\btheta)$ yields similar results to minimizing its continuum counterpart $R(\rho)$:
\begin{align}
\Big|\inf_{\btheta}R_N(\btheta)-\inf_{\rho}R(\rho)\Big|\le \frac{K}{N}\, . \label{eq:SimpleBoundApp}
\end{align}
\item We establish the condition for $\rho_*$ to be a minimizer:
\begin{align}
\supp(\rho_*)\subseteq \arg\min_{\btheta\in\reals^D}\Psi(\btheta;\rho_*)\, . \label{eq:GeneralMinConditionApp}
\end{align}
\end{enumerate}

First notice that, for any $\btheta = (\btheta_i)_{i\le N}$, we have 
\begin{align}
R_N(\btheta) \ge \inf_{\rho} R(\rho)\, .
\end{align}
Indeed, $R_N(\btheta) = R(\rho)$ for $\rho = (1/N)\sum_{i=1}^N\delta_{\btheta_i}$. 

In order to prove Eq.~(\ref{eq:SimpleBoundApp}), let $\rho_*\in \cuP(\reals^D)$ be such that $R(\rho_*)= R_*$
under  assumption $(a)$, or  $R(\rho_*)\le  R_*+\eps$ under  assumption $(b)$. Let $(\btheta_i)_{i\le N} \sim_{iid} \rho_*$. A simple calculation shows that
\begin{align}
\E_\btheta [R_N(\btheta)] -R(\rho_*) &= \frac{1}{N}\left\{\int U(\btheta, \btheta)\, \rho_*(\de \btheta) - \int U(\btheta_1,\btheta_2)\, \rho_*(\de\btheta_1)\, \rho_*(\de\btheta_2) \right\}\\
&\le \frac{1}{N}\, \int U(\btheta, \btheta)\, \rho_*(\de \btheta) \le \frac{K}{N}\, ,
\end{align}
where the first inequality follows since $\int U(\btheta_1,\btheta_2)\, \rho_*(\de\btheta_1)\, \rho_*(\de\btheta_2) = \E\{y(\bx)^2\}\ge 0$ for
$y(\bx) = \int \sigma_*(\bx;\btheta)\, \rho_*(\de\btheta)$, and the second inequality follows by assumption. It follows that
\begin{align}
\inf_{\btheta} R_N(\btheta) \le R_*+\frac{K}{N}+\eps\, ,
\end{align}
whence the claim (\ref{eq:SimpleBoundApp})  follows since $\eps$ is arbitrary.

We next establish the minimum condition (\ref{eq:GeneralMinConditionApp}).
Notice that since $V(\,\cdot\, )$ is continuous, and $U(\,\cdot\,,\,\cdot\,)$ is bounded below, it follows
from Fatou's lemma that, for any $\rho$, the function $\btheta\mapsto\Psi(\btheta;\rho)$ is lower semicontinous
and takes values in $(-\infty,\infty]$. In particular the set $S_0(\rho)\equiv\arg\min_{\btheta} \Psi(\btheta;\rho)$
must be closed.

We first prove that any minimizer must satisfy  (\ref{eq:GeneralMinConditionApp}).
Let $\rho_*$ be a minimizer and define $\Psi_*=\inf_{\btheta}\Psi(\btheta;\rho_*)$.
By rearranging terms, for  any probability measure $\rho$, we have
\begin{align}
R(\rho)- R(\rho_*) = 2\<\Psi(\,\cdot\, ;\rho_*),(\rho-\rho_*)\> + \<U,(\rho-\rho_*)^{\otimes 2}\>\, . \label{eq:RhoVariation}
\end{align}
First we will assume $\Psi_*>-\infty$ (whence, by lower semicontinuity, $S_0(\rho_*)$ must be a non-empty closed set).
Let $\btheta_1\in S_0(\rho_*)$, and assume by contradiction that there exist $\btheta_0\in\supp(\rho_*)$, $\btheta_0\not\in S_0(\rho_*)$.
Let $\Ball(\btheta_0;\eps)$ be a ball of radius $\eps$ around $\btheta_0$. By lower semicontinuity, we can find $\eps_0, \Delta> 0 $
such that $\inf_{\btheta\in\Ball(\btheta_0;\eps_0)}\Psi(\btheta;\rho_*) =\Psi_*+\Delta>\Psi_*$. Further $t_0 \equiv \rho_*(\Ball(\btheta_0;\eps_0)) >0$
because $\btheta_0\in\supp(\rho_*)$. 

Let $\nu \equiv \bfone_{\Ball(\btheta_0;\eps_0)} \rho_*/t_0$ (i.e. $\nu$ is the conditional distribution given $\btheta\in\Ball(\btheta_0;\eps_0)$).
Define, for $t\in [0,t_0]$, the probability measure
\begin{align}
\rho_t = \rho_*-t\nu+t\delta_{\btheta_1}\, .
\end{align}
Using Eq.~(\ref{eq:RhoVariation}), we get
\begin{align}
R(\rho_t)- R(\rho_*) &= 2\<\Psi(\,\cdot\, ;\rho_*),(\delta_{\btheta_1}-\nu)\>\,  t+ \<U,(\delta_{\btheta_1}-\nu)^{\otimes 2}\>t^2\\
& \le 2(\Psi_*-\Psi_*-\Delta)\, t +C_0\, t^2 = -2\Delta\, t+C_0\, t^2\, ,
\end{align}
where the second inequality follows from the fact that  $U$ is continuous and $\delta_{\btheta_1}$, $\nu$ have bounded support.
By taking $t$ small enough, we get $R(\rho)<R(\rho_*)$ hence reaching a contradiction.

Next consider the case in which $\Psi_*\equiv \inf_{\btheta}\Psi(\btheta;\rho_*)=-\infty$. For $M\in\naturals$, $M\ge 1$, let $\btheta_M\in\reals^D$ be such that $\Psi(\btheta_M;\rho_*)\le -M$. For $\btheta_0\in\supp(\btheta_*)$, construct $\nu$
as before. Note that, and call $\inf_{\btheta\in\Ball(\btheta_0;\eps_0)} \Psi(\btheta;\rho_*)=\Psi_0$. Define, for $t\in[0, t_0]$
\begin{align}
\rho_{M,t} = \rho_*-t\nu+t\delta_{\btheta_M}\, .
\end{align}
By applying again Eq.~(\ref{eq:RhoVariation}), we get
\begin{align}
R(\rho_{M,t})- R(\rho_*) &= 2\<\Psi(\,\cdot\, ;\rho_*),(\delta_{\btheta_M}-\nu)\>\,  t+ \<U,(\delta_{\btheta_M}-\nu)^{\otimes 2}\>t^2\\
& \le -2(M+\Psi_0)\, t +C_0(M)\, t^2 \, .
\end{align}
By selecting $t=t_M = \min(t_0,(M+\Psi_0)/C_0(M))$ (which is positive for all $M$ large enough), we obtain $R(\rho_{M,t})- R(\rho_*)<0$
for all $M$ large and hence reach a contradiction.

We finally prove that condition (\ref{eq:GeneralMinConditionApp}) is sufficient for $\rho_*$ to be a minimizer.
Indeed, for any non-negative  measurable function $\mu:\reals^D\to \reals$, letting $\Psi_*=\min_{\btheta}\Psi(\btheta;\rho_*)$,
\begin{align}
R(\rho) &\ge R_{\#} +2\<V,\rho\>+\<U,\rho^{\otimes 2}\> - \<\mu,\rho\>\\
& = R(\rho_*) +2\<\Psi(\,\cdot\,;\rho_*),\rho-\rho_*\>+\<U,(\rho-\rho_*)^{\otimes 2}\>-\<\mu,\rho\>\\
& = R(\rho_*) +2\<\Psi(\,\cdot\,;\rho_*)-\Psi_*,\rho-\rho_*\>+\<U,(\rho-\rho_*)^{\otimes 2}\>-\<\mu,\rho\>\, .
\end{align}
Setting $\mu= 2[\Psi(\,\cdot\,;\rho_*)-\Psi_*]$, and noticing that  condition (\ref{eq:GeneralMinConditionApp})  implies
$\<\Psi(\,\cdot\,;\rho_*)-\Psi_*,\rho_*\>=0$, we get $R(\rho) \ge R(\rho_*) +\<U,(\rho-\rho_*)^{\otimes 2}\>\ge R(\rho_*)$.

\subsection{Some additional results}

We often find empirically that the optimal density $\rho_*$ is supported on a set of Lebesgue measure $0$
(sometimes on a finite set of points). The following consequence of the previous results partially explains these findings.
\begin{corollary}
Assume $\btheta\mapsto V(\btheta)$ to be an analytic function and $(\btheta_1,\btheta_2)\mapsto U(\btheta_1,\btheta_2)$ to be analytic with respect to $\btheta_1$,
 uniformly in $\btheta_2$. Namely there exists a locally bounded function $\btheta\mapsto B(\btheta)$ such that $\|\nabla^{k}_{\btheta_1}U(\btheta_1,\btheta_2)\|_2\le k! B(\btheta_1)^k$
for all $k$, $\btheta_1$, $\btheta_2$. If $\rho_*$ is a minimizer of $R(\rho)$, 
then one of the following holds
\begin{enumerate}
\item[$(a)$] $\Psi(\btheta;\rho_*)=\Psi_*$ for some constant $\Psi_*$ and all $\btheta\in\reals^D$. 
\item[$(b)$] The support of $\rho_*$ has zero Lebesgue measure.
\end{enumerate}
If $D=1$, then $(b)$ can be replaced by: $(b')$ $\rho_*$ is a convex combination of countably many point masses with no accumulation point
(finitely many if $\Psi(\theta;\rho_*)\to\infty$ as $|\theta|\to\infty$).
\end{corollary}
\begin{proof}
Note that, under the stated conditions $f(\btheta) \equiv \int U(\btheta,\btheta') \, \rho_*(\de\btheta')$ is analytic. 
Indeed, by a standard dominated convergence argument, we have that $\nabla^kf$ is given by the integral of $\int \nabla^k U(\btheta_1,\btheta_2)\, \rho_*(\de\btheta_2)$ for any $k\ge 0$.
Further, by an application of the intermediate value theorem there exists $t_{\btheta_1,\btheta_2,\bdelta} \in [0, 1]$ such that
\begin{align}
\left|f(\btheta_1+\bdelta) - \sum_{\ell=0}^{k-1}\frac{1}{\ell !}\, \<\nabla^{\ell}f(\btheta_1),\bdelta^{\otimes \ell}\>\right|& \le 
\frac{1}{k!}\left|\int \<\nabla_{\btheta_1}^{k}U(\btheta_1 +t_{\btheta_1,\btheta_2,\bdelta}\bdelta, \btheta_2),\bdelta^{\otimes k}\> \, \rho_*(\de\btheta_2)\right|\\
& \le \int B(\btheta_1 +t_{\btheta_1,\btheta_2,\bdelta}\bdelta)^k\|\bdelta\|_2^k \, \rho_*(\de\btheta_2)\\
& \le \sup_{\btheta\in\Ball(\btheta_1; \| \bdelta\|_2)} B(\btheta)^k \, \|\bdelta\|_2^k \, ,
\end{align}
which vanishes as $k\to\infty$ for uniformly over $\|\bdelta\|_2\le \delta_0$ for $\delta_0$ small enough. 

Let $\Psi_* = \min_{\btheta\in\reals^D}\Psi(\btheta;\rho_*)$. We thus have that $\btheta\mapsto \Psi(\btheta;\rho_*)$ is also analytic and
so is $\btheta\mapsto \Psi(\btheta;\rho_*)-\Psi_*$.
Since $\supp(\rho_*)\subseteq\{\btheta: \Psi(\btheta;\rho_*) = \Psi_*\}$, the claim follows from the fact that the set of zeros of a non-trivial analytic function has
vanishing Lebesgue measure \cite{mityagin2015zero}. In the case $D=1$, the set of zeros of an analytic function cannot have any accumulation point \cite{lang2013complex},
which therefore allows to replace $(b)$ with $(b')$.
\end{proof}

\section{General results: Dynamics}

In this section we consider the SGD dynamics with step size $s_k = \eps \xi(k\eps)$, under the assumptions ${\sf A1},  {\sf A2},  {\sf A3}$ 
stated in the main text. For the readers convenience, we reproduce here the form of the limiting PDE 
\begin{align}
\partial_t\rho_t(\btheta) =& 2 \xi(t) \nabla\cdot \big[\rho_t(\btheta)\nabla\Psi(\btheta;\rho_t)\big]\,, \label{eq:GeneralPDE_App}\\
\Psi(\btheta; \rho) =& V(\btheta) + \int U(\btheta, \btheta')\; \rho(\de \btheta')\,.
\end{align}
Recall that this is an evolution in the space of probability measures in $\reals^D$, and is to be interpreted in weak sense. Namely $\rho_t$ is a solution of Eq.~(\ref{eq:GeneralPDE_App}), if, for any bounded differentiable function
$\varphi:\reals^D\to\reals$ with bounded gradient:
\begin{align}
\frac{\de\phantom{t}}{\de t}\<\rho_t,\varphi\> = -2\xi(t)\, \int \<\nabla \varphi(\btheta) ,\nabla \Psi(\btheta;\rho_t)\>\, \rho_t(\de\btheta)\, . \label{eq:WeakSolution}
\end{align}
For background on this and similar PDEs  (and the analogous ones at finite temperature, cf. Section \ref{sec:FiniteT}), we refer to
\cite{markowich2000trend,carrillo2003kinetic,carrillo2006contractions,ambrosio2008gradient,carrillo2011global}. Our treatment will be mostly self-contained
because of some differences between our setting and the one in these papers.

\begin{remark} \label{rmk:ExistenceUniqueness}
Recall assumptions {\sf A1}, {\sf A2}, {\sf A3} in the main text. 
By \cite[Theorem 1.1]{sznitman1991topics}, assumptions {\sf A1} and {\sf A3} are sufficient  for the existence and uniqueness of solution of PDE (\ref{eq:GeneralPDE_App}). 

\end{remark}

A very useful tool for the analysis of the PDE  (\ref{eq:GeneralPDE_App}) is provided by the following \emph{nonlinear dynamics}.
We introduce trajectories $(\obtheta^t_i)_{1\le i\le N, \, t\in\reals_{\ge 0}}$ by letting $\obtheta_i^0 = \btheta_i^0$ to be the same initialization as for SGD
and, for $t\ge 0$ (here $\rP_X$ denotes the law of the random variable $X$):
\begin{align}
\obtheta^t_i &=\btheta^0_i-2\int_{0}^t\xi(s)\, \nabla\Psi(\obtheta_i^s;\rho_s)\, \de s\, ,\label{eq:NonLinearEvolution}\\
\rho_s & = {\rm P}_{\obtheta^s_i}\, .
\end{align}
This should be regarded as an equation for the law of the trajectory $(\obtheta_i^t)_{t\in\reals_{\ge 0}}$, with boundary condition determined by  $\obtheta_i^0\sim \rho_0$. 
As implied by \cite[Theorem 1.1]{sznitman1991topics}, under the same assumptions {\sf A1} and {\sf A3}, the nonlinear dynamics
 has a unique solution, with $\rho_t$ satisfying Eq.~(\ref{eq:GeneralPDE_App}).

\begin{lemma}\label{lemma:Lipschitz_continuity_of_rho_and_theta}
Assume conditions {\sf A1} and {\sf A3} hold. Let $(\rho_t)_{t \ge 0}$ be the solution of the PDE (\ref{eq:GeneralPDE_App}). Let $(\obtheta_i^t)_{t\ge 0}$ be the solution of nonlinear dynamics (\ref{eq:NonLinearEvolution}). Then $t\mapsto\obtheta^t_i$ is $K_1 K_3$-Lipschitz continuous, and $t\mapsto \rho_t$ is $K_1 K_3$-Lipschitz continuous in $W_2$ Wasserstein distance, with $K_1$ and $K_3$ as per conditions {\sf A1} and {\sf A3}. In particular, $t\mapsto\rho_t$ is continuous in the topology of weak convergence. 
\end{lemma}

\begin{proof}

Since $\xi$ and $\nabla\Psi$ are $K_1$ and $K_3$ bounded respectively, $t\mapsto\obtheta^t_i$ is $K_1 K_3$-Lipschitz continuous. 
Further, Eq.~(\ref{eq:BL_dual}) implies that
$t\mapsto \rho_t$ is Lipschitz continuous in $W_2$ Wasserstein distance, namely 
\begin{align}
d_{\sBL}(\rho_t,\rho_s)\le& W_2(\rho_t,\rho_s) \le (\E[\| \obtheta^t_i - \obtheta^s_i \|_2^2])^{1/2} \le K_1 K_3 \l t-s\l. 
\end{align}
\end{proof}

We notice that, under the nonlinear dynamics, the trajectories $(\obtheta_1^t)_{t\in\reals_{\ge 0}}$, \dots, $(\obtheta_N^t)_{t\in\reals_{\ge 0}}$
are independent and identically distributed. In particular, this implies  that, almost surely,
\begin{align}
\frac{1}{N}\sum_{i=1}^N\delta_{\obtheta^t_i}\stackrel{{\rm d}}{\Rightarrow} \rho_t\, .
\end{align}

\subsection{Proof of Theorem \ref{thm:GeneralPDE}: Convergence to the PDE}
\label{sec:ProofPDE}

The proof follows a `propagation of chaos' argument
\cite{sznitman1991topics}. Throughout this proof, we will use $K$ to denote generic constant depending on the constants $K_1, K_2, K_3$ in conditions
{\sf A1}, {\sf A2}, {\sf A3}.

It is convenient to introduce the notations $\bz_k = (\bx_k,y_k)$ to denote the $k$-th example and define
\begin{align}
\bF_i(\btheta;\bz_k) &= \big(y_k - \hy(\bx_k;\btheta)\big)\,\nabla_{\btheta_i}\sigma_*(\bx_k;\btheta_i)\,,~~~~~~~~ \btheta = (\btheta_i)_{i \le N} \in \reals^{D \times N}, \\
\bG(\btheta;\rho) &= -\nabla \Psi(\btheta;\rho) = -\nabla V(\btheta) -\int \nabla_{\btheta}U(\btheta,\btheta') \; \rho(\de\btheta')\,, ~~~~ \btheta \in \reals^D.
\end{align}
Note that the assumption of bounded Lipschitz $\nabla V$, $\nabla_1 U$ (here and below $\nabla_1 U(\btheta_1,\btheta_2)$ denotes the gradient of $U$ with respect to its first argument) implies $\|\bG(\btheta;\rho)\|_2\le K$ and $\|\bG(\btheta_1;\rho)-\bG(\btheta_2;\rho)\|_2\le K\|\btheta_1-\btheta_2\|_2$.  Further
\begin{align}
\|\bG(\btheta;\rho_1)-\bG(\btheta;\rho_2)\|_2= \Big\|\int \nabla_{\btheta} U(\btheta;\btheta') (\rho_1-\rho_2)(\de\btheta')\Big\|_2\le K\, d_{\sBL}(\rho_1,\rho_2)\, .
\label{eq:LipschitzG}
\end{align}

With these notations, we can rewrite the SGD dynamics [\ref{eq:First_SGD}] in the main text  as
\begin{align}
\btheta_i^{k+1}& = \btheta_i^{k}+2\eps\, \xi(k\eps) \, \bF_i(\btheta_i^k;\bz_{k+1})\, ,
\end{align}
which yields
\begin{align}
\btheta^k_i = \btheta^0_i +2\eps\sum_{\ell=0}^{k-1}\xi(\ell\eps)\, \bF_i(\btheta_i^{\ell};\bz_{\ell+1})\, .\label{eq:SGDCont}
\end{align}
Recall $(\btheta^0_i)_{i \le N} \sim \rho_0$ independently. 

For $t\in\reals_{\ge 0}$ we will define $[t] = \eps\lfloor t/\eps\rfloor$. Eq. (\ref{eq:SGDCont}) should be compared with the nonlinear dynamics 
(\ref{eq:NonLinearEvolution}), which reads
\begin{align}\label{eq:NonLinearEvolution2}
\obtheta^t_i &=\btheta^0_i+2\int_{0}^t\xi(s)\, \bG(\obtheta_i^s;\rho_s)\, \de s\, .
\end{align}

We next state and prove the key estimate controlling the difference between the original dynamics and the nonlinear dynamics.
\begin{lemma}\label{lemma:NonlinearDynamics}
Under the assumptions of Theorem \ref{thm:GeneralPDE}, there exists a constant $K$ depending uniquely on $K_1, K_2, K_3$ in conditions {\sf A1}, {\sf A2}, and {\sf A3}, such that for any $T\ge 0$, we have
\begin{align}
\max_{i \le N}\sup_{k \in  [0, T/\eps] \cap \mathbb N}\big\|\btheta^{k}_i-\obtheta^{k \eps}_i\big\|_2
\le K e^{KT} \cdot \sqrt{1/N \vee \eps } \cdot \Big[\sqrt{D + \log(N (T/ \eps \vee 1))} + z \Big]
\end{align}
with probability at least $1- e^{-z^2}$.
\end{lemma}
\begin{proof}
Consider for simplicity of notation $t\in \naturals\eps\cap [0,T]$.
Taking the difference of Eqs.~(\ref{eq:SGDCont}) and (\ref{eq:NonLinearEvolution2}), we get
\begin{equation}
\begin{aligned}
\big\|\btheta^{t/\eps}_i-\obtheta^t_i \big\|_2 =& 2\Big\|\int_{0}^t\xi(s)\, \bG(\obtheta_i^s;\rho_s)\, \de s-\eps\sum_{k=0}^{t/\eps-1}\xi(k\eps)\, \bF_i(\btheta^{k};\bz_{k+1})\, \de s\Big\|_2\\
 \le &2\int_{0}^t\Big\|\xi(s)\, \bG(\obtheta_i^s;\rho_s)-\xi([s])\, \bG(\obtheta_i^{[s]};\rho_{[s]})\Big\|_2\, \de s \\
 &+ 2\int_{0}^t\Big\|\xi([s])\, \bG(\obtheta_i^{[s]};\rho_{[s]}) -\xi([s])\, \bG(\btheta_i^{\lfloor s/\eps\rfloor};\rho_{[s]})\Big\|_2\, \de s\\
&+ 2\Big\|\eps \sum_{k=0}^{t/\eps-1}\xi(k\eps)\, \Big\{\bF_i(\btheta^{k};\bz_{k+1}) -\bG(\btheta_i^{k};\rho_{k\eps})\Big\}\Big\|_2\\
\equiv& 2 E_1^i(t)+2 E_2^i(t)+2E_3^i(t)\, . \label{eq:ErrorsDynamics}
\end{aligned}
\end{equation}
We next consider the three terms above. Using the Lipschitz continuity of $\bG(\btheta;\rho)$ with respect to 
$\btheta$ and $\rho$ (see Eq.~(\ref{eq:LipschitzG})), and due to condition {\sf A1} and Lemma \ref{lemma:Lipschitz_continuity_of_rho_and_theta} (implying that $\xi$, $\obtheta^t_i$, and $\rho_s$ are Lipschitz continuous), we get 
\begin{align}
E_1^i(t)&\le t\sup_{s\in [0,t]}\Big\{\big\|\xi(s)\, \bG(\obtheta_i^s;\rho_s)-\xi([s])\, \bG(\obtheta_i^{s};\rho_s)\big\|_2 +
\big\|\xi([s])\, \bG(\obtheta_i^s;\rho_s)-\xi([s])\, \bG(\obtheta_i^{[s]};\rho_{s})\big\|_2\nonumber\\
&\phantom{\le t\max_{s\in [0,t]}\Big\{}+ \big\|\xi([s])\, \bG(\obtheta_i^{[s]};\rho_s)-\xi([s])\, \bG(\obtheta_i^{[s]};\rho_{[s]})\big\|_2\Big\}\nonumber\\
&\le K\, t\, \eps\, . \label{eq:E1Bound}
\end{align}
Bounding the second term yields (by using the Lipschitz continuity of $\bG$ with respect to its first argument):
\begin{align}
E_2^i(t)& \le K\, \int_0^t \big\|\bG(\obtheta_i^{[s]};\rho_{[s]}) - \bG(\btheta_i^{\lfloor s/\eps\rfloor};\rho_{[s]})\big\|_2\de s \le K^2 \int_0^t  \big\|\obtheta_i^{[s]}-\btheta_i^{\lfloor s/\eps\rfloor }\big\|_2\de s\, . \label{eq:E2Bound}
\end{align}
In order to bound the last term we denote by $\cF_k$, for $k\in\naturals$, the sigma-algebra generated by $(\btheta^0_i)_{i\le N}$
and $\bz_1$,\dots,$\bz_k$.
Note that 
\begin{align}
\E\big\{\bF_i(\btheta^{k};\bz_{k+1})\big|\cF_{k}\big\}&= -\nabla V(\btheta_i^k) -\frac{1}{N}\sum_{j=1}^N\nabla_1 U(\btheta^k_i,\btheta^k_j) =  \bG(\btheta_i^k;\hrho^{(N)}_k)\,,
\end{align}
where $\hrho^{(N)}_k \equiv (1/N) \sum_{i \le N} \delta_{\btheta^k_i}$. Hence
\begin{align}
E_3^i(t)& \le \Bigg\|\eps \sum_{k=0}^{t/\eps-1}\xi(k\eps) \Big\{\bG(\btheta_i^{k};\hrho^{(N)}_{k})-\bG(\btheta_i^{k};\rho_{k\eps})\Big\}\Bigg\|_2+\Bigg\|\eps \sum_{k=0}^{t/\eps-1}\xi(k\eps)\bZ_k^i\Bigg\|_2\\
& \equiv E_{3,0}^i(t)+Q_1^i(t)\, ,\label{eq:E3Bound_First}
\end{align}
where we introduced the martingale differences $\bZ_k^i\equiv \bF_i(\btheta^{k};\bz_{k+1}) -\E\big\{\bF_i(\btheta^{k};\bz_{k+1})\big|\cF_{k}\big\}$.
We can apply  Azuma-Hoeffding inequality, cf. Lemma \ref{lemma:AH}.
Indeed, condition (\ref{eq:MG_diff}) follows from the fact that $\sigma_*(\bx;\btheta)$ is bounded and $\nabla_{\btheta}\sigma_*(\bx;\btheta)$ is
sub-Gaussian (the product of a sub-Gaussian random vector and a bounded random variable is sub-Gaussian, cf. for instance Lemma 1.$(d)$ 
in \cite{mei2016landscape}), hence each $\xi(k \eps) \bZ_k^i$ are $K^2$-sub-Gaussian. We therefore get
\begin{align}
\prob\Big(\max_{k \in  [0, t/\eps] \cap \mathbb N} Q_1^i( k \eps)\ge K\sqrt{t\eps } (\sqrt{D} + u) \Big)\le e^{-u^2}\, \,,\label{eq:BoundW10}
\end{align}
and taking union bound over $i \le N$, we get
\begin{align}
\prob\Big(\max_{i \le N}\max_{k \in  [0, t/\eps] \cap \mathbb N} Q_1^i( k \eps)\le K\sqrt{t\eps}\, ( \sqrt{D + \log N} + z)\Big)\ge  1 - e^{-z^2}\, \,.\label{eq:BoundW1}
\end{align}
For the term $E_{3,0}^i(t)$, we use the Lipschitz continuity property (\ref{eq:LipschitzG}), whence
\begin{equation}
\begin{aligned}
&\big\|\bG(\btheta_i^{k};\hrho^{(N)}_{k})-\bG(\btheta_i^{k};\rho_{k\eps})\big\|_2\\
\le& \Big\|\frac{1}{N}\sum_{j=1}^N\big[\nabla_1 U(\btheta_i^{k},\btheta_j^k)-\nabla_1 U(\btheta_i^{k},\obtheta_j^{k\eps})\big]\Big\|_2 +\Big\|\frac{1}{N}\sum_{j=1}^N\big[\nabla_1 U(\btheta_i^{k},\obtheta_j^{k\eps})-\E_{\obtheta} \nabla_1 U(\btheta_i^{k},\obtheta_j^{k\eps})\big]\Big\|_2\\
\le& \frac{K}{N}\sum_{j=1}^N\|\btheta_j^k-\obtheta_j^{k\eps}\|_2+Q_2^i(k\eps)+\frac{K}{N}\, .
\end{aligned}
\end{equation}
Here $Q_2^i(k\eps)$ for $k \in \mathbb N $ is defined as
\begin{align*}
Q_2^i(k\eps) = \Big\|\frac{1}{N}\sum_{j \le N, j \neq i }\big[\nabla_1 U(\btheta_i^{k},\obtheta_j^{k\eps})-\E_{\obtheta} \nabla_1 U(\btheta_i^{k },\obtheta_j^{k\eps})\big]\Big\|_2\, .
\end{align*}
Since for any fixed $k$, $(\obtheta_j^{k\eps})_{j \le N, j \neq i}$ are i.i.d. and independent of $\btheta_i^{k}$, and $\nabla_1 U$ is bounded, we get by another application of Azuma-Hoeffding inequality, cf. Lemma \ref{lemma:AH}, 
\begin{align}
\prob \Big( Q_{2}^i(k\eps) \ge K \sqrt{1/N} (\sqrt{D} + u) \Big)\le \, e^{-u^2}\, .
\end{align}
Therefore, the union bound for $k \in  [0, t/\eps] \cap \mathbb N$, and $i \le N$ gives
\begin{align}
\prob\Big(\max_{i \le N} \max_{k \in  [0, t/\eps] \cap \mathbb N} Q_{2}^i(k \eps) \le K \sqrt{1/N} \cdot\Big(\sqrt{D + \log(N (t/\eps \vee 1))}+z\Big)\Big) \ge  1 -e^{-z^2}\,  . \label{eq:BoundW2}
\end{align}
Conditional on the good events in Eq. (\ref{eq:BoundW1}) and (\ref{eq:BoundW2}), Eq. (\ref{eq:E3Bound_First}) thus yields
\begin{align}
E_3^i(t)& \le \frac{K}{N}\sum_{j=1}^N\int_0^t \|\btheta_j^{\lfloor s/\eps\rfloor}-\obtheta_j^{[s]}\|_2\, \de s+ Q(t)+\frac{Kt}{N}\, ,\label{eq:E3Bound}
\end{align}
where 
\begin{equation}\label{eqn:BoundW0}
\begin{aligned}
Q(t) \equiv & \max_{i \le N}Q_1^i(t)+t \cdot \max_{i \le N} \max_{k \in  [0, t/\eps] \cap \mathbb N}Q_2^i(k \eps) \\
\le& K\sqrt{t\eps}\, \Big( z + \sqrt{D + \log N}\Big) + t K \sqrt{1/N}\Big(\sqrt{D + \log(N (t/\eps \vee 1))}+z\Big)\\
\le& K(\sqrt{t} \vee t) \cdot \sqrt{1/N \vee \eps } \cdot \Big[\sqrt{D + \log (N (t/ \eps \vee 1))} + z \Big] . 
\end{aligned} 
\end{equation}
with probability at least $1 - e^{-z^2}$. 

We finally define the random variable
\begin{align}
\Delta(t;N,\eps) &\equiv \max_{i \le N}\sup_{k \in  [0, t/\eps] \cap \mathbb N}\|\btheta^{k}_i-\obtheta^{k\eps}_i\|_2\, .
\end{align}
Using  the bounds (\ref{eq:E1Bound}), (\ref{eq:E2Bound}), (\ref{eq:E3Bound}) in Eq.~(\ref{eq:ErrorsDynamics}), we get
\begin{align}
\Delta(t;N,\eps) &\le K\int_0^t\Delta(s;N,\eps) \de s+K\, t\eps+ \frac{Kt}{N} + Q(t)\,  .
\end{align}
By Gronwall's inequality, we have
\begin{align}
\Delta(t;N,\eps) &\le K \, e^{Kt}\Big\{\eps+ \frac{1}{N} +Q(t)\Big\}\,  .
\end{align}
Using the bound (\ref{eqn:BoundW0}), the claim follows.
\end{proof}



\begin{lemma}\label{lem:bound_risk_SGD_ODE}
Under the assumptions of Theorem \ref{thm:GeneralPDE}, we have 
\begin{align}\label{eqn:bound_risk_SGD_ODE}
\max_{k \in  [0, T/\eps] \cap \mathbb N} \Big \l R_N(\obtheta^{k\eps})-R_{N}(\btheta^k)\Big \l \le K \cdot \max_{i\le N}\max_{k \in  [0, T/\eps] \cap \mathbb N}\big\|\btheta^{k}_i-\obtheta^{k \eps}_i\big\|_2. 
\end{align}
\end{lemma}

\begin{proof}

Let $\btheta = (\btheta_1,\dots,\btheta_i,\dots,\btheta_n)$ and $\btheta' = (\btheta_1,\dots,\btheta_i',\dots,\btheta_n)$ be two
configurations that differ only in position $i$. Then
\begin{equation}\label{eqn:risk_coordinate_bound}
\begin{aligned}
& \big \l R_N(\btheta)-R_N(\btheta')\big \l \\
\le& \frac{1}{N} \l V(\btheta_i)-V(\btheta_i')\l  +\frac{1}{N^2} \l U(\btheta_i,\btheta_i)-U(\btheta_i',\btheta'_i) \l +\frac{2}{N^2}\sum_{j \le N, j \neq i} \l U(\btheta_i,\btheta_j)-U(\btheta_i',\btheta_j)\l \\
\le& \frac{K}{N} ( \| \btheta_i - \btheta_i' \|_2 \wedge 1 ).
\end{aligned}
\end{equation}
Then, Eq. (\ref{eqn:bound_risk_SGD_ODE}) follows immediately. 

\end{proof}

\begin{lemma}\label{lemma:RiskDynamics}
Under the assumptions of Theorem \ref{thm:GeneralPDE}, we have,
\begin{align}
\max_{k \in  [0, T/\eps] \cap \mathbb N} \Big \l R_N(\obtheta^{k\eps})-R(\rho_{k \eps})\Big\l \le
 K \sqrt{1/N} \cdot \Big(\sqrt{D + \log(N(T / \eps \vee 1))}+z\Big)
\end{align}
with probability at least $1- e^{-z^2}$.
\end{lemma}
\begin{proof}
By Eq. (\ref{eqn:risk_coordinate_bound}) and by Azuma-H\"oeffding inequality and union bound, we get 
\begin{align}
\max_{k \in  [0, T/\eps] \cap \mathbb N}  \Big\l R_N(\obtheta^{k\eps})-\E R_{N}(\obtheta^{k \eps})\Big\l  \le K \sqrt{1/N} \cdot \Big(\sqrt{D + \log(N(T / \eps \vee 1))}+z\Big)
\end{align}
with probability at least $1 - e^{- z^2}$. 
The claim follows since
\begin{align}
\Big\l \E R_N(\obtheta^t) - R(\rho_t)\Big\l = \frac{1}{N}\Big\l \int U(\btheta,\btheta) \, \rho_t(\de\btheta)-\int U(\btheta_1,\btheta_2) \, \rho_t(\de\btheta_1)\, \rho_t(\de\btheta_2)\Big\l\le \frac{K}{N}\, .
\end{align}
\end{proof}

The proof of the theorem follows from a straightforward application of Lemma \ref{lemma:NonlinearDynamics}, \ref{lem:bound_risk_SGD_ODE}, \ref{lemma:RiskDynamics}. The proof for any bounded Lipschitz function $f$ follows the same argument as Lemma \ref{lem:bound_risk_SGD_ODE}, \ref{lemma:RiskDynamics}. As a result, for any sequence $(N, \eps = \eps_N)$ such that $N \to \infty$ and $\eps_N \to 0$ with $N / \log(N/\eps_N) \to \infty$ and $\eps_N \log (N /\eps_N) \to 0$, we have $\hat \rho_{\lfloor k/\eps\rfloor}^{(N)}$ converges weakly to $\rho_{t}$ almost surely immediately.

\subsection{Proof of Theorem \ref{thm:GeneralPDE}: Generalization to $\beta<\infty$}

Here we generalize the proof given in the previous section to noisy SGD at finite temperature $\beta<\infty$. Since the proof follows the same
scheme as in the noiseless case, we will limit ourselves to describing the differences. 

Throughout this section we assume that conditions ${\sf A1}$, ${\sf A2}$, ${\sf A3}$ hold. We also let
\begin{align}
\Psi_\lambda(\btheta; \rho) =& \frac{\lambda}{2} \| \btheta \|_2^2 + V(\btheta) + \int U(\btheta, \btheta') \rho(\btheta') \de \btheta'
\end{align}
for some $\lambda \le 1$. Further we assume $\rho_0$ is $K_0^2$-sub-Gaussian. Finally, we assume $1 \le \beta < \infty$. 

For the reader's convenience, we reproduce here the form of the limiting PDE
\begin{align}\label{eqn:diffusion_PDE_SGD_to_PDE}
\partial_t\rho_t(\btheta) =& 2 \xi(t) \nabla_\btheta \cdot \big[\rho_t(\btheta)\nabla_\btheta \Psi_\lambda(\btheta;\rho_t)\big]+2\xi(t)/\beta \cdot \,\Delta_\btheta \rho_t(\btheta)\,,
\end{align}
which again should be interpreted in weak sense. 

\begin{remark} \label{rmk:ExistenceUniqueness_Noisy}
Recall conditionss {\sf A1}, {\sf A2}, {\sf A3} in the main text. 
By a modified argument of \cite[Theorem 1.1]{sznitman1991topics}, conditions {\sf A1} and {\sf A3} are sufficient  for the existence and uniqueness of solution of PDE (\ref{eqn:diffusion_PDE_SGD_to_PDE}) in weak sense. Section \ref{sec:FiniteT} provides further information of this PDE, including a proof of existence and uniqueness. 
\end{remark}

As in the noiseless case, there is an equivalent formulation of this PDE as a fixed point distribution for the following nonlinear dynamics, which is an integration form of a stochastic differential equation, 
\begin{align}
\obtheta^t_i &=\btheta^0_i+2\int_{0}^t\xi(s)\, \bG(\obtheta_i^s;\rho_s)\, \de s + \int_0^t \sqrt{2\xi(s) / \beta}\, \de \bW_i(s)\, ,\label{eq:NonLinearEvolutionNoisy}\\
\rho_s & = {\rm P}_{\obtheta^s_i}\, ,
\end{align}
where $\{\bW_i(s)\}_{s\ge 0}$ for $i \le N$ are independent $D$-dimensional Brownian motions, and  $\bG(\btheta;\rho) \equiv -\nabla\Psi_\lambda(\btheta;\rho)$. The assumptions on $U$, $V$, $\lambda$, and $\xi$ ensures that this nonlinear dynamics has a unique continuous solution. 

This nonlinear dynamics should be compared with the noisy SGD dynamics [\ref{eq:Noisy_SGD}] in the main text that can be written as follows for $k \in \mathbb N$:
\begin{align}
\btheta^k_i &=\btheta^0_i+2\eps\sum_{\ell=0}^{k-1}\xi(\ell\eps)\, \bF_i(\btheta^{\ell};\bz_{\ell}) + \int_0^{k \eps} \sqrt{2\xi([s]) / \beta}\, \de \bW_i(s)\,,\label{eq:SGD_Noisy}
\end{align}
where 
\begin{align}
\bF_i(\btheta;\bz_k) &= - \lambda \btheta_i +  \big(y_k - \hy(\bx_k;\btheta)\big)\,\nabla_{\btheta_i}\sigma_*(\bx_k; \btheta_i),~~~~ \btheta = (\btheta_i)_{i \le N} \in \reals^{D \times N}. 
\end{align}

It is convenient to collect some standard estimates about the solution of the stochastic differential equation (\ref{eq:NonLinearEvolutionNoisy}).

\begin{lemma}\label{lem:SDE_norm_bound}
Assume $\rho_0$ is $K_0^2$-sub-Gaussian, $\xi(s)$ and $\bG(\bzero; \rho_s)$ are $K_0$-bounded, $\bG(\btheta; \rho_s)$ is $K_0$-Lipschitz in $\btheta$, and $\beta \ge 1$. Let $(\obtheta_i^t)_{t \ge 0}$ for $i \le N$ be the solution of (\ref{eq:NonLinearEvolutionNoisy}) with independent initialization $(\btheta_i^0)_{i \le N} \sim \rho_0$. Let $(\rho_t)_{t \ge 0}$ be the solution of PDE (\ref{eqn:diffusion_PDE_SGD_to_PDE}). Then there exists a constant $K$ depending uniquely on $K_0$, such that
\begin{equation}\label{eq:BoundNormSDE_probability}
\prob\Big( \sup_{i \le N} \sup_{t \in [0, T]} \|\obtheta_i^t\|_2\le K e^{KT}  [\sqrt{D + \log N} + z] \Big)  \ge 1 - e^{-z^2},
\end{equation}
and 
\begin{equation}\label{eqn:SDE_difference_bound}
\prob\Big(\sup_{i \le N} \sup_{k \in [0, T/\eps] \cap \mathbb N}\sup_{u\in [0, \eps]}\| \obtheta_i^{k\eps + u} - \obtheta_i^{k\eps}\|_2 \le K e^{KT} \Big[\sqrt{D + \log (N(T/\eps \vee 1))} + z\Big] \sqrt{\eps} \Big) \ge 1 - e^{-z^2}, 
\end{equation}
and for any $t,h\ge 0$, $t + h \le T$, 
\begin{equation}\label{eq:BoundDist_FiniteT}
 d_{\sBL}(\rho_t,\rho_{t+h})\le W_2(\rho_t,\rho_{t+h})\le K e^{KT}\, \sqrt{D\, h}\, .
\end{equation}

%

\end{lemma}

\begin{proof}We decompose the proof into three parts. 

\noindent
{\bf Part (a). } First, note that for any $D$-dimensional $K_0^2$-sub-Gaussian random vector $\bX$, we have 
\begin{equation}\label{eqn:moment_of_exp_sub_gaussian_square}
\begin{aligned}
\E_{\bX}[\exp\{ \tau \| \bX \|_2^2 /2 \}] =& \E_{\bX, \bG}[\exp\{ \tau\<\bG, \bX \> \}] \le \E_{\bG}[\exp\{ \tau K_0^2 \| \bG \|_2^2 \}/2] = (1 - \tau K_0^2)^{-D/2}. 
\end{aligned}
\end{equation}
Note that $(\btheta_i^0)_{i \le N} \sim \rho_0$ independently, and $\rho_0$ is $K_0^2$-sub-Gaussian. Therefore
\[
\begin{aligned}
\prob(\| \btheta_i^0 \|_2 \ge u) \le \E[\exp(\tau \| \btheta_i \|_2^2/2)] / \exp\{ \tau z^2/2 \} \le (1 - \tau K_0^2)^{-D/2} \exp \{ - \tau u^2/2 \}.
\end{aligned}
\]
Taking union bound over $i \le N$ gives
\[
\begin{aligned}
\prob\Big(\max_{i \le N}\| \btheta_i^0 \|_2 \ge u\Big) \le (1 - \tau K_0^2)^{-D/2} \exp\{ - \tau u^2/2 + \log N \}. 
\end{aligned}
\]
Taking $\tau = 1/(2 K_0^2)$ and $u = 2K_0 (\sqrt{D + \log N} + z)$, we get 
\begin{equation}\label{eqn:bound_for_Theta_in_finite_beta}
\begin{aligned}
\prob\Big(\max_{i \le N}\| \btheta_i^0 \|_2 \ge 2 K_0(\sqrt{D + \log N} + z) \Big) \le  \exp\{ - z^2 \}. 
\end{aligned}
\end{equation}

Then we define $\bW_{\xi, i}(t) \equiv \int_0^t \sqrt{2\xi(s)} \, \de \bW_i(s)$. We have $\Var(W_{\xi, i}^j(t)) = \int_0^t 2 \xi(s) \de s \le 2 K_0 t$ for $j \le D$. Note $\exp\{ \tau \| \bW_{\xi, i}(t)\|_2^2 \}$ is a submartingale, due to Doob's martingale inequality, we have
\[
\begin{aligned}
\prob\Big( \sup_{t \le T} \| \bW_{\xi, i}(t) \|_2 \ge u \Big) \le \E[ \exp\{ \tau \| \bW_{\xi, i}(T) \|_2^2/2 \} ] \cdot \exp\{ -\tau u^2/2 \} \le (1 - 2 K_0 T \tau )^{-D/2} \exp \{ - \tau u^2/2 \}. 
\end{aligned}
\]
Taking union bound over $i \le N$ gives
\[
\begin{aligned}
\prob\Big( \max_{i \le N}\sup_{t \le T} \| \bW_{\xi, i}(t) \|_2 \ge u \Big) \le (1 - 2 K_0 T \tau )^{-D/2} \exp \{ - \tau u^2/2 + \log N \}. 
\end{aligned}
\]
Taking $\tau = 1/(4 K_0 T)$ and $u = 4 \sqrt{K_0 T} (\sqrt{D + \log N} + z)$, we get 
\begin{equation}\label{eqn:bound_for_W_in_finite_beta}
\begin{aligned}
\prob\Big(\max_{i \le N}\sup_{t \le T} \| \bW_{\xi, i}(t) \|_2 \ge 4 \sqrt{K_0 T} (\sqrt{D + \log N} + z) \Big) \le  \exp\{ - z^2 \}. 
\end{aligned}
\end{equation}

By noting that $\xi(s)$, $\bG(\bzero; \rho_s)$ are $K_0$-bounded, and $\bG(\btheta; \rho_s)$ is $K_0$-Lipschitz in $\btheta$, according to Eq. (\ref{eq:NonLinearEvolutionNoisy}), there exists some constant $K$ depending on $K_0$, such that
\[
\begin{aligned}
\Delta_i(t) \le K \int_0^t \Delta_i(s) \de s + K [W /\sqrt{\beta}+ \Theta],
\end{aligned}
\]
where $\Delta_i(t) \equiv \sup_{s \le t} \| \obtheta_i^s\|_2$, $W \equiv \max_{i \le N}\sup_{t \le T} \| \bW_{\xi, i}(t) \|_2$, and $\Theta \equiv \max_{i \le N}\| \btheta_i^0 \|_2$. Due to Gronwall's inequality, we have
\[
\begin{aligned}
\Delta_i(T) \le K \exp(KT) [W / \sqrt{\beta} + \Theta]. 
\end{aligned}
\]
The high probability bound (\ref{eq:BoundNormSDE_probability}) holds by noting the high probability bound for $\Theta$ and $W$ in Eq. (\ref{eqn:bound_for_Theta_in_finite_beta}) and (\ref{eqn:bound_for_W_in_finite_beta}). 


\noindent
{\bf Part (b).} 
Define $\Delta_i(h; k, \eps) = \sup_{0 \le u \le h} \| \obtheta_i^{k\eps + u} - \obtheta_i^{k\eps} \|_2$. By noting that $\xi(s)$, $\bG(\bzero; \rho_s)$ are $K_0$-bounded, and $\bG(\btheta; \rho_s)$ is $K_0$-Lipschitz in $\btheta$, according to Eq. (\ref{eq:NonLinearEvolutionNoisy}), we have
\begin{equation}\label{eqn:bound_SDE_each_segment1}
\begin{aligned}
\Delta_i(h; k, \eps) \le& K \Big[ \sup_{s \le T} \| \obtheta_i^{s} \|_2 + 1\Big] h + \frac{1}{\sqrt \beta} \sup_{0 \le u \le h} \big\|  \bW_{\xi, i, k}(u)\big\|_2, 
\end{aligned}
\end{equation}
where $\bW_{\xi, i, k}(u) \equiv \int_{k \eps}^{k\eps + u} \sqrt{2\xi(s)} \, \de \bW_i(s)$. Similar to the bound Eq. (\ref{eqn:bound_for_W_in_finite_beta}), we have 
\begin{equation}\label{eqn:bound_BM_each_segment}
\begin{aligned}
\prob\Big(\max_{i \le N}\sup_{k \in [0, T/\eps] \cap \mathbb N} \sup_{0 \le u \le h} \| \bW_{\xi, i, k}(u) \|_2 \le 4 \sqrt{K_0 h} \Big(\sqrt{D + \log (N(T/\eps \vee 1))} + z\Big) \Big) \ge  1 - e^{-z^2}. 
\end{aligned}
\end{equation}
Plugging the bound Eq. (\ref{eq:BoundNormSDE_probability}) and Eq. (\ref{eqn:bound_BM_each_segment}) into Eq. (\ref{eqn:bound_SDE_each_segment1}), we have 
\[
\begin{aligned}
\max_{i \le N} \sup_{k \in [0, T/\eps] \cap \mathbb N} \Delta_i(h; k, \eps) \le& K e^{KT}  [\sqrt{D + \log N} + z] h + K\Big(\sqrt{D + \log (N(T/\eps \vee 1))} + z\Big) \sqrt{h}\\
\le&  K e^{KT} \Big[\sqrt{D + \log (N (T / \eps \vee 1))} + z\Big] \sqrt{h}
\end{aligned}
\]
with probability at least $1 - e^{-z^2}$. 

\noindent
{\bf Part (c). } Equation (\ref{eq:BoundDist_FiniteT}) holds directly by noting that 
\[
\begin{aligned}
W_2(\rho_t,\rho_{t+h})^2 \le& \E\big\{\|\obtheta^t-\obtheta^{t+h}\big\|_2^2\big\} \\
\end{aligned}
\]
and applying a integration over $z$ in a modified version of Eq. (\ref{eqn:SDE_difference_bound}) without union bound over $i \le N$ and $k \in [0, T/\eps] \cap \mathbb N$. 

\end{proof}

As in the noiseless case, the key step consists in bounding the difference between the nonlinear dynamics and the SGD dynamics.
\begin{lemma}\label{lemma:NonlinearDynamicsNoisy}
Under the assumptions of Theorem \ref{thm:GeneralPDE}, there exists a constant $K$ depending uniquely on $K_0, K_1, K_2, K_3$, such that for any $T\ge 0$, we have
\begin{align}
\max_{i \le N}\sup_{k \in [0, T/\eps] \cap \mathbb N}\|\btheta^{k}_i-\obtheta^{k\eps}_i\|_2
\le K e^{KT} \cdot \sqrt{1/N \vee \eps } \cdot \Big[\sqrt{D + \log (N (T/ \eps \vee 1))} + z \Big]
\end{align}
with probability at least $1- \, e^{-z^2}$.
\end{lemma}
\begin{proof}
We take the difference of Eqs.~(\ref{eq:SGD_Noisy}) and (\ref{eq:NonLinearEvolutionNoisy}), for $t\in \naturals\eps\cap [0,T]$:
\begin{equation}\label{eq:ErrorsDynamicsNoisy}
\begin{aligned}
\big\|\btheta^{t/\eps}_i-\obtheta^t_i \big\|_2 \le &2\Big\|\int_{0}^t\Big[\xi(s)\, \bG(\obtheta_i^s;\rho_s)-\xi([s])\, \bG(\obtheta_i^{[s]};\rho_{[s]})\Big]\de s\Big\|_2\\
 &+
 2\int_{0}^t\Big\|\xi([s])\, \bG(\obtheta_i^{[s]};\rho_{[s]}) -\xi([s])\, \bG(\btheta_i^{\lfloor s/\eps\rfloor};\rho_{[s]})\Big\|_2\, \de s\\
&+ 2\Big\|\eps \sum_{k=0}^{t/\eps-1}\xi(k\eps)\, \Big\{\bF_i(\btheta^{k};\bz_{k+1}) -\bG(\btheta_i^{k};\rho_{k\eps})\Big\}\Big\|_2 \\
&+\Big\|\int_0^t \big(\sqrt{2\xi(s) /\beta}-\sqrt{2\xi([s]) / \beta}\big)\, \de \bW_i(s)\Big\|_2\\
 \equiv & 2 E_1^i(t)+2 E_2^i(t)+2E_3^i(t) + E_4^i(t)\, . 
\end{aligned}
\end{equation}
Terms $E_2^i(t)$, $E_3^i(t)$ can be bounded the same as in Lemma \ref{lemma:NonlinearDynamics}, i.e., Eq. (\ref{eq:E2Bound}) and (\ref{eq:E3Bound}), by noting that the replacement of $\Psi$ by $\Psi_\lambda$ does not affect these estimates. 

To bound $E_4^i(t)$, notice that $\bW_{\xi, i} \equiv \int_0^T \big(\sqrt{2\xi(s)}-\sqrt{2\xi([s])}\big)\, \de \bW_i(s)$ is a Gaussian random vector, $\bW_{\xi, i}\sim \normal(\bzero,\tau^2\id_D)$, 
where, using the Lipschitz continuity of $\xi$, 
\[
\tau^2 = \int_0^T \Big(\sqrt{2\xi(s)}-\sqrt{2\xi([s])}\Big)^2 \, \de s\le  K\, T \eps\, .
\]
By Gaussian concentration
\[
\prob\big(\|\bW_{\xi, i}\|_2\ge (\sqrt{D}+z)\tau\big) \le e^{-z^2/2}\,, 
\]
and therefore by applying Doob's inequality to the submartingale $t\mapsto E_4^i(t)$, we get
\[
\prob\Big(\max_{s\le T} E_4^i(s) \ge K(\sqrt{D}+z) \sqrt{T\eps} \Big)\le e^{-z^2/2},
\]
and hence
\begin{align}\label{eq:E4Bound}
\prob\Big(\max_{i \le N} \max_{s\le T} E_4^i(s) \le K(\sqrt{D + \log N} + z) \sqrt{T \eps}\Big)\ge 1 - e^{-z^2/2}. 
\end{align}

We need to modify the proof of Lemma \ref{lemma:NonlinearDynamics} to bound terms $E_1^i(t)$. 
\begin{equation}
\begin{aligned}
E_1^i(t)\le& \Big\| \int_{0}^t \big[\xi(s) -\xi([s])\big] \bG(\obtheta_i^{s};\rho_s) \de s\Big\|_2 + \Big\| \int_0^t \xi([s]) \big[ \bG(\obtheta_i^{s};\rho_s)- \bG(\obtheta_i^{s};\rho_{[s]}) \big] \de s\Big\|_2\\
&+\Big\|\int_0^t\xi([s])\, \big[\bG(\obtheta_i^s;\rho_{[s]})- \bG(\obtheta_i^{[s]};\rho_{[s]}) \big]\de s\Big\|_2 \\
\equiv& E_{1, A}^i(t) + E_{1, B}^i(t) + E_{1, C}^i(t). 
\end{aligned}
\end{equation}

To bound the first term $E_{1, A}^i(t)$, due to the Lipschitz property of $\bG(\btheta; \rho)$ and the boundedness of $\bG(\bzero; \rho)$, with probability at least $1 - e^{-z^2}$, we have for all $i \le N$ and $t \le T$, 
\begin{equation}\label{eqn:bound_E_1_A}
\begin{aligned}
E_{1, A}^i(t) \le& T K \eps \cdot \sup_{s \in [0, T]} \| \bG(\obtheta_i^s; \rho_s) \|_2 \le T K \eps \cdot \Big[ K \sup_{s \in [0, T]} \| \obtheta_i^s \|_2 + K\Big]\\
\le&  K e^{KT}  [\sqrt{D + \log N} + z] \eps. 
\end{aligned}
\end{equation}
Here the last inequality is due to Eq. (\ref{eq:BoundNormSDE_probability}) in Lemma \ref{lem:SDE_norm_bound}.

To bound the second term $E_{1, B}^i(t)$, using the fact that $\nabla_{1} U$ is bounded Lipschitz, we have for all $i \le N$ and $t \le T$, 
\begin{equation}\label{eqn:bound_E_1_B}
\begin{aligned}
E_{1, B}^i(t) \le& T K  \cdot \sup_{\btheta \in \reals^D} \| \nabla_1 U(\btheta; \rho_s) - \nabla_1 U(\btheta; \rho_{[s]}) \|_2 \le  T K^2 \cdot d_{\sBL}(\rho_s, \rho_{[s]}) \le K e^{KT} \sqrt{D \eps}. 
\end{aligned}
\end{equation}
Here the last inequality is due to Eq. (\ref{eq:BoundDist_FiniteT}) in Lemma \ref{lem:SDE_norm_bound}.

To bound the third term $E_{1, C}^i(t)$, with probability at least $1 - e^{-z^2}$, we have for all $i \le N$ and $t \le T$, 
\begin{equation}\label{eqn:bound_E_1_C}
\begin{aligned}
E_{1, C}^i(t) \le& T K \cdot \sup_{s \in [0, T]} \big\| \bG(\obtheta_i^s;\rho_{[s]})- \bG(\obtheta_i^{[s]};\rho_{[s]}) \big\|_2 \\
\le& T K^2 \cdot \sup_{s \in [0, T]} \big\| \obtheta_i^s - \obtheta_i^{[s]} \big\|_2 \le K e^{KT} \Big[\sqrt{D + \log (N(T/\eps\vee 1))} + z\Big] \sqrt{\eps}.
\end{aligned}
\end{equation}
Here the last inequality is due to Eq. (\ref{eqn:SDE_difference_bound}) in Lemma \ref{lem:SDE_norm_bound}. 

As a result, combining Eq. (\ref{eq:E2Bound}), (\ref{eq:E3Bound}), (\ref{eqn:BoundW0}), (\ref{eq:ErrorsDynamicsNoisy}), (\ref{eq:E4Bound}), (\ref{eqn:bound_E_1_A}), (\ref{eqn:bound_E_1_B}), and (\ref{eqn:bound_E_1_C}), defining
\begin{align}
\Delta(t;N,\eps) &\equiv \max_{i \le N}\sup_{k \in [0, T/\eps] \cap \mathbb N}\|\btheta^{k}_i-\obtheta^{k\eps}_i\|_2\,,
\end{align}
we get 
\begin{equation}
\begin{aligned}
\Delta(t;N,\eps) &\le K\int_0^t\Delta(s;N,\eps) \de s + \frac{Kt}{N} + E(T),
\end{aligned}
\end{equation}
where
\begin{equation}
E(T) = K e^{KT} \cdot \sqrt{1/N \vee \eps } \cdot \Big[\sqrt{D + \log (N (T/ \eps \vee 1) )} + z \Big]. 
\end{equation}
Applying Gronwall's inequality gives the desired result. 

\end{proof}

The generalization of Theorem \ref{thm:GeneralPDE} to $\beta<\infty$ follows from this lemma exactly as in the previous section.

\subsection{Proof of Proposition \ref{thm:FixedPoints}: Monotonicity of the risk}

By simple algebra, we have 
\begin{align}
R(\rho_{t+h})-R(\rho_t) = 2\int \Psi(\btheta;\rho_t)  (\rho_{t+h}-\rho_t)(\de\btheta) +\<U,  (\rho_{t+h}-\rho_t)^{\otimes 2}\>\, .
\end{align}
By Lemma \ref{lemma:Lipschitz_continuity_of_rho_and_theta}, $t\mapsto \rho_t$ is Lipschitz continuous in Wasserstein distance
$W_2(\rho_{t_1},\rho_{t_2}) \le K|t_1-t_2|$. Hence, we get
\begin{align}
R(\rho_{t+h})-R(\rho_t) &= 2\int \Psi(\btheta;\rho_t)  (\rho_{t+h}-\rho_t)(\de\btheta) + O(h^2)\\
& = - 4\xi(t) \, \int \big\|\nabla \Psi(\btheta;\rho_t)\big\|_2^2\,\rho_t(\de\btheta)\, h+o(h)\, ,
\end{align}
where in the second step we used Eq.~(\ref{eq:WeakSolution}). This immediately implies that $R(\rho_t)$ is non-increasing in $t$. 

Let $\rho$ be a fixed point of Eq.~(\ref{eq:GeneralPDE_App}).  Since $\partial_tR(\rho_t)|_{\rho_0=\rho}=0$, the above formula implies
\begin{align}
\int \big\|\nabla \Psi(\btheta;\rho)\big\|_2^2\,\rho(\de\btheta) = 0\, ,
\end{align}
and therefore $\rho$ is supported in the set of $\btheta$'s such that $\nabla \Psi(\btheta;\rho)=\bzero$.

Vice versa, if this is the case, setting $\rho_0=\rho$,  Eq.~(\ref{eq:WeakSolution}) implies $\partial_t\<\varphi,\rho_t\> = 0$, then $\rho_t \equiv \rho_0$ is a fixed point.

\subsection{A general continuity result}

It is useful to notice that the solution $(\rho_t)_{t\ge 0}$ of the PDE (\ref{eq:GeneralPDE_App}) is continuous with respect to changes in $V(\, \cdot\,)$, $U(\,\cdot\,,\,\cdot\,)$. 
Namely, we consider the following two PDEs:
\begin{align}
\partial_t\rho_t(\btheta) =  2 \xi(t) \nabla\cdot \big[\rho_t(\btheta)\nabla \Psi(\btheta;\rho_t)\big]\,, \label{eq:OriginalPDE}\\
\partial_t\trho_t(\btheta) =  2 \xi(t) \nabla\cdot \big[\trho_t(\btheta)\nabla\tilde\Psi(\btheta;\trho_t)\big]\,, \label{eq:PerturbedPDE}
\end{align}
where
\begin{align}
\Psi(\btheta;\rho) = V(\btheta)+ \int U(\btheta, \btheta')\; \rho(\de\btheta')\, ,\\
\tilde\Psi(\btheta;\trho) = \tV(\btheta)+ \int \tU(\btheta, \btheta')\; \trho(\de\btheta')\,.
\end{align}

\begin{lemma}\label{lem:perturbation_convergence}
Let assumptions {\sf A1}, {\sf A3} hold both for $V,U$ and $\tV,\tU$, and consider the solutions of  Eqs.~(\ref{eq:OriginalPDE}) and (\ref{eq:PerturbedPDE})
with initial conditions $\rho_0$, $\trho_0$. Then there exists  $K< \infty$ depending only on the constants $K_1$, $K_3$ in the assumptions
 (independent of $D$), such that
\begin{align}
 \sup_{t \in [0, T]} d_{\sBL}(\rho_t,\trho_t) \le K\, e^{KT}\,\cdot [ d_{\sBL}(\rho_0, \trho_0) + \eps_0] \, ,
\end{align}
where
\begin{align}\label{eqn:perturbation_error_in_perturbation_lemma}
\eps_0 \equiv \sup_{\btheta, \btheta' \in \reals^D} \Big[ \| \nabla V(\btheta) - \nabla \tV(\btheta) \|_2 + \| \nabla_1 U(\btheta, \btheta') - \nabla_1 \tU(\btheta, \btheta') \|_2 \Big]. 
\end{align}
\end{lemma}

\begin{proof}
The proof adapts the argument used to establish uniqueness in \cite{sznitman1991topics}. Without loss of generality, we fix $\xi(t) \equiv 1/2$. We further
denote by $K$ generic constants depending on $K_1$, $K_3$.

The assumption of bounded Lipschitz $\nabla V$ and $\nabla U$ implies that $\nabla \Psi(\btheta; \rho)$ is $K$-bounded Lipschitz with respect to argument $(\btheta, \rho)$, that is,
\begin{align}\label{eq:BoundedLipschitzInequality}
\Big\| \nabla \Psi(\btheta_1; \rho_1) - \nabla \Psi(\btheta_2; \rho_2) \Big\|_2 \le K \Big[\| \btheta_1 - \btheta_2 \|_2 \wedge 1 + d_{\sBL}(\rho_1, \rho_2)\Big].
\end{align}
The assumption of bounded Lipschitz $\nabla \tilde V$ and $\nabla \tilde U$ implies that $\nabla \tilde \Psi(\btheta; \rho)$ is $K$-bounded Lipschitz. Under these conditions, according to \cite[Theorem 1.1]{sznitman1991topics}, there is existence and uniqueness of PDE (\ref{eq:OriginalPDE}) and (\ref{eq:PerturbedPDE}). We denote their solutions at time $t$ to be $\rho_t, \trho_t \in \cuP(\reals^D)$ respectively. 

 Let $\gamma_0 \in \cuP(\reals^D \times \reals^D)$ be a coupling of $\rho_0$ and $\trho_0$ that achieves 
$2 d_{\sBL}(\rho_0, \trho_0)$.
Given these fixed $(\rho_t)_{t \ge 0}$ and $(\trho_t)_{t \ge 0}$, consider the nonlinear dynamics 
\begin{align} 
\btheta^t =& \btheta^0 - \int_0^t \nabla \Psi(\btheta^s;\rho_s) \de s , ~~\label{eqn:non_linear_in_perturbation_proof_1}\\
\tbtheta^t =& \tbtheta^0 - \int_0^t \nabla \tilde \Psi(\tbtheta^s;\trho_s) \de s,~~\label{eqn:non_linear_in_perturbation_proof_2}
\end{align}
with initialization $(\btheta^0,\tbtheta^0)\sim\gamma_0$.
As implied by \cite[Theorem 1.1]{sznitman1991topics}, since we have $\btheta^0 \sim \rho_0$, $\tbtheta^0 \sim \trho_0$, 
it follows that $\btheta_t\sim \rho_t$, $\tbtheta_t\sim \trho_t$, and therefore
\begin{equation}\label{eqn:BLBoundByDistance}
d_{\sBL}(\rho_t, \trho_t) \le 2 \int\Big( \| \btheta^t- \tbtheta^t \|_2  \wedge 1\Big)\; \gamma_0(\de \btheta^0, \de \tbtheta^0)\, .
\end{equation}

Taking the difference of Eqs.~(\ref{eqn:non_linear_in_perturbation_proof_1}) and (\ref{eqn:non_linear_in_perturbation_proof_2}), for any $(\btheta^0, \tbtheta^0) \in \supp(\gamma_0)$, 
\begin{equation}\label{eqn:Bound_difference_theta_in_perturbation_proof}
\begin{aligned}
&\| \btheta^t - \tbtheta^t \|_2  \le \int_0^t \Big \| \nabla \Psi(\btheta^s; \rho_s) -  \nabla \tilde \Psi(\tbtheta^s; \trho_s) \Big\|_2 \de s + \| \btheta^0 - \tbtheta^0 \|_2\\
\le& \int_0^t  \Big \| \nabla \Psi(\btheta^s; \rho_s) -  \nabla \Psi(\tbtheta^s; \trho_s) \Big\|_2 \de s + \int_0^t \Big \| \nabla \Psi(\tbtheta^s; \trho_s) -  \nabla \tilde \Psi(\tbtheta^s; \trho_s) \Big\|_2 \de s+ \| \btheta^0 - \tbtheta^0 \|_2\\
\equiv & E_1(t) + E_2(t) + \| \btheta^0 - \tbtheta^0 \|_2.
\end{aligned}
\end{equation}
Using bound (\ref{eq:BoundedLipschitzInequality}), the first term $E_1(t)$ is simply bounded by
\begin{align}
E_1(t) \le  K  \int_0^t \Big [ \|\btheta^s - \tbtheta^s \|_2 \wedge 1 + d_{\sBL}(\rho_s, \trho_s) \Big] \cdot  \de s .\label{eq:Bound_E1_in_perturbation_proof}
\end{align}
To bound the second term $E_2(t)$, we have
\begin{align}
E_2(t) \le&t \times  \sup_{\btheta \in \reals^D, \rho \in \cuP(\reals^D)} \| \nabla \Psi(\btheta; \rho) - \nabla \tilde \Psi(\btheta; \rho) \|_2 \nonumber\\
 \le& t \times \sup_{\btheta, \btheta' \in \reals^D} \Big[ \| \nabla V(\btheta) - \nabla \tV(\btheta) \|_2 + \| \nabla_1 U(\btheta, \btheta') - \nabla_1 \tU(\btheta, \btheta') \|_2 \Big] = t \cdot \eps_0, \label{eq:Bound_E2_in_perturbation_proof}
\end{align}
with the definition of $\eps_0$ given by Equation (\ref{eqn:perturbation_error_in_perturbation_lemma}).

Combining Equation (\ref{eqn:Bound_difference_theta_in_perturbation_proof}), (\ref{eq:Bound_E1_in_perturbation_proof}), and (\ref{eq:Bound_E2_in_perturbation_proof}), we have
\begin{align}
\| \btheta^t - \tbtheta^t \|_2 \wedge 1 \le K \int_0^t \Big [ \| \btheta^s - \tbtheta^s \|_2 \wedge 1 + d_{\sBL}(\rho_s, \trho_s)  \Big] \cdot  \de s  + t \cdot \eps_0 + \| \btheta^0 - \tbtheta^0 \|_2 \wedge 1.
\end{align}
Averaging the above inequality over $(\btheta^0, \tbtheta^0) \sim \gamma_0$, and using inequality (\ref{eqn:BLBoundByDistance}), we have
\begin{align}
\int \| \btheta^t - \tbtheta^t \|_2 \wedge 1 \cdot \de \gamma_0 \le 2 d_{\sBL}(\rho_0, \trho_0) + 3 K  \int_0^t \Big [ \int \| \btheta^s - \tbtheta^s \|_2 \wedge 1 \cdot \de \gamma_0\Big] \cdot  \de s  + t \cdot \eps_0.
\end{align}
Using Gronwall's inequality, for any $t \in \reals$, we have 
\[
\begin{aligned}
\int \| \btheta^t - \tbtheta^t \|_2 \wedge 1 \cdot \gamma_0(\de \btheta^0, \de \tbtheta^0) \le K\exp(K t ) \cdot [d_{\sBL}(\rho_0, \trho_0) + \eps_0]. 
\end{aligned}
\]
Applying Equation (\ref{eqn:BLBoundByDistance}), the result follows.
\end{proof}

\subsection{Some properties of the solution of the PDE (\ref{eq:GeneralPDE_App})}

In this section we prove four lemmas on the properties of the solution of the PDE (\ref{eq:GeneralPDE_App}), under conditions {\sf A1} and {\sf A3}.
All of these facts are quite standard, but we provide complete proofs for them for reader's convenience.

We will use several times the following notations.
Let $\rho_t$ be a solution of the PDE (\ref{eq:GeneralPDE_App}) with initialization $\rho_0$.  Let $(\btheta^t)_{t\ge 0}$ be the solution
of the ordinary differential equation (ODE)
\begin{align}
\dot{\btheta^t} = -2 \xi(t) \nabla \Psi(\btheta^t;\rho_t)\, ,\label{eq:ODE_Density_generalized}
\end{align}
with initial condition $\btheta^0$. Without loss of generality, we will assume $\xi(t) = 1/2$ throughout this section. If $\btheta^0 \sim \rho_0$, then for any $t\ge 0$, we have $\btheta^t\sim\rho_t$. 
We will denote by $\bphi^t:\reals^D\mapsto\reals^D$ the map between initial conditions of this ODE, and its state at time $t$
(i.e. $\bphi^t(\btheta^0)=\btheta^t$). Since $\nabla \Psi(\,\cdot\, ;\rho_t)$ is bounded and Lipschitz continuous, 
it follows that $\bphi^t$ is a homeomorphism on its image by Picard's theorem. 

With these notations, $\rho_t$ is the push forward of $\rho_0$ under $\bphi^t$: $\rho_t = \bphi_*^t\rho_0$.
In other words, for any Borel set $B$, $\rho_t(\bphi^t(B)) = \rho_0(B)$.

\begin{lemma}\label{lem:MassIncreasing}
Assume conditions {\sf A1}, {\sf A3} hold. Let $(\rho_t)_{t\ge 0}$ be the solution of the PDE (\ref{eq:GeneralPDE_App}) with initialization $\rho_0$. Let $\Omega \subseteq \reals^D$ be a Borel set. Suppose $\bphi^t(\Omega) \subseteq \Omega$, then we have $\rho_t(\Omega) \ge \rho_0(\Omega)$. 
\end{lemma}

\begin{proof}
The lemma holds immediately by noting that $\rho_t(\Omega) \ge \rho_t(\bphi^t(\Omega)) = \rho_0(\Omega)$. 
\end{proof}

\begin{lemma}\label{lemma:Density_generalized}
Assume conditions {\sf A1}, {\sf A3} hold. Further assume there exists a constant $K < \infty$ such that
\begin{align}\label{eqn:CoordinatewiseLipschitzPsi}
\l \partial_i \Psi(\btheta; \rho) \l \le & K \cdot \theta_i,
\end{align}
for any $\btheta \in (0, \infty)^D$ and $\rho \in \cP([0, \infty]^D)$. Let $(\rho_t)_{t\ge 0}$ be the solution of the PDE (\ref{eq:GeneralPDE_App}) with initial condition $\rho_0$ with $\rho_0((0, \infty)^D) = 1$. Then for any $t < \infty$, $\rho_t((0, \infty)^D) = 1$. 
\end{lemma}

\begin{proof}

According to Eqs. (\ref{eqn:CoordinatewiseLipschitzPsi}) and (\ref{eq:ODE_Density_generalized}), we have for $i \in [d]$, 
\begin{equation}\label{eqn:bounded_evolution_inequality}
\theta_{i}^0 \cdot \exp\{ - K t \} \le \theta_{i}^t \le \theta_{i}^0 \cdot \exp\{ Kt\}. \\
\end{equation}
Denote 
\begin{equation}
\Omega_k(t) = [1/k \cdot \exp\{ -Kt \}, k \cdot \exp\{Kt\} ]^D.
\end{equation}
Then according to (\ref{eqn:bounded_evolution_inequality}), we have $\bphi^t(\Omega_k(0)) \subseteq \Omega_k(t)$. Note $\Omega_k(t)$ is increasing in $k$ for fixed $t$, and $\cup_{k} \Omega_k(t) = \cup_{k} \Omega_k(0) = (0, \infty)^D$. Hence, 
\begin{align}
\rho_t((0, \infty)^D) = \lim_{k \to \infty} \rho_t( \Omega_k(t)) \ge \lim_{k \to \infty} \rho_t(\bphi^t(\Omega_k(0))) = \lim_{k \to \infty} \rho_0(\Omega_k(0)) = \rho_0((0, \infty)^D) = 1. 
\end{align}
\end{proof}

\begin{lemma}\label{lem:ConnectedLimitingSet}
Let $(\rho_t)_{t \ge 0}$ be a continuous curve in a compact metric space $(\Omega, d)$. Denoting 
\[
\cS_\star \equiv \{ \rho_\star \in \Omega: \exists ( t_k )_{k \ge 1}, \lim_{k\to \infty} t_k = \infty, s.t., \lim_{k \to \infty} d(\rho_{t_k}, \rho_\star) = 0 \}
\]
to be the set of all limiting points of $(\rho_t)_{t \ge 0}$. Then $\cS_\star$ is a connected compact set.  
\end{lemma}
\begin{proof}
First, it is easy to see that $\cS_\star$ should be closed. Note that $\Omega$ is a compact space, then $\cS_\star$ should be a compact set. If $\cS_\star = \{ \rho_\star \}$ is a singleton, this lemma holds automatically. Therefore, we would like to consider the case when $\cS_\star$ is not a singleton. 

For any $\rho_1, \rho_2 \in \cS_\star$, and  $d(\rho_1, \rho_2) > 0$. We would like to show $\rho_1$ and $\rho_2$ are connected in $\cS_\star$. 

We use proof by contradiction. Now suppose $\rho_1$ and $\rho_2$ are not connected. Define $\cA \subseteq \cS_\star$ to be the maximal connected subset of $\cS_\star$ containing $\rho_1$. It is easy to see that $\cA$ is compact. It is also easy to see that its complement $\cB \equiv \cS_\star \setminus \cA$ is also a compact set, and $\rho_2 \in \cB$. As a result, we have $\cA \cup \cB = \cS_\star$, $\cA \cap \cB = \emptyset$, and $\rho_1 \in \cA$, $\rho_2 \in \cB$. 

Note that $\Omega$ is a metric space, so it satisfies T4 separation axiom. Since $\cA$ and $\cB$ are closed sets and $\cA \cap \cB = \emptyset$, there exists an open set $\cO$, such that $\cA \subseteq \cO$, $\cO \cap \cB = \emptyset$. Hence, $\partial\cO \subseteq  \cS_\star^c$. 

Note that $\rho_1$ and $\rho_2$ are limiting points of $(\rho_t)_{t \ge 0}$ which is a continuous curve in $\Omega$. Therefore, it must cross the boundary $\partial \cO$ infinite times. That is, there is a sequence $(t_k)_{k \ge 1}$ of time with $\lim_{k \to \infty} t_k = \infty$,  such that $\rho_{t_k} \in \partial \cO$. But since $\partial \cO$ is compact, there exists a limiting point $\rho_\star \in \partial \cO$, so that a subsequence of sequence $\rho_{t_k}$ converges to $\rho_\star$. Therefore, $\rho_\star$ should be a limiting point of $(\rho_{t})_{t\ge 0}$. But this contradict with $\partial \cO \subseteq \cS_\star^c$. 
\end{proof}

\begin{lemma}\label{lemma:Density}
Under the assumptions of {\sf A1} and {\sf A3}, further assume that $U, V$ are twice continuous differentiable, and that $\rho_0$ has density
with respect to the Lebesgue measure, bounded by $M_0$. Then $\rho_t$ also has a density, 
bounded by $M_t = K\,M_0 \exp\{KD t\}$ (where $K$ depends on the constants in the assumptions).
\end{lemma}
\begin{proof}
Let $\bJ(\btheta;t)$ for the Jacobian of $\bphi^t(\,\cdot\, )$ at $\btheta^0=\btheta$. Then Eq.~(\ref{eq:ODE_Density_generalized}) implies
that $\bJ(\btheta;t)$ satisfies
\begin{align}
\frac{\de\phantom{t}}{\de t}\, \bJ(\btheta;t) = -\nabla^2\Psi(\bphi^t(\btheta);\rho_t)\,\bJ(\btheta;t)\,, 
\end{align}
with initial condition $\bJ(\btheta;0)=\id_D$. This implies 
\begin{align}
\frac{\de\phantom{t}}{\de t}\lambda_{\min}\big(\bJ(\btheta;t)\big)\ge -\|\nabla^2\Psi(\bphi^t(\btheta);\rho_t)\|_{\op}\,
\lambda_{\min}\big(
\bJ(\btheta;t)\big)\, .
\end{align}
Therefore, using the fact that  $\| \nabla^2\Psi(\btheta;\rho_t)\|_{\op}$ is $K$-bounded, we obtain $\lambda_{\min}\big(\bJ(\btheta;t)\big)\ge 
\exp(-Kt)$. Finally, since $\bphi^t$ is a diffeomorphism, we have
\begin{align}
\rho_{t}(\btheta) &= \rho_{0}\big((\bphi^{t})^{-1}(\btheta)\big)\,\Big|\det(\bJ((\bphi^{t})^{-1}(\btheta);t)\big)\Big|^{-1}\\
&\le \rho_{0}\big((\bphi^{t})^{-1}(\btheta)\big)\, \exp(K Dt)\,.
\end{align}
This completes the proof.
\end{proof}

\subsection{Proof of Theorems \ref{thm:StabilityDelta}: Stability conditions}
\label{sec:StabDelta}

In this section, we will prove the stability result in Theorem \ref{thm:StabilityDelta}. Throughout the proof we can assume, without loss of generality, $\xi(t) =1/2$. Indeed $\xi(t)$ amounts just of a change of time. Further we introduce the matrix $\bH_1=\bH_1(\delta_{\btheta_*})\in\reals^{D\times D}$ by
\begin{align}
\bH_1(\delta_{\btheta_*}) & = \nabla^2 V(\btheta_*) +\nabla^2_{1,1} U(\btheta_*,\btheta_*) +\nabla^2_{1,2} U(\btheta_*,\btheta_*)\, ,\\
& = \bH_0(\delta_{\btheta_*}) +\nabla^2_{1,2} U(\btheta_*,\btheta_*)\,,
\end{align}
where $\bH_0 \equiv \bH_0(\delta_{\btheta_*}) = \nabla^2 V(\btheta_*) +\nabla^2_{1,1} U(\btheta_*,\btheta_*)$. 
Notice that
\begin{align}
\<\bu,\nabla^2_{1,2} U(\btheta_*,\btheta_*)\bu\> = \E\{\<\bu,\nabla_{\btheta}\sigma_*(\bx;\btheta_*)\>^2\}\, ,
\end{align}
and therefore $\nabla^2_{1,2} U(\btheta_*,\btheta_*)\succeq \bzero$, whence $\bH_1\succeq\bH_0$.

We first establish the condition for $\rho_*=\delta_{\btheta_*}$ to be a fixed point. Note that
$\Psi(\btheta;\rho_*) = V(\btheta)+U(\btheta,\btheta_*)$ and $\supp(\rho_*)=\{\btheta_*\}$.
Hence the condition [\ref{eq:FP_Condition}] in the main text is satisfied if and only if  $\nabla_\btheta \Psi(\btheta;\rho_*)|_{\btheta=\btheta_*}=\bzero$,
i.e. $\nabla V(\btheta_*)+\nabla_1 U(\btheta_*,\btheta_*)=\bzero$.

To establish the stability result of Theorem \ref{thm:StabilityDelta}, the following lemma provides a key estimate.
\begin{lemma}\label{lemma:Stability}
Under the assumptions of Theorem \ref{thm:StabilityDelta},
let $\lambda\equiv \lambda_{\min}(\bH_0)>0$. Then there exists $r_1, \eps_1, \gamma > 0$ such that the following hold

\begin{enumerate}
\item[$(i)$] If  $\supp(\rho)\subseteq \Ball(\btheta_*;r_1)
\equiv \{\btheta:\; \|\btheta-\btheta_*\|_2\le r_1\}$, then,
\begin{align}
\int \<(\btheta-\btheta_*),\nabla\Psi(\btheta;\rho)\> \, \rho(\de\btheta)\ge \frac{\lambda}{2}\, \int \|\btheta-\btheta_*\|_2^2\, \rho(\de\btheta)\, .
\label{eq:2normContraction}
\end{align}
\item[$(ii)$] If $\int \|\btheta-\btheta_*\|_2^2\,\rho(\de\btheta)\le \eps^2_1$ and $\supp(\rho)\subseteq \Ball(\btheta_*;r_1)$, then for any $\btheta \in \Ball(\btheta_*;r_1)\setminus \Ball(\btheta_*;r_1/2)$, 
\begin{align}
\<(\btheta-\btheta_*),\nabla\Psi(\btheta;\rho)\> \ge \gamma>0\, .\label{eq:NoEscape}
\end{align}
\end{enumerate}
\end{lemma}
\begin{proof}
Note that
\begin{align}
\nabla^2\Psi(\btheta;\rho) = \nabla^2 V(\btheta) +\int \nabla_1^2U(\btheta,\btheta')\, \rho(\de\btheta')\, .
\end{align}
Since   $\nabla^2 V(\btheta)$ is continuous and  $\nabla_1^2U(\btheta,\btheta')$ is bounded continuous, it follows that $\btheta\mapsto\nabla^2\Psi(\btheta;\rho)$
is continuous, and $\rho\mapsto \nabla^2\Psi(\btheta;\rho)$ is continuous in the weak topology, and in fact $(\btheta,\rho)\mapsto\nabla^2\Psi(\btheta;\rho)$ is
continuous in the product topology.

Further, we have
\begin{align}
\nabla^2\Psi(\btheta_*;\rho_*) = \nabla^2 V(\btheta_*) +\nabla_{11}^2U(\btheta_*,\btheta_*) = \bH_0\, .
\end{align}
Since  $\bH_0\succ \bzero$ strictly, for any $\delta>0$ we can choose $r_1=r_1(\delta)>0$ such that 
\begin{align}
&\nabla^2\Psi(\btheta;\rho)\succeq (1-\delta) \, \bH_0\, , \\
&\|\nabla^2_{12}U(\btheta_*,\btheta)-\nabla_{12}^2 U(\btheta_*,\btheta_*)\|_{\op}  \le \delta\, ,
\end{align}
for all $\btheta \in \Ball(\btheta_*;r_1)$, and $\rho$ such that $\supp(\rho)\subseteq \Ball(\btheta_*;r_1)$.
If these conditions hold
\begin{align}
\<(\btheta-\btheta_*),\nabla\Psi(\btheta;\rho)\>& =\<(\btheta-\btheta_*),\nabla\Psi(\btheta;\rho)-\nabla\Psi(\btheta_*;\rho)\>+\<(\btheta-\btheta_*),\nabla\Psi(\btheta_*;\rho)\>\\
 & =\<(\btheta-\btheta_*),\nabla^2\Psi(\tilde{\btheta};\rho)\,(\btheta-\btheta_*)\>+\<(\btheta-\btheta_*),\nabla\Psi(\btheta_*;\rho)\>\\
& \ge (1-\delta)\, \<(\btheta-\btheta_*),\bH_0\,(\btheta-\btheta_*)\>+\<(\btheta-\btheta_*),\nabla\Psi(\btheta_*;\rho)\>\, .\label{eq:H1plus}
\end{align}
In order to bound the second term, note that, since $\nabla\Psi(\btheta_*;\rho_*)=\bzero$,
\begin{align}
\nabla\Psi(\btheta_*;\rho) =& \int \big[\nabla_1 U(\btheta_*,\btheta')-\nabla_1U( \btheta_*,\btheta_*)\big] \, \rho(\de\btheta') = \nabla^{2}_{12}U( \btheta_*,\btheta_*)\bmu+\bxi\, ,\\
\bmu =& \int (\btheta-\btheta_*)\,\rho(\de\btheta)\,,\\
\bxi =&\int \big[\nabla_1 U(\btheta_*,\btheta')-\nabla_1U( \btheta_*,\btheta_*)-\nabla^{2}_{12}U( \btheta_*,\btheta_*) (\btheta'-\btheta_*) \big] \, \rho(\de\btheta') \, .
\end{align}
Substituting in Eq.~(\ref{eq:H1plus}), we obtain
\begin{align}
\<(\btheta-\btheta_*),\nabla\Psi(\btheta;\rho)\>\ge (1-\delta)\<(\btheta-\btheta_*),\bH_0(\btheta-\btheta_*)\>+
\<(\btheta-\btheta_*),(\bH_1-\bH_0)\bmu\> +\<(\btheta-\btheta_*),\bxi\>\, .  \label{eq:BasicLB_Stability}
\end{align}

By the intermediate value theorem, for any $\bv\in\reals^D$, there exists $\tbtheta=\tbtheta(\bv,\btheta) \in [\btheta_*,\btheta]$ such that
\begin{align}
\<\bv,\bxi\> &= \int \<\bv,[\nabla^{2}_{12}U( \tbtheta,\btheta_*) -\nabla^{2}_{12}U( \btheta_*,\btheta_*)]  (\btheta-\btheta_*) \>\, \rho(\de\btheta)\\
&\ge - \int \|\bv\|_2\big\|\nabla^{2}_{12}U( \tbtheta,\btheta_*) -\nabla^{2}_{12}U( \btheta_*,\btheta_*)\big\|_{\op} \|\btheta-\btheta_*\|_2\, \rho(\de\btheta)\\
& \ge -\delta\|\bv\|_2\int \|\btheta-\btheta_*\|_2\, \rho(\de\btheta)\\
&\ge -\delta\|\bv\|_2\, \sqrt{\Tr(\bQ)+\|\bmu\|_2^2}\label{eq:vbxi}\\
&\ge -\delta\|\bv\|_2\, \sqrt{\Tr(\bQ)}-\delta\|\bv\|_2\|\bmu\|_2\, ,
\end{align}
where $\bQ= \int (\btheta-\mu)(\btheta-\mu)^{\sT}\, \rho(\de\btheta)$ is the covariance of $(\btheta-\btheta_*)$.

Let now consider the claim at point $(i)$. Integrating Eq.~(\ref{eq:BasicLB_Stability}) with respect to $\rho(\de\btheta)$, 
we get
\begin{align}
\int\<(\btheta-\btheta_*),&\nabla\Psi(\btheta;\rho)\>\, \rho(\de\btheta) \ge  (1-\delta) \<\bH_0,\bQ+\bmu\bmu^{\sT}\>+
\<\bmu,(\bH_1-\bH_0)\bmu\> +\<\bmu,\bxi\>\\
&\ge (1-\delta)\<\bH_0,\bQ\>+\<\bmu,(\bH_1-\delta\bH_0)\bmu\>-\delta\|\bmu\|_2\, \sqrt{\Tr(\bQ)}-\delta\|\bmu\|_2^2\\
& \ge (1-\delta)\<\bH_0,\bQ\>+\<\bmu,(\bH_1-\delta\bH_0)\bmu\>-\frac{3\delta}{2}\|\bmu\|_2^2-\frac{\delta}{2}\,\Tr(\bQ)\\
&= \<(1-\delta)\bH_0-\frac{\delta}{2}\id,\bQ\>+\<\bmu,(\bH_1-\delta\bH_0-\frac{3\delta}{2}\id)\bmu\>
\, .
\end{align}
By choosing $\delta$ sufficiently small, we can ensure that $(1-\delta)\bH_0-(\delta/2)\id\succeq \lambda_{\min}(\bH_0)\id/2$,
$\bH_1-\delta\bH_0-(3\delta/2)\id\succeq \lambda_{\min}(\bH_1)\id/2$, and therefore
\begin{align}
\int\<(\btheta-\btheta_*),\nabla\Psi(\btheta;\rho)\>\, \rho(\de\btheta) \ge  \frac{1}{2}\lambda_{\min}(\bH_0)\, \Tr(\bQ)+
\frac{1}{2}\lambda_{\min}(\bH_1)\, \|\bmu\|_2^2\, ,
\end{align}
which yields the claim (\ref{eq:2normContraction}).

Next consider point $(ii)$. In this case, Eq.~(\ref{eq:vbxi}) implies
\begin{align}
\<(\btheta-\btheta_*),\bxi\>\ge -\delta\eps_1\|\btheta-\btheta_*\|_2\, .
\end{align}
Substituting in Eq.~(\ref{eq:BasicLB_Stability}), and using $\|\bmu\|_2\le \eps_1$, we get 
\begin{align}
\<(\btheta-\btheta_*),\nabla\Psi(\btheta;\rho)\>&\ge (1-\delta)\<\bH_0,(\btheta-\btheta_*)^{\otimes 2}\>-
\eps_1(\lambda_{\max}(\bH_1)+\lambda_{\max}(\bH_0)+\delta)\|\btheta-\btheta_*\|_2\nonumber \\
&\ge (1-\delta)\lambda  \left(\frac{r_1}{2}\right)^2-\eps_1(\lambda_{\max}(\bH_1)+\lambda_{\max}(\bH_0)+\delta)r_1\, .
\end{align}
This is strictly positive for all $\eps_1$ small enough, hence implying the claim (\ref{eq:NoEscape}).
\end{proof}

We are now in position of proving Theorem  \ref{thm:StabilityDelta}.
\begin{proof}[Proof of Theorem  \ref{thm:StabilityDelta}]
Let $r_0=\min(r_1/2,\eps_1/2)$ and assume, without loss of generality $t_0=0$,
so that $\supp(\rho_0)\subseteq\Ball(\btheta_*;r_0)$.  We also define
\begin{align}
T_{1} & \equiv \inf\Big\{ t:\;\; \int \|\btheta-\btheta_*\|_2^2\, \rho_t(\de\btheta)> \eps_1^2\Big\}\, ,\\
T_{2} & \equiv \inf\Big\{ t:\;\; \supp(\rho_t)\not\subseteq \Ball(\btheta_*;r_1)\Big\}\, ,\\
T_* &\equiv \min(T_1,T_2)\, .
\end{align}
As usual, we adopt the convention that the infimum of an empty set is equal to $+\infty$.

Define $\varphi_1(\btheta) = h(\| \btheta - \btheta_* \|_2)$, with $h$ to be an non-decreasing function with
\begin{align}
h(r) = \begin{cases}
0 & \mbox{ if $r < r_1/2$,}\\
\text{smooth intropolation} & \mbox{ if $r_1/2 \le r < 5r_1 / 8$,}\\
2r/r_1-1 & \mbox{ if $5r_1/8 \le r < 7r_1/8$,}\\
\text{smooth intropolation} & \mbox{ if $7 r_1/8 \le r < r_1$,}\\
1 & \mbox{ if $r \ge r_1$.}\\
\end{cases}
\end{align}
For any $t<T_*$, the PDE (\ref{eq:GeneralPDE_App}) implies
\begin{align}
\partial_t\<\varphi_1,\rho_t\> & = -\int \<\nabla \varphi_1(\btheta),\nabla\Psi(\btheta;\rho_t)\> \, \rho_t(\de\btheta)\\
& = -\frac{2}{r_1}\int h'(\| \btheta - \btheta_* \|_2)\<\frac{(\btheta-\btheta_*)}{\|\btheta-\btheta_*\|_2},\nabla\Psi(\btheta;\rho_t)\> \, \rho_t(\de\btheta)\\
& \le -\frac{4\gamma}{r_1^2}\, \rho_t\Big(\Ball(\btheta_*;7r_1/8)\setminus\Ball(\btheta_*;5r_1/8)\Big)\, ,\label{eq:Bounded}
\end{align}
where, in the last inequality, we used Lemma \ref{lemma:Stability}.$(ii)$.
Next, define 
\begin{align}
\varphi_2(\btheta) = \frac{1}{2}\|\btheta-\btheta_*\|_2^2\, .
\end{align}
Applying again  Eq.~(\ref{eq:GeneralPDE_App}), we get, for $t\le T_*$, 
\begin{align}
\partial_t\<\varphi_2,\rho_t\> & = -\int \<\nabla \varphi_2(\btheta),\nabla\Psi(\btheta;\rho_t)\> \, \rho_t(\de\btheta)\\
&  = -\int \<(\btheta-\btheta_*),\nabla\Psi(\btheta;\rho_t)\> \, \rho_t(\de\btheta)\\
& \le -\lambda\, \<\varphi_2,\rho_t\>\, .\label{eq:Contraction}
\end{align}
Together the last two bounds imply $T_*=\infty$. Indeed assume by contradiction $T_*<\infty$. 
Then either $T_1\le T_2$, $T_1<\infty$, or $T_2 < T_1$, $T_2<\infty$.

Consider the first case: $T_1\le T_2$, $T_1<\infty$. Since $\<\rho_{T_1},\varphi_2\> \ge \eps_1^2$ but $\<\rho_0,\varphi_2\>\le r_0^2\le \eps_1^2/4$,
there exists $t<T_*$ such that $\partial_t\<\rho_0,\varphi_2\> >0$. However this contradicts Eq.~(\ref{eq:Contraction}).
Consider then the second case: $T_2 < T_1$, $T_2<\infty$. This implies $\<\rho_{T_2},\varphi_1\> >0$, but on the other hand $\<\rho_0,\varphi_1\>=0$.
Hence, there exists $t<T_*$ such that $\partial_t\<\rho_0,\varphi_1\> >0$. However this contradicts Eq.~(\ref{eq:Bounded}).

We conclude that $T_*=\infty$ and hence we can apply Eq.~(\ref{eq:Contraction}) for any $t$, thus obtaining
$\partial_t\<\varphi_2,\rho_t\> \le -\lambda\, \<\varphi_2,\rho_t\>$ and hence $\<\varphi_2,\rho_t\>\le (r_0^2/2)e^{-\lambda t}$, which concludes the proof.
\end{proof}

\subsection{Proof of Theorem  \ref{thm:InstabilityDelta}: Instability conditions}
\label{sec:InstDelta}

In this section we will prove the instability result of Theorem \ref{thm:InstabilityDelta}. Throughout the section, we assume $\xi(t) \equiv 1/2$. 
We will use several times the nonlinear dynamics, defined for $\rho_t$ a solution of  Eq. (\ref{eq:GeneralPDE_App})
with initial condition $\rho_0$:
\begin{align}
\dot{\btheta^t} = -\nabla \Psi(\btheta^t;\rho_t)\, .\label{eq:ODE}
\end{align}

\begin{lemma}\label{lemma:Coupling}
Let $\nu$ be a probability measure on $\reals^d$, absolutely continuous with respect to the Lebesgue measure, 
with density bounded by $M$, and let $\bu\in\reals^d$ be a unit vector. Further assume that, for some $\bx_0\in\reals^d$, $r>0$,
we have $\nu(\Ball(\bx_0;r))\ge 1-\eps$, with $0< \eps<1/20$. Then there exists a coupling $\gamma$ of $\nu$
with itself (i.e. a probability distribution on $\reals^d\times \reals^d$ with marginals $\int \gamma(\,\cdot\,,\de\bx) =
\int \gamma(\de\bx,\,\cdot\,)=\nu(\,\cdot\,)$) and a constant $L = L(d,r,M)$ such that the following holds.  If $(\bx_1,\bx_2)\sim \gamma$, then
\begin{align}
\gamma\left(\<\bu,\bx_1-\bx_2\>\ge \frac{1}{L};\;\; \bP^{\perp}_{\bu} (\bx_1-\bx_2)=\bzero\right)\ge \frac{9}{10}\, ,\label{eq:Coupling}
\end{align}
where $\bP^{\perp}_{\bu}=\id-\bu\bu^{\sT}$ is the projector orthogonal to vector $\bu$.
\end{lemma}
\begin{proof}
First consider the case $d=1$: in this case, the assumption $\nu(\Ball(\bx_0;r))\ge 1-\eps$ is not required.
Denote by $F$ the distribution function associated to $\nu$ (i.e. $F(x) \equiv \nu((-\infty,x])$).
By assumption $F$ is differentiable with $F'(x)\le M$.
In order to construct the desired coupling, let $Z$ be a random variable uniformly distributed in $[0,1]$. For a small constant $\xi_0>0$,
define the random variables $(X_1,X_2)$ by letting
\begin{align}
X_1 & = F^{-1}(Z)\, ,\\
X_2 & =\begin{cases}
F^{-1}(Z-\xi_0) & \mbox{ if $Z>\xi_0$,}\\
F^{-1}(Z+1-\xi_0) & \mbox{ if $Z<\xi_0$.}
\end{cases}
\end{align}
(Note that $X_2$ is not defined for $Z=\xi_0$ but this is a zero-probability event.)
On the event $\{Z>\xi_0\}$ (which has probability $1-\xi_0$), we have,
for some $W\in [X_1,X_2]$,
\begin{align}
\xi_0 = F'(W) \, (X_1-X_2)\le M(X_1-X_2)\, .
\end{align}
By choosing $\xi_0$ small enough, this proves the claim for $d=1$.

Consider next $d>1$ and assume without loss of generality $\bu = \bfe_1$. 

Let   $\onu(\,\cdot\,) = \nu(\,\cdot\,|\bX\in\Ball(\bx_0;r))$, $\bX_a^b \equiv(X_a,\dots,X_b)$,
and denote by $f_{1|[2,d]}$ the density of $\onu(X_1\in\,\cdot\,|\bX_2^n)$,  and by $f_{[a,b]}$ the density of $\onu(\bX_a^b\in\,\cdot\,)$.
We then have
\begin{align}
f_{1|[2,d]}(x_1|\bx_2^d) = \frac{f_{[1,d]}(\bx_1^d)}{f_{[2,d]}(\bx_2^d)}\le \frac{M}{f_{[2,d]}(\bx_2^d)}\, .
\end{align}
Further, we have 
\begin{align}
\onu(\{\bx: \; f_{[2,d]}(\bx_2^d)\le\Delta\}) & = \int \bfone_{f_{[2,d]}(\bx_2^d)\le\Delta}\, f_{[2,d]}(\bx_2^d)\, \de \bx_2^d\\
& \le \Delta \int_{\Ball((\bx_0)_2^d ;r)}\, \de \bx_2^d \le C_d\Delta\, r^{d-1}\, .
\end{align}
In order to construct the coupling, we sample $\bZ\sim \nu$. If $\bZ\not\in\Ball(\bx_0;r)$, then we take $\bX_1 = \bX_2 = \bZ$.
If $\bZ\in\Ball(\bx_0;r)$ and $\max_{x_1}f_{1|[2,d]}(x_1|\bZ_2^d) > M/\Delta$, we also take $\bX_1 = \bX_2 = \bZ$. Otherwise we have $\bZ\in\Ball(\bx_0;r)$ and $\max_{x_1}f_{1|[2,d]}(x_1|\bZ_2^d)\le M/\Delta$, then we sample
$(X_{1,1},X_{2,1})$ from the coupling developed in the case $d=1$ applied to $f_{1|[2,d]}(\,\cdot\, |\bZ_2^d)$, and set $\bX_1 = (X_{1,1},\bZ_2^d)$, $\bX_2 = (X_{2,1},\bZ_2^d)$. Now define $\gamma$ to be the
joint distribution of $\bX_1, \bX_2$. Then $\gamma$ is a coupling of $\nu$ with itself.


The above analyisis yields
\begin{align}
\gamma\left(\<\bu,\bX_1-\bX_2\>\ge \frac{\xi_0\Delta}{M};\;\; \bP^{\perp}_{\bu} (\bX_1-\bX_2)=\bzero\right)\ge 1-\xi_0-C_d\Delta\, r^{d-1}-\eps\, .
\end{align}
Hence, we can choose $\Delta, \xi_0$ small enough so that the claim (\ref{eq:Coupling}) holds.
\end{proof}

For any $u \in \reals$, define the level set $\tcL(u)$, 
\begin{align}\label{eqn:LevelSet_tilde}
\tcL(u) &\equiv \{\btheta\in\reals^D:\; \Psi(\btheta;\rho_*)\le u\}\, .
\end{align}
According to the notation of Theorem \ref{thm:InstabilityDelta}, we have $\cL(\eta) = \tcL( \Psi(\btheta_*; \rho_*) - \eta)$ for any $\eta \in \reals$.

\begin{lemma}\label{lemma:LevelSets}
For any $u\in\reals$, $\Delta>0$ such that $\partial \tcL(u_0)$ is compact for all $u_0 \in (u - \Delta, u)$, there exists
$\eps_{0,\#}>0$ such that the following holds. Let $(\rho_t)_{t\ge t_0}$ be a solution of the PDE (\ref{eq:GeneralPDE_App}) such that $d_{\sBL}(\rho_t,\rho_*)\le \eps_{0,\#}$ for all $t\ge t_0$.
Let  $(\btheta^t)_{t\ge t_0}$ be a solution of the ODE (\ref{eq:ODE}) with $\Psi(\btheta^{t_0};\rho_*)\le u-\Delta$. Then $\Psi(\btheta^{t};\rho_*)\le u$ for all $t\ge t_0$.
\end{lemma}
\begin{proof}
By Sard's theorem \cite{guillemin2010differential}, there exists  $u_0\in (u-\Delta, u)$ such that the boundary $\partial \tcL(u_0)$
contains no critical points of $\Psi(\,\cdot\, ;\rho_*)$. If we define $g_0 = \min_{\btheta \in \partial \tcL(u_0)}\|\nabla\Psi(\btheta;\rho_*)\|_2$, the minimum is achieved
by compactness, and therefore we have $g_0>0$ strictly.
Notice that by the differentiability  assumptions on $V$ and $U$, $\partial \tcL(u_0)$ is a $C^1$ submanifold of $\reals^D$, with $\nabla\Psi(\btheta;\rho_*)$ 
orthogonal to $\partial \tcL_0(u_0)$ and directed toward the exterior. Further, as observed already above,
\begin{align}
\|\nabla\Psi(\btheta;\rho_t)-\nabla\Psi(\btheta;\rho_*)\|_2&= \left\|\int \nabla_{\btheta} U(\btheta;\btheta') (\rho_t-\rho_*)(\de\btheta')\right\|_2\\
&\le K\, d_{\sBL}(\rho_t,\rho_*)\le K\, \eps_{0,\#}\, .
\end{align}
By choosing $\eps_{0,\#}$ small enough, we can ensure $\|\nabla\Psi(\btheta;\rho_t)-\nabla\Psi(\btheta;\rho_*)\|_2\le g_0/3$ for all $\btheta$
and all $t\ge t_0$. 

Assume by contradiction that  $\Psi(\btheta^{t_1};\rho_*)> u$ for some $t_1\ge t_0$, 
and let $t_*=\sup\{t\le t_1:\; \Psi(\btheta^t;\rho_*)\le u_0\}$. Note that, by continuity of the trajectory,  $\btheta^{t_*}\in \partial \tcL(u_0)$.
We then must have 
\begin{align}
0 &\le \frac{\de\phantom{t}}{\de t}\Psi(\btheta^{t_*};\rho_*) = -\<\nabla\Psi(\btheta^{t_*};\rho_{t_*}),\nabla\Psi(\btheta^{t_*};\rho_*)\>\\ 
& \le  -\|\nabla\Psi(\btheta^{t_*};\rho_*)\|_2^2+\|\nabla\Psi(\btheta^{t_*};\rho_*)\|_2 \|\nabla\Psi(\btheta^{t_*};\rho_{t_*})-\nabla\Psi(\btheta^{t_*};\rho_*)\|_2\\
&\le  -\frac{2}{3} g_0 \, \|\nabla\Psi(\btheta^{t_*};\rho_*)\|_2\, ,
\end{align}
which leads to a contradiction since $\btheta^{t_*}\in \partial \tcL(u_0)$ and hence $\|\nabla\Psi(\btheta^{t_*};\rho_*)\|_2>0$.
\end{proof}

To prove Theorem \ref{thm:InstabilityDelta}, let now assume by contradiction that $\rho_t\Rightarrow \rho_*= p_*\delta_{\btheta_*} +(1-p_*)\trho_*$ weakly. Then for any $\eps_0, r_0>0$ (to be chosen below),
we can find $t_0= t_0(\eps_0, r_0)$ such that 
\begin{align}\label{eqn:ApproxDistance}
d_{\sBL}(\rho_t,\rho_*)\le \eps_0, ~~ |\rho_t(\Ball(\btheta_*;r_0))-p_*|\le \eps_0
\end{align}
for all $t\ge t_0$. 
Let $\orh_{t_0}$ be the conditional probability measure of $\rho_{t_0}$ given $\btheta\in \Ball(\btheta_*;r_0)$.
By Lemma \ref{lemma:Density}, $\orh_{t_0}$ has a density upper bounded by  a constant $M= M(\eps_0, t_0)$
(note that $\orh_{t_0}(S)\le \rho_{t_0}(S)/(p_*-\eps_0)$).

Set $\bH_0 = \bH_0(\rho_*) = \nabla^2\Psi(\btheta_*;\rho_*)$.
Since $\btheta_*$ is a critical point of $\btheta\mapsto \Psi(\btheta;\rho_*)$, for any $\delta>0$, we can find $r_1(\delta)>0$ such that
\begin{align}
\btheta\in\Ball(\btheta_*;r_1)\;\;\Rightarrow \;\;\big\|\nabla^2\Psi(\btheta;\rho_*)-\bH_0\big\|_{\op}\le\frac{\delta}{2}\, ,\;\;\;\;\;
\|\nabla\Psi(\btheta_*;\rho_*)\|_2 =0\, .
\end{align}
As shown in the proof of Theorem \ref{thm:StabilityDelta}, the function $(\btheta,\rho)\mapsto \nabla^2\Psi(\btheta;\rho)$ is continuous when the space
of probability distributions $\rho$ is endowed with the weak topology.
Analogously $\rho\mapsto \nabla\Psi(\btheta_*;\rho)$ is continuous in the weak topology. Hence for this $\delta>0$ and
 $r_1(\delta)>0$, there exists  $\eps_{0,*}(\delta, r_1)>0$ small enough such that, the following inequalities hold
\begin{align}
\btheta\in\Ball(\btheta_*;r_1),\;\;d_{\sBL}(\rho,\rho_*)\le \eps_{0, *} \Rightarrow \;\;\big\|\nabla^2\Psi(\btheta;\rho)-\bH_0\big\|_{\op}\le\delta\, ,\;\;\;\;\;
\|\nabla\Psi(\btheta_*;\rho)\|_2\le \delta^2\, r_1/2 \, .
\label{eq:ApproxHessian}
\end{align}
Let us emphasize that $r_1$ depends on $\delta$ but can be taken to be independent of $\eps_0$. Further,
by an application of the intermediate value theorem, for all  $\btheta\in\Ball(\btheta_*;r_1)$,
\begin{align}
\left|\Psi(\btheta;\rho_*)-\Psi(\btheta_*;\rho_*)-\frac{1}{2}\<(\btheta-\btheta_*),\bH_0 (\btheta-\btheta_*)\>\right|\le \frac{1}{2}\delta \|\btheta-\btheta_*\|_2^2\, .
\label{eq:ApproxValue}
\end{align}

For $r_0<r_1$, $\btheta^{t_0}\in \Ball(\btheta_*;r_0)$, we let $(\btheta^t)_{t\ge t_0}$ be the solution of Eq.~(\ref{eq:ODE}) with this initial condition.
We then define 
\begin{align}
t_{\exit}(\btheta^{t_0},r_1) & = \inf\big\{t\ge t_0\,:\;\; \btheta^t\not\in \Ball(\btheta_*;r_1)\big\}\, ,\\
t_{\return}(\btheta^{t_0},r_0,r_1)& = \inf\big\{t> t_{\exit}(\btheta^{t_0},r_1) \,:\;\; \btheta^t\in \Ball(\btheta_*;r_0)\big\}\, .
\end{align}

\begin{lemma}\label{lem:Instability_lemma_first}
Under the conditions of Theorem \ref{thm:InstabilityDelta}, there exists $r_1>0$ and $\eps_{0,*}>0$ such that, for all $r_0\le r_1$, $\eps_0\le \eps_{0,*}$, there exists 
$T_{\sUB}(\eps_0,r_0,r_1, t_0)$ such that the following happens. 
If $d_{\sBL}(\rho_t,\rho_*)\le \eps_0$ and $|\rho_t(\Ball(\btheta_*;r_0))-p_*|\le \eps_0$ for all $t \ge t_0$ for some $t_0$, then
\begin{align}
\rho_{t_0}\Big(\big\{\btheta^{t_0} \in \Ball(\btheta_*;r_0)\, :\; t_{\exit}(\btheta^{t_0},r_1) \le T_{\sUB}(\eps_0,r_0,r_1, t_0)\,\big\}\Big)\ge \frac{1}{3}\, p_*\, .
\label{eq:ExitClaim}
\end{align}
\end{lemma}
\begin{proof}

Let $\bu$ be an eigenvector of $\bH_0$ corresponding to the eigenvalue $\lambda_{\rm min}(\bH_0) = -\lambda_1$. By condition {\sf B1} of Theorem \ref{thm:InstabilityDelta}, we have $\lambda_1 > 0$. Let $-\lambda_2$ denote the second smallest eigenvalue (which can be positive).
We further denote by $\bP\in\reals^{D\times D}$ the orthogonal projector onto the eigenspace
corresponding to $\lambda_{\rm min}(\bH_0)$ and  by $\bPp=\id-\bP$ the projector onto the orthogonal subspace.

We fix a $\delta \le (\lambda_1 - \lambda_2)/10$. Then we choose $r_1 > 0$ and $\eps_{0, *} > 0$ such that Eq. (\ref{eq:ApproxHessian}) holds, with an additional requirement that $\eps_{0, *} < p_* / 10$. We will prove this lemma with this choice of $r_1$ and $\eps_{0, *}$. 

We always denote $(\btheta_i^t)_{t \ge t_0}$ to be the solution of Eq.~(\ref{eq:ODE}) with initial condition $\btheta_i^{t_0}$, for $i = 1, 2$. First we claim that, for $0 < \delta \le (\lambda_1 - \lambda_2)/10$, assuming
\begin{align}\label{eqn:condition_of_claim_in_unstable_proof}
\big\|\nabla^2\Psi(\btheta; \rho_t)-\bH_0\big\|_{\op}\le\delta\,,~~~~  \forall t \ge t_0,~~~ \forall \btheta \in \Ball(\btheta_*; r_1),
\end{align}
then for any $\btheta^{t_0}_1, \btheta^{t_0}_2 \in \Ball(\btheta_*; r_1)$ with $\bPp(\btheta^{t_0}_1 -\btheta^{t_0}_2) = \bzero$, we have 
\begin{align}\label{eqn:Unstable_manifold_escape}
\| \btheta^t_1-\btheta^t_2 \|_2 \ge \| \btheta^{t_0}_1 -\btheta^{t_0}_2 \|_2 \,e^{\lambda_1 (t - t_0) / 2}
\end{align}
for all $t \in [t_0, t_{\exit}(\btheta^{t_0}_1,r_1)\wedge t_{\exit}(\btheta^{t_0}_2,r_1)]$.  

For now we assume this claim holds. Fix $r_0 \le r_1$ and $\eps_0 \le \eps_{0, *}$. Define $\gamma$  to be the coupling of Lemma \ref{lemma:Coupling} corresponding to $\bu$ which is the eigenvector corresponding to the least eigenvalue of $\bH_0$, and 
$\nu=\orh_{t_0}$ which is the conditional measure of $\rho_{t_0}$ given $\btheta^{t_0}\in\Ball(\btheta_*;r_0)$. Note $\orh_{t_0}$ has a density upper bounded by  a constant $M= M(\eps_0, t_0)$. By Lemma \ref{lemma:Coupling}, we have $\gamma(\cE) \ge 9/10$, where 
\begin{align}
\cE& \equiv \left\{(\btheta^{t_0}_1,\btheta^{t_0}_2)\in\Ball(\btheta_*;r_0)\times\Ball(\btheta_*;r_0):\; \<\bu,\btheta^{t_0}_1-\btheta^{t_0}_2\>\ge \frac{1}{Z};\;\; \bP^{\perp}_{\bu} 
(\btheta^{t_0}_1-\btheta^{t_0}_2)=\bzero\right\}\,
\end{align}
for some $Z = Z(\eps_0, r_0, t_0) > 0$. 
%
Now we take $(\btheta_1^{t_0}, \btheta_2^{t_0}) \in \cE$. Note the assumption of this lemma gives $d_{\sBL}(\rho_t,\rho_*)\le \eps_0 \le \eps_{0, *}$ for all $t \ge t_0$. According to Eq. (\ref{eq:ApproxHessian}), we have Eq. (\ref{eqn:condition_of_claim_in_unstable_proof}) holds, and due to this claim, we have $\| \btheta^t_1-\btheta^t_2 \|_2 \ge (1/Z) e^{\lambda_1 (t - t_0) / 2}$ for all $t \in [t_0, t_{\exit}(\btheta^{t_0}_1,r_1)\wedge t_{\exit}(\btheta^{t_0}_2,r_1)]$.

Define $T_{\sUB}(\eps_0,r_0,r_1, t_0) = (2/\lambda_1) \log(2Z \, r_1)$. Then for $t>T_{\sUB}$, we have $\|\btheta_1^t-\btheta^t_2\|_2 \ge 2r_1$. This is impossible if $\btheta_1^t,\btheta^t_2\in \Ball(\btheta_*;r_1)$ and hence we deduce  
$(t_{\exit}(\btheta^{t_0}_1,r_1)\wedge t_{\exit}(\btheta^{t_0}_2,r_1))\le T_{\sUB}$ for all $(\btheta^{t_0}_1,\btheta^{t_0}_2)\in\cE$.

Therefore, we get
\begin{align*}
\frac{9}{10}& \le \gamma(\cE) \le \gamma\Big(\Big\{(\btheta_1^{t_0},\btheta_2^{t_0})\in\Ball(\btheta_*;r_0)\times \Ball(\btheta_*;r_0)
\; :\;\; t_{\exit}(\btheta^{t_0}_1,r_1)\wedge t_{\exit}(\btheta^{t_0}_2,r_1)\le T_{\sUB}\Big\}\Big)\\
& \le \gamma\Big(\Big\{\btheta_1^{t_0}\in\Ball(\btheta_*;r_0)\; :\;\; t_{\exit}(\btheta^{t_0}_1,r_1)\le T_{\sUB}\Big\}\Big)+
\gamma\Big(\Big\{\btheta_2^{t_0}\in\Ball(\btheta_*;r_0)\; :\;\; t_{\exit}(\btheta^{t_0}_2,r_1)\le T_{\sUB}\Big\}\Big)\\
& = 2\, \orh_{t_0}\Big(\Big\{\btheta^{t_0}\in\Ball(\btheta_*;r_0)\; :\;\; t_{\exit}(\btheta^{t_0},r_1)\le T_{\sUB}\Big\}\Big)\, .
\end{align*}
Denoting by $S$ the event in the last expression, we obtain $\rho_{t_0}(S)\ge (p_*-\eps_0)\orh_{t_0}(S)\ge (9/20)(p_*-\eps_0)\ge p_*/3$ by noting that $\eps_0 < p_* / 10$. 

\noindent
{\bf Proof of the claim. }
Define the quantities
\begin{align}
x_{\parallel}(t) =& \|\bP(\btheta_1^t-\btheta^t_2)\|_2^2\, ,\\
x_{\perp}(t) =& \|\bP_{\perp}(\btheta_1^t-\btheta^t_2)\|_2^2\,.
\end{align}
We then have, for $t \in  [t_0, t_{\exit}(\btheta^{t_0}_1,r_1)\wedge t_{\exit}(\btheta^{t_0}_2,r_1)]$,
\begin{align*}
\dot{x}_{\parallel}(t)& = 2\<\bP(\btheta_1^t-\btheta_2^t), -\nabla\Psi(\btheta_1^t;\rho_t) +\nabla\Psi(\btheta_2^t;\rho_t)\>\\
& \stackrel{(a)}{=} 2\<\bP(\btheta_1^t-\btheta_2^t), -\nabla^2\Psi(\tbtheta^t;\rho_t)(\btheta_1^t-\btheta^t_2)\>\\
& = - 2\<(\btheta_1^t-\btheta_2^t), \bP\nabla^2\Psi(\tbtheta^t;\rho_t) \bP(\btheta_1^t-\btheta^t_2)\> -
2\<(\btheta_1^t-\btheta_2^t), \bP \nabla^2\Psi(\tbtheta^t;\rho_t) \bPp(\btheta_1^t-\btheta^t_2)\>\\
&\stackrel{(b)}{\ge} 2(\lambda_1-\delta) \| \bP(\btheta_1^t-\btheta^t_2)\|_2^2-2\delta \|\bP(\btheta_1^t-\btheta^t_2)\|_2 \|\bPp(\btheta_1^t-\btheta^t_2)\|_2\\
&\ge 2(\lambda_1-\delta) x_{\parallel}(t) -\delta(x_{\parallel}(t)+x_{\perp}(t))\, ,
\end{align*}
where in $(a)$ we used the intermediate value theorem (with $\tbtheta^t$ a point between $\btheta_1^t$ and $\btheta_2^t$),
and in $(b)$ we used Eq.~(\ref{eqn:condition_of_claim_in_unstable_proof}).

Proceeding analogously for $x_{\perp}(t)$, we get (for a new choice of $\tbtheta^t$)
\begin{align*}
\dot{x}_{\perp}(t)& = 2\<\bPp(\btheta_1^t-\btheta_2^t), -\nabla\Psi(\btheta_1^t;\rho_t) +\nabla\Psi(\btheta_2^t;\rho_t)\>\\
& = 2\<\bPp(\btheta_1^t-\btheta_2^t), -\nabla^2\Psi(\tbtheta^t;\rho_t)(\btheta_1^t-\btheta^t_2)\>\\
& = - 2\<(\btheta_1^t-\btheta_2^t), \bPp \nabla^2\Psi(\tbtheta^t;\rho_t) \bPp(\btheta_1^t-\btheta^t_2)\> -
2\<(\btheta_1^t-\btheta_2^t), \bPp \nabla^2\Psi(\tbtheta^t;\rho_t) \bP(\btheta_1^t-\btheta^t_2)\>\\
& \le 2(\lambda_2+\delta) \| \bPp(\btheta_1^t-\btheta^t_2)\|_2^2+2\delta \|\bP(\btheta_1^t-\btheta^t_2)\|_2 \|\bPp(\btheta_1^t-\btheta^t_2)\|_2\\
&\le 2(\lambda_2+\delta) x_{\parallel}(t) +\delta(x_{\parallel}(t)+x_{\perp}(t))\, .
\end{align*}
Summarizing, we obtained the inequalities 
\begin{align}
\dot{x}_{\parallel}(t)& \ge (2\lambda_1-3\delta) x_{\parallel}(t) -\delta\, x_{\perp}(t)\, ,\label{eq:x_xp_1}\\
\dot{x}_{\perp}(t)& \le \delta x_{\parallel}(t) + (2\lambda_2+3\delta) x_{\perp}(t) \, . \label{eq:x_xp_2}
\end{align}
The matrix of coefficients on the right-hand side is
\begin{align}
\bA = \left(
\begin{matrix}
2\lambda_1-3\delta & -\delta\\
\delta & 2\lambda_2+3\delta
\end{matrix}\right)\, .
\end{align}
This has a (un-normalized)  left eigenvectors $(1, -v)$, $(-v,1)$ with eigenvalues $\xi_{\pm}$
given by:
\begin{align}
v & =\frac{1}{\delta}\Big[\lambda_1-\lambda_2-3\delta-\sqrt{(\lambda_1-\lambda_2-3\delta)^2-\delta^2}\Big] \, ,\\
\xi_{\pm} & = \lambda_1+\lambda_2\pm \sqrt{(\lambda_1-\lambda_2-3\delta)^2-\delta^2}\, .
\end{align}
%
Note we took $\delta < (\lambda_1 - \lambda_2) / 10$, we have $v > 0$,
and $\xi_+ \ge \lambda_1$. 

Multiplying the inequalities (\ref{eq:x_xp_1}), (\ref{eq:x_xp_2}) by $(1,-v)$, we thus obtain
\begin{align}
\frac{\de\phantom{t}}{\de t} \big(x_{\parallel}(t)-v\, x_{\perp}(t)\big)\ge \xi_+ \, \big(x_{\parallel}(t)-v\, x_{\perp}(t)\big)\, .
\end{align}
Since we assumed $x_{\perp}(t_0) = 0$, whence, for all $t  \in [t_0, t_{\exit}(\btheta^{t_0}_1,r_1)\wedge t_{\exit}(\btheta^{t_0}_2,r_1)]$, we have
\begin{align}
x_{\parallel}(t) \ge x_{\parallel}(t)-v\, x_{\perp}(t) \ge x_{\parallel}(t_0)\, e^{\xi_+(t-t_0)} \ge x_{\parallel}(t_0)\, e^{\lambda_1 (t-t_0)}.
\end{align}

\end{proof}

We next strengthen the last lemma and prove that trajectories that exit $\Ball(\btheta_*;r_1)$ do not re-enter
$\Ball(\btheta_*;r_0)$.
\begin{lemma}\label{lemma:NoReturn}
Under the conditions of Theorem \ref{thm:InstabilityDelta}, there exists $r_{0,*},r_1>0$ (with $r_{0,*}<r_1$) and $\eps_{0,*}>0$ such that, for all $r_0\le r_{0,*}$, $\eps_0\le \eps_{0,*}$, there exists 
$T_{\sUB}(\eps_0,r_0,r_1, t_0)$ such that the following happens. 
If $d_{\sBL}(\rho_t,\rho_*)\le \eps_0$ and $|\rho_t(\Ball(\btheta_*;r_0))-p_*|\le \eps_0$ for all $t \ge t_0$ for some $t_0$, then
\begin{align}
\rho_{t_0}\Big(\big\{\btheta^{t_0} \in \Ball(\btheta_*;r_0)\, :\; t_{\exit}(\btheta^{t_0},r_1) \le T_{\sUB}(\eps_0,r_0,r_1, t_0), \;  t_{\return}(\btheta^{t_0},r_0,r_1)=\infty\big\}\Big)\ge \frac{1}{3}\, p_*\, .
\label{eq:ExitClaim}
\end{align}
\end{lemma}
\begin{proof}

Let $\bP_+$ be the projector onto the eigenspace of $-\bH_0$ corresponding to positive eigenvalues, and 
$\bP_-$ the projector onto the subspace corresponding to negative eigenvalues, and let $\lambdaknot\equiv \min_{i\le D} \l \lambda_i(\bH_0) \l$ to be the least absolute value of eigenvalue of $\bH_0$. By condition {\sf B1} of Theorem \ref{thm:InstabilityDelta}, we have $\lambdaknot > 0$. Let $\lambda_{\max}$ denote the largest absolute value of eigenvalue of $\bH_0$.

Fix a $\delta$ such that $0 < \delta \le \min\{\lambdaknot/(1 + \lambdaknot + \lambda_{\max}), \sqrt{\lambdaknot/\lambda_{\max}}, \lambda_1 - \lambda_2, 1\} /10$, where $\lambda_1, \lambda_2$ are as defined in Lemma \ref{lem:Instability_lemma_first}. Next we choose $r_1$ as per Lemma \ref{lem:Instability_lemma_first}, and we further require $\lambda_0 r_1^2 \le \eta_0$, where $\eta_0$ is as per condition {\sf B3} in the statement of Theorem \ref{thm:InstabilityDelta}. We take $\eps_{0, *}$ to be the minimum of the parameter $\eps_{0, *}$ as per Lemma \ref{lem:Instability_lemma_first} and the parameter $\eps_{0, \#}$ as per Lemma \ref{lemma:LevelSets}, where in Lemma \ref{lemma:LevelSets}, we choose $u = \Psi(\btheta_*;\rho_*) - \lambdaknot r_1^2/8$, and $\Delta = \lambda_0 r_1^2/8$. Then we will choose smaller $r_1$ and $\eps_{0, *}$ so that Eq. (\ref{eq:ApproxHessian}) holds. Finally, we take $r_{0, *} = \delta r_1 < r_1$. We will prove this lemma with this choice of $r_1$, $\eps_{0, *}$, and $r_{0, *}$, and with the same function $T_{\sUB}$ as per Lemma \ref{lem:Instability_lemma_first}.

Define 
\begin{align}
t_*(\btheta^{t_0};r_1,\delta)\equiv\sup\big\{t\in (t_0,t_{\exit}(\btheta^{t_0},r_1)):\; \|\btheta_1^t-\btheta_*\|_2< \delta r_1\big\}\, ,
\end{align}
and define
\begin{align}
z_+(t) = \|\bP_+(\btheta^t-\btheta_*)\|_2^2\, ,\\
z_-(t) = \|\bP_-(\btheta^t-\btheta_*)\|_2^2\,.
\end{align}
We bound the evolution of these quantities following the same argument
used above for $x_{\parallel}(t)$, $x_{\perp}(t)$. Namely 
\begin{align*}
\dot{z}_+(t)=& 2\<\bP_+(\btheta^t-\btheta_*), -\nabla\Psi(\btheta^t;\rho_t) +\nabla\Psi(\btheta_*;\rho_t)\> - 2\<\bP_+(\btheta^t-\btheta_*),\nabla\Psi(\btheta_*;\rho_t)\>\\
=& - 2\<\bP_+(\btheta^t-\btheta_*), \nabla^2\Psi(\tbtheta^t;\rho_t)(\btheta^t-\btheta_*)\> - 2\<\bP_+(\btheta^t-\btheta_*), \nabla\Psi(\btheta_*;\rho_t)\>\\
 =& -2\<(\btheta^t-\btheta_*), \bP_+\nabla^2\Psi(\tbtheta^t;\rho_t) \bP_+(\btheta^t-\btheta_*)\>\\
&-2\<(\btheta^t-\btheta_*), \bP_+ \nabla^2\Psi(\tbtheta^t;\rho_t) \bP_-(\btheta^t-\btheta_*)\> - 2\<\bP_+(\btheta^t-\btheta_*), \nabla\Psi(\btheta_*;\rho_t)\>\\
\ge& 2(\lambdaknot-\delta) \| \bP_+(\btheta^t-\btheta_*)\|_2^2- 2 \delta \|\bP_+(\btheta^t-\btheta_*)\|_2 \|\bP_-(\btheta_1^t-\btheta^t_2)\|_2-\delta^2r_1 \|\bP_+(\btheta^t-\btheta_*)\|_2 \\
\ge& (2\lambdaknot-3\delta) z_+(t) -\delta z_-(t) -\delta^2 r_1\sqrt{z_+(t)}\, .
\end{align*}

For $t\in [t_*(\btheta^{t_0};r_1,\delta),t_{\exit}(\btheta^{t_0};r_1)]$, we have $\sqrt{z_+(t)+z_-(t)}\ge \delta r_1$. Using the inequality $\sqrt{a(a+b)}\le a+b$ holding for non-negative $a$ and $b$, we have
\begin{align}
\dot{z}_+(t)&\ge (2\lambdaknot-3\delta) z_+(t) -\delta z_-(t) -\delta^2 r_1\sqrt{z_+(t)}\\
& \ge (2\lambdaknot-3\delta) z_+(t) -\delta z_-(t) -\delta\sqrt{z_+(t) (z_+(t) +z_-(t))}\\
&\ge  (2\lambdaknot-3\delta) z_+(t) -\delta z_-(t)-\delta z_+(t) -\delta z_-(t)\\
&\ge  (2\lambdaknot-4\delta) z_+(t) -2\delta z_-(t)\, .
\end{align}
Proceeding analogously for $z_-$, we arrive at the inequalities
\begin{align}
\dot{z}_+(t)&\ge (2\lambdaknot-4\delta) z_+(t) -2\delta z_-(t) \, ,\\
\dot{z}_-(t)&\le 2\delta z_+(t) -(2\lambdaknot-4\delta) z_-(t) \,,  
\end{align}
for $t\in [t_*(\btheta^{t_0};r_1,\delta),t_{\exit}(\btheta^{t_0};r_1)]$. The matrix of coefficients on the right-hand side has a left eigenvector of the form $(-w,1)$ with corresponding eigenvalue $-\txi$,
whereby $\txi = \sqrt{\lambdaknot^2-4\delta^2}$ and $w = (\lambdaknot-\sqrt{\lambdaknot^2-4\delta^2})/(2\delta)$. In particular,
since $\delta < \lambdaknot/10$, we have $\txi \ge \lambdaknot/2 >0$ and $w >0$.
Multiplying the above inequalities by $(-w,1)$, we get
\begin{align}
\frac{\de\phantom{t}}{\de t}\big(-w z_+(t)+z_-(t)\big)&\le -\txi \big(-w z_+(t)+z_-(t)\big)\, ,
\end{align}
and therefore, for all $t\in[t_*(\btheta^{t_0};r_1,\delta),t_{\exit}(\btheta^{t_0};r_1)]$,
$z_-(t)\le w\, z_+(t) +e^{-\txi t}\big(-w\, z_+(0)+z_-(0)\big)\le wz_+(t)+\delta^2r_1^2$.
In particular, for $t=t_{\exit}(\btheta^{t_0};r_1)$, using $z_+(t_{\exit})+z_-(t_{\exit}) =r_1^2$, we finally obtain 
\begin{align}
\big\|\bP_+(\btheta^{t_{\exit}}-\btheta_*)\big\|_2^2&\ge r_1^2\left(\frac{1-\delta^2}{1+w}\right) \ge r_1^2(1- \delta)\, ,\\
\big\|\bP_-(\btheta^{t_{\exit}}-\btheta_*)\big\|_2^2&\le r_1^2\delta \,. 
\end{align}
Using Eq.~(\ref{eq:ApproxValue}), we obtain
\begin{align}
\Psi(\btheta^{t_{\exit}};\rho_*)&\le \Psi(\btheta_*;\rho_*) + \frac{1}{2}\<(\btheta-\btheta_*),\bH_0(\btheta-\btheta_*)\> +\frac{1}{2}\delta r_1^2\\
&\le   \Psi(\btheta_*;\rho_*) - \frac{1}{2} \lambdaknot \big\|\bP_+(\btheta^{t_{\exit}}-\btheta_*)\big\|_2^2+\frac{1}{2}\lambda_{\max}
\, \big\|\bP_-(\btheta^{t_{\exit}}-\btheta_*)\big\|_2^2 +\frac{1}{2}\delta r_1^2\\
& \le \Psi(\btheta_*;\rho_*) -\frac{1}{2} \lambdaknot r_1^2+ \frac{1}{2}(1 + \lambdaknot + \lambda_{\max})\delta r_1^2\, .
\end{align}
Since $\delta \le \lambdaknot/(10(1 + \lambdaknot + \lambda_{\max}))$, we can ensure that $\Psi(\btheta^{t_{\exit}};\rho_*)\le  \Psi(\btheta_*;\rho_*) - \lambdaknot r_1^2/4$. By Lemma \ref{lemma:LevelSets}, since $d_{\sBL}(\rho_t, \rho_*) \le \eps_{0, *} \le \eps_{0, \#}$ for all $t \ge t_0$, we have $\Psi(\btheta^t; \rho_*) \le \Psi(\btheta_*;\rho_*) - \lambdaknot r_1^2/8$ for all $t \ge t_{\exit}(\btheta^{t_0}; r_1)$. Note for all $\btheta \in \Ball(\btheta_*; \delta r_1)$, we have $\Psi(\btheta; \rho_*) \ge  \Psi(\btheta_*;\rho_*) - \lambda_{\max} \delta^2 r_1^2/2$. Since $\delta \le \sqrt{\lambdaknot/\lambda_{\max}}/10$, we have $\btheta^t \not\in\Ball(\btheta_*; \delta r_1)$ for all $t \ge t_{\exit}(\btheta^{t_0};r_1)$. 

This implies that, for any $\btheta^{t_0} \in \Ball(\btheta_*;r_0)$ for $r_0 \le r_{0, *}$ with $t_{\exit}(\btheta^{t_0},r_1) \le T_{\sUB}(\eps_0,r_0,r_1, t_0) < \infty$, it will never return to $\Ball(\btheta_*; r_0)$. This gives the desired result. 
\end{proof}

Finally we upper bound the probability that $\btheta^t\in\Ball(\btheta_*;r_0)$ for some $t > t_0$, given that 
$\btheta^{t_0}\not \in\Ball(\btheta_*;r_0)$. We define
\begin{align}
t_{\enter}(\btheta^{t_0},r_0) & = \inf\big\{t\ge t_0\,:\;\; \btheta^t\in \Ball(\btheta_*;r_0)\big\}\, .
\end{align}

\begin{lemma}\label{lemma:Enter}
Under the conditions of Theorem \ref{thm:InstabilityDelta}, for any $\eta >0$, there exists $r_{0,*}>0$ and $\eps_{0,*}>0$ such that, for all $r_0\le r_{0,*}$, $\eps_0\le \eps_{0,*}$, the following happens. 
If $d_{\sBL}(\rho_t,\rho_*)\le \eps_0$ and $|\rho_t(\Ball(\btheta_*;r_0))-p_*|\le \eps_0$  for all $t \ge t_0$ for some $t_0$, then
\begin{align}
\rho_{t_0}\Big(\big\{\btheta^{t_0} \not\in \Ball(\btheta_*;r_0)\, :\; t_{\enter}(\btheta^{t_0},r_0) =\infty\big\}\Big)\ge 1-p_*-\eta \, .
\label{eq:ExitClaim2}
\end{align}
\end{lemma}
\begin{proof}
Due to condition {\sf B2} of Theorem \ref{thm:InstabilityDelta}, we can choose $u_1$ with $\Psi(\btheta_*; \rho_*) - \eta_0 < u_1<\Psi(\btheta_*;\rho_*)$ (where $\eta_0$ is as per condition {\sf B3} of Theorem \ref{thm:InstabilityDelta}) such that $\rho_*(\tcL(u_1))\ge 1-p_*-\eta /2$ (recall the notation $\tcL$ defined as Eq. (\ref{eqn:LevelSet_tilde})). By taking $\eps_{0,*}$ small enough,
and since $\btheta\mapsto \Psi(\btheta;\rho_*)$ is Lipschitz continuous, we can also choose $u_2\in (u_1,\Psi(\btheta_*;\rho_*))$ such that
$\rho_{t_0}(\tcL(u_2)) \ge 1-p_*-\eta $. Fix $u_3\in (u_2,\Psi(\btheta_*,\rho_*))$. Applying Lemma \ref{lemma:LevelSets}, we can further reduce $\eps_{0,*}$,
so that for any initialization  $\btheta^{t_0}\in\tcL(u_2)$, we have $\btheta^t\in\tcL(u_3)$ for all $t$. 
Further, by continuity of $\Psi(\,\cdot\,;\rho_*)$, we can choose $r_{0,*}$ small enough so that $\Ball(\btheta_*;r_{0,*})\cap \tcL(u_3) = \emptyset$, whence
\begin{align}
&\rho_{t_0}\Big(\big\{\btheta^{t_0}\not\in \Ball(\btheta_*;r_0)\, :\; t_{\enter}(\btheta^{t_0},r_0)=\infty\big\}\Big)\\
\ge& \rho_{t_0}\Big(\big\{\btheta^{t_0}:\; \Psi(\btheta^{t_0};\rho_*)<u_2,\;\; 
 t_{\enter}(\btheta^{t_0},r_0) =\infty\big\}\Big)\\
=&  \rho_{t_0}\Big(\big\{\btheta^{t_0}:\; \Psi(\btheta^{t_0};\rho_*)<u_2\}) \ge 1-p_*-\eta \, .
\label{eq:ExitClaim2}
\end{align}
\end{proof}

The proof of Theorem \ref{thm:InstabilityDelta} follows immediately from Lemma \ref{lemma:NoReturn} and Lemma \ref{lemma:Enter}. Indeed, let $\eta = p_* / 10$. Take $\eps_0 \le \min\{ \eps_{0, *}, p_*/10\}$ where $\eps_{0, *}$ is the minimum of $\eps_{0, *}$ as per Lemma \ref{lemma:NoReturn} and \ref{lemma:Enter}. Take $r_1$ as per Lemma \ref{lemma:NoReturn}. Take $r_0 \le \min\{ r_{0, *}, r_1 \}$ where $r_{0, *}$ is the minimum of $r_{0, *}$ as per Lemma \ref{lemma:NoReturn} and \ref{lemma:Enter}. With this choice of $\eps_0$ and $r_0$, there exists $t_0 > 0$ such that Eq. (\ref{eqn:ApproxDistance}) holds for all $t \ge t_0$. Setting $t_* = T_{\sUB}(\eps_0,r_0,r_1, t_0) \ge t_0$ with $T_{\sUB}$ given in Lemma \ref{lemma:NoReturn}. Denoting by $\prob_{t_0,\rho_{t_0}}$ be the probability distribution over trajectories of 
(\ref{eq:ODE}) with $\btheta^{t_0}\sim\rho_{t_0}$, we have
\begin{align*}
\rho_{t_*}(\Ball(\btheta_*;r_0)) =&\prob_{t_0,\rho_{t_0}}\big(\btheta^{t_*}\in \Ball(\btheta_*;r_0)\big)\\
=& \prob_{t_0,\rho_{t_0}}\big(\btheta^{t_0}\in \Ball(\btheta_*;r_0);\; \btheta^{t_*}\in \Ball(\btheta_*;r_0)\big) +\prob_{t_0,\rho_{t_0}}\big(\btheta^{t_0}\not\in \Ball(\btheta_*;r_0);\; \btheta^{t_*}\in \Ball(\btheta_*;r_0)\big)\\
\le&  \prob_{t_0,\rho_{t_0}}\big(\btheta^{t_0}\in \Ball(\btheta_*;r_0)\big)-\prob_{t_0,\rho_{t_0}}\big(\btheta^{t_0}\in \Ball(\btheta_*;r_0)\, ;\;t_{\exit}(\btheta^{t_0};r_1)<t_*, t_{\return}(\btheta^{t_0};r_0)=\infty\big)\\
&+\prob_{t_0,\rho_{t_0}}\big(\btheta^{t_0}\not\in \Ball(\btheta_*;r_0)\big) - \prob_{t_0,\rho_{t_0}}\big(\btheta^{t_0}\not\in \Ball(\btheta_*;r_0);\; t_{\enter}(\btheta^{t_0},r_0) =\infty\big)\\
\le& 1 -\frac{1}{3}p_* - (1 - p_* - \eta ) = 2p_*/3 + \eta \, .
\end{align*}
Since we also had $\rho_t(\Ball(\btheta_*;r_0)) \ge p_*-\eps_0$ for all $t \ge t_0$, note $\eta, \eps_0 \le p_* / 10$, we reached a contradiction.

\section{Centered isotropic Gaussians}\label{sec:IsotropicGaussian}

In this section we consider the centered isotropic Gaussians example discussed in the main text. That is, we assume the joint law of $(y,\bx)$ to be as follows:
\begin{itemize}
\item[] With probability $1/2$: $y=+1$, $\bx\sim\normal(\bzero, (1 + \Delta)^2 \id_d)$.
\item[] With probability $1/2$: $y=-1$, $\bx\sim\normal(\bzero, (1 - \Delta)^2 \id_d)$.
\end{itemize}
We assume $0 < \Delta < 1$, and choose $\sigma_*(\bx;\btheta_i) = \sigma(\<\bx,\bw_i\>)$ for some activation function $\sigma$. Define $q(r) \equiv \E\{\sigma(rG)\}$ for $G\sim\normal(0,1)$. We assume $\sigma(\,\cdot \,)$ satisfies the following conditions {\sf S0} - {\sf S4}: 
\begin{itemize}
\item[{\sf S0}] $x\mapsto \sigma(x)$ is bounded, non-decreasing, Lipschitz continuous.
Its weak derivative $x \mapsto \sigma'(x)$ is  Lipschitz in a neighborhood of $0$. 
\item[{\sf S1}] $q$ is analytic on $(0, \infty)$ with  $\sup_{r\in[0,\infty]} q''(r)<\infty$.
\item[{\sf S2}] $q'(r)>0$ for all $r\in (0,\infty)$, with $\sup_{r\in[0,\infty]} q'(r)<\infty$, and  $\lim_{r\to 0} q'(r) = \lim_{r\to\infty}q'(r) = 0$.
\item[{\sf S3}] $-\infty < q(0+)<-1$, $1 < q(+\infty) <\infty$, and $-1 < (q(0+) + q(+ \infty))/2 < 1$.
\item[{\sf S4}] Letting $Z(r) \equiv q'(\tau_- r)/q'(\tau_+ r)$ for some $\tau_+ > \tau_- > 0$ we have $Z'(r) >0$ for all $r\in (0,\infty)$. 
\end{itemize}

Note that condition {\sf S1} and part of {\sf S2} are implied by {\sf S0}, but we list them here for conveniency. Some of these assumptions can be relaxed at 
the cost of extra technical work. In the interest of simplicity, we prefer to 
avoid being overly general. 

As our running example we will use
\begin{align}
\sigma(t) = 
\begin{cases} 
s_1 & \mbox{ if $t\le t_1$,}\\
(s_2(t-t_1)+s_1(t_2-t))/(t_2-t_1) & \mbox{ if $t\in(t_1,t_2)$,}\\
s_2 & \mbox{ if $t\ge t_2$.}
\end{cases}\label{eq:SimpleSigma}
\end{align}
In particular, we choose $s_1 = -2.5$, $s_2 = 7.5$, $t_1 = 0.5$, $t_2 = 1.5$ in our simulations. In section \ref{sec:AssumptionsExample}, we check that this choice satisfies the above assumptions.

Throughout this section, we set $\tau_{\pm} = (1\pm \Delta)$ and $q_+(r) = q(\tau_+r)$, $q_-(r) = q(\tau_-r)$. Also, we will assume $\xi(t)=1/2$, since other choices of $\xi(\,\cdot\,)$ merely amounts to a time reparametrization.

Before analyzing our model, we introduce the function space and space of probability measures we will work on. We equip the set $[0, \infty]$ with a metric $\db$, where $\db(x, y) = \l 1/(1 + x ) - 1/(1 + y) \l$ for any $x, y \in [0, \infty]$. 
Then $([0, \infty], \db)$ is a compact metric space, and we will still denote it by $[0, \infty]$ for simplicity in notations. We denote $C_b([0, \infty])$ to be the set of bounded continuous functions on $[0, \infty]$, where continuity is defined using the topology generated by $\db$. More explicitly, we have isomorphism
\begin{equation}
C_b([0, \infty]) \simeq \{ f \in C([0, \infty)): \exists f(+\infty) \equiv \lim_{r \to +\infty} f(r),  \sup_{r \in [0, \infty]} f(r) < \infty \}. 
\end{equation}
Because of condition ${\sf S2}$ and ${\sf S3}$, we have $q, q' \in C_b([0, \infty])$. 

Let $\cuP([0, \infty])$ be the set of probability measures on $[0, \infty]$. Due to Prokhorov's theorem, there exists a complete metric $\db_\cuP$ on $\cuP([0, \infty])$ equivalent to the topology of weak convergence, so that $(\cuP([0, \infty]), \db_\cuP)$ is a compact metric space. In this section, we will denote by $\ocD = \cuP([0, \infty])$. 

\subsection{Statics}

Since the distribution of $\bx$ is invariant under rotations for each of the two classes, so are the functions
\begin{align}
V(\bw) & =  v(\|\bw\|_2)\, ,\;\;\;\;\;\;\;\;\;\;
U(\bw_1, \bw_2)=  u_0(\|\bw_1\|_2,\|\bw_2\|_2,\<\bw_1,\bw_2\>)\,.
\end{align}
These take the form 
\begin{align}
v(r) & = -\frac{1}{2}\, q(\tau_+r)+\frac{1}{2}\, q(\tau_-r)\, ,\;\;\;\;\; q(t) =\E\{\sigma( t G)\} \\
u_0(r_1,r_2,r_1r_2\cos\alpha)& = \frac{1}{2}\E\{\sigma(\tau_+r_1G_1) \sigma(\tau_+r_2G_2)\}+\frac{1}{2}\E\{\sigma(\tau_-r_1G_1) \sigma(\tau_+r_2G_2)\}\, ,
\end{align}
where expectations are with respect to standard normals $G,G_1,G_2\sim\normal(0,1)$, with $(G_1, G_2)$ jointly Gaussian and $\E\{G_1 G_2\} = \cos\alpha$. 

In order to minimize $R(\rho)$, it is sufficient to restrict ourselves to distributions  that are invariant under
rotations.  Indeed, for any probability distribution $\rho$ on $\reals^d$, we can define its symmetrization by letting,
for any Borel set $Q\subseteq \reals^d$,
\begin{align}
\rho_s(Q) \equiv \int\, \rho(\bR \, Q) \;\;  \mu_{\mbox{\tiny\rm Haar}} (\de \bR)\, ,
\end{align}
where $\mu_{\mbox{\tiny\rm Haar}}$ is the Haar measure over the group of orthogonal rotations. Since $\rho\mapsto R(\rho)$ is 
convex, $R(\rho_s)\le R(\rho)$.

We therefore restrict ourselves to $\rho$'s that are invariant under rotations.
In other words, under $\rho$, the vector $\bw$ is uniformly random conditional on $\|\bw\|_2$. 
We denote by $\rad$ the probability distribution of $\|\bw\|_2$ when $\bw\sim\rho$ and we let $\barR_d(\rad)$ denote the resulting risk. We then have
\begin{align}
\barR_d(\rad) =& 1+2 \int v(r) \, \rad(\de r) + \int u_d(r_1,r_2) \, \rad(\de r_1)\,\rad(\de r_2)\, ,\\
u_d(r_1,r_2) =& \E[u_0(r_1,r_2,r_1r_2 \cos\Theta)] .\label{eqn:u_d_isotropic}
\end{align}
where $\Theta \sim (1/Z_d) \sin^{d - 2} \theta \cdot \bfone\{\theta \in [0, \pi] \} \de \theta$. 

As $d\to\infty$, we have $\lim _{d\to\infty}u_d(r_1,r_2) = u_{\infty}(r_1,r_2)$ (uniformly over compact sets), with
\begin{equation}\label{eqn:u_infty_isotropic}
u_{\infty}(r_1,r_2)=\frac{1}{2}\Big[q(\tau_+r_1) q(\tau_+r_2)+ q(\tau_-r_1) q(\tau_+r_2)\Big],
\end{equation}
and the risk function converges to
\begin{align}\label{eqn:Risk_infinite_isotropic}
\barR_{\infty}(\rad)= \frac{1}{2}\left(1-\int q(\tau_+r)\, \rad(\de r)\right)^2+\frac{1}{2}\left(1+\int q(\tau_- r)\, \rad(\de r)\right)^2\, .
\end{align}
We also define
\begin{align}
\psi_d(r;\rad) = v(r) +\int u_d(r,r') \, \rad(\de r')\, .
\end{align}
For  $d=\infty$, we have the simpler expression
\begin{align}
\psi_\infty(r; \rad) &= \lambda_+(\rad) \cdot q_+(r) + \lambda_-(\rad) \cdot q_-(r) \label{eq:psi_infty_isotropic_description},\\
\lambda_+(\rad) =& \frac{1}{2} [ \< q_+, \rad\> - 1 ], \label{eq:LambdaPlus}\\
\lambda_-(\rad) =& \frac{1}{2} [ \< q_-, \rad\> + 1 ].  \label{eq:LambdaMinus}
\end{align}

The following theorem provides a characterization of global minimizers of $\barR_d(\rad)$.  
\begin{proposition}[Lemma \ref{lemma:OneDeltaCondition} in the main text]\label{propo:MinGeneralD}
For any $d\le\infty$, define 
\begin{align}
\psi_d(r;\rad) \equiv v(r)+\int u_d(r, r')\;\rad(\de r')\, .
\end{align}
Then 
\begin{enumerate}
\item $\rad_*$ is a global minimizer
of $\barR_d(\rad)$ if and only if $\supp(\rad_*)\subseteq \arg\min_{r} \psi_d(r;\rad_*)$.
\item In particular, $\rad_* =\delta_{r_*}$ is a global minimizer or $\barR_d(\rad)$ if and only if $v(r)+u_d(r,r_*)\ge v(r_*)+u(r_*,r_*)$ 
for all $r$.
\end{enumerate}
\end{proposition}
\begin{proof}
Point 1 is essentially a special case of the second part of Proposition \ref{thm:NtoInfty} in the main text
 (cf. Eq.~(\ref{eq:GeneralMinConditionApp})) and follows by the same argument.
Point 2 is follows by taking $\rad_*=\delta_{r_*}$.
\end{proof}

Given the last result, it is interesting to understand whether the optimal radial distribution $\rad_*$
is a single point mass or not. Under the ansatz $\rad = \delta_r$ (a single point mass at radius $r$)
we obtain an effective risk $\barR_d^{(1)}(r) \equiv \barR_d(\delta_r)$ defined by
$\barR_d^{(1)}(r)  = 1+2 v(r)+u_d(r,r)$, which is plotted in Figure
\ref{fig:num_iso_riskSingleDelta} for the case of our running example
(\ref{eq:SimpleSigma}), and $\Delta=0.4$.

%

Let $r_*=r_*(\Delta,d)$ be the minimizer of $\barR_d^{(1)}(r)$, and define, for $d\le \infty$,
\begin{align}
\Delta_d = \sup\big\{\Delta: \;\; v(r)+u_d(r,r_*)\ge v(r_*)+u_d(r_*,r_*),\; \forall r\ge 0 \big\}\, . \label{eq:DefDelta_D}
\end{align}

In the case $d=\infty$, the minimization problem simplifies further. Either the minimum risk is $0$, 
or it is achieved at a point mass $\rad_*=\delta_{r_*}$.
\begin{theorem}\label{thm:global_minimizer_infinite_d_isotropic}
Consider $d=\infty$. Recall $\ocD = \cuP([0, \infty])$. In this case $\Delta_{\infty}$ defined as per Eq.~(\ref{eq:DefDelta_D}) is such that $\Delta_{\infty}\in (0,1)$.
Further
\begin{enumerate}
\item For $\Delta < \Delta_\infty$, $\inf_{\rad\in\ocD}\barR_\infty(\rad)>0$ and  the unique global minimizer of risk function $\barR_\infty(\rad)$ is a point mass located at some $r_*(\Delta)\in(0,\infty)$. 
\item For  $\Delta \ge \Delta_\infty$, all global minimizers of risk function $\barR_\infty(\rad)$ have risk zero, and there exists a global minimizer that has compact support
bounded away from $0$. 
\end{enumerate}
\end{theorem}

\begin{proof}[Proof of Theorem \ref{thm:global_minimizer_infinite_d_isotropic}]
Recall the definitions $q_+(r) = q(\tau_+ r)$ and $q_-(r) = q(\tau_- r)$. Further, we define the set $\Gamma\subseteq [0,1]$ by
\begin{equation}\label{eqn:Gamma_set_in_isotropic_proof}
\Gamma = \{ \Delta: \exists r \in (0, +\infty), \text{ s.t., } q_+( r) \ge 1 \text{ and } q_-( r ) \le -1 \}.
\end{equation}
According to condition {\sf S3}, for $\Delta = 1$, we have $q_-(r) = q(0) < -1$ and $q_+(+\infty) = q(+\infty) > +1$. 
Since $q$ is continuous, it is easy to see that there exists an $\eps > 0$, such that $[1 - \eps,1] \subseteq  \Gamma$. 
Further, for $\Delta = 0$ we have $q_+(r) = q_-(r)$. By continuity, there exists an $\eps > 0$, such that $[0,\eps] \in [0,1]\setminus \Gamma$.  

Since $q$ is an increasing function, we have
\begin{equation}\label{eqn:Delta_infty_in_isotropic_proof}
\Gamma = [\Delta_{\infty},1]\,,\;\;\;\;\;\; \Delta_\infty = \inf_{\Delta \in \Gamma} \Delta. 
\end{equation}
By the remarks above, we have $0 < \Delta_\infty < 1$. Notice that this definition does not coincide with the one in Eq.~(\ref{eq:DefDelta_D}).
However, the proof below (together with Proposition \ref{propo:MinGeneralD}) implies that the two definitions actually coincide.

\noindent
{\bf Part (1): $\Delta < \Delta_\infty$. } 

\noindent
{\bf Step 1. Prove that $\inf_{\rad \in \ocD} \barR_\infty(\rad) > 0$ as $\Delta < \Delta_\infty$. }

First, we consider the optimization problem
\begin{equation}\label{eqn:f_star_in_isotropic_proof}
f_* \equiv 
\sup_{\rad \in \ocD} \Big\{\< q_+, \rad\>  - 1\; \;\;\; \mbox{s.t.} \;\;\, \< q_-, \rad \> \le -1\Big\}.
\end{equation}
We claim that, for $\Delta < \Delta_\infty$ we have $f_* < 0$. 
Indeed, for any  $\lambda \in [0, +\infty)$, we have the following upper bound
\begin{align}
f_* \le \sup_{\rad \in \ocD} \{ L(\rad, \lambda) \equiv  \< q_+, \rad\> - 1 - \lambda (\< q_-, \rad\> + 1 )\}. 
\end{align}
Since $q_+ - \lambda\, q_- \in C_b([0, +\infty])$, then $L(\, \cdot \,, \lambda)$ is continuous in $\rad$ in weak topology. By the compactness of $\ocD$, the supremum of $L(\, \cdot \,, \lambda)$ is attained by some $\rad_\lambda \in \ocD$. This $\rad_\lambda$ should satisfy
\[
\supp(\rad_\lambda) \subseteq\argmax_{r \in [0, +\infty]} \{ q_+(r) - \lambda q_-(r)\}. 
\]
Let $h(r) \equiv q_+(r) - \lambda q_-(r)$. Note the supremum of $h$ should either satisfy 
\begin{equation}\label{eqn:stationary_equation_for_minimizer_in_isotropic_proof}
h'(r) = q_+'(r) - \lambda q_-'(r) = 0,
\end{equation}
for $r \in (0, \infty)$, or the supremum should be attained at the boundary $0$ or $+\infty$. According to condition {\sf S4}, $[q_-'(r) / q_+'(r)]' > 0$ for $r \in (0, \infty)$, the equation (\ref{eqn:stationary_equation_for_minimizer_in_isotropic_proof}) has at most one solution $r_\star \in (0, \infty)$. 

Assume that  there exists $r_\star \in (0, \infty)$ such that $h'(r_\star) = 0$. Then we have $h'(r)>0$ for  $0 < r < r_\star$, and $h'(r)<0$ for  $r_\star < r < +\infty$, whence 
$\supp(\rad_\lambda) = \{r_\star\}$. If $h'(r) = 0$ does not have a solution in $(0, \infty)$, the only supremum of $h(r)$ could be achieved at $0$ or $+\infty$. Therefore, $\supp(\rad_\lambda) = \{ 0\}$ or $\supp(\rad_\lambda) = \{ +\infty \}$. This concludes that, for any $\lambda \in [0, +\infty)$, $\sup_{\rad \in \ocD} L(\rad, \lambda)$ is achieved by a point mass. Therefore, we have
\[
f_\star \le \inf_{\lambda \in [0, +\infty)}\sup_{r \in [0, +\infty]} ~ \{ q_+(r) - 1 - \lambda (q_-(r) + 1 ) \} = q_+(q_-^{-1}(-1)) - 1. 
\]
For $\Delta < \Delta_\infty$, the right hand side of the above inequality is less than $0$.  Therefore, we cannot have a probability distribution $\rad$ such that $\< q_+, \rad \> = 1$ and $\< q_-, \rad \> = -1$. The infimum of the risk cannot be $0$. 

\noindent
{\bf Step 2. Show that the global minimizer should be a delta function for $\Delta < \Delta_\infty$. }

According to Proposition \ref{thm:NtoInfty}, the global minimizer $\rad_\star \in \ocD$ should satisfy
\[
\supp(\rad_\star) \subseteq \arg\min_{r \in [0, +\infty]}
\psi_\infty(r; \rad_\star)\, ,
\]
with $\psi_\infty$ given in Eq.~(\ref{eq:psi_infty_isotropic_description}).

As proved in the last step, as $\Delta < \Delta_\infty$, we cannot have both $\lambda_+(\rad_\star) = 0$ and $\lambda_-(\rad_\star) = 0$. The argument
given above also implies that $\psi_\infty(r; \rad_\star)$ is minimized at a unique point, and hence the support of $\rad_\star$ should be
a single point. This proves the first part of the theorem. 

\noindent
{\bf Part (2): $\Delta \ge \Delta_\infty$. }

For $\Delta \ge \Delta_\infty$, there exists $r > 0$, such that $q(\tau_+ r) \ge 1$, and $q(\tau_- r) \le -1$. Therefore, there exists $r_\star > 0$ such that $q(\tau_+ r_\star) -1 = -1 - q(\tau_- r_\star) = \eps_\star \ge 0$. Consider the following probability measure on $[0, +\infty]$, 
\[
\rad_\star = \frac{1}{1 + \eps_\star} \delta_{r_\star} +\frac{\eps_\star}{(1 + \eps_\star)(q(+\infty) - q(0))}  [ q(+\infty)\delta_0 - q(0) \delta_{+\infty}].
\]
It can be checked that $\barR_\infty(\rad_\star) = 0$. 

We would like to show further that there exists a global minimizer that is compactly supported. We construct this global minimizer as following. First, define 
\[
r_0 = \inf\{r: q_-(r) \ge -1 \}. 
\]
Then we know that $q_-(r_0) = -1$ and $q_+(r_0) \ge 1$. Now for any $0 \le r \le r_0$, define $u(r) = q_-^{-1}(-2 - q_-(r))$. According to condition {\sf S3}, we have $-1 < [q(0) + q(+\infty)] /2 < 1$, then $u(r)$ is well defined on $[0, r_0]$. It is easy to see that $u(r_0) = r_0$, and $[q_-(r) + q_-(u(r))]/2 = -1$ for any $0 \le r \le r_0$. Now we consider the function $z(r) = [q_+(r) + q_+(u(r))]/2 - 1$. Note that $z(r_0) > 0$, and $z(0) \le [q(0) + q(\infty)]/2 - 1 < 0$. Therefore, there exists $r_\star$ satisfying $0 < r_\star \le r_0$ such that $z(r_\star) = 0$. Consider the following probability measure on $(0, +\infty)$, 
\[
\rad_\star = \frac{1}{2} [\delta_{r_\star} + \delta_{u(r_\star)}]. 
\]
It is easy to see that $\barR_\infty(\rad_\star) = 0$. 
\end{proof}

\subsection{Dynamics: Fixed points}

We specialize the general evolution (\ref{eq:GeneralPDE_App}) to the present case. Assuming $\rho_0$ to be spherically symmetric, then $\rho_t$ is spherically symmetric for any $t \ge 0$.
We let $\rad_t$ denote the distribution of $\|\bw\|_2$ when $\bw\sim\rho_t$. This satisfies the following PDE: 
\begin{align}
\partial_t\rad_t(r) = 2\xi(t)\partial_r\big[\rad_t(r) \partial_r\psi_d(r;\rad_t) \big]\, . \label{eq:PDERadial}
\end{align}
We will view this as an evolution in the space of probability distribution on the completed half-line
$\cuP([0, \infty])$.

In analogy with Proposition \ref{thm:FixedPoints}, we can prove the following characterization of fixed points.
\begin{proposition}\label{propo:FixedPointdRadial}
A distribution $\rad\in \cuP([0,\infty])$ is a fixed point of the PDE (\ref{eq:PDERadial}) if and only if 
\begin{align}
\supp(\rad) \subseteq \{r\in [0,\infty]: \;\partial_r \psi_d(r; \rad) = 0 \}.\label{eq:FixedPointRadial}
\end{align}
%
\end{proposition}
Notice, in particular, global minimizers of $\barR_d(\rad)$ are fixed points of this evolution, but not vice-versa. The next result classifies fixed points.

\begin{theorem}\label{thm:local_minimizer_infinite_d_isotropic}
Consider $d = \infty$ and recall the definition of $\lambda_+(\rad)$ and $\lambda_-(\rad)$ given by Eqs.~(\ref{eq:LambdaPlus}) and (\ref{eq:LambdaMinus}). Then the fixed points of the PDE (\ref{eq:PDERadial}) 
(i.e. the probability measures $\rad\in\cuP([0,\infty])$ satisfying (\ref{eq:FixedPointRadial})) are of one of the following types 
\begin{enumerate}
\item[$(a)$] A fixed point with zero risk. 
\item[$(b)$] A point mass $\rad_{r_\star} = \delta_{r_\star}$ at some location $r_\star \not \in\{ 0, +\infty\}$, but not of type $(a)$.
\item[$(c)$] A mixture of the type $\rad = a_0 \delta_0 + a_\infty \delta_{+\infty} + a \delta_{r_\star}$, but not of type $(a)$ or $(b)$.
\end{enumerate}

For  $\Delta < \Delta_\infty$, the PDE has a unique fixed point of type $(b)$, with $\lambda_+(\rad_\star) < 0$ and $\lambda_-(\rad_\star) > 0$; it has no type-$(a)$ fixed points;
it has possibly  fixed points  of type $(c)$.

For $\Delta > \Delta_\infty$, the PDE has some fixed points of type $(b)$, with $\lambda_+(\rad_\star) > 0$ and $\lambda_-(\rad_\star) < 0$; it also has some type-$(a)$ fixed points; 
it has possibly  fixed points  of type $(c)$.

For $\Delta = \Delta_\infty$, the PDE has a unique fixed point of type $(a)$ which is also a delta function at some location $r_\star$, and no type $(b)$ fixed points;
 it has possibly  fixed points  of type $(c)$.
\end{theorem}

\begin{proof}
We use the characterization of fixed points in Proposition \ref{propo:FixedPointdRadial}.
Recall that $\psi_\infty(r; \rad_\star)$ is defined as in Equation (\ref{eq:psi_infty_isotropic_description}). The derivative $\partial_r \psi_\infty(r; \rad)$ gives
\begin{equation}\label{eqn:stationary_in_isotropic_proof}
\begin{aligned}
\partial_{r} \psi_\infty(r; \rad) =& \lambda_+(\rad) q_+'(r) + \lambda_-(\rad) q_-'(r).\\
\end{aligned}
\end{equation}
If a fixed point has $\lambda_+(\rad_\star) = \lambda_-(\rad_\star)=0$, then $\barR_\infty(\rad_\star)=0$. This is type-$(a)$ fixed point. 
Consider then the case $(\lambda_+(\rad_\star), \lambda_-(\rad_\star)) \neq (0, 0)$.
For the same reason as in the proof of Theorem \ref{thm:global_minimizer_infinite_d_isotropic},
we conclude that $\partial_r \psi_\infty(r; \rad_\star)$ has at most three zeros, two of which are located at $0$ and $+\infty$. 
This proves that all fixed points are of type $(a)$, $(b)$ or $(c)$. 

We already proved in Theorem \ref{thm:global_minimizer_infinite_d_isotropic}  that, for $\Delta < \Delta_\infty$, $\inf_{\rad}\barR_\infty(\rad)>0$. 
Therefore, for $\Delta < \Delta_\infty$, there is no type $(a)$   fixed points.

We next prove that, as $\Delta < \Delta_\infty$, fixed point of type $(b)$ is always unique. The location of the delta fixed point should satisfy
\begin{equation}\label{eqn:delta_stationary_in_isotropic_proof}
\partial_r\psi_{\infty}(r_*;\delta_{r_*})= [q_+'(r_\star)(q_+(r_\star) -1) + q_-'(r_\star)(q_-(r_\star) +1)]/2 = 0.
\end{equation}
Note that $\partial_r\psi_{\infty}(r_*;\delta_{r_*})< 0$ for $r > 0$ small enough, and $\partial_r\psi_{\infty}(r_*;\delta_{r_*})> 0$ for $r$ large enough, whence this equation 
has at least one solution $r_\star \in (0, \infty)$. 
In order to prove that it  has a unique solution in $(0, +\infty)$, define $r_+ \equiv \inf\{r: q_+(r) \ge 1\}$ and $r_- \equiv \inf\{r: q_-(r) \ge -1\}$. 
Note that $q'_+(r_\star),q'_-(r_\star)>0$ and that, in order to satisfy Eq.~(\ref{eqn:delta_stationary_in_isotropic_proof}), the terms $\lambda_+(\delta_{r_\star}) = 1/2 \cdot (q_+(r_\star) -1)$ and
$\lambda_-(\delta_{r_\star}) = 1/2 \cdot (q_-(r_\star) + 1)$ must have opposite signs.
For $\Delta < \Delta_\infty$, we must have $\lambda_+(\delta_{r_\star}) < 0$ and $\lambda_-(\delta_{r_\star}) > 0$, and all stationary points should be within $[r_-, r_+]$. Note that $q_-'(r) / q_+'(r)$ is strictly increasing, and $[1 - q_+(r)]/[1 + q_-(r)]$ is decreasing on $[r_-, r_+]$. Therefore, the fixed point of type $\delta_{r_\star}$ with $r_\star \in (0, \infty)$ is unique. 

For $\Delta > \Delta_\infty$, we must have $\lambda_+(\rad_\star) > 0$ and $\lambda_-(\rad_\star) < 0$, and all solutions should be within $[r_+, r_-]$. There could possibly be multiple fixed points of type $\delta_{r_\star}$ with $r_\star \in [r_+, r_-]$.

If  $\Delta = \Delta_\infty$, it is easy to see that, $\rad_\star = \delta_{r_\star}$ at some $r_\star \in (0, \infty)$ is the unique fixed point with zero risk, and the unique fixed point as a point mass. 
\end{proof}

\subsection{Dynamics: Convergence to global minimum for $d=\infty$}

In this section, denote $\cuP_{\good}$ to be 
\begin{align}\label{eqn:goodset_isotropic_infty}
\cuP_{\good} = \{\rad_0 \in \cuP((0, \infty)): \barR_\infty(\rad_0) < 1, \rad_0 \text{ has bounded density on } (0, \infty) \}.
\end{align}

We then prove that the $d=\infty$ dynamics converges to a global minimizer from any initialization 
in $\cuP_{\good}$.

\begin{theorem}\label{thm:PDE_converge_to_global_minimizer}
Consider the PDE (\ref{eq:PDERadial}) for $d=\infty$, with initialization $\rad_0\in\cuP_{\good}$. It has a unique solution $(\rad_t)_{t \ge 0}$, such that
\[
\lim_{t\rightarrow +\infty} \barR_\infty(\rad_t) = \inf_{\rad \in \ocD} \barR_\infty(\rad)\, .
\] 
\end{theorem}

\begin{proof} Without loss of generality, we assume $\xi(t) = 1/2$. First we show the existence and uniqueness of solution of the PDE.

\vskip 0.2cm
\noindent
{\bf Step 1. Existence and uniqueness of solution. Mass $\rad_t((0, \infty)) = 1$ for all $t$. }

According to conditions {\sf S1} - {\sf S3}, $q(r)$, $q'(r)$, and $q''(r)$ are uniformly bounded on $[0, \infty]$. Recall that
\[
\begin{aligned}
v(r) =& 1/2 \cdot [q_-(r) - q_+(r)],\\
u_\infty(r_1, r_2) =& 1/2 \cdot [q_+(r_1) q_+(r_2) + q_-(r_1) q_-(r_2)].\\
\end{aligned}
\]
Hence $v'(r), \partial_1 u_\infty(r_1, r_2), v''(r), \partial_{11}^2 u_\infty(r_1, r_2), \partial_{12}^2 u_\infty(r_1, r_2)$ are uniformly bounded. Recall we further assumed $\xi(t) \equiv 1/2$. Therefore, conditions {\sf A1} and {\sf A3} are satisfied with $D = 1$, $V = v$, and $U = u$. By Remark \ref{rmk:ExistenceUniqueness}, there is the existence and uniqueness of solution of PDE (\ref{eq:PDERadial}) for $d = \infty$. Denote this solution to be $(\rad_t)_{t \ge 0}$. 

Recall the formula of $\partial_r \psi_\infty(r; \rad)$ given in Equation (\ref{eqn:stationary_in_isotropic_proof}), it is easy to see that the assumption of Lemma \ref{lemma:Density_generalized} is satisfied with $d = 1$ and $\Psi = \psi_\infty$. Hence, we have $\rad_t((0, \infty)) = 1$ for any $t < \infty$.

\vskip 0.2cm
\noindent
{\bf Step 2. Classify the limiting set $\cS_\star$. }

Recall the definition of $(\cuP([0, +\infty]), \db_\cuP)$ at the beginning of Section \ref{sec:IsotropicGaussian}. Since $(\cuP([0, +\infty]), \db_\cuP)$ is a compact metric space, and $( \rad_t )_{t \ge 0}$ is a continuous curve in this space, then there exists a subsequence $(t_k)_{k \ge 1}$ of times, such that $(\rad_{t_k})_{k \ge 1}$ converges in metric $\db_\cuP$ to a probability distribution $\rad_\star \in \cuP([0, +\infty])$. 

Analogously to Proposition \ref{thm:FixedPoints} (using Eq.~(\ref{eq:PDERadial})),  we have
\[
\partial_t \barR_\infty(\rad_t) = -\int  [\partial_r \psi_\infty(r; \rad_t)]^2\, \rad_t(\de r)\, . 
\]
Since $\barR_\infty(\rad_t)\ge 0$, we have
\[
\lim_{t\rightarrow+\infty} \int  [\partial_r \psi_\infty(r; \rad_t)]^2\, \rad_t(\de r) = 0. 
\]

Recall the definition of $\lambda_+(\rad)$ and $\lambda_-(\rad)$ given by Eq. (\ref{eq:LambdaPlus}) and (\ref{eq:LambdaMinus}). Since $q \in C_b([0, \infty])$, we have 
\begin{align}
\lim_{k \to \infty} \lambda_+(\rad_{t_k}) = \lambda_+(\rad_\star),~~~~ \lim_{k \to \infty} \lambda_-(\rad_{t_k}) = \lambda_-(\rad_\star).
\end{align}
Note $\partial_r \psi_\infty(r; \rad)$ is given by Eq. (\ref{eqn:stationary_in_isotropic_proof}), and $q' \in C_b([0, +\infty])$, hence
\[
\lim_{k \rightarrow+\infty} \< [\partial_r \psi_\infty(\, \cdot \,; \rad_{t_k})]^2, \rad_{t_k}\> = \< [\partial_r \psi_\infty(\, \cdot \,; \rad_\star )]^2, \rad_\star \>,
\]
which implies
\[
\< [\partial_r \psi_\infty(\, \cdot \,; \rad_\star )]^2, \rad_\star \> = 0.
\]
In other words, any limiting point $\rad_\star$ of the PDE is a fixed point of the PDE (\ref{eq:PDERadial}). 

Note $\barR_\infty(\rad) = 1/2 \cdot [\lambda_+(\rad)^2 + \lambda_-(\rad)^2]$, we have
\[
\lim_{k\rightarrow +\infty} \barR_\infty(\rad_{t_k}) = \barR_\infty(\rad_\star). 
\]
Note $\barR_\infty(\rad_t)$ is decreasing with $t$, hence
\[
\lim_{t\rightarrow +\infty} \barR_\infty(\rad_t) = \barR_\infty(\rad_\star). 
\]

Let $\cS_\star = \cS_\star(\rad_0)$ be the set of all limiting points of the $(\rad_t)_{t \ge 0}$, 
\[
\cS_\star = \{\rad_\star \in \cuP([0, \infty]): \exists (t_k)_{k \ge 1}, \lim_{k\rightarrow \infty} t_k = +\infty, s.t., \lim_{k \to \infty} \db_\cuP(\rad_\star, \rad_{t_k}) = 0 \}. 
\]
Due to Lemma \ref{lem:ConnectedLimitingSet}, $\cS_\star$ is a connected compact set. Since $\barR_\infty(\rad_t)$ is decreasing as $t$ increases, we have $\barR_\infty(\rad_\star) \equiv \barR_\star$ is a constant for all $\rad_\star \in \cS_\star$. Since we assumed $\barR_\infty(\rad_0) < 1$, and $\barR_\infty(\rad_t)$ is decreasing in $t$, we have $\barR_\star < 1$.

Let $\rad_\star$ be a fixed point of PDE such that $\lambda_+(\rad_\star) \ge 0, \lambda_-(\rad_\star) \ge 0$ or $\lambda_+(\rad_\star) \le 0, \lambda_-(\rad_\star) \le 0$ but not both $\lambda_+(\rad_\star)$ and $\lambda_-(\rad_\star)$ equal $0$. In this case, according to Eq. (\ref{eqn:stationary_in_isotropic_proof}), $\partial_r \psi_\infty(r; \rad_\star)$ must be strictly increasing or strictly decreasing in $r$. Since $\supp(\rad_\star) \subseteq \{ r \in [0, \infty]:  \partial_r \psi_\infty(r; \rad_\star) = 0\}$, $\rad_\star$ must be a combination of two delta functions located at $0$ and $+\infty$, i.e., $\rad_\star = a_0 \delta_0 + (1 - a_0) \delta_\infty$. But for a fixed point of this type, it is easy to see that $\barR_\infty(\rad_\star) \ge 1$. Such fixed points $\rad_\star$ cannot be one of the limiting points of the PDE since $\barR_\infty(\rad_0) < 1$. 

Let $L$ be a mapping $L: \cuP([0, +\infty]) \to \reals^2$, $\rad \mapsto (\lambda_+(\rad), \lambda_-(\rad))$. The above argument implies that for any $\rad_0 \in \cuP_{\good}$, we have
\[
L(\cS_\star(\rad_0)) \cap ( \{ (\lambda_+, \lambda_-): \lambda_+ \ge 0, \lambda_- \ge 0, \text{ or } \lambda_+ \le 0, \lambda_- \le 0\} \setminus \{(0, 0)\} ) = \emptyset.
\]
Since $\cS_\star$ is a connected set, $L(\cS_\star)$ should also be a connected set. Further notice that $\barR_\infty(\rad_\star) = 1/2 \cdot [\lambda_+(\rad_\star)^2 + \lambda_-(\rad_\star)^2]$, and $\barR_\infty(\rad_1) = \barR_\infty(\rad_2)$ for any $\rad_1, \rad_2 \in \cS_\star$. Therefore, we can only have $L(\cS_\star) \subseteq \cP_2 \equiv \{(\lambda_+, \lambda_-): \lambda_+ > 0, \lambda_- < 0\}$, or $L(\cS_\star) \subseteq \cP_1 \equiv \{(\lambda_+, \lambda_-): \lambda_+ < 0, \lambda_- > 0\}$, or $L(\cS_\star) = \{ (0, 0) \}$. 

\vskip 0.2cm
\noindent
{\bf Step 3. Finish the proof using two claims. }

We make the following two claims. 

\begin{enumerate}
\item[] Claim $(1)$. If $L(\cS_\star) \subseteq \cP_1$, then for any $\rad_\star \in \cS_\star$, we have $\rad_\star((0, \infty)) = 1$. 
\item[] Claim $(2)$. We cannot have $L(\cS_\star) \subseteq \cP_2$. 
\end{enumerate}

Here we assume these two claims hold, and use them to prove our results. For $\Delta < \Delta_\infty$, we proved in Theorem \ref{thm:local_minimizer_infinite_d_isotropic} that, there is not a fixed point such that $L(\rad_\star) = (0, 0)$. Therefore, we cannot have $L(\cS_\star) = \{ (0, 0) \}$. Due to Claim $(2)$, we cannot have $L(\cS_\star) \subseteq \cP_2$. Hence, we must have $L(\cS_\star) \subseteq \cP_1$. According to Theorem \ref{thm:local_minimizer_infinite_d_isotropic}, for $\Delta < \Delta_\infty$, the only fixed point of PDE with $\rad_\star((0, \infty)) = 1$ is a point mass at some location $r_\star$. Furthermore, this delta function fixed point is unique and is also the global minimizer of the risk. Therefore, we conclude that, as $\Delta < \Delta_\infty$, the PDE will converge to this global minimizer. 

For $\Delta \ge \Delta_\infty$, according to Claim $(1)$, if $\rad_\star$ is a limiting point such that $L(\rad_\star) \in \cP_1$, then $\rad_\star((0, \infty)) = 1$. According to Theorem \ref{thm:local_minimizer_infinite_d_isotropic}, a fixed point $\rad_\star$ with $\rad_\star((0, \infty)) = 1$ and $L(\rad_\star) \neq (0, 0)$ must be a point mass at some location $r_\star$, with $L(\rad_\star) \in \cP_2$. Therefore, we cannot have $L(\cS_\star) \subseteq \cP_1$. Claim $(2)$ also tells us that we cannot have $L(\cS_\star) \subseteq \cP_2$. Hence, we must have $L(\cS_\star) = \{ (0, 0) \}$. In this case, all the points in the set $\cS_\star$ have risk $0$. Therefore, we conclude that, as $\Delta \ge \Delta_\infty$, the PDE will converge to some limiting set with risk $0$.

\vskip 0.2cm
\noindent
{\bf Step 4. Proof of the two claims. }

We are left with the task of proving the two claims above. Before that, we introduce some useful notations. Recall $Z(r) = q_-'(r) / q_+'(r)$ for $r \in (0, +\infty)$. According to condition {\sf S4}, $Z'(r) > 0$ for $r \in (0, +\infty)$. This implies that $Z(0 +) \equiv Z_0 \ge 0$ and $Z(+\infty) \equiv Z_\infty \le \infty$ exist. We rewrite $\partial_r \psi_\infty(r; \rad)$ as
\begin{align}\label{eqn:psi_in_convergence_proof_isotropic}
\partial_r \psi_\infty(r; \rad) = \lambda_+(\rad) q_+'(r) + \lambda_-(\rad) q_-'(r) =\lambda_-(\rad)q_+'(r) [ \lambda_+(\rad) / \lambda_-(\rad) + Z(r) ]. 
\end{align}

\vskip 0.2cm
\noindent
{\bf Proof of Claim (1). If $L(\cS_\star) \subseteq \cP_1$, then for any $\rad_\star \in \cS_\star$, we have $\rad_\star(\{0, \infty \}) = 0$. }

Assume $L(\cS_\star) \subseteq \cP_1$. Then, we must have $L(\cS_\star) \subseteq \cP_1 \cap \{(\lambda_+, \lambda_-) : Z_0 < - \lambda_+ / \lambda_- < Z_\infty\}$. Otherwise suppose there exists $\rad_\star \in \cS_\star$, such that $- \lambda_+(\rad_\star) / \lambda_-(\rad_\star) \ge Z_\infty$ or $- \lambda_+(\rad_\star) / \lambda_-(\rad_\star) \le Z_0$, according to Eq. (\ref{eqn:psi_in_convergence_proof_isotropic}), $\psi_\infty(r; \rad_\star)$ must be strictly increasing or strictly decreasing in $r$. Since $\supp(\rad_\star) \subseteq \{ r \in [0, \infty]:  \partial_r \psi_\infty(r; \rad_\star) = 0\}$, then $\rad_\star$ must be a combination of two delta functions located at $0$ and $+\infty$. But such $\rad_\star$ must have $\barR_\infty(\rad_\star) \ge 1$, and thus $\rad_\star$ cannot be a limiting point of the PDE. Hence the claim that $L(\cS_\star) \subseteq \cP_1 \cap \{(\lambda_+, \lambda_-) : Z_0 < - \lambda_+ / \lambda_- < Z_\infty\}$ holds. 

Since $\cS_\star$ is a compact set, and $L$ is a continuous map, then $L(\cS_\star)$ is a compact set. Therefore, there must exist $\eps_0 > 0$, so that for any $\rad_\star \in \cS_\star$, we have $Z_0 + 3\eps_0 < -\lambda_+(\rad_\star) / \lambda_-(\rad_\star) < Z_\infty - 3\eps_0$. For this $\eps_0 > 0$, since $\cS_\star$ contains all the limiting points of PDE starting from $\rad_0$, there exists $t_0$ large enough, so that as $t \ge t_0$, we have $Z_0 + 2\eps_0 < - \lambda_+(\rad_t) / \lambda_-(\rad_t) < Z_\infty - 2\eps_0$, and $\lambda_+(\rad_t) < 0$, $\lambda_-(\rad_t) > 0$. For the same $\eps_0$, since $Z(r)$ is continuous at $0$ and $+\infty$, there exists $0 < r_0 < r_\infty < \infty$, so that $Z(r) < Z_0 + \eps_0$ for $r \in (0, r_0)$, and $Z(r) > Z_\infty - \eps_0$ for $r \in (r_\infty, \infty)$. Therefore, for any $t \ge t_0$, $\partial_r \psi_\infty(r; \rad_t) < 0$ for any $r \in (0, r_0)$, and $\partial_r \psi_\infty(r; \rad_t) > 0$ for any $r \in (r_\infty, +\infty)$. 

As a result, according to the equation (\ref{eqn:psi_in_convergence_proof_isotropic}), we must have $\partial_r \psi_\infty(r; \rad_t) < 0$ for any $r \in (0, r_0)$ and $t \ge t_0$, and $\partial_r \psi_\infty(r; \rad_t) > 0$ for any $r \in (r_\infty, \infty)$ and $t \ge t_0$. 

Due to Lemma \ref{lemma:Density_generalized}, $\rad_{t_0}((0, \infty)) = 1$. Denoting $\Omega_k = [1/k, k]$, then $\lim_{k \to \infty} \rad_{t_0}(\Omega_k) = 1$. With this choice of $\Omega_k$, for any $k \ge \{ r_\infty, 1/r_0 \}$, and for any $t \ge t_0$, we have $ \< \partial_r \psi_\infty( r ; \rad_t), \bn(r)\> > 0$ for $r \in \partial \Omega_k$ where $\bn(r)$ is the normal vector point outside $\Omega_k$. Therefore, if we consider the ODE 
\begin{align}\label{eqn:ODE_in_isotropic_proof}
\dot r(t) = - \partial \psi_\infty(r(t); \rad_t). 
\end{align}
starting with $r(t_0) \in \Omega_k$, $r(t)$ cannot leak outside $\Omega_k$ from either boundaries of $\Omega_k$, and we must have $r(t) \in \Omega_k$ for any $t \ge t_0$. Due to Lemma \ref{lem:MassIncreasing}, $\rad_{t}(\Omega_k) \ge \rad_{t_0}(\Omega_k)$ for any $t \ge t_0$. As a result, we conclude that for any $\rad_\star \in \cS_\star$, 
\begin{equation}
\rad_\star(\cup_k  \Omega_k) \ge \lim_{k\to \infty} \rad_\star( \Omega_k) \ge \lim_{k\to \infty} \rad_{t_0}(\Omega_k) = 1.
\end{equation}
Note $\cup_k \Omega_k = (0, \infty)$. This gives $\rad_\star(\{0, \infty \}) = 0$, which proves Claim $(1)$.

\vskip 0.2cm
\noindent
{\bf Proof of Claim $(2)$, step $(1)$. If $L(\cS_\star) \subseteq \cP_2$, then $\cS_\star$ must be a singleton. } 

In the case $L(\cS_\star) \subseteq \cP_2$, the argument is similar to the proof of Claim $(1)$, and hence will be presented in a synthetic form. First, we must have $L(\cS_\star) \subseteq \cP_2 \cap \{(\lambda_+, \lambda_-) : Z_0 < - \lambda_+ / \lambda_- < Z_\infty \}$. Therefore, there must exist $\eps_0 > 0$, so that for any $\rad_\star \in \cS_\star$, we have $Z_0 + 3\eps_0 < -\lambda_+(\rad_\star) / \lambda_-(\rad_\star) < Z_\infty - 3\eps_0$. For this $\eps_0 > 0$, there exists $t_0$ large enough, so that as $t \ge t_0$, we have $Z_0 + 2\eps_0 < - \lambda_+(\rad_t) / \lambda_-(\rad_t) < Z_\infty - 2\eps_0$, and $\lambda_+(\rad_t) > 0$, $\lambda_-(\rad_t) < 0$. Further, there exists $0 < r_0 < r_\infty < \infty$, so that $\partial_r \psi_\infty(r; \rad_t) > 0$ for any $r \in (0, r_0)$ and $t \ge t_0$, and $\partial_r \psi_\infty(r; \rad_t) < 0$ for any $r \in (r_\infty, \infty)$ and $t \ge t_0$. 

Therefore, if we consider the ODE (\ref{eqn:ODE_in_isotropic_proof}) starting with $r(t_0) \in [0, r_0)$, we must have $r(t) \in [0, r_0)$ for any $t \ge t_0$; if we start with $r(t_0) \in (r_\infty, \infty]$, we must have $r(t) \in (r_\infty, \infty]$ for any $t \ge t_0$. Due to Lemma \ref{lem:MassIncreasing}, $\{\rad_t([0, r))\}_{t \ge t_0}$ for $0 < r \le r_0$ and $\{\rad_t((r, +\infty]) \}_{t \ge t_0}$ for $r \ge r_\infty$ must be non-decreasing in $t$. According to Theorem \ref{thm:local_minimizer_infinite_d_isotropic}, we can express $\rad_\star \in \cS_\star$ in the form $\rad_\star = a_0(\rad_\star) \delta_0 + a_\infty(\rad_\star) \delta_\infty + a(\rad_\star) \delta_{r_\star}$. By the stated monotonicity property, for any $\rad_1, \rad_2 \in \cS_\star$, it holds that $a_0(\rad_1) = a_0(\rad_2)$, $a_\infty(\rad_1) = a_\infty(\rad_2)$, and hence $a(\rad_1) = a(\rad_2)$. We denote them in short as $a_0$, $a_\infty$, and $a$.

For any such fixed point $\rad_\star \in \cS_\star$, since we must have $\supp(\rad_\star) \subseteq \{r: \partial_r \psi_\infty(r; \rad_\star) = 0 \}$, $r_\star \in (0, +\infty)$ should be a solution of $\phi(r) = 0$ where
\[
\phi(r) = (a_0 q(0) + a_\infty q_\infty + a q_+(r) - 1) q_+'(r) + (a_0 q(0) + a_\infty q_\infty + a q_-(r) + 1) q_-'(r).
\] 
By condition {\sf S1}, the function $\phi(r)$ is analytic, and it is not constant. Therefore, the set of all its zeros $\{ r_\star^{i} \}_{i \in \mathbb N} \subseteq (0,  +\infty )$ is a countable set, and it does not have accumulation points in $(0, +\infty)$. Furthermore, according to Lemma \ref{lem:ConnectedLimitingSet}, the limiting set $\cS_\star$ should be a connected compact set with respect to the metric $\db_\cuP$. Therefore, the limiting set could only be a singleton. That is, $\cS_\star = \{a_0 \delta_0 + a_\infty \delta_\infty + a \delta_{r_\star} \}$ for some $r_\star$. 

\vskip 0.2cm
\noindent
{\bf Proof of Claim $(2)$, step $(2)$. If $\rad_\star$ is a fixed point with $L(\rad_\star) \in \cP_2$, then $\rad_\star$ is unstable. }

We apply Theorem \ref{thm:InstabilityDelta} to $\rad_\star = a_0 \delta_0 + a_\infty \delta_\infty + a \delta_{r_\star}$. We will check the conditions of Theorem \ref{thm:InstabilityDelta} to show that this type of fixed point is unstable. 

First we check condition {\sf B1}. Since $[q_-'(r)/q_+'(r)]' > 0$ and $q_+'(r) > 0$ for $r \in (0, +\infty)$, we have
\begin{equation}
q_-''(r_\star) q_+'(r_\star) - q_+''(r_\star) q_-'(r_\star) > 0. 
\end{equation}
Note the stationary condition of the PDE implies
\begin{equation}
\partial_r \psi(r_\star; \rad_\star) = \lambda_+(\rad_\star) q_+'(r_\star) + \lambda_-(\rad_\star) q_-'(r_\star) = 0,
\end{equation}
and $\lambda_+(\rad_\star) > 0$, $\lambda_-(\rad_\star) < 0$. Combined with the equation above, we have
\begin{equation}
\begin{aligned}
\partial_r^2 \psi_\infty(r_\star; \rad_\star) =& \lambda_+(\rad_\star) q_+''(r_\star) + \lambda_-(\rad_\star) q_-''(r_\star) \\
=&[q_+'(r_\star) q_-''(r_\star) - q_-'(r_\star) q_+''(r_\star)] \cdot \lambda_-(\rad_\star)/ q_+'(r_\star) < 0.
\end{aligned}
\end{equation}
This verifies condition ${\sf B1}$ of Theorem \ref{thm:InstabilityDelta}.

Second, since $\lambda_+(\rad_\star) > 0$ and $\lambda_-(\rad_\star) < 0$, according to Equation (\ref{eqn:psi_in_convergence_proof_isotropic}), we must have $\partial_r \psi_\infty(r; \rad_\star) > 0$ for $r \in (0, r_\star)$, and $\partial_r \psi_\infty(r; \rad_\star) < 0$ for $r \in (r_\star, \infty)$. Therefore, we have $\psi_\infty(0; \rad_\star) < \psi_\infty(r_\star; \rad_\star)$ and $\psi_\infty(+\infty; \rad_\star) < \psi_\infty(r_\star; \rad_\star)$. Note $\cL(\eta) \equiv \{ r: \psi_\infty (r; \rad_\star) \le \psi_\infty(r_\star; \rad_\star) - \eta \}$. For any $\eta > 0$ small enough, $\rad_\star(\cL(\eta)) = 1 - a$, which verifies condition ${\sf B2}$. It is also easy to see that, for any $\eta > 0$, $\partial \cL(\eta)$ is a compact set, hence condition {\sf B3} holds. Note that we assumed further that $\rad_0$ has a bounded density with respect to Lebesgue measure, all the assumptions of Theorem \ref{thm:InstabilityDelta} are satisfied. Theorem \ref{thm:InstabilityDelta} implies that the PDE cannot converge to $\rad_\star$. As a result, we conclude that we cannot have $L(\cS_\star(\rad_0)) \subseteq \cP_2$ for $\rad_0 \in \cuP_{\good}$. This proves Claim $(2)$.

\end{proof}

\subsection{Proof of Theorem \ref{thm:ConvergenceIsotropic}}

The key step consists in proving that the dynamics for large but finite $d$ is
well approximated by the dynamics at $d=\infty$. The key estimate is provided by the next lemma.

\begin{lemma}\label{lem:perturbation_bound_isotropic_dynamics}
Assume $\sigma$ satisfies condition {\sf S0}, recall the definition of $u_d$ and $u_\infty$ given by Equation (\ref{eqn:u_d_isotropic}) and (\ref{eqn:u_infty_isotropic}). Then we have
\[
\lim_{d \rightarrow \infty} \sup_{r_1, r_2 \in [0, \infty)} \l u_d(r_1, r_2) - u_\infty(r_1, r_2) \l = 0,
\]
and
\[
\lim_{d \rightarrow \infty} \sup_{r_1, r_2 \in [0, \infty)} \l \partial_1 u_d(r_1, r_2) - \partial_1 u_\infty(r_1, r_2) \l = 0.
\]
\end{lemma}

\begin{proof}
Recall that $u_d$ is given by
\[
\begin{aligned}
u_d(r_1, r_2) =& 1/2 \cdot [u_{d, 1}(r_1, r_2) + u_{d, 2}(r_1, r_2)], \\
u_{d, 1}(r_1, r_2) =& \E[\sigma(r_1 (1 + \Delta) G_1) \sigma(r_2 (1 + \Delta) (G_1\cos\Theta + G_2 \sin\Theta))],\\
u_{d, 2}(r_1, r_2) =& \E[\sigma(r_1 (1 - \Delta) G_1) \sigma(r_2 (1 - \Delta) (G_1 \cos\Theta + G_2 \sin \Theta ))],
\end{aligned}
\]
where $(G_1, G_2) \sim \normal(0, \id_2)$, and $\Theta \sim (1/Z_d) \sin(\theta)^{d-2} \cdot \bfone\{ \theta \in [0, \pi] \} \de \theta$ are mutually independent. 

Define $G_3 = G_1 \cos \Theta + G_2 \sin \Theta$, then
\begin{equation}
\begin{aligned}
& \l u_{d, 1}(r_1, r_2) -  u_{\infty, 1}(r_1, r_2) \l \\
=& \l \E[\sigma(r_1 (1 + \Delta) G_1) [\sigma(r_2 (1 + \Delta) G_3) - \sigma(r_2 (1 + \Delta) G_2)]] \l\\
\le& \| \sigma\|_\infty \E[\l \sigma(r_2 (1 + \Delta) G_3) - \sigma(r_2 (1 + \Delta) G_2)\l ],
\end{aligned}
\end{equation}
and
\begin{equation}
\begin{aligned}
& \l \partial_1 u_{d, 1}(r_1, r_2) - \partial_1 u_{\infty, 1}(r_1, r_2) \l \\
=& \l \E[(1 + \Delta) G_1 \cdot \sigma'(r_1 (1 + \Delta) G_1) [\sigma(r_2 (1 + \Delta) G_3) - \sigma(r_2 (1 + \Delta) G_2)]] \l\\
\le& (1 + \Delta) \| \sigma' \|_\infty \E[G_1^2]^{1/2} \E[[\sigma(r_2 (1 + \Delta) G_3) - \sigma(r_2 (1 + \Delta) G_2)]^2]^{1/2}\\
\le& (1 + \Delta) \| \sigma' \|_\infty  (2 \| \sigma \|_\infty^{1/2}) \cdot \E[\l \sigma(r_2 (1 + \Delta) G_3) - \sigma(r_2 (1 + \Delta) G_2)\l ]^{1/2}.
\end{aligned}
\end{equation}
According to condition {\sf S0}, $\| \sigma'\|_\infty$ and $\| \sigma \|_\infty$ are bounded, it is sufficient to bound the following quantity uniformly for $r \in [0, \infty)$
\begin{equation}
T(r) \equiv 1/2 \cdot \E\big\{ \l \sigma(r G_2) - \sigma(r G_3)] \l \}= \E\big\{ [ \sigma(r G_2) - \sigma(r G_3)]\, \bfone_{G_2>G_3}\big\}\, .
\end{equation}

We claim that, for any $a\in\reals$,
\begin{align}\label{eqn:Gaussian_inequality_in_perturbation_proof}
\prob(G_3 \le a, G_2 \ge a) \le \prob(G_3 \le 0, G_2 \ge 0) = \E[\l \pi/2-\Theta\l /(2\pi)].
\end{align}
Assuming this claim holds, let us show that it implies the desired bound on $T(r)$. We have 
\begin{align*}
T(r) =& \E\left\{\int_\reals \sigma'(t) \, \bfone_{rG_2\ge t\ge rG_3}\, \de t\right\} =  \int_\reals \sigma'(t) \, \prob\big\{G_2\ge t/r\ge G_3\big\}\, \de t\\
\le & \sup_{a\in \reals} \prob(G_3 \le a, G_2 \ge a) \, \int_\reals \sigma'(t)\, \de t \le  2\|\sigma\|_{\infty} \cdot \E[\l \pi/2-\Theta\l /(2\pi)]\, .
\end{align*}
Note that $\cos(\Theta) \ed Z_1/\|\bZ\|_2$ for $\bZ\sim\normal(0,\id_d)$ and hence $\E\{|\Theta-\pi/2|\} \le K/\sqrt{d}$
for a universal constant $K$. We therefore obtain 
\begin{equation}
\sup_r \l T(r)\l \le (K / \pi) \|\sigma\|_{\infty}/\sqrt{d}.
\end{equation}

We are left with the task of proving Eq.~(\ref{eqn:Gaussian_inequality_in_perturbation_proof}).

Denote $X = G_2$ and $Y = G_3$ for simplicity in notations. Note that $(X, Y)\ed (Y, X)\ed (-X,-Y)$. It follows that we can assume, without loss of generality, $a > 0$. We have
\begin{align*}
\prob(Y \le a, X \ge a) =& \prob(Y \le 0, X \ge a) + \prob(0 \le Y \le a, X \ge a), \\
\prob(Y \le 0, X \ge 0) =& \prob(Y \le 0, X \ge a) + \prob(Y \le 0, 0 \le X \le a), 
\end{align*}
suffice to prove that
\[
\prob(0 \le Y \le a, X \ge a) \le \prob(Y \le 0, 0 \le X \le a). 
\]
Define $U = (X - Y)/2$, $V = (X + Y) /2$, and $\cA_1 = \{0 \le Y \le a, X \ge a\}$, $\cA_2 = \{Y \le 0, 0 \le X \le a\}$. It is easy to see that $[U \l \Theta = \theta]$ and $[V \l \Theta = \theta]$ are independent normal random variables. Therefore,
it is sufficient to show $\prob(\cA_1 \l U = u, \Theta = \theta) \le \prob(\cA_2 \l U = u, \Theta = \theta)$ for $u \ge 0$ and $\theta \in [0, \pi]$ (as $u < 0$, both conditional probability equal $0$). 

Fix an $u \ge 0$ and $\theta \in [0, \pi]$. Consider the closed interval $\cI_i = \cI_i(u) \subseteq \reals$ for $i = 1, 2$, with definition $\cI_i(u) \equiv \{v: \{U = u, V = v, \Theta = \theta \} \subseteq \cA_i \}$. Then $\prob(\cA_i \l U = u, \Theta = \theta) = \int_{\cI_i(u)} p_{V \l \Theta}(v \l \theta) \de v$, where $p_{V \l \Theta}(v \l \theta)$ is the density of $[V\l \Theta = \theta]$ at $v$. It is not hard to see that every element in $\cI_1$ is greater or equal to $a/2$, and every element in $\cI_2$ is less or equal to $a/2$; in the meanwhile, $\cI_1$ and $\cI_2$ are symmetric with respect to $a/2$. Note that $[V\l \Theta = \theta]$ is a Gaussian random variable with zero mean, therefore $p_{V\l \Theta}(a/2 + s \l\theta) \le p_{V\l \Theta}(a/2 - s \l\theta)$ for any $s \ge 0$ and $\theta \in [0, \pi]$. This implies that $\prob(\cA_1 \l U = u, \Theta = \theta) \le \prob(\cA_2 \l U = u, \Theta = \theta)$, for any $u \ge 0$ and $\theta \in [0, \pi]$. 
\end{proof}

\begin{lemma}\label{lem:CheckAssumptions}
Let $y \sim \textup{Unif}(\{-1, + 1\})$, $[\bx \l y = +1] \sim \normal(\bzero, \bSigma_+)$, $[\bx \l y = -1] \sim \normal(0,\bSigma_-)$ with 
$\tau_-^2 \id_D \preceq \bSigma_+, \bSigma_- \preceq
\tau_+^2 \id_D$ for some $0 < \tau_- < \tau_+ < \infty$. Assume that the activation function $\sigma$ satisfies condition {\sf S0}. Define 
\begin{equation}\label{eqn:Defintion_of_V_U_in_isotropic_proof}
\begin{aligned}
V(\btheta) =& -\E[y \, \sigma(\< \bx, \btheta\> )], \\
U(\btheta_1, \btheta_2) =& \E[\sigma(\< \bx, \btheta_1\>) \sigma(\< \bx, \btheta_2\>)]. \\
\end{aligned}
\end{equation}
Then assumptions {\sf A2} and {\sf A3} are satisfied. 
\end{lemma}

\begin{proof}
Note that $\bx$ is sub-Gaussian, and by condition {\sf S0} we have $\sigma'$ is bounded, then $\nabla_\btheta \sigma(\< \bx, \btheta \>) = \sigma'(\< \bx, \btheta \>) \bx$ is also sub-Gaussian (with sub-Gaussian parameter independent of $D$). Condition {\sf S0} also gives that $\sigma$ is bounded, therefore assumption {\sf A2} is satisfied. 

To verify assumption {\sf A3}, it is sufficient to check that $\nabla V$, $\nabla_1 U$, $\nabla_{12}^2 U$, $\nabla^2 V$, and $\nabla_{11}^2 U$ are uniformly bounded
in $\ell_2$ norm (for the gradients) or operator norm (for the Hessians). For any unit vector $\bn$, we have 
\begin{align}
\< \nabla V(\btheta), \bn\> =& -\E[y\sigma'(\< \bx, \btheta\>) \< \bx, \bn\>], \\
\< \nabla_1 U(\btheta_1, \btheta_2), \bn \>=& \E[\sigma'(\< \bx, \btheta_1\>) \< \bx, \bn\> \sigma(\< \bx, \btheta_2 \>)], \\
\<\nabla_{12}^2U(\btheta_1, \btheta_2), \bn^{\otimes 2}\> =& \E[\sigma'(\< \bx, \btheta_1\>) \< \bx, \bn\>^2 \sigma'(\< \bx, \btheta_2\>)].
\end{align}
Since $\| \sigma \|_\infty, \| \sigma'\|_\infty < \infty$, applying Cauchy-Schwarz inequality, we have $\nabla V, \nabla_1 U, \nabla_{12}^2U$ are uniformly bounded. 

It is difficult to bound $\nabla^2 V$ and $\nabla_1^2 U$ directly because $\sigma'$ may not be differentiable. We will use a longer argument to bound them. 

First, for a bounded-Lipschitz function $f$, and for $g \in \{ 1, \sigma\}$, define
\begin{align}
W_{f, g}(\btheta_1, \btheta_2) = \E_\bG[f(\< \btheta_1, \bG\>) g(\< \btheta_2, \bG \>)], 
\end{align}
where $\bG \sim \normal(0, \id_d)$. Since we have $\tau_-^2 \id_D \preceq \bSigma_+, \bSigma_- \preceq \tau_+^2 \id_D$ for some $0 < \tau_- < \tau_+ < \infty$, in order to bound $\nabla^2 V$ and $\nabla_1^2 U$, it is sufficient to bound $\nabla_1^2 W_{\sigma, 1}$ and $\nabla_1^2 W_{\sigma, \sigma}$. 

Since $\sigma'$ is $K_0$-Lipschitz on $[- 2\delta_0, 2\delta_0]$ for some $\delta_0 > 0$ and $K_0 < \infty$, then, there exists a function $\sigma_0: \reals \to \reals$, so that $\sigma_0$ is non-decreasing and $K$-bounded-Lipschitz, $\sigma_0'$ is $K$-bounded-Lipschitz, and $\sigma_0(r) = \sigma(r)$ for $r \in [-\delta_0, \delta_0]$. For this $\sigma_0$, a
 second weak derivative exists and $\l \sigma_0''\l \le K$. Hence 
\begin{equation}
\begin{aligned}
\< \nabla_1^2 W_{\sigma_0, g}(\btheta_1, \btheta_2), \bn^{\otimes 2}\> = \E[\sigma_0''(\< \btheta_1, \bG\>) \< \bG, \bn\>^2g(\< \btheta_2, \bG \>)]
\end{aligned}
\end{equation}
is uniformly bounded for $g = 1$ or $g = \sigma$. Let $h = \sigma - \sigma_0$, then $h = 0$ for $r \in [-\delta_0, \delta_0]$, and $h$ is $K$-bounded-Lipschitz for some constant $K$. It is sufficient to bound $\nabla_1^2 W_{h, g}$ for $g \in \{1, \sigma\}$.

Since $\bG$ is Gaussian, using Stein's formula, for any unit vector $\bn$, we have
\begin{equation}
\begin{aligned}
&\< \nabla_1 W_{h, g}(\btheta_1, \btheta_2), \bn\> = \E[h'(\< \btheta_1, \bG\>) \< \bn, \bG\> g(\<\btheta_2, \bG \>)] \\
=& \underbrace{\frac{1}{\| \btheta_1 \|_2^2} \E[h(\< \btheta_1, \bG\>) \< \btheta_1, \bG\> \< \bn, \bG\> g(\<\btheta_2, \bG \>)]}_{E_{1}(\btheta_1, \btheta_2, \bn)} - \underbrace{\frac{1}{ \| \btheta_1 \|_2^2} \E[h(\< \btheta_1, \bG\>) \< \btheta_1, \bn\>  g(\<\btheta_2, \bG \>)]}_{E_{2}(\btheta_1, \btheta_2, \bn)} \\
&- \underbrace{\frac{1}{ \| \btheta_1 \|_2^2} \E[h(\< \btheta_1, \bG\>) \< \bn, \bG\> g'(\<\btheta_2, \bG \>) \< \btheta_2, \btheta_1\> ]}_{E_{3}(\btheta_1, \btheta_2, \bn)}. 
\end{aligned}
\end{equation}
Taking directional derivatives of $E_1$ and $E_2$, we have
\begin{equation}
\begin{aligned}
&\<\nabla_1 E_{1}(\btheta_1, \btheta_2, \bn), \bn\> = \underbrace{\frac{1}{\| \btheta_1 \|_2^2} \E[h'(\< \btheta_1, \bG\>) \< \btheta_1, \bG\> \< \bn, \bG\>^2 g(\<\btheta_2, \bG \>)]}_{E_{11}}\\
&+ \underbrace{\frac{1}{\| \btheta_1 \|_2^2} \E[h(\< \btheta_1, \bG\>) \< \bn, \bG\>^2 g(\<\btheta_2, \bG \>)]}_{E_{12}} - \underbrace{\frac{2\< \btheta_1, \bn\>}{\| \btheta_1 \|_2^4} \E[h(\< \btheta_1, \bG\>) \< \btheta_1, \bG\> \< \bn, \bG\> g(\<\btheta_2, \bG \>)]}_{E_{13}},
\end{aligned}
\end{equation}
and 
\begin{equation}
\begin{aligned}
&\<\nabla_1 E_{2}(\btheta_1, \btheta_2, \bn), \bn\> = \underbrace{\frac{1}{ \| \btheta_1 \|_2^2} \E[h'(\< \btheta_1, \bG\>) \< \btheta_1, \bn\> \< \bG, \bn\>  g(\<\btheta_2, \bG \>)]}_{E_{21}}\\
&+\underbrace{\frac{1}{ \| \btheta_1 \|_2^2} \E[h(\< \btheta_1, \bG\>) g(\<\btheta_2, \bG \>)]}_{E_{22}} - \underbrace{\frac{2\< \btheta_1, \bn \>}{ \| \btheta_1 \|_2^4} \E[h(\< \btheta_1, \bG\>) \< \btheta_1, \bn\>  g(\<\btheta_2, \bG \>)]}_{E_{23}}.
\end{aligned}
\end{equation}

To bound $E_{11}$, note $h'(r) = 0$ for $r \in (-\delta_0, \delta_0)$, and $\l h'(r) \l \le K$ for $r \in \reals$, we have
\begin{equation}
\begin{aligned}
E_{11} \le& \frac{K}{\| \btheta_1 \|_2} \E\Big[\bfone\{\l \< \btheta_1, \bG\> \l \ge \delta_0\} 
\cdot \l \< \btheta_1 / \| \btheta_1 \|_2, \bG\> \l \cdot \< \bn, \bG\>^2 \big|g(\<\btheta_2, \bG \>)\big|\Big] \\
\le& \frac{K}{\| \btheta_1 \|_2} \cdot \prob(\l \< \btheta_1, \bG\> \l \ge \delta_0)^{1/2} \cdot \{\E[ (\< \btheta_1 / \| \btheta_1 \|_2, \bG\>^2 \< \bn, \bG\>^4 g(\<\btheta_2, \bG \>))^2]\}^{1/2}.
\end{aligned}
\end{equation}
Take $r = \| \btheta_1 \|_2$, then
\begin{align}
1/\| \btheta_1 \|_2 \cdot \prob(\l \< \btheta_1, \bG\> \l \ge \delta_0)^{1/2} \le 1/ r \cdot \exp\{ - \delta_0^2/(4 r^2) \}
\end{align}
is uniformly bounded for $r \in [0, \infty]$. Hence $E_{11}$ is uniformly bounded. Using a similar argument, we can show that each terms $E_{12}$, $E_{13}$, $E_{21}$, $E_{22}$, and $E_{23}$ are uniformly bounded.

Now we look at $\nabla_1 E_{3}(\btheta_1, \btheta_2, \bn)$. We have
\begin{equation}
\begin{aligned}
&\<\nabla_1 E_3(\btheta_1, \btheta_2, \bn), \bn\> = \underbrace{\frac{1}{ \| \btheta_1 \|_2^2} \E[h'(\< \btheta_1, \bG\>) \< \bn, \bG\>^2 g'(\<\btheta_2, \bG \>) \< \btheta_2, \btheta_1\> ]}_{E_{31}}\\
& + \underbrace{\frac{1}{ \| \btheta_1 \|_2^2} \E[h(\< \btheta_1, \bG\>) \< \bn, \bG\> g'(\<\btheta_2, \bG \>) \< \btheta_2, \bn\> ]}_{E_{32}}
 - \underbrace{\frac{2\< \btheta_1, \bn\> }{ \| \btheta_1 \|_2^4} \E[h(\< \btheta_1, \bG\>) \< \bn, \bG\> g'(\<\btheta_2, \bG \>) \< \btheta_2, \btheta_1\> ]}_{E_{33}}.
\end{aligned}
\end{equation}


In order to bound $E_{32}$, we apply Stein's formula to get
\begin{equation}
\begin{aligned}
E_{32} =& \frac{\< \btheta_2, \bn\>}{ \| \btheta_1 \|_2^2 \| \btheta_2\|_2^2} \Big\{  \E[h(\< \btheta_1, \bG\>) \< \bn, \bG\> g(\<\btheta_2, \bG \>) \<\btheta_2, \bG \>] \\
& - \E[h(\< \btheta_1, \bG\>) \< \bn, \btheta_2\> g(\<\btheta_2, \bG \>)]  - \E[h'(\< \btheta_1, \bG\>) \< \btheta_1, \btheta_2\> \< \bn, \bG\> g(\<\btheta_2, \bG \>)] \Big\}.
\end{aligned}
\end{equation}
For each terms above, we can bound them using the same argument as for bounding $E_{11}$. Similarly, we can bound $E_{33}$. We cannot apply directly Stein's formula to $E_{31}$ similar to what we did for $E_{32}$, because $h' = \sigma' - \sigma_0'$ may not have weak derivative. However, recall that $h'(r) = 0$ for $r\in [-\delta_0,\delta_0]$ and $h'$ is $K$-bounded. 
Therefore, we can find a function $h_0: \reals \to \reals$, such that $\l h'(r) \l \le h_0(r)$ for $r \in \reals$, $h_0(r) = 0$ for 
$r \in [-\delta_0/2, \delta_0/2]$, and $h_0$ is $K$-bounded-Lipschitz (for some larger constant $K$). Hence, recalling that $g'(r)\ge 0$, we get
\begin{equation}
\begin{aligned}
E_{31} \le \frac{1}{ \| \btheta_1 \|_2} \E[h_0(\< \btheta_1, \bG\>) \< \bn, \bG\>^2 g'(\<\btheta_2, \bG \>) \| \btheta_2 \|_2 ]. 
\end{aligned}
\end{equation}
We can apply Stein's formula to the right hand side of the last equation. Using the same argument
as above, we obtain that $E_{31}$ is uniformly bounded. 

As a result, $\nabla^2 V$ and $\nabla_1^2 U$ are uniformly bounded. Therefore, assumption {\sf A3} is satisfied.

\end{proof}

We are now in position to prove Theorem  \ref{thm:ConvergenceIsotropic}.

\begin{proof}[Proof of Theorem \ref{thm:ConvergenceIsotropic}]

First we consider PDE (\ref{eq:PDERadial}) for $d = \infty$. We fix an initial radial density $\rad_0 \in \cuP_{\good}$. Due to Theorem \ref{thm:PDE_converge_to_global_minimizer}, for any $\eta > 0$, there exists $T = T(\eta, \rad_0, \Delta) > 0$, so that the solution $(\rad_t^\infty)_{t \ge 0}$ of PDE (\ref{eq:PDERadial}) for $d = \infty$ with initialization $\rad_0$ satisfies
\[
\barR_\infty(\rad_t^\infty) \le \inf_{\rad \in \ocD}\barR_\infty(\rad) + \eta / 5
\]
for any $t \ge T$. 


Then we consider the general PDE
\begin{align}
\partial_t\rho_t(\btheta) =& 2 \xi(t) \nabla\cdot \big[\rho_t(\btheta)\nabla\Psi(\btheta;\rho_t)\big]\,, \label{eq:GeneralPDE_App_isotropic}
\end{align}
with initialization $\rho_0$ the distribution of $r\bn$, where $(r, \bn) \sim \rad_0 \times \text{Unif}(\mathbb S^{d - 1})$. Due to Lemma \ref{lem:CheckAssumptions}, we have the existence and uniqueness of the solution of PDE (\ref{eq:GeneralPDE_App_isotropic}), and let $(\rho_t)_{t \ge 0}$ be the solution. Let $\rad_t^d$ be the radial marginal distribution of $\rho_t$. It is easy to see that $(\rad_t^d)_{t \ge 0}$ is the unique solution of (\ref{eq:PDERadial}) for $d$ finite.

Now, we would like to bound the distance of $\rad_t^d$ and $\rad_t^\infty$ using Lemma \ref{lem:perturbation_convergence}. We take $D = 1$, $V = v$, $U = u_d$, $\tilde V = v$, $\tilde U = u_\infty$ in Lemma \ref{lem:perturbation_convergence}.  Let $\eps_0(d)$ be defined as in Eq. (\ref{eqn:perturbation_error_in_perturbation_lemma}). Due to Lemma \ref{lem:perturbation_bound_isotropic_dynamics}, we have $\eps_0(d) \rightarrow 0$ as $d \rightarrow \infty$. Therefore, according to Lemma \ref{lem:perturbation_convergence}, we have $\lim_{d \to \infty} \sup_{t \le 10 T}d_{\sBL}(\rad_t^{d}, \rad_t^{\infty}) = 0$. Further note that $\barR_\infty$ is uniformly continuous with respect to $\rad$ in bounded-Lipschitz distance. Therefore, there exists $d_0 = d_0 (\eta, \rad_0, \Delta)$ large enough, so that for $d \ge d_0$ we have
\[
\l \barR_\infty(\rad_t^d) - \barR_\infty(\rad_t^\infty) \l \le \eta / 5.
\]
for any $t \le 10T$. 

Next we would like to bound the difference of $\barR_\infty(\rad)$ and $\barR_d(\rad)$ for any $\rad$. Note 
\begin{align}
\l \barR_\infty(\rad) - \barR_d(\rad) \l \le \int \l u_d(r_1, r_2) - u_\infty(r_1, r_2)\l\, \rad(\de r_1) \rad(\de r_2).
\end{align}
By Lemma \ref{lem:perturbation_bound_isotropic_dynamics}, there exists $d_0 = d_0(\eta, \Delta)$ large enough, so that for $d \ge d_0$, we have
\begin{align}
\sup_{\rad \in \ocD} \l \barR_\infty(\rad) - \barR_d(\rad) \l \le \eta / 5.
\end{align}

Finally, let $(\btheta^{k})_{k \ge 1}$ be the trajectory of SGD, with step size $s_k = \eps \xi(k \eps)$, and initialization $\bw_i^0 \sim_{iid} \rho_0$ for $i \le N$. We apply Theorem \ref{thm:GeneralPDE} to bound the difference of the law of trajectory of SGD and the solution of PDE (\ref{eq:GeneralPDE_App_isotropic}). The assumptions of Theorem \ref{thm:GeneralPDE} are verified by Lemma \ref{lem:CheckAssumptions}. As a consequence, there exists constant $K$ (which depend uniquely on the constants in assumptions {\sf A1} {\sf A2} {\sf A3}), such that for any $t \le 10 T$, we have
\[
R_{N}(\btheta^{\lfloor t/\eps \rfloor}) - \barR_d(\rad_t^d) \le K e^{10 KT} \cdot \err_{N, d}(z).
\]
with probability $1 - e^{-z^2}$, where
\[
\err_{N, d}(z) = \sqrt{1/N \vee \eps } \cdot \Big[\sqrt{d + \log (N (1/ \eps \vee 1))} + z \Big]. 
\]

As a consequence, for any $\delta > 0$, there exists $C_0 = C_0(\delta, \eta, \rad_0, \Delta)$, so that as $N, 1/\eps \ge C_0 d$ and $\eps \ge 1/N^{10}$, for any $t \le 10 T$, we have
\[
R_{N}(\btheta^{\lfloor t/\eps \rfloor}) - \barR_d(\rad_t^d) \le \eta/5
\]
with probability at least $1 - \delta$. 

Therefore, the trajectory $\btheta^{\lfloor t/\eps \rfloor}$ of SGD as $t \in [T, 10 T]$ satisfies
\[
\begin{aligned}
R_{N}(\btheta^{\lfloor t / \eps\rfloor}) \le&  \barR_d(\rad_t^d) + \eta/5 \le \barR_\infty(\rad_t^d) + 2\eta/5  \le \barR_\infty(\rad_t^\infty) + 3\eta / 5 \\
\le& \inf_{\rad \in \ocD} \barR_\infty(\rad) + 4\eta/5 \le \inf_{\rad \in \ocD} \barR_d(\rad) + \eta = \inf_{\rho \in \cuP(\reals^d)} R(\rho) + \eta \\
\le&\inf_{\btheta \in \reals^{d \times N}} R_N(\btheta) + \eta
\end{aligned}
\]
with probability at least $1 - \delta$. This gives the desired result. 

\end{proof}

\subsection{Checking conditions ${\sf S0}$--${\sf S4}$ for the running example}
\label{sec:AssumptionsExample}

\begin{lemma}\label{lem:unique_solution_condition}
Consider the activation function $\sigma$ with definition in Equation (\ref{eq:SimpleSigma}), with $s_1 < s_2$, $s_1 < -1$, $(s_1 + s_2)/2 > 1$, $(3s_1 + s_2)/4 \in (-1, 1)$, $0 < t_1 < t_2$. For $r \in (0, +\infty)$, define $q(r) = \E_G[\sigma(r G)]$ where $G \sim \normal(0, 1)$. Then conditions ${\sf S0}$--${\sf S4}$ hold. 
\end{lemma}

\begin{remark}
The requirements of Lemma \ref{lem:unique_solution_condition} are not restrictive. An example of parameters that satisfies all conditions gives $s_1 = -2.5$, $s_2 = 7.5$, $t_1 = 0.5$, $t_2 = 1.5$. 
\end{remark}

\begin{proof}

It is straightforward to see that condition {\sf S0} holds. To show condition {\sf S1}, denote by $\sigma'(r)$ the weak derivative of $\sigma(r)$, we calculate the function $q'(r)$ for $r > 0$ explicitly, 
\begin{equation}\label{eqn:derivative_of_q}
\begin{aligned}
q'(r) =& \E[\sigma'( r G) G] = \frac{s_2 - s_1}{t_2 - t_1} \int_\reals \bfone\{ r x \in [t_1, t_2] \} \cdot \frac{1}{\sqrt{2 \pi}} \exp \Big\{ - \frac{x^2}{2} \Big\} \cdot x \cdot \de x\\
=& \frac{s_2 - s_1}{\sqrt{2 \pi} (t_2 - t_1)} \Big\{\exp\Big[- \frac{t_1^2 }{ 2  r^2}\Big] - \exp\Big[- \frac{t_2^2}{2 r^2}\Big]\Big\}.
\end{aligned}
\end{equation}
Since $s_1 < s_2 $ and $0 < t_1 < t_2$, it is easy to see that $q'(r)$ is analytic on $(0,\infty)$, and hence $q(r)$ is analytic on $(0,\infty)$. Differentiating $q'(r)$ in Eq. (\ref{eqn:derivative_of_q}), it is easy to see that $\lim_{r\rightarrow \infty} q''(r) = 0$, and $q''(0 +) = 0$. Hence, we have $\sup_{r \in [0, +\infty]} q''(r) < \infty$. Then condition {\sf S1} holds. 

Since $s_2 > s_1$, $0 <  t_1 < t_2$, we have $q'(r) > 0$ for $r \in (0, +\infty)$, $\lim_{r\rightarrow \infty} q'(r) = 0$, and $q'(0 +) = 0$. Hence, we have $\sup_{r \in [0, +\infty]} q'(r) < \infty$. Then condition {\sf S2} holds. Note that $q(0) = \sigma(0) = s_1  < -1$, and $q(+\infty) = (s_1 + s_2) / 2 > 1$. In addition, $[q(0) + q(+\infty)]/2 = (3 s_1 + s_2)/4 \in (-1, 1)$. Therefore, condition {\sf S3} holds. 

Finally, we show that condition {\sf S4} holds. Define $p(r) = \exp[- t_1^2 / (2 r^2)] - \exp[- t_2^2 / (2 r^2)]$, which is a positively scaled version of $q'(r)$. To show that for $r \in (0, \infty)$, 
\[
[q'(\tau_- r)/q'(\tau_+ r)]' = [\tau_- \cdot q''(\tau_- r) q'(\tau_+ r) - \tau_+ \cdot q'(\tau_- r) q''(\tau_+ r)] / [q'(\tau_+ r)]^2> 0,
\]
we only need to show that for $r \in (0, \infty)$
\[
F_1(r) \equiv \tau_- \cdot p'(\tau_- r) p(\tau_+ r) - \tau_+ \cdot p'(\tau_+ r) p(\tau_- r) > 0. 
\]

We have
\[
\begin{aligned}
F_1(r)
=&+  1 / (\tau_-^2 r^3) \cdot\{ t_1^2  \exp[- t_1^2 / (2 \tau_-^2 r^2)] - t_2^2 \exp[- t_2^2 / (2 \tau_-^2 r^2)] \} \\
&~~~~ \times  \{\exp[- t_1^2 / (2 \tau_+^2 r^2)] -  \exp[- t_2^2 / (2 \tau_+^2 r^2)]\} \\
&-   1 / (\tau_+^2 r^3)  \cdot\{ t_1^2 \exp[- t_1^2 / (2 \tau_+^2 r^2)]  - t_2^2 \exp[- t_2^2 / (2 \tau_+^2 r^2)] \} \\
&~~~~ \times  \{\exp[- t_1^2 / (2  \tau_-^2 r^2)] -  \exp[- t_2^2 / (2 \tau_-^2 r^2)]\}. \\
\end{aligned}
\]
Define $x \equiv t_2^2 / (2 \tau_+^2 r^2) > 0$, $s  \equiv \tau_+^2 / \tau_-^2 > 1$, $0 < c  \equiv t_1^2 /t_2^2 < 1$, we have
\[
\begin{aligned}
F_1(r)=&+  t_2^2 / (\tau_+^2 r^3) \cdot\{  cs \cdot \exp[-xsc] - s \exp[-xs] \} \cdot  \{\exp[- xc] -  \exp[- x]\} \\
&-   t_2^2 / (\tau_+^2 r^3)  \cdot\{ c\cdot \exp[- xc]  - \exp[- x] \} \cdot  \{\exp[- xsc] -  \exp[- xs]\}\\
=& t_2^2/(\tau_+^2 r^3) \{(cs - c) \exp[- xc - xsc] + (c - s) \exp[- xs - xc] \\
&+ (1 - cs ) \exp[- x - xsc] + (s - 1) \exp[-x - xs]\}\\
=& t_2^2/(\tau_+^2 r^3) \exp\{ - x - xsc \}\{(cs - c) \exp[x - xc] \\
&+ (c - s) \exp[x - xs - xc + xsc] + (1 - cs) + (s - 1) \exp[xsc - xs]\}.\\
\end{aligned}
\]

Define
\[
F_2(x; s, c) = (cs - c) \exp[x - xc] + (c - s) \exp[x - xs - xc + xsc] + (1 - cs) + (s - 1) \exp[xsc - xs]. 
\]
It is sufficient to show that $F_2(x; s, c) > 0$ for $x > 0$, $s > 1$, and $0 < c < 1$. Note that $F_2(0 +; s, c) = 0$. Hence it is sufficient to show that $\partial_x F_2(x; s, c) > 0$ for $x > 0$. 

We have 
\[
\begin{aligned}
\partial_x F_2(x; s, c) =& c(s - 1)(1 - c) \exp[x - xc] + (s - c) (s - 1) (1 - c) \exp[x - xs - xc + xsc] \\
&+ (s - 1) s (c - 1) \exp[xsc - xs] \\
=& (s - 1)(1 - c) \exp[xsc - xs] \{c \cdot \exp[x - xc - xsc + xs] + (s - c) \exp[x  - xc] - s  \}. 
\end{aligned}
\]

Define 
\[
F_3(x; s, c) = c \cdot \exp[x - xc - xsc + xs] + (s - c) \exp[x  - xc] - s. 
\]
Note that $s > 1$ and $0 \le c < 1$, $F_3(0 +; s, c) = 0$. It is therefore sufficient to show that $\partial_x F_3(x; s, c) > 0$ for $x > 0$. 

We have
\[
\partial_x F_3(x; s, c) = c(1 - c) (1 + s) \exp[x - xc - xsc + xs]  + (s - c) (1 - c) \exp[x - xc]. 
\]
Since $0 < c < 1$, $s > 1$, and $x > 0$, we have $\partial_x F_3(x; s, c) > 0$, and hence condition {\sf S4} holds. 
\end{proof}

\section{Centered anisotropic Gaussians}\label{sec:AnisotropicGaussian}

In this section we consider the centered anisotropic Gaussian example discussed in the main text. That is, we assume the joint law of $(y,\bx)$ to be as follows: 
\begin{itemize}
\item[] With probability $1/2$: $y=+1$, $\bx\sim\normal(\bzero,\bSigma_+)$.
\item[] With probability $1/2$: $y=-1$, $\bx\sim\normal(\bzero,\bSigma_-)$.
\end{itemize}
We will assume $\bSigma_+, \bSigma_+$ to be diagonalizable in the same orthonormal basis, and to differ only on a subspace 
of dimension $s_0$. We want to study whether and how the neural network will identify this subspace of relevant features.
Without loss of generality, we can assume that the eigenvalues correspond to the standard basis. In order to focus on the simplest 
possible model of this type, we will choose:
\begin{align}
\bSigma_+ &={\rm Diag}\big(\underbrace{(1+\Delta)^2,\dots,(1+\Delta)^2}_{s_0},\underbrace{1,\dots,1}_{d-s_0}\big)\, ,\\
\bSigma_- &={\rm Diag}\big(\underbrace{(1-\Delta)^2,\dots,(1-\Delta)^2}_{s_0},\underbrace{1,\dots,1}_{d-s_0}\big)\, .
\end{align}
We assume $0 < \Delta < 1$. As in the previous section, we choose $\sigma_*(\bx;\btheta_i) = \sigma(\<\bx,\bw_i\>)$ for some activation function $\sigma$. Define $q(r) \equiv \E\{\sigma(rG)\}$ for $G\sim\normal(0,1)$. We assume $\sigma(\,\cdot\,)$ satisfies conditions {\sf S0} - {\sf S4} stated at the beginning of Section \ref{sec:IsotropicGaussian}. We will still use the specific $\sigma$ in Eq. (\ref{eq:SimpleSigma}) as our running example. 

Throughout this section, we assume $s_0 = \gamma \cdot d$ for some fixed $0 < \gamma < 1$. Therefore, as $d \to \infty$, we have $s_0 = \gamma \cdot d \to \infty$ and $d - s_0 = (1 - \gamma)\cdot d \to \infty$. For any $\bw \in \reals^d$, we denote $\bw_1 \in \reals^{s_0}$ and $\bw_2 \in \reals^{d - s_0}$ by writing $\bw = (\bw_1, \bw_2)$. We denote $\tau_{+} = 1 + \Delta$ and $\tau_- = 1 - \Delta$. Then we have $0 < \tau_{-} < 1 < \tau_{+} < 2$. Denote $q_{+}(r) = q(\tau_{+} r)$ and $q_{-}(r) = q(\tau_{-} r)$. For any $\ba = (a_1, a_2) \in \reals^2$, denote
\begin{align}
r_+(\ba) = (\tau_{+}^2 a_1^2 + a_2^2)^{1/2}, ~~ r_-(\ba) = (\tau_{-}^2 a_1^2 + a_2^2)^{1/2}. 
\end{align}

Before analyzing our model, we introduce the function space and space of probability measures we will work on. Let $E_2 \equiv [0, +\infty)^2 \cup \{ \infty \}$. Note there is a bijection $\iota$ between $E_2$ and $\mathbb S^2 \cap \{(x, y, x) \in \reals^3: x, y \ge 0\}$. Indeed, for any $\br = (r_1, r_2) \in [0, +\infty)^2$, consider the line crossing $(r_1, r_2, 0)$ and $(0, 0, 1)$. This line intersects with $\mathbb S^2$ at two points. One intersection point is $(0, 0, 1)$, and we denote the other intersection point as $\iota(\br)$. Moreover, let $\iota(\infty) = (0, 0, 1)$. With this bijection $\iota$, we equip $E_2$ with a metric $\db$ induced by the usual round metric on $\mathbb S^2$. Then $(E_2, \db)$ is a compact metric space, and we will still denote it as $E_2$ for simplicity in notations. We denote $C_b(E_2)$ to be the set of bounded continuous functions on $E_2$, where continuity is defined using the topology generated by $\db$. More explicitly, we have isomorphism
\begin{equation}
C_b(E_2) \simeq \{ f \in C([0, \infty)^2): \exists f(\infty) \equiv \lim_{\| \br \|_2 \to \infty} f(\br),  \sup_{\br \in E_2} f(\br) < \infty \}. 
\end{equation}
Because of condition ${\sf S2}$ and ${\sf S3}$, we have $q \circ r_+, q \circ r_-, q' \circ r_+, q' \circ r_- \in C_b(E_2)$. 

Let $\cuP(E_2)$ be the set of probability measures on $E_2$. Due to Prokhorov's theorem, there exists a complete metric $\db_\cuP$ on $\cuP(E_2)$ equivalent to the topology of weak convergence, so that $(\cuP(E_2), \db_\cuP)$ is a compact metric space. In this section, we will denote by $\ocD = \cuP(E_2)$. 

\subsection{Statics}

Since the distribution of $\bx$ is invariant under rotations in first $s_0$ coordinates, and invariant under rotations in last $d - s_0$ coordinates, so are the functions 
\begin{align}
V(\ba) =& v(\| \ba_1 \|_2, \| \ba_2 \|_2), \\
U(\ba, \bb)=&  u_0(\|\ba_1\|_2, \|\bb_1\|_2,\<\ba_1,\bb_1\>, \| \ba_2 \|_2, \| \bb_2 \|_2, \< \ba_2, \bb_2\> )\,.
\end{align}
These take the form
\[
\begin{aligned}
v(a_1, a_2) & = -\frac{1}{2}\, q(r_+(a_1, a_2))+\frac{1}{2}\, q(r_-(a_1, a_2))\, ,\;\;\;\;\; q(t) =\E\{\sigma( t G)\}
\end{aligned}
\]
and
\[
\begin{aligned}
&u_0(a_1,b_1,a_1b_1\cos\alpha, a_2, b_2, a_2b_2\cos\beta) \\
= &\frac{1}{2}\E\{\sigma(\tau_{+} a_1 F_1 +  a_2 G_1) \sigma(\tau_{+} b_1 F_2 + b_2 G_2)\} +\frac{1}{2}\E\{\sigma(\tau_{-} a_1 F_1 + a_2 G_1) \sigma(\tau_{-} b_1 F_2 + b_2 G_2)\}\, ,
\end{aligned}
\]
where expectations are with respect to standard normals $G,F_1,F_2, G_1, G_2\sim\normal(0,1)$, with $(F_1, F_2)$ independent of $(G_1, G_2)$. Moreover, $(F_1, F_2)$ are jointly Gaussian, $(G_1, G_2)$ are jointly Gaussian, and covariance $\E\{ F_1 F_2\} = \cos \alpha$, $\E\{G_1 G_2\} = \cos\beta$.

In order to minimize $R(\rho)$, it is sufficient to restrict ourselves to distributions  that are invariant under
product of rotations. Indeed, for any probability distribution $\rho$ on $\reals^d$, we can define its symmetrization by letting,
for any Borel set $Q_1 \subseteq \reals^{s_0}$, $Q_2 \subseteq \reals^{d - s_0}$,
\begin{align}
\rho_s(Q_1 \times Q_2) \equiv \int\, \rho((\bR_1 \, Q_1) \times (\bR_2 \, Q_2)) \;\;  \mu_{\mbox{\tiny\rm Haar}} (\de \bR_1) \mu_{\mbox{\tiny \rm Haar}}(\de \bR_2)\, ,
\end{align}
where $\mu_{\mbox{\tiny\rm Haar}}$ is the Haar measure over the group of orthogonal rotations. Since $\rho\mapsto R(\rho)$ is 
convex, $R(\rho_s)\le R(\rho)$.

We therefore restrict ourselves to $\rho$'s that are invariant under product of rotations.
In other words, under $\rho$, the vector $\bw = (\bw_1, \bw_2) \in \reals^d$ is sampled as following: $\bw_1 \in \reals^{s_0}$ is uniformly random conditional on $\|\bw_1\|_2$, and $\bw_2 \in \reals^{d - s_0}$ is uniformly random conditional on $\| \bw_2 \|_2$. We denote by $\rad \in \cuP(E_2)$ the probability distribution of $(\|\bw_1\|_2, \| \bw_2 \|_2)$ when $\bw\sim\rho$ and we let $\barR_d(\rad)$ denote the corresponding risk. We then have
\begin{align}
\barR_d(\rad) &= 1+2 \int v(r_1, r_2) \, \rad(\de \br) + \int u_d(a_1, a_2, b_1, b_2) \, \rad(\de \ba)\,\rad(\de \bb)\,,
\end{align}
and
\begin{align}\label{eqn:u_d_anisotropic}
 u_d(a_1, a_2, b_1, b_2) = \E_{\Theta_1,\Theta_2}[u_0(a_1, b_1, a_1 b_1 \cos\Theta_1, a_2, b_2, a_2 b_2 \cos\Theta_2 )], 
\end{align}
where $\Theta_1 \sim (1/Z_{s_0}) \sin^{s_0 -2} \theta \cdot \bfone\{ \theta \in [0, \pi] \} \de \theta$ and $\Theta_2 \sim (1/Z_{d - s_0}) \sin^{d - s_0 - 2} \theta \cdot \bfone\{ \theta \in [0, \pi] \} \de \theta$ are independent.

As $d\to\infty$, we have $\lim_{d\to\infty}u_d(a_1, a_2, b_1, b_2) = u_{\infty}(a_1, a_2, b_1, b_2)$, with
\begin{align}\label{eqn:u_infty_anisotropic}
u_{\infty}(a_1, a_2, b_1, b_2) =\frac{1}{2 }\Big[q(r_+(a_1, a_2)) q(r_+(b_1, b_2))+ q(r_-(a_1, a_2)) q(r_-(b_1, b_2))\Big],
\end{align}
and the risk function converges to (for $\ba = (a_1, a_2)$)
\begin{align}\label{eqn:Risk_infinite_anisotropic}
\barR_{\infty}(\rad)= \frac{1}{2}\left(1-\int q(r_+(\ba))\, \rad(\de \ba)\right)^2+\frac{1}{2}\left(1 + \int q(r_-(\ba))\, \rad(\de \ba)\right)^2\,.
\end{align}

We also define
\begin{align}
\psi_d(\ba;\rad) = v(\ba) +\int u_d(\ba, \bb) \, \rad(\de \bb)\, .
\end{align}
For $s_0 = \gamma \cdot d$ with $0 < \gamma < 1$ and $d \to \infty$, we have the simpler expression
\begin{align}
\psi_\infty(\ba; \rad) =& \lambda_+(\rad) \cdot q(r_+(\ba)) + \lambda_-(\rad) \cdot q(r_-(\ba)), \label{eqn:psi_in_anisotropic_proof}\\
\lambda_+(\rad) =& \frac{1}{2} [ \< q \circ r_+, \rad\> - 1 ], \label{eq:LambdaPlusAnisotropic}\\
\lambda_-(\rad) =& \frac{1}{2} [ \< q \circ r_-, \rad\> + 1 ].  \label{eq:LambdaMinusAnisotropic}
\end{align}

The following theorem provides a characterization of the global minimizers of $\barR_\infty(\rad)$. 

\begin{theorem}\label{thm:global_minimizer_infinite_d_anisotropic}
Consider $d=\infty$. Recall $\ocD = \cuP(E_2)$ where $E_2 \equiv [0, +\infty)^2 \cup \{\infty\}$. Then there exists $\Delta_{\infty}\in (0,1)$, such that
\begin{enumerate}
\item For $\Delta < \Delta_\infty$, $\inf_{\rad\in\ocD}\barR_\infty(\rad)>0$ and the unique global minimizer of risk function $\barR_\infty(\rad)$ is a point mass located at $(r_*, 0)$ for some $r_* = r_*(\Delta)\in(0,\infty)$. 
\item For  $\Delta \ge \Delta_\infty$, all global minimizers of risk function $\barR_\infty(\rad)$ have risk zero, and there exists a global minimizer that has finite support. 
\end{enumerate}
\end{theorem}

\begin{proof}


Throughout the proof, we will denote $\barR_\infty^{(1)}: \cuP([0, \infty]) \to \reals$ as the risk function defined as in Eq. (\ref{eqn:Risk_infinite_isotropic}), and $\barR_\infty^{(2)}: \cuP(E_2) \to \reals$ as the risk function defined as in Eq. (\ref{eqn:Risk_infinite_anisotropic}). Recall the definition $\tau_{+} = 1 + \Delta$, $\tau_{-} = 1 - \Delta$, $q_{+}(r) = q(\tau_{+} r)$, $q_{-}(r) = q(\tau_{-} r)$, $r_+(\ba) = (\tau_{+}^2 a_1^2 + a_2^2)^{1/2}$, and $r_-(\ba) = (\tau_{-}^2 a_1^2 + a_2^2)^{1/2}$ for $\ba = (a_1, a_2) \in E_2$. 


%
%
Suppose $\rad_2^\star \in \argmin_{\rad_2 \in \cuP(E_2)} \barR_\infty^{(2)}(\rad_2)$. Then we must have $\< q \circ r_+, \rad_2^\star\> \le 1$ and $\< q \circ r_-, \rad_2^\star \> \ge -1$. Indeed, if either $\< q \circ r_+, \rad_2^\star\> > 1$ or $\< q \circ r_-, \rad_2^\star\> < -1$, since $q(+\infty) > 1$ and $q(0) < -1$, the distribution $\rad_2' = a_0 \delta_\bzero + a_\infty \delta_\infty + (1-a_0 - a_\infty) \rad_2^\star$ with appropriate choice of $a_0$ and $a_\infty$ will give a lower risk. 

This $\rad_2^\star \in \cuP(E_2)$ induces a $\rad_1 \in \cuP([0, \infty])$ as follows: for any Borel set $B \subseteq [0, \infty]$, $\rad_1(B) = \rad_2^\star(\{\br \in E_2 : \| \br \|_2 \in B \})$. For this $\rad_1$, it is easy to see that $\<q_-, \rad_1 \> \le \<q \circ r_-, \rad_2^\star \>$ and $\<q_+, \rad_1 \> \ge \<q \circ r_+, \rad_2^\star \>$, and the equalities hold if and only if $\rad_2^\star(E_1) = 1$, where $E_1 \equiv  ([0, +\infty) \times \{ 0 \}) \cup \{ \infty \}$. Since $q(+\infty) > 1$ and $q(0) < -1$, we can take $\rad_1^\star = a_0 \delta_0 + a_\infty \delta_\infty + (1 - a_0 - a_\infty) \rad_1$ with appropriate choice of $a_0$ and $a_\infty$, so that $\< q \circ r_+, \rad_2^\star\> \le \< q_+, \rad_1^\star \> \le 1$ and $\< q\circ r_-, \rad_2^\star\> \ge \< q_-, \rad_1^\star\> \ge - 1$. Therefore, we always have $\inf_{\rad_1 \in \cuP([0, \infty])} \barR_\infty^{(1)}(\rad_1) \le \inf_{\rad_2 \in \cuP(E_2)} \barR_\infty^{(2)}(\rad_2)$, and $\rad_2^\star(E_1) = 1$ for any $\rad_2^\star \in \argmin_{\rad_2 \in \cuP(E_2)} \barR_\infty^{(2)}(\rad_2)$. Note that $\barR_\infty^{(2)}(\rad_1 \times \delta_0) = \barR_\infty^{(1)}(\rad_1)$ for any $\rad_1 \in \cuP([0, \infty])$. Hence, we must have $\inf_{\rad_1 \in \cuP([0, \infty])} \barR_\infty^{(1)}(\rad_1) = \inf_{\rad_2 \in \cuP(E_2)} \barR_\infty^{(2)}(\rad_2)$. 

Due to the above argument, we reduced our analysis to the centered isotropic Gaussians case. All the conclusions can be proved using the same argument as in the proof of Theorem \ref{thm:global_minimizer_infinite_d_isotropic}. 


\end{proof}

\subsection{Dynamics: Fixed points}

We specialize the general evolution (\ref{eq:GeneralPDE_App}) to the present case. Assuming $\rho_0$ to be invariant with respect to products of orthogonal transformations, the same happens for $\rho_t$. We let $\rad_t \in \cuP(E_2)$ denote the distribution of $(\|\bw_1\|_2, \| \bw_2 \|_2)$ when $\bw\sim\rho_t$. Then $\rad_t$ satisfies the following PDE: 
\begin{align}
\partial_t\rad_t(\br) = 2\xi(t)\nabla \cdot \big[\rad_t(\br) \nabla \psi_d(\br ;\rad_t) \big]\, . \label{eq:PDERadialAnisotropic}
\end{align}
We will view this as an evolution in the space of probability distribution on $\ocD = \cuP(E_2)$.

In analogy with Proposition \ref{thm:FixedPoints}, we can prove the following characterization of fixed points.

\begin{proposition}\label{propo:FixedPointdRadialAnisotropic}
A distribution $\rad\in \ocD$ is a fixed point of the PDE (\ref{eq:PDERadialAnisotropic}) if and only if 
\begin{align}
\supp(\rad) \subseteq \{\br \in E_2: \;\nabla_\br \psi_d(\br ; \rad) = \bzero \}.\label{eq:FixedPointRadialAnisotropic}
\end{align}
\end{proposition}

Notice, in particular, global minimizers of $\barR_d(\rad)$ are fixed points of this evolution, but not vice-versa. The next result classifies fixed points.

\begin{theorem}\label{thm:local_minimizer_infinite_d_anisotropic}
Consider $d = \infty$, and recall the definition of $\lambda_+(\rad)$ and $\lambda_-(\rad)$ given by Eq. (\ref{eq:LambdaMinusAnisotropic}) and (\ref{eq:LambdaPlusAnisotropic}). Then the  fixed points of the PDE (\ref{eq:PDERadialAnisotropic}) (i.e. the probability measures $\rad\in\ocD$ satisfying (\ref{eq:FixedPointRadialAnisotropic})) must
be of one of the following types 
\begin{enumerate}
\item[$(a)$] A fixed point with zero risk. 
\item[$(b)$] A point mass $\rad_{r_\star} = \delta_{(r_\star, 0)}$ at some location $(r_\star, 0)$ with $r_\star \not \in\{ 0, +\infty\}$, but not of type $(a)$.
\item[$(c)$] A mixture of the type $\rad = a_0 \delta_\bzero + a_\infty \delta_{\infty} + a_1 \delta_{(r_{\star 1}, 0)} + a_2 \rad_2$ with $\supp(\rad_2) \subseteq \{0 \} \times (0, \infty)$, but not of type $(b)$ and $(a)$. 
\end{enumerate}

For  $\Delta < \Delta_\infty$, the PDE has a unique fixed point of type $(b)$, with $\lambda_+(\rad_\star) < 0$ and $\lambda_-(\rad_\star) > 0$; it has no type-$(a)$ fixed points;
it has possibly  fixed points  of type $(c)$.

For $\Delta > \Delta_\infty$, the PDE has some fixed points of type $(b)$, with $\lambda_+(\rad_\star) > 0$ and $\lambda_-(\rad_\star) < 0$; it also has some type-$(a)$ fixed points;
it has possibly  fixed points  of type $(c)$.

For $\Delta = \Delta_\infty$, the PDE has a unique fixed point of type $(a)$ which is also a delta function at some location $(r_{\star1}, 0)$, and no type $(b)$ fixed points;
it has possibly  fixed points  of type $(c)$.

\end{theorem}

\begin{proof}
We use the characterization of fixed points in Proposition \ref{propo:FixedPointdRadialAnisotropic}.
Recall that $\psi_\infty(\br; \rad_\star)$ is defined as in Eq. (\ref{eqn:psi_in_anisotropic_proof}). The gradient $\nabla \psi_\infty(\br; \rad)$ is given by
\begin{equation}\label{eqn:stationary_in_anisotropic_proof}
\begin{aligned}
\partial_{r_1} \psi_\infty(\br; \rad) =& \lambda_+(\rad) q'(r_+(\br)) \tau_{+}^2 r_1 / r_+(\br) + \lambda_-(\rad) q'(r_-(\br)) \tau_{-}^2 r_1 / r_-(\br),\\
\partial_{r_2} \psi_\infty(\br; \rad) =& \lambda_+(\rad) q'(r_+(\br)) r_2 / r_+(\br) + \lambda_-(\rad) q'(r_-(\br)) r_2 / r_-(\br).
\end{aligned}
\end{equation}

If a fixed point $\rad_\star$ gives $\lambda_+(\rad_\star) = \lambda_-(\rad_\star)=0$, then $\barR_\infty(\rad_\star)=0$. This is type-$(a)$ fixed point. 
Consider then the case $(\lambda_+(\rad_\star), \lambda_-(\rad_\star)) \neq (0, 0)$. 

Suppose $\rad_\star((0, +\infty)^2) > 0$. Since $q'(r) > 0$ and $\tau_{+} > 1 > \tau_{-}$, in order for $\nabla \psi_\infty(\br; \rad_\star) = \bzero$ for some $\br \in (0, +\infty)^2$, we must have $(\lambda_+(\rad_\star), \lambda_-(\rad_\star)) = (0, 0)$. Therefore, as $\rad_\star$ is a fixed point with $(\lambda_+(\rad_\star), \lambda_-(\rad_\star)) \neq (0, 0)$, we must have $\rad_\star((0, +\infty)^2) = 0$. That is, we can write $\rad_\star = a_0 \delta_\bzero + a_\infty \delta_\infty + a_1 \rad_1 + a_2 \rad_2$, with $\supp(\rad_1) \in (0, \infty) \times \{0\}$, and $\supp(\rad_2) \in \{0\} \times (0, \infty) $. 

The solutions of $\nabla \psi_\infty((r_1, r_2); \rad_\star) = 0$ with $r_2 = 0$ are of the form $\bzero$, $(r_{\star1}, 0)$, and $\infty$. Therefore, $\rad_1 = \delta_{(r_{\star1}, 0)}$ for some $r_{\star1} \in (0, \infty)$. Hence, as $\rad_\star$ is not a type-$(a)$ stationary point, it must be a type-$(b)$ or type-$(c)$ stationary point. 

This proves that all fixed points are of type $(a)$, $(b)$, or $(c)$. The remaining claims follows the same argument as the proof of Theorem \ref{thm:local_minimizer_infinite_d_isotropic}. 

\end{proof}

\subsection{Dynamics: Convergence to global minimum for $d=\infty$}

In this section, denote $\cuP_{\good}$ to be 
\begin{align}\label{eqn:goodset_anisotropic_infty_2d}
\cuP_{\good} = \{\rad_0 \in \cuP((0, \infty)^2): \barR_\infty(\rad_0) < 1 \}.
\end{align}

We then prove that the $d=\infty$ dynamics converges to a global minimizer from any initialization $\rad_0 \in \cuP_{\good}$.

\begin{theorem}\label{thm:PDE_converge_to_global_minimizer_anisotropic}
Consider the PDE (\ref{eq:PDERadialAnisotropic}) for $d=\infty$, with initialization $\rad_0\in\cuP_{\good}$. It has a unique solution $(\rad_t)_{t \ge 0}$, such that
\[
\lim_{t\rightarrow +\infty} \barR_\infty(\rad_t) = \inf_{\rad \in \ocD} \barR_\infty(\rad)\, .
\]
\end{theorem}

\begin{proof} Without loss of generality, we assume $\xi(t) = 1/2$. First we show the existence and uniqueness of solution of the PDE. 

\vskip 0.2cm
\noindent
{\bf Step 1. Existence and uniqueness of solution. Mass $\rad_t((0, \infty)^2) = 1$ for all $t$. }

According to conditions {\sf S1} - {\sf S3}, $q(r)$, $q'(r)$, and $q''(r)$ are uniformly bounded on $[0, \infty]$. Note 
\[
\begin{aligned}
v(\br) =& 1/2 \cdot [q(r_-(\br)) - q(r_+(\br))],\\
u_\infty(\br_1, \br_2) =& 1/2 \cdot [q(r_+(\br_1)) q(r_+(\br_2)) + q(r_-(\br_1)) q(r_-(\br_2))].\\
\end{aligned}
\]
Then $\nabla v(\br), \nabla_1 u_\infty(\br_1, \br_2), \nabla^2 v(\br), \nabla_{11}^2 u_\infty(\br_1, \br_2), \nabla_{12}^2 u_\infty(\br_1, \br_2)$ are uniformly bounded. Therefore, conditions {\sf A1} and {\sf A3} are satisfied with $D = 2$, $V = v$, and $U = u$. Then, there is the existence and uniqueness of solution of PDE (\ref{eq:PDERadialAnisotropic}) for $d = \infty$. Denote this solution to be $(\rad_t)_{t \ge 0}$. 

Recall the expression for $\nabla \psi_\infty(\br; \rad)$ in Eq. (\ref{eqn:stationary_in_anisotropic_proof}). It is easy to see that the assumption of Lemma \ref{lemma:Density_generalized} is satisfied with $d = 2$ and $\Psi = \psi_\infty$. Hence, we have $\rad_t((0, \infty)^2) = 1$ for any fixed $t < \infty$. 

\vskip 0.2cm
\noindent
{\bf Step 2. Classify the limiting set $\cS_\star$. }

Recall the definition of $(\cuP(E_2), \db_\cuP)$ at the beginning of Section \ref{sec:AnisotropicGaussian}. Since $(\cuP(E_2), \db_\cuP)$ is a compact metric space, and $( \rad_t )_{t \ge 0}$ is a continuous curve in this space, then there exists a subsequence $(t_k)_{k \ge 1}$ of times, such that $(\rad_{t_k})_{k \ge 1}$ converges in metric $\db_\cuP$ to a probability distribution $\rad_\star \in \cuP(E_2)$. 

For any $\rad_0 \in \cuP_{\good}$, let $\cS_\star = \cS_\star(\rad_0)$ be the set of limiting points of the PDE, 
\[
\cS_\star = \{\rad_\star \in \cuP(E_2): \exists (t_k)_{k \ge 1}, \lim_{k\rightarrow \infty} t_k = +\infty, s.t., \lim_{k \to \infty} \db_{\cuP}(\rad_\star, \rad_{t_k}) = 0 \}. 
\]
Analogous to the proof of Theorem \ref{thm:PDE_converge_to_global_minimizer}, we have the following properties for $\cS_\star$: 
\begin{enumerate}
\item $\cS_\star$ is connected and compact. 
\item For any $\rad_\star \in \cS_\star$, $\rad_\star$ is a fixed point of PDE. 
\item For any $\rad_\star \in \cS_\star$, $\barR_\infty(\rad_\star) = \barR_\star < 1$. 
\end{enumerate}

Recall the definition of $\lambda_+(\rad_\star)$ and $\lambda_-(\rad_\star)$ given by Equation (\ref{eq:LambdaPlusAnisotropic}) and (\ref{eq:LambdaMinusAnisotropic}). Let $\rad_\star$ be a fixed point of PDE such that $\lambda_+(\rad_\star) \ge 0, \lambda_-(\rad_\star) \ge 0$ or $\lambda_+(\rad_\star) \le 0, \lambda_-(\rad_\star) \le 0$ but not both $\lambda_+(\rad_\star)$ and $\lambda_-(\rad_\star)$ equal $0$. In this case, according to Eq. (\ref{eqn:stationary_in_anisotropic_proof}), both $\partial_{r_1} \psi_\infty(\br; \rad_\star)$ and $\partial_{r_2} \psi_\infty(\br; \rad_\star)$ must be strictly positive or strictly negative. Since $\supp(\rad_\star) \subseteq \{ \br \in E_2 :  \nabla_\br \psi_\infty(\br; \rad_\star) = \bzero \}$, $\rad_\star$ must be a combination of two delta functions located at $\bzero$ and $\infty$, i.e., $\rad_\star = a_0 \delta_\bzero + (1 - a_0) \delta_\infty$. But for a fixed point like this, it is easy to see that $\barR_\infty(\rad_\star) \ge 1$. Such fixed points $\rad_\star$ cannot be one of the limiting points of the PDE since $\barR_\infty(\rad_0) < 1$. 

Let $L$ be a mapping $L: \cuP(E_2) \to \reals^2$, $\rad \mapsto (\lambda_+(\rad), \lambda_-(\rad))$. The above argument concludes that for any $\rad_0 \in \cuP_{\good}$, we have
\[
L(\cS_\star(\rad_0)) \cap ( \{ (\lambda_+, \lambda_-): \lambda_+ \ge 0, \lambda_- \ge 0, \text{ or } \lambda_+ \le 0, \lambda_- \le 0\} \setminus \{(0, 0)\} ) = \emptyset.
\]
Since $\cS_\star$ is a connected set, $L(\cS_\star)$ should also be a connected set. Further notice that $\barR_\infty(\rad_\star) = 1/2 \cdot [\lambda_+(\rad_\star)^2 + \lambda_-(\rad_\star)^2]$, and $\barR_\infty(\rad_1) = \barR_\infty(\rad_2)$ for any $\rad_1, \rad_2 \in \cS_\star$. Therefore, we can only have $L(\cS_\star) \subseteq \cP_2 \equiv \{(\lambda_+, \lambda_-): \lambda_+ > 0, \lambda_- < 0\}$, or $L(\cS_\star) \subseteq \cP_1 \equiv \{(\lambda_+, \lambda_-): \lambda_+ < 0, \lambda_- > 0\}$, or $L(\cS_\star) = \{ (0, 0) \}$. 

\vskip 0.2cm
\noindent
{\bf Step 3. Finish the proof using two claims. }

We make the following two claims. 
\begin{enumerate}
\item[] Claim (1). If $L(\cS_\star) \subseteq \cP_1$, then for any $\rad_\star \in \cS_\star$, we have $\rad_\star((0, \infty) \times \{0\} ) = 1$. 
\item[] Claim (2). We cannot have $L(\cS_\star) \subseteq \cP_2$. 
\end{enumerate}

Here we assume these two claims holds, and use it to prove our results. For $\Delta < \Delta_\infty$, we proved in Theorem \ref{thm:local_minimizer_infinite_d_anisotropic} that, there is no fixed point such that $L(\rad_\star) = (0, 0)$. Therefore, we cannot have $L(\cS_\star) = \{ (0, 0) \}$. Due to Claim $(2)$, we cannot have $L(\cS_\star) \subseteq \cP_2$. Hence, we must have $L(\cS_\star) \subseteq \cP_1$. According to Theorem \ref{thm:local_minimizer_infinite_d_anisotropic}, for $\Delta < \Delta_\infty$, the only fixed point of PDE with $\rad_\star( (0, \infty) \times \{0\} ) = 1$ is a point mass at some location $\br_\star = (r_{\star1}, 0)$. Furthermore, this delta function fixed point is unique and is also the global minimizer of the risk. Therefore, we conclude that, for $\Delta < \Delta_\infty$, the PDE will converge to this global minimizer. 

For $\Delta \ge \Delta_\infty$, according to Claim (1), if $\rad_\star$ is a limiting point such that $L(\rad_\star) \in \cP_1$, then $\rad_\star( (0, \infty) \times \{0\} ) = 1$. According to Theorem \ref{thm:local_minimizer_infinite_d_anisotropic}, a fixed point $\rad_\star$ with $\rad_\star( (0, \infty) \times \{0\} ) = 1$ and $L(\rad_\star) \neq (0, 0)$ must be a point mass at some location $\br_\star = (r_{\star1}, 0)$, with $L(\rad_\star) \in \cP_2$. Therefore, we cannot have $L(\cS_\star) \subseteq \cP_1$. Claim $(2)$ also tells us that we cannot have $L(\cS_\star) \subseteq \cP_2$. Hence, we must have $L(\cS_\star) = \{ (0, 0) \}$. In this case, all the points in the set $\cS_\star$ have risk $0$. Therefore, we conclude that, as $\Delta \ge \Delta_\infty$, the PDE will converge to some limiting set with risk $0$.

\vskip 0.2cm
\noindent
{\bf Step 4. Proof of the two claims. }

We are left with the task of proving the two claims above. Before that, we introduce some useful notions used in the proof. Define $Z(\br)$ for $\br \in E_2$, 
\begin{equation}
Z(\br) \equiv [q'(r_-(\br)) r_-(\br)] / [q'(r_+(\br)) r_+(\br)]. 
\end{equation}
Define $Z_l(r) \equiv Z((r, lr))$ for $r, l \in [0, \infty]$. Then we have
\begin{align}\label{eqn:Z_in_anisotropic_proof}
Z_l(r) = [q'((\tau_{-}^2 + l^2 )^{1/2} r) / q'((\tau_{+}^2 + l^2 )^{1/2} r)] \cdot [(\tau_{-}^2 + l^2 )^{1/2} / (\tau_{+}^2 + l^2 )^{1/2}].
\end{align}
According to condition {\sf S4}, for any fixed $l \in [0, \infty]$, $Z_l(r)$ is increasing in $r$.  

Recall the formula of $\nabla_\br \psi_\infty(\br; \rad)$ given by Equation (\ref{eqn:stationary_in_anisotropic_proof}). Define 
\begin{align}
\chi_{\rm nm}(\br; \rad) \equiv& \<\nabla_\br \psi_\infty(\br; \rad) ,\br/ \| \br \|_2\>,\\
\chi_{\rm tg}(\br; \rad) \equiv& \<\nabla_\br \psi_\infty(\br; \rad), (- r_2, r_1)/ \| \br \|_2\>.
\end{align}
Then we have
\begin{equation}\label{eqn:NormalGradient_in_anisotropic_proof}
\begin{aligned}
\chi_{\rm nm}(\br; \rad) =& \lambda_+(\rad) q'(r_+(\br)) r_+(\br) / \| \br \|_2 + \lambda_-(\rad) q'(r_-(\br)) r_-(\br) / \| \br \|_2, \\
=& \lambda_-(\rad) q'(r_+(\br)) r_+(\br) / \| \br \|_2 \cdot [\lambda_+(\rad) / \lambda_-(\rad) + Z(\br)],
\end{aligned}
\end{equation}
and
\begin{equation}\label{eqn:TangentGradient_in_anisotropic_proof}
\begin{aligned}
\chi_{\rm tg}(\br; \rad) =& [+\lambda_+(\rad) (1 - \tau_{+}^2) q'(r_+(\br)) / r_+(\br) \\
&+ \lambda_-(\rad) (1 - \tau_{-}^2)  q'(r_-(\br)) / r_-(\br)]  \times r_1 r_2/ \|\br \|_2. 
\end{aligned}
\end{equation}

\vskip 0.2cm
\noindent
{\bf Proof of Claim $(1)$. If $L(\cS_\star) \subseteq \cP_1$, then for any $\rad_\star \in \cS_\star$, we have $\rad_\star( (0, \infty) \times \{0\} ) = 1$. }  

Assume $L(\cS_\star) \subseteq \cP_1$. There must exist $t_0$ large enough, so that as $t \ge t_0$, we have $\lambda_+(\rad_t) < 0$, and $\lambda_-(\rad_t) > 0$. Therefore, we must have $\chi_{\rm tg}(\br; \rad_t) > 0$ for any $\br \in (0, \infty)^2$. We denote
\begin{equation}\label{eqn:Gamma_area_in_anisotropic_proof_case1}
\Gamma_k \equiv \{\br \in [0, \infty)^2: r_2 \le k \cdot r_1 \}.
\end{equation}
Consider the ODE 
\begin{align}\label{eqn:ODE_in_anisotropic_proof}
\dot \br(t) = - \nabla_{\br} \psi_\infty(\br(t); \rad_t),
\end{align}
starting with $\br(t_0) \in \Gamma_k$ for some $k \in (0, \infty)$, we claim $\br(t) \in \Gamma_k$ for any $t \ge t_0$. Indeed, for any $\br \in \partial \Gamma_k \cap \{ \br: r_2 = k r_1 > 0\}$, its normal vector pointing outside $\Gamma_k$ gives $\bn(\br) = (-r_2, r_1)/ \| \br \|_2$, and hence $\<\nabla_\br \psi_\infty(\br; \rad), \bn(\br)\>  =  \chi_{\rm tg}(\br; \rad_t) > 0$. Therefore, $\br(t)$ cannot leak outside $\Gamma_k$ from this boundary. Further note that $\br(t)$ cannot reach the boundary $([0, \infty) \times \{ 0 \}) \cup \{ \infty \}$ for any finite time $t$. This proves the claim that $\br(t) \in \Gamma_k$ for any $t \ge t_0$. 

According to Lemma \ref{lem:MassIncreasing}, we have $\rho_t(\Gamma_k) \ge \rho_{t_0}(\Gamma_k)$ for any $k \in (0, \infty)$. Furthermore, according to Lemma \ref{lemma:Density_generalized}, $\rad_{t_0}((0, \infty)^2) = 1$, hence $\lim_{k \to \infty} \rad_{t_0}(\Gamma_k) = 1$. Therefore, for any $\rad_\star \in \cS_\star$, we must have 
\begin{equation}
\rad_\star(\{ 0 \} \times (0, \infty)) \le \lim_{k \to \infty} \rad_\star([0, \infty)^2 \setminus \Gamma_k) \le \lim_{k \to \infty}\rad_{t_0}([0, \infty)^2 \setminus \Gamma_k )= 0. 
\end{equation}
Theorem \ref{thm:local_minimizer_infinite_d_anisotropic} implies that for any such fixed point $\rad_\star$, we have $\supp(\rad_\star) \subseteq ([0, \infty) \times \{0\}) \cup \{\infty\}$.  

In this case, we claim $L(\cS_\star) \subseteq \cP_1 \cap \{(\lambda_+, \lambda_-) : Z_0(0) < - \lambda_+ / \lambda_- < Z_0(\infty)\}$. Indeed, suppose there exists $\rad_\star \in \cS_\star$, such that $-\lambda_+(\rad_\star)/\lambda_-(\rad_\star) \ge Z_0(\infty)$ or $-\lambda_-(\rad_\star)/\lambda_-(\rad_\star) \le Z_0(0)$, according to Equation (\ref{eqn:NormalGradient_in_anisotropic_proof}), $\chi_{\rm nm}((r, 0); \rad_\star)$ must be strictly positive or strictly negative. However, we know $\supp(\rad_\star) \in \{\br: \nabla \psi_\infty(\br; \rad_\star) = \bzero \}$.  Hence, $\rad_\star$ should be a combination of two delta functions located at $\bzero$ and $\infty$. Such fixed point $\rad_\star$ has risk $\barR_\infty(\rad_\star) \ge 1$, hence $\rad_\star$ cannot be a limiting point of the PDE. Hence the claim holds.

Since $\cS_\star$ is a compact set, and $L$ is a continuous map, then $L(\cS_\star)$ is a compact set. Therefore, there must exist $\eps_0 > 0$, so that for any $\rad_\star \in \cS_\star$, we have $Z_0(0) + 3\eps_0 < -\lambda_+(\rad_\star) / \lambda_-(\rad_\star) < Z_0(\infty) - 3\eps_0$. For this $\eps_0 > 0$, we take $t_0$ large enough, so that for $t \ge t_0$, we have $Z_0(0) + 2\eps_0 < - \lambda_+(\rad_t) / \lambda_-(\rad_t) < Z_0(\infty) - 2\eps_0$, and $\lambda_+(\rad_t) < 0$, $\lambda_-(\rad_t) > 0$. 

According to the conditions {\sf S0} - {\sf S4} on $q(r)$, for any fixed $l$, $Z_l(r)$ is an increasing function of $r$, and for any fixed $r$, $Z_l(r)$ is continuous in $l$. Therefore, for the fixed $\eps_0 > 0$, there exists $0 < r_0 < r_\infty < \infty$ and $b > 0$, such that
\begin{align}
\sup_{r \in [0, r_0]}\sup_{l \in [0, b]} Z_l(r) <& Z_0(0) + \eps_0,\\
\inf_{r \in [r_\infty, \infty]}\inf_{l \in [0, b]} Z_l(r) >& Z_0(\infty) - \eps_0.
\end{align}
As a result, for any $t \ge t_0$, we have 
\begin{equation}\label{eqn:Boundary1_in_anisotropic_proof_case1}
\begin{aligned}
\chi_{\rm nm}(\br; \rad_t) <& 0, ~~ \forall \br \in \Ball(\bzero; r_0) \cap \Gamma_b,\\
\chi_{\rm nm}(\br; \rad_t) >& 0, ~~ \forall \br \in \Ball(\bzero; r_\infty)^c \cap \Gamma_b,
\end{aligned}
\end{equation}
where $\Gamma_{(\cdot)}$ is defined as in Equation (\ref{eqn:Gamma_area_in_anisotropic_proof_case1}). 

%
According to Lemma \ref{lemma:Density_generalized}, $\rad_{t_0}((0, \infty)^2) = 1$. Define 
\begin{equation}\label{eqn:Ok_in_anisotropic_proof_case1}
O_k = \Gamma_k \cap \Ball(\bzero; k) \cap \Ball(\bzero; 1/k)^c. 
\end{equation}
We have $O_k$ is increasing in $k$, and $\cup_k O_k \supset (0, \infty)^2$. Hence $\lim_{k\to \infty} \rad_{t_0}(O_k) = 1$. Now we fix a parameter $k$. 

Recall the formula for $\chi_{\rm nm}$ and $\chi_{\rm tg}$ given by Equation (\ref{eqn:NormalGradient_in_anisotropic_proof}) and (\ref{eqn:TangentGradient_in_anisotropic_proof}). It is easy to see that, there exists $0 < u_{k1}, u_{k2} < \infty$ depending on $(b, k, \tau_{+}, \tau_{-}, Z_0(0), Z_0(\infty), \eps_0)$, such that for any $\br \in (0, \infty)^2$ with $b \cdot r_1 \le r_2 \le k \cdot r_1$, and $t \ge t_0$, we have
\begin{align}
\chi_{\rm tg}(\br; \rad_t) \ge&   u_{k1} \l \lambda_+(\rad_t)\l q'(r_+(\br)) > 0,\label{eqn:Boundary2_in_anisotropic_proof_case1}\\
\l \chi_{\rm nm}(\br; \rad_t) \l  \le&  u_{k2} \l \lambda_+(\rad_t)\l q'(r_+(\br)) < \infty,
\end{align}
and hence
\begin{equation}\label{eqn:GradientFieldInequality_anisotropic_proof_case1}
\begin{aligned}
\l \chi_{\rm nm}(\br; \rad_t) \l / \chi_{\rm tg}(\br; \rad_t) \le u_{k2} / u_{k1} \equiv u_k < \infty. 
\end{aligned}
\end{equation}


Consider the following spiral curve $\br_k^\infty(s) = (r_{k1}^\infty(s), r_{k2}^\infty(s))$, with
\begin{equation}
\begin{aligned}
r_{k1}^\infty(s) =& k \cdot \cos(\arctan(k) - s) \exp\{ 2 u_k s \},\\
r_{k2}^\infty(s) =& k \cdot \sin(\arctan(k) - s) \exp\{ 2 u_k s \},\\
\end{aligned}
\end{equation}
and another spiral curve $\br_k^0(s) = (r_{k1}^0(s), r_{k2}^0(s))$, with
\begin{equation}
\begin{aligned}
r_{k1}^0(s) =& 1/k \cdot \cos(\arctan(k) - s) \exp\{ - 2 u_k s \},\\
r_{k2}^0(s) =& 1/k \cdot \sin(\arctan(k) - s) \exp\{ - 2 u_k s \},\\
\end{aligned}
\end{equation}
for $s \in [0, s_{k\star}]$ with $s_{k\star} = \arctan(k) - \arctan(b)$. 

Because of inequality (\ref{eqn:GradientFieldInequality_anisotropic_proof_case1}), along the curve $\br_k^\infty(s)$, denoting $\bn(\br_k^\infty(s))$ to be its normal vector with $[\bn(\br_k^\infty(s))]_2 > 0$, we have for any $t \ge t_0$ and $s \in [0, s_{k\star}]$, 
\begin{equation}\label{eqn:Boundary4_in_anisotropic_proof_case1}
\< \nabla \psi_\infty(\br_k^\infty(s); \rad_t), \bn(\br_k^\infty(s)) \> > 0.
\end{equation}
Along the curve $\br_k^0(s)$, denoting $\bn(\br_k^0(s))$ to be its normal vector with $[\bn(\br_k^0(s))]_2 > 0$, we have for any $t \ge t_0$ and $s \in [0, s_{k\star}]$,
\begin{equation}\label{eqn:Boundary3_in_anisotropic_proof_case1}
\< \nabla \psi_\infty(\br_k^0(s); \rad_t), \bn(\br_k^0(s)) \> > 0,
\end{equation}

Define the set $\Omega_k$ to be 
\begin{equation}
\begin{aligned}
\Omega_k = & \Gamma_k \cap \Ball(\bzero; k \cdot \exp\{2 u_k s_{k\star} \}) \cap \Ball(\bzero; 1/k \cdot \exp\{-2u_k s_{k\star} \})^c \\
& \cap \{\br: \exists s \in [0, s_{k\star}], s.t., r_1 = r_{k1}^\infty(s), r_2 \ge r_{k2}^\infty(s)\}^c \\
& \cap \{\br: \exists s \in [0, s_{k\star}], s.t., r_1 = r_{k1}^0(s), r_2 \ge r_{k2}^0(s)\}^c. 
\end{aligned}
\end{equation}
Consider the ODE (\ref{eqn:ODE_in_anisotropic_proof}) starting with $\br(t_0) \in \Omega_k$ for any $k \ge \{r_\infty, 1/r_0 \}$, we claim $\br(t) \in \Omega_k$ for any $t \ge t_0$. Indeed, combining Eq. (\ref{eqn:Boundary1_in_anisotropic_proof_case1}), (\ref{eqn:Boundary2_in_anisotropic_proof_case1}), (\ref{eqn:Boundary3_in_anisotropic_proof_case1}), and (\ref{eqn:Boundary4_in_anisotropic_proof_case1}), for any $\br \in \partial \Omega_k \setminus ( ([0, \infty) \times \{ 0 \}) \cup \{ \infty \})$ and $t \ge t_0$, the gradient $\nabla \psi_\infty(\br; \rad_t)$ pointing outside $\Omega_k$. Therefore, $\br(t)$ cannot leak outside $\Gamma_k$ from this boundary. Further note that $\br(t)$ cannot reach the boundary $([0, \infty) \times \{ 0 \}) \cup \{ \infty \}$ for any finite time $t$. This proves the claim that $\br(t) \in \Omega_k$ for any $t \ge t_0$. According to Lemma \ref{lem:MassIncreasing}, $\rad_t(\overline  \Omega_k) \ge \rad_{t_0} (\overline \Omega_k)$ for any $k \ge \{r_\infty, 1/r_0 \}$ and $t \ge t_0$. 

Recall the definition of $O_k$ given by Equation (\ref{eqn:Ok_in_anisotropic_proof_case1}). Note that $O_k\subseteq \Omega_k$, and $\lim_{k \to \infty} \rad_{t_0}(\overline O_k) = 1$, which implies $\lim_{k \to \infty} \rad_{t_0}(\overline \Omega_k) = 1$. Hence, for any $\rad_\star \in \cS_\star$,
\begin{equation}
\rad_\star(\cup_k \overline \Omega_k) \ge \lim_{k \to \infty} \rad_\star(\overline \Omega_k) \ge \lim_{k \to \infty} \rad_{t_0}(\overline \Omega_k) = 1. 
\end{equation}
It is easy to see that $\cup_k \overline \Omega_k = (0, \infty) \times [0, \infty)$. Combining with the fact that $\rad_\star((0, \infty)^2) = 0$ for any $\rad_\star \in \cS_\star$, claim $(1)$ holds. 

\vskip 0.2cm
\noindent
{\bf Proof of Claim (2). We cannot have $L(\cS_\star) \subseteq \cP_2$. } 

In the case $L(\cS_\star) \subseteq \cP_2$, the argument is similar to the proof of Claim (1), and hence will be presented in a synthetic form. First, there exists $t_0$ large enough, so that as $t \ge t_0$, we have $\lambda_+(\rad_t) > 0$, and $\lambda_-(\rad_t) < 0$. Then $\chi_{\rm tg}(\br; \rad_t) < 0$ for any $\br \in (0, \infty)^2$. Letting
\begin{equation}\label{eqn:Gamma_area_in_anisotropic_proof_case2}
\Gamma_k \equiv \{\br \in [0, \infty)^2: r_1 \le k \cdot r_2 \},
\end{equation}
According to the same argument as in the proof of Claim (1), we have $\rho_t(\Gamma_k) \ge \rho_{t_0}(\Gamma_k)$ for any $k \in (0, \infty)$ and $t \ge t_0$. As a result, we have $\supp(\rad_\star) \subseteq (\{0\} \times [0, \infty)) \cup \{\infty\}$. 

However, the fixed point $\rad_\star$ with support on $(\{0\} \times [0, \infty)) \cup \{\infty\}$ has risk $\barR_\infty(\rad_\star) \ge 1$. Therefore, we cannot have $L(\cS_\star) \subseteq \cP_2$. This proves claim (2).

\end{proof}

\subsection{Dynamics: Proof of Theorem \ref{thm:ConvergenceAnisotropic}}

We will prove that the dynamics for large but finite $d$ is
well approximated by the dynamics at $d=\infty$. The key estimate is provided by the next lemma.

\begin{lemma}\label{lem:perturbation_bound_anisotropic_dynamics}
Assume $\sigma$ satisfies condition {\sf S0}, recall the definition of $u_d$ and $u_\infty$ given by Equation (\ref{eqn:u_d_anisotropic}) and (\ref{eqn:u_infty_anisotropic}). Assuming $k = \gamma \cdot d$ for some $\gamma \in (0, 1)$, then we have
\[
\lim_{d \rightarrow \infty} \sup_{\ba, \bb \in [0, \infty)^2} \l u_d(\ba, \bb) - u_\infty(\ba, \bb) \l = 0.
\]
and
\[
\lim_{d \rightarrow \infty} \sup_{\ba, \bb \in [0, \infty)^2} \| \nabla_\ba u_d(\ba, \bb) - \nabla_\ba u_\infty(\ba, \bb) \|_2 = 0. 
\]
\end{lemma}
%
%
%
%
%
%
%

\begin{proof}
We rewrite $u_d$ here as
\[
\begin{aligned}
u_d(\ba, \bb) =& 1/2 \cdot [u_{d, 1}(\ba, \bb) + u_{d, 2}(\ba, \bb)], \\
u_{d, 1}(\ba, \bb) =& \E[\sigma(\tau_{+} a_1  F_1 + a_2  G_1) \sigma(\tau_{+} b_1 (F_1 \cos \Theta_1 + F_2 \sin \Theta_1)  + b_2 (G_1 \cos \Theta_2 + G_2 \sin \Theta_2))],\\
u_{d, 2}(\ba, \bb) =& \E[\sigma(\tau_{-} a_1  F_1 + a_2  G_1) \sigma(\tau_{-} b_1 (F_1 \cos \Theta_1 + F_2 \sin \Theta_1)  + b_2 (G_1 \cos \Theta_2 + G_2 \sin \Theta_2))],\\
\end{aligned}
\]
where 
\begin{align}
(F_1, F_2, G_1, G_2) \sim& \normal(0, \id_4),\\
\Theta_1 \sim& (1/Z_{s_0})\sin(\theta)^{s_0 -2} \bfone\{ \theta \in [0, \pi] \} \de \theta,\\
\Theta_2 \sim& (1 / Z_{d - s_0}) \sin(\theta)^{d - s_0 -2} \bfone\{ \theta \in [0, \pi] \} \de \theta,
\end{align}
are mutually independent. 

Define $F_3 = F_1 \cos \Theta_1 + F_2 \sin \Theta_1$, $G_3 = G_1 \cos \Theta_2 + G_2 \sin \Theta_2$, then
\begin{equation}
\begin{aligned}
& \l  u_{d, 1}(\ba, \bb) -  u_{\infty, 1}(\ba, \bb) \l \\
=& \l \E\{\sigma(\tau_{+} a_1  F_1 + a_2  G_1) [\sigma(\tau_{+} b_1  F_3 + b_2  G_3) - \sigma(\tau_{+} b_1  F_2 + b_2  G_2)] \} \l\\
\le& \| \sigma \|_\infty \cdot \E\{\l \sigma(\tau_{+} b_1  F_3 + b_2  G_3) - \sigma(\tau_{+} b_1  F_2 + b_2  G_2) \l \},
\end{aligned}
\end{equation}
and
\begin{equation}
\begin{aligned}
& \l \partial_{a_1} u_{d, 1}(\ba, \bb) - \partial_{a_1} u_{\infty, 1}(\ba, \bb) \l \\
=& \l \E\{ \tau_{+} F_1 \cdot \sigma'(\tau_{+} a_1  F_1 + a_2  G_1) [\sigma(\tau_{+} b_1  F_3 + b_2  G_3) - \sigma(\tau_{+} b_1  F_2 + b_2  G_2)] \} \l\\
\le&  \tau_{+} \| \sigma' \|_\infty \E[F_1^2]^{1/2} \E\{ [\sigma(\tau_{+} b_1  F_3 + b_2  G_3) - \sigma(\tau_{+} b_1  F_2 + b_2  G_2)]^2 \}^{1/2}\\
\le&  \tau_{+} \| \sigma' \|_\infty  (2 \| \sigma \|_\infty^{1/2}) \cdot \E\{\l \sigma(\tau_{+} b_1  F_3 + b_2  G_3) - \sigma(\tau_{+} b_1  F_2 + b_2  G_2) \l \}^{1/2}.
\end{aligned}
\end{equation}
We have similar bounds for $\l \partial_{a_2} u_{d, 1}(\ba, \bb) - \partial_{a_2} u_{\infty, 1}(\ba, \bb) \l$. 

According to condition {\sf S0}, $\| \sigma'\|_\infty$ and $\| \sigma \|_\infty$ are bounded, it is sufficient to bound the following quantity uniformly for $r \in [0, \infty)$ and $\ba \in \mathbb S^1$,
\begin{equation}
T(r, \ba) \equiv 1/2 \cdot \E\big\{ \l \sigma(r H_2) - \sigma(r H_3) \l \}= \E\big\{ [ \sigma(r H_2) - \sigma(r H_3)]\, \bfone_{H_2>H_3}\big\}\,,
\end{equation}
where
\begin{align}
H_2 = H_2(\ba) = & [\tau_{+} a_1 F_2 + a_2 G_2]/[\tau_{+}^2 a_1^2 +  a_2^2]^{1/2},\\
H_3 = H_3(\ba) =& [\tau_{+} a_1 F_3 + a_2 G_3]/[\tau_{+}^2 a_1^2 + a_2^2]^{1/2}. 
\end{align}

We denote $\Theta_3 = \Theta_3(\ba)= \arcsin\{\E[H_2 H_3 \l \Theta_1, \Theta_2]\}$. It is easy to see that $H_2, H_3 \sim \normal(0, 1)$ with 
\begin{equation}\label{eqn:Theta3_in_anisotropic_proof}
\sin(\Theta_3) = \E[H_2 H_3 \l \Theta_1, \Theta_2] = [\tau_{+}^2 a_1^2 \sin \Theta_1 +  a_2^2 \sin \Theta_2]/ [\tau_{+}^2 a_1^2 + a_2^2]. 
\end{equation}
Using the same argument as in the proof of Theorem \ref{lem:perturbation_bound_isotropic_dynamics}, we have for any $z \in\reals$,
\begin{align}\label{eqn:Gaussian_inequality_in_perturbation_proof_anisotropic}
\prob(H_3 \le z, H_2 \ge z) \le \prob(H_3 \le 0, H_2 \ge 0) = \E[\l \pi/2-\Theta_3\l /(2\pi)].
\end{align}
%
Hence, we have 
\begin{align*}
T(r, \ba) =& \E\left\{\int_\reals \sigma'(t) \, \bfone_{rH_2\ge t\ge rH_3}\, \de t\right\} =  \int_\reals \sigma'(t) \, \prob\big\{H_2\ge t/r\ge H_3\big\}\, \de t\\
\le & \sup_{z\in \reals} \prob(H_3 \le z, H_2 \ge z) \, \int_\reals \sigma'(t)\, \de t \le  2\|\sigma\|_{\infty} \cdot \E[\l \pi/2-\Theta_3\l /(2\pi)]\, .
\end{align*}
Note that $\cos(\Theta_1) \ed Y_{1}/\|\bY\|_2$, for $\bY \sim\normal(0,\id_{s_0})$, and $\cos(\Theta_2) \ed Z_{1}/\|\bZ\|_2$, for $\bZ\sim\normal(0,\id_{d-s_0})$. Hence, there exists a universal constant $K$, such that $\E\{|\Theta_1-\pi/2|\} \le K/\sqrt{s_0}$, $\E\{|\Theta_2-\pi/2|\} \le K/\sqrt{d - s_0}$. 

Note the relationship of $\Theta_3 = \Theta_3(\ba)$ with $(\Theta_1, \Theta_2)$ is given by Eq. (\ref{eqn:Theta3_in_anisotropic_proof}), which yields 
\begin{align}
\sin(\Theta_3(\ba)) \ge \min\{ \sin\Theta_1, \sin \Theta_2 \},
\end{align}
hence
\begin{align}
\l \pi /2 - \Theta_3(\ba) \l \le \max\{ \l \pi /2 - \Theta_1 \l, \l \pi /2 - \Theta_2 \l \}. 
\end{align}
As a result, 
\begin{align}
\sup_{\ba \in \mathbb S^1}\E\{|\Theta_3(\ba)-\pi/2|\} \le& K \cdot \max\{ 1/\sqrt{s_0}, 1 / \sqrt{d - s_0} \}.
\end{align}
We therefore obtain 
\begin{align}
\sup_{r \in \reals, \ba \in \mathbb S^1} \l T(r, \ba) \l \le K / \pi \cdot \|\sigma\|_{\infty} \cdot \max\{ 1/\sqrt{s_0}, 1/ \sqrt{d - s_0} \}.
\end{align}
The lemma holds by noting that as $d \to \infty$, we have $s_0 \to \infty$ and $d - s_0 \to \infty$.
\end{proof}
%
%
%

\begin{proof}[Proof of Theorem \ref{thm:ConvergenceAnisotropic}]

Recall the definition of $\barR_\infty$ given by Eq. (\ref{eqn:Risk_infinite_anisotropic}), and $R$ given by Eq. (\ref{eq:R-Rho-def}). Recall the set of good initialization given by
\[
\cuP_{\good} = \{\rad_0 \in \cuP((0, \infty)): \lim_{d \to \infty} R(\rad \times \text{Unif}(\mathbb S^{d-1})) < 1 \}.
\]
Define $\cuP_{\good}^1$ and $\cuP_{\good}^2$ to be
\begin{align}
\cuP_{\good}^1 =& \{\rad_0^1 \in \cuP( (0, \infty) ):\barR_\infty(\rad_0^2) < 1,  \text{ where } \rad_0^2 \sim (\gamma^{1/2} u, (1 - \gamma)^{1/2} u) \text{ with } u \sim \rad_0^1 \}, \\
\cuP_{\good}^2 =& \{\rad_0^2 \in \cuP((0, \infty)^2): \barR_\infty(\rad_0^2) < 1 \}.
\end{align}
With this definition, it is easy to see that $\cuP_{\good}^1 = \cuP_{\good}$. 

For any $\rad_0^1 \in \cuP_{\good}^1$, let $u \sim \rad_0^1$, $Y_1 \sim \chi^2(\gamma \cdot d)$, and $Y_2 \sim \chi^2((1 - \gamma) \cdot d)$ be independent. We take $u_{d1} = u \cdot [Y_1 / (Y_1 + Y_2)]^{1/2}$, $u_{d2} = u \cdot [Y_2 / (Y_1 + Y_2)]^{1/2}$, $\bu_d = (u_{d1}, u_{d2})$, $u_{\infty1} = u \cdot [s_0 / d]^{1/2} = u \cdot \gamma^{1/2}$, $u_{\infty2} = u \cdot [(d - s_0) / d]^{1/2} = u \cdot (1 - \gamma)^{1/2}$, and $\bu_\infty = (u_{\infty1}, u_{\infty2})$. Denote $\rad_{0}^{2, d}$ to be the distribution of $\bu_d$, and $\rad_{0}^{2, \infty}$ to be the distribution of $\bu_\infty$. Then we have $\rad_{0}^{2, \infty} \in \cuP_{\good}^2$. Further, if we sample $(r, \bn) \sim \rad_0^1 \times \text{Unif}(\mathbb S^{d-1})$ and $ (\br, \bn_1, \bn_2) \sim \rad_{0}^{2, d} \times \text{Unif}(\mathbb S^{k - 1}) \times \text{Unif}(\mathbb S^{d - k - 1})$, then $r \bn \ed (r_1 \bn_1, r_2 \bn_2)$. 

Here we bound $d_{\sBL}(\rad_{0}^{2, d}, \rad_{0}^{2, \infty})$. Note the joint distribution of $\bu_d$ and $\bu_\infty$ is a coupling of $\rad_{0}^{2, d}$ and $\rad_{0}^{2, \infty}$, hence 
\begin{equation}
\begin{aligned}
d_{\sBL}(\rad_{0}^{2, d}, \rad_{0}^{2, \infty}) \le& \E[ \| \bu_d - \bu_\infty \|_2 \wedge 1] \\
=& \E[ \{  u [ ( (Y_1/(Y_1 + Y_2))^{1/2} - \gamma^{1/2} )^2 + ( (Y_2/(Y_1 + Y_2))^{1/2} - (1-\gamma)^{1/2} )^2 ]^{1/2}  \} \wedge 1].
\end{aligned}
\end{equation}
It is easy to see that $\lim_{d \to \infty} Y_1 / (Y_1 + Y_2) = \gamma$ almost surely. Bounded convergence theorem implies that $\lim_{d \to \infty} d_{\sBL}(\rad_{0}^{2, d}, \rad_{0}^{2, \infty}) = 0$. 

Now we consider the PDE (\ref{eq:PDERadialAnisotropic}) for $d = \infty$. We fix its initialization $\rad_{0}^{2, \infty} \in \cuP_{\good}^2$ induced by $\rad_0^1 \in \cuP_{\good}^1$. Denote the solution of PDE (\ref{eq:PDERadialAnisotropic}) to be $(\rad_t^\infty)_{t \ge 0}$. Due to Theorem \ref{thm:PDE_converge_to_global_minimizer_anisotropic}, for any $\eta > 0$, there exists $T = T(\eta, \rad_0^1, \gamma, \Delta) > 0$, so that its solution $(\rad_t^\infty)_{t \ge 0}$ satisfies
\[
\barR_\infty(\rad_t^\infty) \le \inf_{\rad \in \cuP(E_2)}\barR_\infty(\rad) + \eta / 5 
\]
for any $t \ge T$. 

Then we consider the general PDE
\begin{align}
\partial_t\rho_t(\btheta) =& 2 \xi(t) \nabla\cdot \big[\rho_t(\btheta)\nabla\Psi(\btheta;\rho_t)\big]\,, \label{eq:GeneralPDE_App_anisotropic}
\end{align}
with initialization $\rho_0$ the distribution of $r\bn$, where $(r, \bn) \sim \rad_0^1 \times \text{Unif}(\mathbb S^{d - 1})$. Due to Lemma \ref{lem:CheckAssumptions} and Remark \ref{rmk:ExistenceUniqueness}, we have the existence and uniqueness of the solution of PDE (\ref{eq:GeneralPDE_App_anisotropic}). We denote its solution to be $(\rho_t)_{t \ge 0}$. Let $\rad_t^d$ be the distribution of $(\| \bw_1 \|_2, \|\bw_2\|_2)$ with $\bw = (\bw_1, \bw_2) \sim \rho_t$, $\bw_1 \in \reals^{s_0}$ and $\bw_2 \in \reals^{d - s_0}$. It is easy to see that $(\rad_t^d)_{t \ge 0}$ is the unique solution of (\ref{eq:PDERadialAnisotropic}) with initialization $\rad_{0}^{2, d}$.

Now, we would like to bound the distance of $\rad_t^d$ and $\rad_t^\infty$ using Lemma \ref{lem:perturbation_convergence}. We take $D = 2$, $V = v$, $U = u_d$, $\tilde V = v$, $\tilde U = u_\infty$ in Lemma \ref{lem:perturbation_convergence}.  Let $\eps_0(d)$ be as defined in Eq. (\ref{eqn:perturbation_error_in_perturbation_lemma}). Due to Lemma \ref{lem:perturbation_bound_anisotropic_dynamics}, we have $\lim_{d \to \infty} \eps_0(d) = 0$. We also showed that $\lim_{d \to \infty} d_{\sBL}(\rad_{0}^{2, d}, \rad_{0}^{2, \infty} ) = 0$. Therefore, according to Lemma \ref{lem:perturbation_convergence}, we have $\lim_{d \to \infty} \sup_{t \le 10 T}d_{\sBL}(\rad_t^{2, d}, \rad_t^{2, \infty}) = 0$. Further note $\barR_\infty$ is uniformly continuous with respect to $\rad$ in bounded-Lipschitz distance. Therefore, there exists $d_0 = d_0 (\eta, \rad_0^1, \gamma, \Delta)$ large enough, so that for $d \ge d_0$ we have
\[
\l \barR_\infty(\rad_t^d) - \barR_\infty(\rad_t^\infty) \l \le \eta / 5.
\]
for any $t \le 10 T$. 

Then we would like to bound the difference of $\barR_\infty(\rad)$ and $\barR_d(\rad)$ for any $\rad$. Note 
\begin{align}
\l \barR_\infty(\rad) - \barR_d(\rad) \l \le \int \l u_d(\ba, \bb) - u_\infty(\ba, \bb)\l \rad(\de \ba) \rad(\de \bb).
\end{align}
By Lemma \ref{lem:perturbation_bound_anisotropic_dynamics}, there exists $d_0 = d_0(\eta, \Delta)$ large enough, so that for $d \ge d_0$, we have
\begin{align}
\sup_{\rad \in \cuP(E_2)} \l \barR_\infty(\rad) - \barR_d(\rad) \l \le \eta / 5.
\end{align}

Finally, let $(\btheta^{k})_{k \ge 1}$ be the trajectory of SGD, with step size $s_k = \eps \xi(k \eps)$, and initialization $\bw_i^0 \sim_{iid} \rho_0$ for $i \le N$. We apply Theorem \ref{thm:GeneralPDE} to bound the difference of the law of trajectory of SGD and the solution of PDE (\ref{eq:GeneralPDE_App_anisotropic}). The assumptions of Theorem \ref{thm:GeneralPDE} are verified by Lemma \ref{lem:CheckAssumptions}. As a consequence, there exists constant $K$ (which depend uniquely on the constants in assumptions {\sf A1} {\sf A2} {\sf A3}), such that 
\[
R_{N}(\btheta^{\lfloor t/\eps \rfloor}) - \barR_d(\rad_t^d) \le K e^{10 KT} \cdot \err_{N, d}(z).
\]
with probability $1 - e^{-z^2}$ for any $t \le 10 T$, where
\[
\err_{N, d}(z) =  \sqrt{1/N \vee \eps } \cdot \Big[\sqrt{D + \log (N (1/ \eps \vee 1))} + z \Big]. 
\]

As a consequence, for any $\delta > 0$, there exists $C_0 = C_0(\delta, \eta, \rad_0^1, \gamma, \Delta)$, so that as $N, 1/\eps \ge C_0 d$ and $\eps \ge 1/N^{10}$, for $t \le 10 T$, we have 
\[
R_{N}(\btheta^{\lfloor t/\eps \rfloor}) - \barR_d(\rad_t^d) \le \eta / 5
\]
with probability at least $1 - \delta$. 

Therefore, the trajectory $\btheta^{\lfloor t/\eps \rfloor}$ of SGD as $t \in [T, 10 T]$ satisfies
\[
\begin{aligned}
R_{N}(\btheta^{\lfloor t / \eps\rfloor}) \le&  \barR_d(\rad_t^d) + \eta/5 \le \barR_\infty(\rad_t^d) + 2\eta/5  \le \barR_\infty(\rad_t^\infty) + 3\eta / 5 \\
\le& \inf_{\rad \in \ocD} \barR_\infty(\rad) + 4\eta/5 \le \inf_{\rad \in \ocD} \barR_d(\rad) + \eta = \inf_{\rho \in \cuP(\reals^d)} R(\rho) + \eta \\
\le&\inf_{\btheta \in \reals^{d \times N}} R_N(\btheta) + \eta
\end{aligned}
\]
with probability at least $1 - \delta$. This gives the desired result. 
\end{proof}

\section{Finite temperature}\label{sec:FiniteT}

We will states the lemma regarding statics properties of the finite temperature free energy in Section \ref{sec:Statics_finite_T}, and regarding dynamics properties in Section \ref{sec:Dynamics_finite_T}. We will prove Proposition \ref{thm:FixedPoints_Temp_finite}, Theorem \ref{thm:GeneralPDE_Noisy_fixed_point}, and Theorem \ref{thm:GeneralPDE_Noisy} in Section \ref{sec:Proof_finite_T}. Throughout Section \ref{sec:Statics_finite_T} and \ref{sec:Dynamics_finite_T}, to distinguish the dimension of parameters with the generalized differential operator, we will denote the dimension of parameters by $d$ instead of $D$. This should not be confused with the dimension of feature vectors, which never appears throughout this section. 

We introduce the set $\cK$ of admissible probability densities, 
\begin{align}
\cK = \Big\{ \rho: \reals^d \to [0, +\infty) \text{ measurable } : \int_{\reals^d} \rho(\btheta) \de \btheta = 1, M(\rho) < \infty \Big\},
\end{align}
where
\begin{align}
M(\rho) \equiv \int_{\reals^d} \| \btheta \|_2^2 \rho(\btheta) \de \btheta. 
\end{align}
Recall
\begin{align}
R(\rho) =&  R_{\#}+2 \int_{\reals^d} V(\btheta) \rho(\btheta) \de \btheta + \int_{\reals^d \times \reals^d} U(\btheta, \btheta') \rho(\btheta) \rho(\btheta') \de \btheta \de \btheta',\\
R_{\#}  =& \E\{y^2\}\, ,\;\;\;\;\;\; V(\btheta) = - \E\big\{y\,\sigma_*(\bx;\btheta)\big\}\, ,\\
U(\btheta_1, \btheta_2)  =& \E\big\{\sigma_*(\bx;\btheta_1) \sigma_*(\bx;\btheta_2)\big\}\,,\\
\Psi(\btheta;\rho) =& V(\btheta)+ \int_{\reals^d} U(\btheta, \btheta')\; \rho(\btheta') \de \btheta'\, .
\end{align}
Let 
\begin{align}
R_\lambda(\rho)=& \lambda M(\rho) + R(\rho),\\
\Psi_\lambda(\btheta; \rho) =& \lambda/2 \cdot \| \btheta \|_2^2 + V(\btheta) + \int_{\reals^d} U(\btheta, \btheta') \rho(\btheta') \de \btheta',\\
\ent(\rho) =& - \int_{\reals^d} \rho(\btheta) \log \rho(\btheta) \de \btheta,\\
F_{\beta, \lambda}(\rho) =& 1/2 \cdot [\lambda M(\rho) + R(\rho)] - 1/\beta \cdot \ent(\rho). 
\end{align}

\subsection{Statics}
\label{sec:Statics_finite_T}

\begin{lemma}\label{lem:Bound_H_by_M}
For any $\rho \in \cK$, we have 
\begin{align}\label{eqn:Bound_H_by_M}
\ent(\rho) \le \int_{\reals^d} \rho(\btheta) \cdot \l \min(\log \rho(\btheta), 0 )\l \cdot \de \btheta \le 1 + M(\rho)/\sigma^2 + d \cdot \log (2\pi \sigma^2)
\end{align}
for any $\sigma^2 > 0$. 
\end{lemma}

\begin{proof}

Define $\Omega_0 = \{ \btheta: 1/(\sqrt{2\pi} \sigma)^d \cdot \exp\{ - \| \btheta \|_2^2/(2\sigma^2) \} \le \rho(\btheta)^{1/2} \le 1 \}$. Then we have 
\[
\begin{aligned}
\ent(\rho) =& - \int_{\reals^d} \rho(\btheta) \log \rho(\btheta) \de \btheta \le \int_{\reals^d} \rho(\btheta) \cdot \l \min(\log \rho(\btheta), 0 )\l \cdot \de \btheta\\
\le&\int_{\Omega_0} \rho(\btheta) \cdot \l \min(\log \rho(\btheta), 0 )\l \cdot \de \btheta + \int_{\Omega_0^c} \rho(\btheta) \cdot \l \min(\log \rho(\btheta), 0 )\l \cdot \de \btheta.
\end{aligned}
\]
The first term is bounded by 
\[
\begin{aligned}
\int_{\Omega_0} \rho(\btheta) \cdot \l \min(\log \rho(\btheta), 0 )\l \cdot \de \btheta \le \int_{\reals^d} \rho(\btheta) [\| \btheta \|_2^2/\sigma^2 + d \cdot \log(2 \pi \sigma^2)] \de \btheta = M(\rho)/\sigma^2 + d \cdot \log(2 \pi \sigma^2). 
\end{aligned}
\]
Noting that $\l \rho \log\rho\l \le \sqrt{\rho}$ for any $\rho \in [0, 1]$, the second term is bounded by
\[
\begin{aligned}
&\int_{\Omega_0^c} \rho(\btheta) \cdot \l \min(\log \rho(\btheta), 0 )\l \cdot \de \btheta \le \int_{\Omega_0^c} \rho(\btheta)^{1/2} \bfone\{ \rho(\btheta) \le 1 \} \de \btheta \\
\le& \int_{\reals^d} 1/(\sqrt{2\pi} \sigma)^d \cdot \exp\{ - \| \btheta \|_2^2/(2\sigma^2) \} \de \btheta =1. 
\end{aligned}
\]

\end{proof}

\begin{lemma}\label{lem:Free_Energy_Unique_Minimum}
Assume $U$ and $V$ are bounded-Lipschitz. Then for any $\lambda > 0$ and $0 < \beta < \infty$, $F_{\beta, \lambda}(\rho)$ has a unique minimizer $\rho_\star \in \cK$. Moreover, we have
\begin{align}\label{eqn:F_lower_bound_by_M}
F_{\beta, \lambda}(\rho) \ge 1/2 \cdot R(\rho) + \lambda/4 \cdot M(\rho) - 1/\beta \cdot [1 + d\cdot \log(8\pi/(\beta \lambda))].
\end{align}
\end{lemma}

\begin{proof}
%
%

First, by Lemma \ref{lem:Bound_H_by_M}, we have 
\[
\begin{aligned}
F_{\beta, \lambda}(\rho) =& 1/2 \cdot R(\rho) + \lambda/2 \cdot M(\rho) - 1/\beta \cdot \ent(\rho)\\
\ge& 1/2 \cdot R(\rho) + \lambda/2 \cdot M(\rho) - 1/\beta \cdot [1 + M(\rho) / \sigma^2 + d \cdot \log(2\pi \sigma^2)].
\end{aligned}
\]
Taking $\sigma^2 = 4/(\beta \lambda)$ gives Eq. (\ref{eqn:F_lower_bound_by_M}) . 

The argument to show the existence and uniqueness of minimizer of
$F_{\beta, \lambda}$ is similar to the proof of \cite[Proposition 4.1]{jordan1998variational}, and we will just give a sketch here. Since $U$, $V$ are  bounded-Lipschitz,
 it follows that $\rho\mapsto R(\rho)$  is continuous with respect to the topology of weak convergence in $L^1(\reals^d)$. Fatou's lemma implies that $M$ is lower semi-continuous. 
\cite[Proposition 4.1]{jordan1998variational} shows the upper semi-continuity of $\ent$. Hence $F_{\beta, \lambda}$ is lower semi-continuous.
Note (as just shown)  $F_{\beta, \lambda}$ is lower bounded, there exists a sequence $( \rho_k )_{k \ge 1} \subset \cK$ such that $\lim_{k \to\infty}F_{\beta, \lambda}(\rho_k) = \inf_{\rho \in \cK} F_{\beta, \lambda}(\rho) > -\infty$. By the same argument as \cite[Proposition 4.1]{jordan1998variational}, we can see that $\{ \int \max\{\rho_k \log \rho_k, 0) \} \de \btheta\}_{k \ge 1}$ and $\{M(\rho_k)\}_{k \ge 1}$ are uniformly upper bounded, and by de la Vall\'ee-Poussin criterion, there exists $\rho_\star \in \cK$ such that there is a subsequence of $(\rho_k)_{k \ge 1}$ converges weakly to $\rho_\star$ in $L^1(\reals^d)$. The lower semi-continuity of $F_{\beta, \lambda}$ implies that $\rho_\star$ is the minimizer of $F_{\beta, \lambda}$. Uniqueness follows by noting that $U$ is positive semi-definite, $\ent$ is strongly concave, and $\< V, \rho\>$ and $M$ are linear in $\rho$, so that $F_{\beta, \lambda}$ is a strongly convex functional. 

\end{proof}

For any $\rho \in \cK$, we call the following equation the Boltzmann fixed point condition
\begin{equation}\label{eqn:BoltzmannEquation_in_proof}
\begin{aligned}
\rho(\btheta) =& 1/Z(\beta, \lambda; \rho) \, \exp\{ - \beta \Psi_\lambda(\btheta; \rho) \},\\
Z(\beta, \lambda; \rho) =& \int  \exp\{ - \beta \Psi_\lambda(\btheta; \rho) \} \de \btheta. 
\end{aligned}
\end{equation}

\begin{lemma}
Under the assumption of Lemma \ref{lem:Free_Energy_Unique_Minimum}, the minimizer $\rho_\star \in \cK$ of 
$F_{\beta, \lambda}(\rho)$ satisfies the Boltzmann fixed point condition. 
\end{lemma}

\begin{proof}
We denote $\mu_0$ to be the Lebesgue measure on $\reals^d$. 

First, we show that $\rho$ is positive almost everywhere. Let $\rho_\star \in \cK$ be a minimizer of $F(\rho)$, and assume by contradiction that there exists a measurable set 
$\Omega_0 \subset \reals^d$, such that $\mu_0(\Omega_0) > 0$, and $\rho_\star(\Omega_0) = 0$. Without loss of generality, we assume that the support of $\Omega_0$ is compact so that $\mu_0(\Omega_0) < \infty$, otherwise we can always consider the intersection of $\Omega_0$ with a large ball. Define $\rho_\eps = (1 - \eps) \rho_\star + \eps  / \mu_0(\Omega_0) \cdot \bfone_{\Omega_0}\in \cK$. It is easy to see that there exists $ \eps_0 > 0$ and $C < \infty$, such that $\l R_\lambda(\rho_\star) - R_\lambda(\rho_\eps) \l \le C \cdot \eps$, and 
\[
\begin{aligned}
\ent(\rho_\eps) =& (1- \eps) \ent(\rho_\star) - (1 - \eps) \log (1 - \eps) + \eps \log (\mu_0(\Omega_0)/\eps) \\
\ge& \ent(\rho_\star)  - C \cdot \eps + \eps \log (\mu_0(\Omega_0)/\eps)
\end{aligned}
\]
for any $\eps < \eps_0$. As $\eps$ is sufficiently small, we have $F_{\beta, \lambda}(\rho_\eps) < F_{\beta, \lambda}(\rho_\star)$. This contradict with the fact that $\rho_\star \in \cK$ is the minimizer of $F_{\beta, \lambda}(\rho)$. 

Next we show that, for all $\btheta\in\reals^d$,
\begin{align}\label{eqn:log_BoltzmannEquation}
\Psi_\lambda(\btheta; \rho_\star) + 1/\beta \cdot \log \rho_\star(\btheta) \equiv \gamma(\beta, \lambda; \rho_\star)
\end{align}
for some constant $\gamma(\beta, \lambda; \rho_\star)$. 

Let $\rho_\star \in \cK$ be the minimizer of $F_{\beta, \lambda}(\rho)$.
Fix $\eps_0 > 0$  and define $\Gamma_{\eps_0} \equiv \{\btheta \in \reals^d: \rho_\star(\btheta) \ge \eps_0 \} \cap \Ball(\bzero; 1/\eps_0)$, and $\cA_{\eps_0} \equiv \{ v \in C^\infty(\reals^d):
\|v\|_{\infty}\le 1, \supp(v) \subseteq \Gamma_{\eps_0},\int_{\reals^d} v(\btheta) \de \btheta = 0\}$.
For any $v \in \cA_{\eps_0}$, define $\rho_{\eps, v} = \rho + \eps v$. Note that, for  $-\eps_0 < \eps < \eps_0$, we have $\rho_{\eps, v} \in \cK$. Since $\rho_\star$ is the minimizer of $F_{\beta, \lambda}(\rho)$, we must have $\lim_{\eps \to 0+} [F_{\beta, \lambda}(\rho_{\eps, v}) - F_{\beta, \lambda}(\rho_\star)]/\eps \ge 0$. It can be easily verified that 
\[
\begin{aligned}
\lim_{\eps \to 0} [F_{\beta, \lambda}(\rho_{\eps, v}) - F_{\beta, \lambda}(\rho_\star)]/\eps = \int_{\reals^d} [\Psi_\lambda(\btheta; \rho_\star) + 1/\beta \cdot \log \rho_\star(\btheta) ] v(\btheta) \de \btheta, 
\end{aligned}
\]
which implies
\begin{align}\label{eqn:Variation_of_F}
\int_{\reals^d} [\Psi_\lambda(\btheta; \rho_\star) + 1/\beta \cdot \log \rho_\star(\btheta) ] v(\btheta) \de \btheta = 0
\end{align}
for any $v \in \cA_{\eps_0}$. This implies that Eq. (\ref{eqn:log_BoltzmannEquation}) holds for any $\btheta \in \Gamma_{\eps_0}$. But note that $\mu_0(\reals^d \setminus(\cup_{\eps_0 > 0} \Gamma_{\eps_0}) ) = 0$. This implies that Eq. (\ref{eqn:log_BoltzmannEquation}) holds almost surely. 

Note we have $\int \rho_\star(\btheta) \de \btheta = 1$. Therefore, we must have $\gamma(\beta, \lambda; \rho_\star) = - 1/\beta \cdot \log Z(\beta, \lambda; \rho_\star)$. This proves that $\rho_\star$ satisfies the Boltzmann fixed point condition. 

\end{proof}

\begin{lemma}\label{lem:Unique_Boltzmann_solution}
Under the assumption of Lemma \ref{lem:Free_Energy_Unique_Minimum}, the Boltzmann fixed point condition has a unique solution in $\cK$.
\end{lemma}

\begin{proof}
The last two lemmas already imply that the Boltzmann fixed point condition has at least one solution.
Assume $\rho_1, \rho_2 \in K$ to be two such solutions. Then $\rho_i$ is positive, and 
\[
\begin{aligned}
\log Z(\beta, \lambda; \rho_i) = - \beta \Psi_\lambda(\btheta; \rho_i) - \log \rho_i(\btheta).
\end{aligned}
\]
Therefore
\[
\begin{aligned}
0 = & \int_{\reals^d} [ \log Z(\beta, \lambda; \rho_1) - \log Z(\beta, \lambda; \rho_2)] \cdot [ \rho_1(\btheta) - \rho_2(\btheta) ]   \de \btheta  \\
=& - \beta \< U, (\rho_1 -  \rho_2)^{\otimes 2}\> - \int_{\reals^d} \log(\rho_1(\btheta) / \rho_2(\btheta)) [\rho_1(\btheta) - \rho_2(\btheta)] \de \btheta.
\end{aligned}
\]
Note the right hand side does not equal $0$ unless $\rho_1 = \rho_2$. 

\end{proof}

\begin{lemma}\label{lem:Bound_global_minimizer_free_energy}
Under the assumption of Lemma \ref{lem:Free_Energy_Unique_Minimum}, and further assume condition {\sf A3} holds. Let $\rho_\star^{\beta, \lambda}$ be the minimizer of $F_{\beta, \lambda}(\rho)$. Then there is a constant $K$ depending on the parameter $K_3$ in condition {\sf A3}, such that for any $\beta \ge 1$, we have
\begin{align}
R(\rho_\star^{\beta, \lambda}) \le \inf_{\rho \in \cuP(\reals^d)} R_\lambda(\rho) + K (1+ \lambda) [d \log (2 + 1/\lambda)] /\beta. 
\end{align} 

\end{lemma}

\begin{proof}

Fix a $\rho \in \cuP(\reals^d)$. Let $g_{\tau}(\btheta)$ be the density for $\normal(\bzero, \tau^2 \id_d)$. Denote $\rho * g_{\tau}$ to be the convolution of $\rho$ and $g_{\tau}$. Now we derive the formula for $F_{\beta, \lambda}(\rho * g_{\tau})$.

Let $\bG, \bG_1, \bG_2 \sim \normal(\bzero, \id_d)$ be independent, we have
\[
\begin{aligned}
R(\rho * g_{\tau} ) =& R(\rho) + 2 \int \{ \E[V(\btheta + \tau \bG)] - V(\btheta) \} \rho(\de \btheta) \\
&+ \int \{\E[U(\btheta_1 + \tau \bG_1, \btheta_2 + \tau \bG_2)] - U(\btheta_1, \btheta_2)\} \rho(\de \btheta_1) \rho(\de \btheta_2).\\
\end{aligned}
\]
Using the intermediate value theorem and Cauchy-Schwarz inequality, and noting that $\nabla^2 V$ is $K_3$-bounded by condition {\sf A3}, we have
\[
\begin{aligned}
&\int \{V(\btheta) - \E[V(\btheta + \tau\bG)]\} \rho(\de \btheta) \\
=& \tau \int \E[\<\nabla V(\btheta), \bG\> ] \rho( \de \btheta) + \frac{\tau^2}{2} \int \E[\< \nabla^2 V(\tilde \btheta) , \bG^{\otimes 2}\> ] \rho(\de \btheta) \le \frac{\tau^2}{2} K_3 d,
\end{aligned}
\]
We have similar bound for the $U$ term. Therefore, 
\begin{align}
R(\rho * g_{\tau} ) \le R(\rho) + 2\tau^2 K_3 d. 
\end{align}

For the term $M(\rho * g_\tau)$, we have 
\begin{align}
M(\rho * g_\tau) = \int \E[\| \btheta + \tau \bG \|_2^2] \rho(\de \btheta) = M(\rho) + \tau^2 d. 
\end{align}
Next we give a lower bound for $\ent(\rho * g_\tau)$:
\begin{align}
\ent(\rho * g_\tau) \ge \ent(g_\tau) = (d/2) \log(2\pi e \tau^2). 
\end{align}

As a result, taking $\tau = 1/\beta$, we have 
\begin{align}
F_{\beta, \lambda}(\rho_\star^{\beta, \lambda}) \le (1/2) R_\lambda(\rho) + (2 K_3 + \lambda) d/(2\beta^2) + d\cdot \log( 2\pi e \beta^2) /(2\beta).
\end{align}
Combining with Eq. (\ref{eqn:F_lower_bound_by_M}), we have 
\begin{align}
R(\rho_\star^{\beta, \lambda}) \le R_\lambda(\rho)  + \frac{(2 K_3 + \lambda)d}{\beta^2} + \frac{2}{\beta} + \frac{d\cdot \log( 2\pi e \beta^2)}{\beta} - \frac{2 d \cdot \log (\lambda \beta/(8\pi))}{\beta}
\end{align}
for any $\rho \in \cuP(\reals^d)$. Hence, the theorem holds by taking infimum over $\rho \in \cuP(\reals^d)$. 
\end{proof}

\subsection{Dynamics}
\label{sec:Dynamics_finite_T}

Recall that the finite-temperature distributional dynamics reads:
\begin{align}\label{eqn:Chaos_PDE_diffusion}
\partial_t \rho_t(\btheta) = 2 \xi(t) \nabla_\btheta \cdot (\nabla_\btheta \Psi_\lambda(\btheta; \rho_t) \rho_t(\btheta)) + 2\xi(t)/\beta \cdot \Delta_\btheta \rho_t(\btheta). 
\end{align}
We say $(\rho_t)_{t \ge 0} \subseteq \cuP(\reals^d)$ is a weak solution of (\ref{eqn:Chaos_PDE_diffusion}), if for any $\zeta \in C_0^\infty(\reals \times \reals^d)$ (the space of smooth functions, decaying to $0$ at infinity), we have
\begin{equation}\label{eqn:Solution_of_PDE_in_weak_form}
\begin{aligned}
 & \int_{\reals^d} \rho_0(\btheta) \zeta_0(\btheta) \de \btheta\\
=&- \int_{(0, \infty) \times \reals^d} [ \partial_t \zeta_t(\btheta) - 2 \xi(t) \< \nabla_\btheta \Psi_\lambda (\btheta; \rho_t), \nabla_\btheta \zeta_t(\btheta)\> + 2 \xi(t)\Delta_\btheta \zeta_t(\btheta)] 
\rho_t(\de \btheta)\,  \de t
\end{aligned}
\end{equation}
Notice that this notion of weak solution is equivalent to the one introduced earlier in Eq.~(\ref{eq:WeakSolution}), see for instance
\cite[Proposition 4.2]{santambrogio2015optimal}.

\begin{lemma}\label{lemma:AbsolutelyContinuous}
Assume conditions {\sf A1}, {\sf A2} and {\sf A3} hold. Let initialization $\rho_0 \in \cK$ so that $F_{\beta, \lambda}(\rho_0) < \infty$. Then, the weak solution $(\rho_t)_{t \ge 0} \subseteq \cuP(\reals^d)$ of PDE (\ref{eqn:Solution_of_PDE_in_weak_form}) exists and is unique. Moreover, for any fixed $t$, $\rho_t \in \cK$ is absolutely continuous with respect to the Lebesgue measure, and $\ent(\rho_t)$ and $M(\rho_t)$ are uniformly bounded in $t$. 
\end{lemma}

\begin{proof} 
Without loss of generality, we assume $\xi(t) \equiv 1/2$. 
%
%


We use the JKO scheme of \cite[Theorem 5.1]{jordan1998variational} to show the existence, uniqueness, and absolute continuousness of solution of PDE (\ref{eqn:Chaos_PDE_diffusion}). Since the proof is basically the same as the proof of \cite[Theorem 5.1]{jordan1998variational}, we will skip several details.

First, we consider the following discrete scheme. Let $\rad_{0}^{h} = \rho_0$, and define $\{\rad_k^h\}_{k\in\naturals}$ recursively by
\begin{align}\label{eqn:Discrete_scheme_Diffusion_PDE}
\rad_{k+1}^h\in \arg\min_{\rho \in \cK} \{ h F(\rho) + (1/2)  W_2^2(\rho, \rad_{k}^{h}) \},
\end{align}
where $W_2(\mu, \nu)$ is the Wasserstein distance between $\mu, \nu \in \cuP(\reals^d)$, with definition
\[
\begin{aligned}
W_2^2(\mu, \nu) 
=& \inf \Big\{ \int_{\reals^d \times \reals^d} \| \btheta_1 - \btheta_2 \|_2^2 \gamma(\de \btheta_1, \de \btheta_2): \gamma \text{ is a coupling of } \mu, \nu \Big\}. 
\end{aligned}
\]
For any $\rad_{k-1}^{h}$, the optimization problem (\ref{eqn:Discrete_scheme_Diffusion_PDE}) has a unique minimizer $\rad_{k}^{h} \in \cK$, where the proof is basically the same as Lemma \ref{lem:Free_Energy_Unique_Minimum}, by additionally noting that $W_2^2(\rho, \rad_{k-1}^h)$ as a function of $\rho$ is lower bounded, lower semi-continuous, and convex over $\rho \in \cK$. 

Hence, we have a sequence of probability densities $(\rad_k^{h})_{k \ge 0}$ with each $\rad_k^{h} \in \cK$. Now we define its interpolation $\rho^h: (0, \infty) \times \reals^d \to [0, \infty)$ by 
\[
\begin{aligned}
\rho^h(t,\,\cdot\,) = \rad_k^{h} ~~ \text{ for } ~~ t \in [k h, (k+1) h) ~~ \text{ and } ~~ k \in \mathbb N. 
\end{aligned}
\]
In the following, we will show that this $\rho^h$ approximately satisfies PDE (\ref{eqn:Solution_of_PDE_in_weak_form}) in the weak form. 

Let $\bxi \in C_0^\infty(\reals^d, \reals^d)$ be a smooth vector field with bounded support, and define the corresponding flux $\{\Phi_\tau \}_{\tau \in \reals}$ by 
\begin{align}
\partial_\tau \Phi_\tau = \bxi \circ \Phi_\tau \text{ for all } \tau \in \reals ~~~ \text{ and } ~~ \Phi_0 = {\rm id}. 
\end{align}
For any $\tau \in \reals$, let the measure $\nu_\tau$ to be the push forward of $\rad_k^h$ under $\Phi_\tau$. This means that
\begin{align}
\int_{\reals^d} \nu_\tau(\btheta) \zeta(\btheta) \de \btheta = \int_{\reals^d} \rad_k^h(\btheta) \zeta(\Phi_\tau(\btheta)) \de \btheta, ~~ \text{ for all } ~~ \zeta \in C(\reals^d). 
\end{align}
Since $\rad_k^h$ is the minimizer of optimization problem (\ref{eqn:Discrete_scheme_Diffusion_PDE}), we have for each $\tau > 0$, 
\begin{align}
\Big(\frac{1}{2} W_2^2(\rad_{k-1}^h, \nu_\tau) + h F(\nu_\tau) \Big) - \Big(\frac{1}{2} W_2^2(\rad_{k-1}^h, \rad_k^h) + h F(\rad_k^h)\Big) \ge 0. 
\end{align}

Using the result in the proof of \cite[Theorem 5.1]{jordan1998variational}, and noting $\nabla V$ is bounded Lipschitz, we have 
\begin{align}
\frac{\de}{\de \tau} [ \< V, \nu_\tau \> ]_{\tau = 0} =& \int_{\reals^d}\< \nabla V(\btheta) ,\bxi(\btheta) \>\, \rad_k^h(\btheta) \de \btheta,\label{eqn:absolute_continuity_eq_begin}\\
\frac{\de}{\de \tau}[\ent(\nu_\tau)]_{\tau = 0} =& \int_{\reals^d} \rad_k^h(\btheta) \cdot {\rm div}(\bxi(\btheta)) \de \btheta,\\
\lim\sup_{\tau \to 0+}\frac{1}{\tau}[M(\nu_\tau) - M( \rad_k^h)] \le& \int_{\reals^d} 2 \<\btheta , \bxi(\btheta)\> \, \rad_k^h(\btheta) \de \btheta,\\
\lim\sup_{\tau \to 0+}\frac{1}{\tau}[W_2^2(\rad_{k-1}^h, \nu_\tau) - W_2^2(\rad_{k-1}^h, \rad_k^h)] \le&  \int_{\reals^d} 2 \<(\btheta_1 - \btheta_2) , \bxi(\btheta_1) \>\,
 p(\de \btheta_1, \de \btheta_2),
\end{align}
where $p$ is an optimal coupling of $\rho_k^h$ and $\rho_{k-1}^h$ in Wasserstein metric. Further we have for any $\zeta \in C_0^\infty(\reals^d)$, 
\begin{align}
\Big\l \int_{\reals^d} (\rad_k^h - \rad_{k-1}^h) \zeta \de \btheta - \int_{\reals \times \reals} \<\btheta_1 - \btheta_2, \nabla \zeta(\btheta_1) \> \de p 
\Big\l \le \frac{1}{2} \sup_{\btheta \in \reals^d} \| \nabla^2 \zeta(\btheta) \|_{\op} W_2^2(\rad_k^h, \rad_{k-1}^h)\, . 
\end{align}
We need to further calculate the derivative of $\< U, \nu_\tau^{\otimes 2}\>$ with respect to $\tau$. Note $U$ is symmetric, we have
\[
\begin{aligned}
&\frac{1}{\tau}[\< U, \nu_\tau^{\otimes 2}\> - \< U, (\rad_k^h)^{\otimes 2}\>] - 2 \int_{\reals^d \times \reals^d} \< \nabla_{\btheta_1} U(\btheta_1, \btheta_2), \bxi(\btheta_1)\>  \rad_k^h(\btheta_1) \rad_k^h(\btheta_2)\de \btheta_1 \de \btheta_2 \\
=& \int_{\reals^d \times \reals^d}\{\frac{1}{\tau}[U(\Phi_\tau(\btheta_1), \Phi_\tau(\btheta_2)) - U(\Phi_\tau(\btheta_1), \btheta_2)] - \< \nabla_{\btheta_2} U(\Phi_\tau(\btheta_1), \btheta_2), \bxi(\btheta_2)\> \} \rad_k^h(\btheta_1) \rad_k^h(\btheta_2) \de \btheta_1 \de \btheta_2\\
&+\int_{\reals^d \times \reals^d}\{\frac{1}{\tau} [U(\Phi_\tau(\btheta_1), \btheta_2) - U(\btheta_1, \btheta_2) ] - \< \nabla_{\btheta_1} U(\btheta_1, \btheta_2), \bxi(\btheta_1)\> \} \rad_k^h(\btheta_1) \rad_k^h(\btheta_2) \de \btheta_1 \de \btheta_2\\
&+\int_{\reals^d \times \reals^d} [ \< \nabla_{\btheta_2} U(\Phi_\tau(\btheta_1), \btheta_2), \bxi(\btheta_2)\> - \< \nabla_{\btheta_2} U(\btheta_1, \btheta_2), \bxi(\btheta_2)\>] \rad_k^h(\btheta_1) \rad_k^h(\btheta_2) \de \btheta_1 \de \btheta_2.
\end{aligned}
\]
According to condition {\sf A3}, $\nabla_{\btheta_1} U(\btheta_1, \btheta_2)$ is Lipschitz in $(\btheta_1, \btheta_2)$, and note $\bxi(\btheta) \in C_0^\infty(\reals^d)$ is uniformly bounded, hence $1/\tau\cdot [U(\Phi_\tau(\btheta_1), \btheta_2) - U(\btheta_1, \btheta_2) ] - \< \nabla_{\btheta_1} U(\btheta_1, \btheta_2), \bxi(\btheta_1)\>$, $1/\tau [U(\Phi_\tau(\btheta_1), \btheta_2) - U(\btheta_1, \btheta_2) ] - \< \nabla_{\btheta_1} U(\btheta_1, \btheta_2), \bxi(\btheta_1)\>$, and $[ \< \nabla_{\btheta_2} U(\Phi_\tau(\btheta_1), \btheta_2), \bxi(\btheta_2)\> - \< \nabla_{\btheta_2} U(\btheta_1, \btheta_2), \bxi(\btheta_2)\>]$ converges to $0$ for $\tau \to 0+$, uniformly over $(\btheta_1, \btheta_2) \in \reals^d \times \reals^d$. Therefore, we have
\begin{align}\label{eqn:absolute_continuity_eq_end}
\frac{\de}{\de \tau} [ \< U, \nu_\tau^{\otimes 2} \> ]_{\tau = 0} =& 2 \int_{\reals^d \times \reals^d} \<\nabla_{\btheta_1} U(\btheta_1, \btheta_2), \bxi(\btheta_1)\> \cdot \rad_k^h(\btheta_1) \rad_k^h (\btheta_2) \de\btheta_1 \de \btheta_2. 
\end{align}
Combining Eq. (\ref{eqn:absolute_continuity_eq_begin}) to (\ref{eqn:absolute_continuity_eq_end}), choosing $\bxi = \nabla \zeta$
and $\bxi = -\nabla \zeta$, we have for any $\zeta \in C_0^\infty(\reals)$, 
\begin{equation}\label{eqn:Discrete_weak_PDE}
\begin{aligned}
\left| \int_{\reals^d} \Big\{ \frac{1}{h} (\rad_k^h - \rad_{k-1}^h) \zeta + (\< \nabla_\btheta \Psi_{\lambda}(\btheta; \rad_{k}^h), \nabla \zeta\> - \Delta \zeta) \rad_k^h\Big\} \de \btheta \right|
\le \frac{1}{2} \sup_{\reals^d} \| \nabla^2 \zeta \|_{\op} \cdot \frac{1}{h} W_2^2(\rad_{k-1}^h, \rad_k^h). 
\end{aligned}
\end{equation}

According to the estimates in \cite[Theorem 5.1]{jordan1998variational}, for any $T < \infty$, there exists a constant $C < \infty$ such that for all $N \in \mathbb N$ and all $h \in (0,1]$ with $N h\le T$, there holds
\begin{align}\label{eqn:Bounded_M_H_E}
\max\Big\{ M(\rad_N^h), ~~ \int_{\reals^d} \max \{ \rad_N^h \log(\rad_N^h), 0 \} \de \btheta, ~~ R(\rad_N^h), ~~ \frac{1}{h} \sum_{k=1}^N W_2^2(\rad_k^h, \rad_{k-1}^h) \Big\} \le C. 
\end{align}
As in \cite[Theorem 5.1]{jordan1998variational}, by de la Vall\'ee-Poussin criterion, the second condition in Eq. (\ref{eqn:Bounded_M_H_E}) implies that  there exists a measurable function $(t,\btheta)\mapsto
\rho(t,\btheta)$ and a sequence $(h_s)_{s \ge 1}$ with $\lim_{s \to \infty} h_s = 0$, such that $(t,\btheta)\mapsto \rho^{h_s}(t,\btheta)$ converges to $\rho$ weakly in $L^1((0, T) \times \reals^d)$ for all $T < \infty$.
Eq. (\ref{eqn:Bounded_M_H_E}) also guarantees that $\rho(t,\, \cdot\, ) \in \cK$ for almost every $t \in (0, \infty)$, and $M(\rho), R(\rho) \in L^\infty((0, T))$ 
for all $T < \infty$. By Eq. (\ref{eqn:Discrete_weak_PDE}) and (\ref{eqn:Bounded_M_H_E}), we have that $\rho$ satisfies Eq. (\ref{eqn:Solution_of_PDE_in_weak_form}). 
Since this equation is not affected by changing $\rho(t,\,\cdot\,)$ for a set of values of $t$ with measure $0$, we can ensure that the $\rho(t,\,\cdot\,)\in \cK$
for all $t$. Therefore, $\rho$ is a solution of the weak form of PDE (\ref{eqn:Solution_of_PDE_in_weak_form}).

The uniqueness of solution of Eq. (\ref{eqn:Solution_of_PDE_in_weak_form}) can be proved using standard method from theory of elliptic-parabolic equations (see, for instance, \cite[Theorem 5.1]{jordan1998variational}). In the proof of uniqueness we need the smoothness property of the solution, which is proved by Lemma \ref{lem:smoothness_of_DD_diffusion_PDE}.

\end{proof}

\begin{lemma}\label{lem:smoothness_of_DD_diffusion_PDE}

Assume conditions {\sf A1} - {\sf A4} hold. Let initialization $\rho_0 \in \cK$ with $F_{\beta, \lambda}(\rho_0) < \infty$. Denote the solution of PDE (\ref{eqn:Chaos_PDE_diffusion}) to be $(\rho_t)_{t \ge 0}$. Then $\rho_t(\btheta)$ as a function of $(t, \btheta)$ is in $C^{1, 2}((0, \infty) \times \reals^d)$, where $C^{1, 2} ((0, \infty) \times \reals^d)$ is the function space of continuous function with continuous derivative in time, and second order continuous derivative in space.  

\end{lemma}

Before proving this lemma, we give some notations in the following. 

For any open set $\Omega \subseteq \reals^d$, and $1 \le p \le \infty$, define $L^{p}( \Omega)$ to be the Banach space 
consisting of all measurable functions on $\Omega$ with a finite norm 
\begin{align}
\| u \|_{L^{p}(\Omega)} \equiv \Big( \int_{\Omega} \l u(\btheta) \l^p \de \btheta \Big)^{1/p}. 
\end{align}
We say $u \in L^{p}_{\rm loc}(\Omega)$ if for any compact subset $\Omega' \subset \Omega$, we have $u \in L^{p}(\Omega')$. We denote $\| \cdot \|_{L^p(\reals^d)}$ simply by $\| \cdot \|_{L^p}$. 

For any nonnegative integer $l$ and $1 \le p \le \infty$, we denote $W_p^{l}(\Omega)$ to be the Banach space (Sobolev space) consisting of the elements of $L^p(S)$ having generalized derivatives of all forms up to order $l$ included, that are $p$'th power integrable on $\Omega$. The norm in $W_p^l(\Omega)$ is defined by the equality 
\begin{align}
\| u \|_{L^p(\Omega)}^{(l)} = \sum_{j = 0}^{l} \<\< u \>\>_{L^p(\Omega)}^{(j)}, ~~~ \< \< u \> \>_{L^p(\Omega)}^{(j)} = \sum_{\l \balpha \l = j} \| D_\btheta^\balpha u \|_{L^p(\Omega)},
\end{align}
where $\balpha = (\alpha_1, \ldots, \alpha_d)$ is a multi-index with $\l \balpha \l = \sum_{i=1}^d \alpha_i$, and $D^\balpha_\btheta u = \partial^{\l \balpha \l} u / \partial \theta_1^{\alpha_1} \cdots \partial \theta_d^{\alpha_d}$.

Let $(t_1, t_2) \subseteq (0, T)$ be an open interval and $\Omega \subseteq \reals^d$ be an open set, in these three paragraphs we temporarily denote $S = (t_1, t_2) \times \Omega$. For any $1 \le r, p \le \infty$, define $L^{r, p}(S)$ to be the Banach space consisting of all measurable functions on $S$ with a finite norm 
\begin{align}
\| u \|_{L^{r, p}(S)} \equiv \Big( \int_{t_1}^{t_2} \Big( \int_{\Omega} \l u(t, \btheta) \l^p \de \btheta \Big)^{r/p} \de t \Big)^{1/r}. 
\end{align}
We say $u \in L^{r, p}_{\rm loc}(S)$ if for any compact subset $[t_1', t_2'] \subset (t_1, t_2)$ and compact subset $\Omega' \subset \Omega$, we have $u \in L^{r, p}( [t_1', t_2'] \times \Omega')$. We will denote $L^{p, p}(S)$ by $L^p(S)$, and $L^{p, p}_{\rm loc}(S)$ by $L^p_{\rm loc}(S)$. 

For nonnegative integer $l$ and $1 \le p \le \infty$, we denote $W_p^{2l, l}(S)$ to be the Banach space consisting of the elements of $L^p(S)$ having generalized derivatives of the form $D_t^r D_\btheta^\balpha$ with $r$ and $\balpha$ satisfying the inequality $2r + \l \balpha\l \le 2l$. The corresponding norm  is defined by 
\begin{align}
\| u \|_{L^p(S)}^{(2l)} = \sum_{j = 0}^{2l} \<\< u \>\>_{L^p(S)}^{(j)}, ~~~ \< \< u \> \>_{L^p(S)}^{(j)} = \sum_{\l \balpha\l + 2 r = j} \| D^r_t D^\balpha_\btheta u \|_{L^p(S)}.
\end{align}

We denote $C^{m, n}(S)$ to be the function space of continuous function with $m$ continuous derivative in time, and $n$ continuous derivatives in space. For example, $u \in C^{1,2}(S)$ if and only if $u, \partial_t u, \nabla_\btheta u, \nabla_\btheta^2 u \in C^{0, 0}(S) \equiv C(S)$. We say $u \in C_c^{m, n}(S)$ if $u \in C^{m, n}(S)$ and the support of $u$ is compact. We will denote $C^{n, n}(S)$ by $C^n(S)$, and $C^{n,n}_c(S)$ by $C^n_c(S)$.

For any measurable functions $f, g$ defined on $\reals^d$, we denote $f * g$ to be their space convolution, which is a measurable function on $\reals^d$, with 
\begin{align}
(f * g)(\btheta) = \int_{\reals^d} f(\btheta') g(\btheta - \btheta') \de \btheta'. 
\end{align}
For any measurable function $u, v$ defined on $\reals \times \reals^d$, we denote $u *_2 v$ to be their space and time convolution, which is a measurable function on $\reals \times \reals^d$, with 
\begin{align}
(u *_2 v)(t, \btheta) = \int_{\reals} \de t' \int_{\reals^d} u(t', \btheta') v(t- t', \btheta - \btheta') \de \btheta'. 
\end{align}
If $u, v$ are defined on a subset of $\reals \times \reals^d$, we define $u *_2 v$ using their zero extensions. 

We denote $G$ to be the heat kernel, where for $t > 0$, we have 
\begin{align}
G(t, \btheta) = t^{- d/2} g(t^{-1/2} \btheta), ~~~ g(\btheta) = (2 \pi )^{- d/2} \exp\{ - 1/2 \cdot \| \btheta \|_2^2 \}.
\end{align}

\begin{proof}
The proof is similar to the one of \cite[Theorem 5.1]{jordan1998variational}, so we will skip some details. Without loss of generality we can set $\beta = 1$,
and $\xi(t)=1/2$ (different choices can be obtained by rescaling $\Psi(\btheta;\rho)$ and reparametrizing time).

Let $\setE = (0, \infty) \times \reals^d$. With a slight abuse of notations, we denote $\Psi(t, \btheta) = \Psi_\lambda(\btheta; \rho_t)$. Since $V \in C^4(\reals^d)$, and $\nabla_1^k U$ are uniformly bounded for $0 \le k \le 4$, we have $\nabla_\btheta^k \Psi \in L_{\rm loc}^\infty(\setE)$ for $0 \le k \le 4$.  

In the following, we will write $\rho(t, \btheta) = \rho_t(\btheta)$ for clarity. When we write $\rho(t)$, we regard it as a function in $L^1(\reals^d)$ at any fixed $t$. For other functions, we also use this convention.

\noindent
{\bf Step 1. Show that $\rho \in L^{\infty, p}_{\rm loc}(\setE)$. }

Taking $G$ to be the heat kernel, it is easy to see that 
\[
\begin{aligned}
\| G(t) \|_{L^p} = t^{(\frac{1}{p} - 1)\frac{d}{2}} \| g \|_{L^p}, ~~
\| \nabla G(t) \|_{L^p} = t^{\frac{1}{p}\frac{d}{2} - \frac{d + 1}{2}} \| \nabla g \|_{L^p}. 
\end{aligned}
\]
Then for any $\eta \in C_c^\infty(\reals^d)$, Duhamel's principle gives 
\begin{equation}\label{eqn:Duhamel_regularity}
\begin{aligned}
\rho(t) \eta =&  \int_{\eps}^{t} [\rho(s) (\Delta \eta - \< \nabla \Psi(s), \nabla \eta\> )] * G(t - s) \de s \\
&+ \int_{\eps}^{t} [\rho(s) (2 \nabla \eta - \eta \nabla \Psi(s))] * \nabla G(t - s) \de s  + (\rho(\eps) \eta ) * G_\eps(t)
\end{aligned}
\end{equation}
for almost every $0 \le \eps < t < \infty$, where $*$ denotes convolution in the $\btheta$-variables, and $G_\eps(t, \btheta) \equiv G(t - \eps, \btheta)$. 
By Young's convolution inequality, we have $\| f * g \|_{L^r} \le C \| f \|_{L^p} \| g \|_{L_q}$ for $1/p + 1/q = 1/r + 1$ and $p, q, r \ge 1$. For fixed $t$, we estimate the $L^p(\reals^d)$ norm of $\rho(t) \eta$, which gives 
\[
\begin{aligned}
\| \rho(t) \eta \|_{L^p} \le& \int_{\eps}^{t} \| \rho(s) (\Delta \eta - \< \nabla \Psi(s), \nabla \eta\>) \|_{L^1} \| G(t - s) \|_{L^p} \de s\\
&+ \int_{\eps}^{t} \| \rho(t) (2 \nabla \eta - \eta \nabla \Psi(t))\|_{L^1} \| \nabla G(t - s) \|_{L^p} \de s + \| \rho(\eps) \eta \|_{L^1} \| G(t - \eps)\|_{L^p}\\
\le&\ess\sup_{s \in [\eps, t]} \| \rho(s) (\Delta \eta - \< \nabla \Psi(s), \nabla \eta\> ) \|_{L^1} \| g\|_{L^p} \int_{0}^{t - \eps} s^{(\frac{1}{p} - 1) \frac{d}{2}} \de s \\
&+ \ess\sup_{s \in [\eps, t]} \| \rho(s) (2 \nabla \eta - \eta \nabla \Psi(s))\|_{L^1} \| \nabla g \|_{L^p}\int_{0}^{t - \eps} s^{\frac{1}{p} \frac{d}{2} - \frac{d+1}{2}} \de s  \\
&+ \| \rho(\eps) \eta \|_{L^1} \| g \|_{L^p} (t - \eps)^{(\frac{1}{p} -1) \frac{d}{2}}\\
\end{aligned}
\]
for almost every $0 \le \eps < t < \infty$. For $p < d/(d-1)$, the $s$-integrals are finite. Therefore, we have $\rho \eta \in L^{\infty, p}( (\delta, T) \times \reals^d )$ for any $\delta, T$ such that $\eps < \delta < T < \infty$. Hence we have $\rho \in L^{\infty, p}_{\rm loc}((0,\infty)\times \reals^d)$. 

\noindent
{\bf Step 2. Show that $\rho \in L^\infty_{\rm loc}((0,\infty)\times \reals^d)$ using bootstrap. }

In what follows, we let $\setE\equiv (0,\infty)\times\reals^d$.

We can iteratively use the strategy in step $1$ to show that $\rho \in L^\infty_{\rm loc}(\setE)$. We will summarize our key estimates in Step 1 as follows. For any measurable function $u$ defined on $\setS = (\delta, T) \times \reals^d$ for some $0 \le \delta < T < \infty$, we have 
\begin{align}
\| u *_2 G  \|_{L^{\infty, p_o}(\setS)} \le& C \| u \|_{L^{\infty, p_i}(\setS)},  \label{eqn:Extremely_powerful_inequalies3}\\
\| u *_2 \nabla G  \|_{L^{\infty, p_o}(\setS)} \le& C \| u \|_{L^{\infty, p_i}(\setS)}, \label{eqn:Extremely_powerful_inequalies4}
\end{align}
provided that the $p_o, p_i$ satisfy the relations
\begin{align}\label{eqn:Condition_for_in_out}
1 \le p_i \le p_o,  ~~ d\cdot (1/p_i - 1/p_o) < 1. 
\end{align}
Here, $C$ is a constant depends only on $T, \delta$ and on $p_i, p_o$.

Define $\varphi_1 \equiv \rho(\Delta \eta - \< \nabla \Psi, \nabla \eta\>) \bfone\{ t > \eps \}$, $\varphi_2 \equiv \rho (2 \nabla \eta - \eta \nabla \Psi) \bfone\{ t > \eps \}$, and $\psi \equiv \rho(\eps) \eta$. Then Eq. (\ref{eqn:Duhamel_regularity}) reads 
\begin{align}\label{eqn:Compact_Duhamel}
\rho \eta = \varphi_1 *_2 G + \varphi_2 *_2 \nabla G + \psi * G_\eps.
\end{align}
Since $\psi = \rho(\eps) \eta \in L^1(\reals^d)$, the behavior of $\psi * G_\eps$ on $\setS = (\delta, T) \times \reals^d$ for $\eps < \delta < T < \infty$ will be extremely nice: for any generalized gradient $D_t^r D^\balpha [\psi * G_\eps]$, 
\begin{align}
\| D_t^r D^\balpha [\psi * G_\eps] \|_{L^{\infty}(\setS)} \le & \| \psi \|_{L^1(\reals^d)} \| D_t^r D^\balpha G_\eps \|_{L^\infty(\setS)} < \infty. 
\end{align}
Hence $D_t^r D^\balpha [\psi * G_\eps] \in L^\infty(\setS)$. From now on, we fix $0 < \eps < \delta < T < \infty$ and take $\setS \equiv (\delta, T) \times \reals^d$.  

According to Eq. (\ref{eqn:Compact_Duhamel}) we have 
\begin{equation}\label{eqn:In_out_iteration}
\begin{aligned}
\| \rho \eta  \|_{L^{\infty, p_o}(\setS)} \le& \| \varphi_1 *_2 G \|_{L^{\infty, p_o}(\setS)} + \| \varphi_2 *_2 \nabla G \|_{L^{\infty, p_o}(\setS)} + \| \psi * G_\eps \|_{L^{\infty, p_o}(\setS)}\\
\le& C\{ \| \varphi_1 \|_{L^{\infty, p_i}(\setS)} + \| \varphi_2 \|_{L^{\infty, p_i}(\setS)} + \|\psi\|_{L^{1}(\reals^d)}\}
\end{aligned}
\end{equation}

Now we assume $\rho \in L_{\rm loc}^{\infty, p_i}(\setE)$ for some $p_i$. Note $\nabla \Psi \in L_{\rm loc}^\infty(\setE)$ so that $\max\{\| \varphi_1 \|_{L^{\infty, p_i}(\setS)}, \| \varphi_2 \|_{L^{\infty, p_i}(\setS)}\} \le C_\eta \| \rho \|_{L^{\infty, p_i}((\delta, T) \times \Omega_2)}$, where $\Omega_2 \supseteq \supp(\eta)$ is a compact set. As a result, for any $\eta \in C_c^\infty(\reals^d)$, we have
\begin{align}\label{eqn:Iterative_bootstrap}
\| \rho \|_{L^{\infty, p_o}((\delta, T) \times \Omega_1)} \le C_\eta (\| \rho \|_{L^{\infty, p_i}((\delta, T) \times \Omega_2)} + 1),
\end{align}
where $\Omega_1 \subseteq \supp(\eta) \subseteq \Omega_2$. Therefore, $\rho \in L_{\rm loc}^{\infty, p_o}(\setE)$, where $p_i, p_o$ satisfy Eq. (\ref{eqn:Condition_for_in_out}). 

Note there exists a sequence $p_{i, l}, p_{o, l}$ for $1 \le l \le k$ and $k < \infty$, so that $p_{i, l+1} = p_{o, l}$, $p_{i, 1} = p<d/(d-1)$, $p_{i, k} = \infty$, and $p_{i, l}, p_{o, l}$ for fixed $l$ satisfies Eq. (\ref{eqn:Condition_for_in_out}). Since we have $\rho \in L^{\infty, p}_{\rm loc}(\setE)$, using Eq. (\ref{eqn:Iterative_bootstrap}) iteratively, we have $\rho \in L^{\infty, p_{o, l}}_{\rm loc}(\setE)$ for any $1 \le l \le k$. As a result, we have $\rho \in L^{\infty}_{\rm loc}(\setE)$. 

\noindent
{\bf Step 3. Derivatives, $D\rho$, $D^2 \rho$, and $D^3 \rho$.}

By \cite[Chapter IV, section 3, (3.1)]{ladyzhenskaia1988linear}, for any function $u$ defined on $\setE = (0, \infty) \times \reals^d$, we have
\begin{align}
\< \< G *_2 u \>\>_{L^p(\setE)}^{(2m+2)} \le& C \< \< u \> \>_{L^p(\setE)}^{(2m)}, \label{eqn:Extremely_powerful_inequalies1}
\end{align}
where $1 < p \le \infty$ and $m$ is a nonnegative integer. 

First, we show the regularity of $D \rho$. Note that $\rho \in L^\infty_{\rm loc}(\setE)$, $\eta \in C_c^\infty(\reals^d)$, $\nabla \Psi \in L_{\rm loc}^\infty(\setE)$, we have $\varphi_1, \varphi_2 \in L^\infty(\setE)$. Due to Eq. (\ref{eqn:Extremely_powerful_inequalies1}), we have $D^2 \{\varphi_1 *_2 G\}, D^2\{ \varphi_2 *_2 G \} \in L^\infty(\setE)$, which also implies $D \{\varphi_1 *_2 G\} \in L^\infty_{\rm loc}(\setE)$. Hence we have $D (\rho \eta) = D \{\varphi_1 *_2 G\} + D^2\{ \varphi_2 *_2 G \}  +  D [\psi * G_\eps] \in L^\infty(\setS)$, which gives $D \rho \in L^\infty_{\rm loc}(\setE)$. 

Then we show the regularity of $D^2 \rho$. Note that $\nabla^2 \Psi \in L_{\rm loc}^\infty(\setE)$, we have $D \varphi_1, D \varphi_2 \in L^\infty(\setE)$. Due to Eq. (\ref{eqn:Extremely_powerful_inequalies1}), we have $D^3 \{\varphi_1 *_2 G\}, D^3\{ \varphi_2 *_2 G \} \in L^\infty(\setE)$, which also implies $D^2 \{\varphi_1 *_2 G\} \in L^\infty_{\rm loc}(\setE)$. Hence we have $D^2 (\rho \eta) = D^2 \{\varphi_1 *_2 G\} + D^3\{ \varphi_2 *_2 G \}  +  D^2 [\psi * G_\eps] \in L^\infty(\setS)$, which gives $D^2 \rho \in L^\infty_{\rm loc}(\setE)$. 

Next we show the regularity of $D^3 \rho$. Note that $\nabla^3 \Psi \in  L_{\rm loc}^\infty(\setE)$, we have $D^2 \varphi_1, D^2 \varphi_2 \in L^\infty(\setE)$. Due to Eq. (\ref{eqn:Extremely_powerful_inequalies1}), we have $D^4 \{\varphi_1 *_2 G\}, D^4\{ \varphi_2 *_2 G \} \in L^\infty(\setE)$, which also implies $D^3 \{\varphi_1 *_2 G\} \in L^\infty_{\rm loc}(\setE)$. Hence we have $D^3 (\rho \eta) = D^3 \{\varphi_1 *_2 G\} + D^4\{ \varphi_2 *_2 G \}  +  D^3 [\psi * G] \in L^\infty(\setS)$, which gives $D^3 \rho \in L^\infty_{\rm loc}(\setE)$. 

\noindent
{\bf Step 4. Derivatives, $D_t\rho$, $D_t D \rho$, and $D_t D^2 \rho$.}

Now we study the regularity of $D_t \rho, D_t D \rho, D_t D^2 \rho$. Note we have $D_t (\rho \eta) = D_t\{ \varphi_1 *_2 G \} - D_t \{ D \varphi_1 *_2 G \} + D_t\{\psi * G_\eps\}$. Due to Eq. (\ref{eqn:Extremely_powerful_inequalies1}), $\varphi_1, D \varphi_2 \in L^\infty(\setE)$ implies that $D_t\{ \varphi_1 *_2 G \}, D_t \{ D \varphi_1 *_2 G \} \in L^\infty(\setE)$ and hence $D_t[\rho \eta] \in L^\infty(\setS)$, $D_t \rho \in L^\infty_{\rm loc}(\setE)$. 

Note we have $D_t D(\rho\eta) = D_t\{ D\varphi_1 *_2 G \} + D_t \{ D^2 \varphi_1 *_2 G \} + D_t\{D \psi * G_\eps\}$. The fact that $D\varphi_1, D^2 \varphi_2 \in L^\infty(\setE)$ implies that $D_t\{ D \varphi_1 *_2 G \}, D_t \{ D^2 \varphi_1 *_2 G \} \in L^\infty(\setE)$ and hence $D_t D \rho \in L^\infty_{\rm loc}(\setE)$. 

Note we have $D_t D^2(\rho\eta) = D_t\{ D^2\varphi_1 *_2 G \} - D_t \{ D^3 \varphi_1 *_2 G \} + D_t\{D^2 \psi * G_\eps\}$. Note that $\nabla^4 \Psi \in L_{\rm loc}^\infty(\setE)$, hence $D^3 \varphi_2 \in L^\infty(\setE)$. Combining with the fact that $D^2\varphi_1 \in L^\infty(\setE)$, we have $D_t\{ D^2 \varphi_1 *_2 G \}, D_t \{ D^3 \varphi_1 *_2 G \} \in L^\infty(\setE)$ and hence $D_t D^2 \rho \in L^\infty_{\rm loc}(\setE)$.

\noindent
{\bf Step 5. Derivatives, $D_t^2\rho$.}

Finally we show the regularity of $D_t^2 \rho$. We have $D_t^2 (\rho\eta) = D_t\{ D_t [\varphi_1 *_2 G ] - D_t [ D \varphi_1 *_2 G ] + D_t[\psi * G_\eps]\}$, and
\begin{align}
D_t[\varphi_1 *_2 G] =& [\Delta \varphi_1 ] *_2 G + \varphi_1(\eps) * G_\eps,\\
D_t [D\varphi_2 *_2 G] =& [D\Delta\varphi_2] *_2 G + [D \varphi_2(\eps)] * G_\eps.
\end{align}
Note that $\nabla^4 \Psi \in L_{\rm loc}^\infty(\setE)$, we have $\Delta \varphi_1, D \Delta \varphi_1 \in L_{\rm loc}^\infty(\setE)$, and $\varphi_1(\eps), D \varphi_2(\eps) \in L^1(\reals^d)$. Hence according to Eq. (\ref{eqn:Extremely_powerful_inequalies1}), we have $D_t \{[\Delta \varphi_1 ] *_2 G \}, D_t\{[D \Delta \varphi_2 ] *_2 G\}$. In addition $D_t \{\varphi_1(\eps) * G_\eps\}, D_t\{[D \varphi_2(\eps)] * G_\eps\} \in L^\infty(S)$. As a result, we have $D_t^2 \rho \in L_{\rm loc}^\infty(\setE)$.  

\noindent
{\bf Step 6. Finish the proof. }

As a result, we have $\rho, D\rho, D^2 \rho, D^3 \rho, D_t \rho, D_t D \rho, D_t D^2 \rho, D_t^2 \rho \in L_{\rm loc}^\infty(\setE)$. Sobolev embedding theorem implies that $\rho, \partial_t \rho, \nabla_\btheta \rho, \nabla^2_\btheta \rho \in C^{0, 0}(\reals^d)$. In other words, $\rho \in C^{1,2}(\setE)$, which is the desired result.

\end{proof}

\begin{lemma}
Assume conditions {\sf A1} - {\sf A4} hold. Let initialization $\rho_0 \in \cK$ with $F_{\beta, \lambda}(\rho_0) < \infty$. Denote the solution of PDE (\ref{eqn:Chaos_PDE_diffusion})  to be $(\rho_t)_{t \ge 0}$. Then $\rho_t (\btheta) > 0$ for any $(t, \btheta) \in (0, \infty) \times \reals^d$. 

\end{lemma}

\begin{proof}

Note that $\rho_t \in C^{1,2}((0, \infty) \times \reals^d)$. By the Harnack's inequality \cite{evans2009partial}, we immediately have $\rho_t(\btheta) > 0$ 
for any $(t, \btheta) \in (0, \infty) \times \reals^d$. 

\end{proof}

We say $\rho_\star$ is a fixed point of PDE (\ref{eqn:Chaos_PDE_diffusion}), if its solution $(\rho_t)_{t \ge 0}$ starting from $\rho_\star$ satisfies $\rho_t \equiv \rho_\star$ for any $t \ge 0$. 

\begin{lemma}\label{lem:Fixed_point_PDE_Boltzman}
Assume conditions {\sf A1} - {\sf A3} hold. Then any fixed point  $\rho_\star$ of PDE (\ref{eqn:Chaos_PDE_diffusion}) with $\rho_\star \in \cK$ must satisfy the Boltzmann fixed point condition (\ref{eqn:BoltzmannEquation_in_proof}). 
\end{lemma}

\begin{proof}
Suppose $\rho_\star \in \cK$ is a fixed point of PDE (\ref{eqn:Chaos_PDE_diffusion}), taking $W(\btheta) \equiv \Psi_\lambda(\btheta;\rho_\star)$, then $\rho_\star \in \cK$ is a fixed point of the Fokker-Planck equation (\ref{eqn:Fokker_Planck_PDE}). 
\begin{align}\label{eqn:Fokker_Planck_PDE}
\partial_t \rho_t(\btheta) = 2 \xi(t) \nabla \cdot (\nabla W(\btheta) \rho_t(\btheta)) + 2\xi(t)/\beta \cdot \Delta_\btheta \rho_t(\btheta). 
\end{align}
Since $\lambda/2 \cdot \| \btheta \|_2^2 - 2K_3 \le \Psi_\lambda(\btheta;\rho_\star) \le \lambda/2 \cdot \| \btheta \|_2^2 + 2K_3$, the  
Fokker-Planck equation has a unique fixed point  \cite{markowich2000trend}, which solves
\[
\begin{aligned}
\rho_\star(\btheta) = \frac{1}{Z_\beta} \, \exp\{ - \beta W(\btheta)\}, ~~~ Z_\beta = \int_{\reals^d} \exp\{ - \beta W(\btheta) \} \de \btheta. 
\end{aligned}
\]
This is exactly the Boltzmann fixed point condition.

\end{proof}

\begin{lemma}\label{lem:F_non_increasing}
Assume conditions {\sf A1} - {\sf A4} hold. Let $(\rho_t)_{t\ge 0}$ be the solution  of PDE (\ref{eqn:Chaos_PDE_diffusion}) for an initialization $\rho_0 \in \cK$. Then the free energy $F_{\beta, \lambda}(\rho_t)$ is differentiable with respect to $t$, with
\begin{align}\label{eqn:F_non_increasing}
\partial_t F_{\beta, \lambda}(\rho_t) =& - 2 \xi(t) \int_{\reals^d} \| \nabla_\btheta (\Psi_\lambda(\btheta; \rho_t) + 1/\beta \cdot \log \rho_t(\btheta)) \|_2^2 \rho_t(\btheta) \de \btheta.
\end{align}
Therefore, $F_{\beta, \lambda}(\rho_t)$ is non-increasing in $t$. 
\end{lemma}

\begin{proof}
Calculate the differential of the free energy along the curve $\rho_t$, we have 
\[
\begin{aligned}
\partial_t F_{\beta, \lambda}(\rho_t) =& \int_{\reals^d}\Psi_\lambda(\btheta; \rho_t) \partial_t\rho_t(\btheta) \de \btheta + 1/\beta \cdot \int \log (\rho_t(\btheta)) \partial_t \rho_t(\btheta) \de \btheta \\
=& -2 \xi(t) \int_{\reals^d} \| \nabla_\btheta (\Psi_\lambda(\btheta; \rho_t) + 1/\beta \cdot \log \rho_t(\btheta)) \|_2^2 \rho_t(\btheta) \de \btheta.
\end{aligned}\label{eq:PartialFreeEnergy}
\]
\end{proof}

\begin{lemma}\label{lem:generalized_Poincare_inequality}
Assume $ K_0 \| \btheta \|_2^2 - K_1 \le \Phi(\btheta) \le  K_0 \| \btheta \|_2^2 + K_1 $ for some positive constant $K_0, K_1$. Define 
\begin{align}
\mu_\star(\de \btheta) = \frac{1}{Z_\star} \exp\{ - \Phi(\btheta) \} \de \btheta,~~ Z_\star = \int_{\reals^d} \exp\{ - \Phi(\btheta) \} \de \btheta
\end{align}
Let $\cD \equiv \{ f \in L^2(\reals^d, \mu_\star) \cap C^1(\reals^d): \| \nabla f \|_2 \in L^2(\reals^d, \mu_\star)\}$. For any $f \in \cD$, define 
\begin{align}
I(f) \equiv \int_{\reals^d} \| \nabla f (\btheta) \|_2^2 \cdot \mu_\star(\de \btheta) < \infty. 
\end{align}
Assume $(f_n)_{n \ge 1} \subseteq \cD$, with $\lim_{n\to \infty} I(f_n) = 0$, and $f_n$ converges weakly to $f_\star$ in $L^2(\reals^d, \mu_\star)$. Then $f_\star(\btheta) \equiv F_\star$ for some constant $F_\star$. 
\end{lemma}

\begin{proof}

First we show that the measure $\mu_\star$ satisfies the Poincare inequality:  for any $f \in \cD$,
\begin{align}\label{eqn:Poincare_for_mu_star}
\mu_\star((f - \mu_\star(f))^2) \le K \cdot I(f),
\end{align}
for some constant $K$. 

Let $\mu$ be the Gaussian distribution $\normal(\bzero, 1/(2 K_0) \cdot \id_d)$. Then for any $\btheta \in \reals^d$,
\begin{align}
\mu(\btheta) \cdot \exp\{ - 2 K_1 \} \le \mu_\star(\btheta) \le \mu(\btheta) \cdot \exp\{2 K_1 \}. 
\end{align}
Therefore, for any nonnegative measurable function $f : \reals^d \to [0, \infty)$ and $g: \reals^d \times \reals^d \to [0, \infty)$, letting $(G, G') \sim \mu \times \mu$ and $(X, X') \sim \mu_\star \times \mu_\star$, we have 
\[
\begin{aligned}
\E[f(G)] \cdot \exp\{ -2 K_1 \} \le& \E[f(X)] \le \E[f(G)] \cdot \exp\{ 2 K_1 \},\\
\E[g(G, G')] \cdot \exp\{ -4 K_1 \} \le& \E[g(X, X')] \le \E[g(G, G')] \cdot \exp\{ 4 K_1 \}.
\end{aligned}
\]
Note we have the Poincare inequality for the Gaussian distribution $\mu$, 
\begin{align}
\text{Var}[f(G)] \le 1/(2 K_0) \cdot \E[\| \nabla f(G) \|_2^2]
\end{align}
for any differentiable $f$. Therefore, we have 
\[
\begin{aligned}
\text{Var}[f(X)] =& \frac{1}{2}\E[(f(X) - f(X'))^2] \le \frac{1}{2}\exp\{ 4 K_1 \} \cdot \E[(f(G) - f(G'))^2]\\
=& \exp\{ 4 K_1 \} \cdot \text{Var}[f(G)] \le 1/(2K_0) \cdot \exp\{ 4 K_1 \} \cdot \E[\| \nabla f(G) \|_2^2] \\
\le& 1/(2K_0) \cdot \exp\{ 6 K_1 \} \cdot \E[\| \nabla f(X) \|_2^2]. 
\end{aligned}
\]
This proves the Poincare inequality (\ref{eqn:Poincare_for_mu_star}) for $\mu_\star$.

Since $\lim_{n\to \infty} I(f_n) = 0$, due to (\ref{eqn:Poincare_for_mu_star}), we immediately have $f_n - \mu_\star(f_n)$ converges to $0$ in $L^2(\reals^d, \mu_\star)$. Note we assumed $f_n$ converges weakly to $f_\star$ in $L^2(\reals^d, \mu_\star)$, and $1 \in L^2(\reals^d, \mu_\star)$, we have 
\[
\begin{aligned}
\lim_{n \to \infty} \mu_\star(f_n) =  \mu_\star(f). 
\end{aligned}
\]
Therefore, $f_n - \mu_\star(f_n)$ converges weakly to $f_\star - \mu_\star(f_\star)$ in $L^2(\reals^d, \mu_\star)$. Hence $f_\star(\btheta) \equiv  \mu_\star(f_\star)$.

\end{proof}

\begin{lemma}\label{lem:limiting_point_Boltzmann_equation}
Assume conditions {\sf A1} - {\sf A4} hold. Then the solution $(\rho_t)_{t\ge 0}$ of PDE (\ref{eqn:Chaos_PDE_diffusion}) for any initialization $\rho_0 \in \cK$ converges weakly to $\rho_\star \in \cK$ as $t \to \infty$, where $\rho_\star$ is the unique solution of  the Boltzmann fixed point condition, which is the global minimizer of $F_{\beta, \lambda}$. 
\end{lemma}

\begin{proof}
According to Lemma \ref{lem:F_non_increasing}, $F_{\beta, \lambda}$ is non-increasing along the solution path. According to Lemma \ref{lem:Free_Energy_Unique_Minimum}, $F_{\beta, \lambda}(\rho_t)$ is lower bounded. Therefore, we have 
\begin{align}\label{eqn:partial_F_converges_to_0}
\lim_{t \to \infty} \int_{\reals^d} \| \nabla_\btheta(\Psi_\lambda(\btheta; \rho_t) + 1/\beta \cdot \log \rho_t(\btheta)) \|_2^2 \rho_t(\btheta) \de \btheta = 0. 
\end{align}

Since $M(\rho_t)$ is uniformly bounded, by Lemma \ref{lemma:AbsolutelyContinuous}, $(\rho_t)_{t\ge 0}$ as a sequence of probability distribution in $\cuP(\reals^d)$ is uniformly tight. Hence there exists $\rho_\star \in \cuP(\reals^d)$ and a subsequence $(\rho_{t_k})_{k \ge 1}$ with $\lim_{k\to\infty} t_k = \infty$ such that $(\rho_{t_k})_{k \ge 1}$ converges weakly to $\rho_\star$. By Lemma \ref{lemma:AbsolutelyContinuous} and Lemma \ref{lem:Bound_H_by_M}, $\{ \int \max\{\rho_{t_k} \log \rho_{t_k}, 0) \} \de \btheta\}_{k \ge 1}$ is uniformly bounded. 
Using de la Vall\'ee-Poussin's criteria, we can show that $(\rho_{t_k})_{k \ge 1}$ is uniformly integrable, and hence $\rho_\star$ is absolute continuous with respect to Lebesgue measure, which means $\rho_\star$ has a density. 

Note we have 
\[
\nabla_\btheta\Psi_\lambda(\btheta; \rho_t) - \nabla_\btheta\Psi_\lambda(\btheta; \rho_\star) = \int_{\reals^d} \nabla_\btheta U(\btheta, \btheta') (\rho_t(\btheta') - \rho_\star(\btheta')) \de \btheta'. 
\]
According to condition {\sf A3}, $\nabla_\btheta U$ is $K_3$-bounded-Lipschitz with respect to $(\btheta, \btheta')$. Therefore, 
\begin{align}
\sup_{\btheta \in \reals^d} \| \nabla_\btheta \Psi_\lambda(\btheta; \rho_t) - \nabla_\btheta \Psi_\lambda(\btheta; \rho_\star)\|_2 \le K_3 \cdot d_{\sBL}(\rho_{t}, \rho_\star) \to 0,
\end{align}
as $d_{\sBL}(\rho_{t}, \rho_\star) \to 0$. Accordingly, we have  
\begin{align}\label{eqn:partial_F_intepolation}
\lim_{k \to \infty} \int_{\reals^d} \| \nabla_\btheta(\Psi_\lambda(\btheta; \rho_{t_k}) - \Psi_\lambda(\btheta; \rho_\star) )\|_2^2 \rho_{t_k}(\btheta) \de \btheta \le K_3^2 \cdot \lim_{k \to \infty} d_{\sBL}(\rho_{t_k}, \rho_\star)^2 = 0.
\end{align}
Combining Eq. (\ref{eqn:partial_F_intepolation}) with Eq. (\ref{eqn:partial_F_converges_to_0}), we have 
\begin{align}\label{eqn:partial_F_converges_to_0_intepolation}
\lim_{k \to \infty} \int_{\reals^d} \| \nabla_\btheta(\Psi_\lambda(\btheta; \rho_\star) + 1/\beta \cdot \log \rho_{t_k}(\btheta)) \|_2^2 \rho_{t_k}(\btheta) \de \btheta = 0.
\end{align}

Note we have 
\begin{equation}\label{eqn:reformulation_fisher_information}
\begin{aligned}
&\int_{\reals^d} \| \nabla_\btheta(\Psi_\lambda(\btheta; \rho_\star) + 1/\beta \cdot \log \rho_{t_k}(\btheta)) \|_2^2 \rho_{t_k}(\btheta) \de \btheta \\
=&\frac{1}{\beta^2} \int_{\reals^d} \| \nabla_\btheta (\rho_{t_k}(\btheta) \exp\{ \beta \Psi_\lambda(\btheta; \rho_\star) \} ) \|_2^2 \cdot \rho_{t_k}(\btheta)^{-1} \exp\{ - 2\beta \Psi_\lambda(\btheta; \rho_\star) \} \de \btheta\\
=& \frac{4}{\beta^2} \int_{\reals^d} \| \nabla_\btheta [(\rho_{t_k}(\btheta) \exp\{ \beta \Psi_\lambda(\btheta; \rho_\star) \} )^{1/2}] \|_2^2 \cdot \exp\{ - \beta \Psi_\lambda(\btheta; \rho_\star) \} \de \btheta. 
\end{aligned}
\end{equation}
Define 
\begin{align}
\mu_\star(\de \btheta) = 1/ Z_\star \cdot \exp\{ - \beta \Psi_\lambda(\btheta; \rho_\star) \} \mu_0(\de \btheta), ~~ Z_\star = \int_{\reals^d}\exp\{ - \beta \Psi_\lambda(\btheta; \rho_\star) \} \mu_0(\de \btheta), 
\end{align}
$f_k(\btheta) = [\exp(\beta \Psi_\lambda(\btheta; \rho_\star)) \rho_{t_k}(\btheta)]^{1/2} \in \cD \equiv \{ f \in L^2(\reals^d, \mu_\star) \cap C^1(\reals^d): \| \nabla f\|_2 \in L^2(\reals^d, \mu_\star) \}$ ($f_{k} \in C^1(\reals^d)$ because $\rho_{t}(\btheta) > 0$ for any $\btheta \in \reals^d$ and $\rho_t(\btheta) \in C^1(\reals^d)$ for fixed $t$), and $f_\star(\btheta) = [\exp(\beta \Psi_\lambda(\btheta; \rho_\star)) \rho_\star (\btheta)]^{1/2} \in L^2(\reals^d, \mu_\star)$. Since we have $\rho_{t_k}$ converges to $\rho_\star$ weakly in $L^1(\reals^d, \mu_0)$, then $f_k$ converges weakly to $f_\star$ in $L^2(\reals^d, \mu_\star)$. Define $I(f) \equiv \int_{\reals^d} \| \nabla f (\btheta) \|_2^2 \cdot \mu_\star(\de \btheta)$. Eq. (\ref{eqn:partial_F_converges_to_0_intepolation}) and (\ref{eqn:reformulation_fisher_information}) give $\lim_{k \to \infty} I(f_k) = 0$. Now we apply Lemma \ref{lem:generalized_Poincare_inequality} with $\Phi(\btheta) = \beta \Psi_\lambda(\btheta; \rho_\star)$. This $\Phi$ satisfies $\beta \lambda/2 \cdot \| \btheta \|_2^2 - 2 \beta K_2 \le \Phi(\btheta) \le \beta \lambda/2 \cdot \| \btheta \|_2^2 + 2 \beta K_2$, where $K_2$ is the constant in Assumption {\sf A2}. Lemma \ref{lem:generalized_Poincare_inequality} implies $f_\star(\btheta) \equiv F_\star$ for some constant $F_\star$. 

This proves that $\rho_\star(\btheta) = F_\star \cdot \exp\{ - \beta \Psi_\lambda(\btheta; \rho_\star) \}$. Combining with the fact that $\int_{\reals^d} \rho_\star(\btheta) \de \btheta = 1$, $\rho_\star$ satisfies the Boltzmann fixed point condition. According to Lemma \ref{lem:Unique_Boltzmann_solution}, the Boltzmann fixed point condition has a unique solution $\rho_\star^{\beta, \lambda}$. Therefore, all the converging weak limit of subsequence of $\rho_t$ converges to the same point $\rho_\star^{\beta, \lambda}$. As a result, $\rho_t$ converges to $\rho_\star^{\beta, \lambda}$ weakly in $L^1(\reals^d)$. 

\end{proof}

\subsection{Proof of Proposition \ref{thm:FixedPoints_Temp_finite}, Theorem \ref{thm:GeneralPDE_Noisy_fixed_point}, and Theorem \ref{thm:GeneralPDE_Noisy}}
\label{sec:Proof_finite_T}

Proposition \ref{thm:FixedPoints_Temp_finite} is given by Lemma \ref{lemma:AbsolutelyContinuous}, \ref{lem:Unique_Boltzmann_solution}, and Lemma \ref{lem:Fixed_point_PDE_Boltzman}. 
Theorem \ref{thm:GeneralPDE_Noisy_fixed_point} is given by Lemma \ref{lem:Free_Energy_Unique_Minimum}, \ref{lem:Unique_Boltzmann_solution}, \ref{lem:Bound_global_minimizer_free_energy}, and \ref{lem:limiting_point_Boltzmann_equation}.

Now we prove Theorem \ref{thm:GeneralPDE_Noisy}. First, according to Lemma \ref{lem:Bound_global_minimizer_free_energy}, for any $\eta > 0$, there exists constant $K$ depending on $\eta, K_0, K_1, K_2, K_3$, such that as we take $\beta \ge K D$, we have
\begin{align}\label{eqn:GeneralSGD_Noisy_Risk_bound1}
R(\rho_\star^{\beta, \lambda}) \le \inf_{\rho \in \cuP(\reals^D)} R_\lambda(\rho) + \eta/3. 
\end{align} 

According to Lemma \ref{lem:limiting_point_Boltzmann_equation}, we have $\rho_t$ converges to $\rho_\star^{\beta, \lambda}$ weakly. Therefore, there exists $T = T(\eta,V,U,\{K_i\},D,\lambda, \beta) < \infty$, so that $d_{\sBL}(\rho_t, \rho_\star^{\beta, \lambda}) \le \eta/(3Z)$ for any $t \ge T$, where $Z = Z(\{K_i\})$ is the bounded-Lipschitz constant of $R$ with respect to $\rho$. Hence, we have 
\begin{align}\label{eqn:GeneralSGD_Noisy_Risk_bound2}
R(\rho_t) \le R(\rho_\star^{\beta, \lambda}) + \eta/3
\end{align}
for any $t \ge T$. 

Finally, according to Theorem \ref{thm:GeneralPDE}, there exists $K'$ depending on $K_i$'s, so that for all $k \le 10 T / \eps$, we have
\[
\l R_{N}(\btheta^k) - R_{\rho_{k \eps}} \l \le K' e^{K' T}\sqrt{1/N \vee \eps } \cdot \Big[\sqrt{D + \log (N  (1 / \eps \vee 1))} + z \Big],
\]
with probability at least $1 - e^{-z^2}$. Hence there exists $C_0 = C_0(\eta, \{K_i\}, \delta)$, so that as $N, 1/\eps \ge C_0 \exp\{ C_0 T \} D$ and $\eps \ge 1/N^{10}$, we have 
\begin{align}\label{eqn:GeneralSGD_Noisy_Risk_bound3}
\l R_{N}(\btheta^k) - R(\rho_{k \eps}) \l \le \eta/3,
\end{align}
with probability at least $1 - \delta$. 

Combining Eq. (\ref{eqn:GeneralSGD_Noisy_Risk_bound1}), (\ref{eqn:GeneralSGD_Noisy_Risk_bound2}), and (\ref{eqn:GeneralSGD_Noisy_Risk_bound3}) we get the desired result.

\subsection{Dependence of convergence time on $D$ and $\eta$}

Theorem \ref{thm:GeneralPDE_Noisy} does not provide any estimate for
the dependence of the convergence time on the problem dimensions $D$ and on the accuracy $\eta$.
However the proof suggests the following heuristic. When $\rho_t$ is sufficiently close to the minimizer $\rho_*$,
we heuristically can approximate the free energy dissipation formula (\ref{eq:PartialFreeEnergy}) as
\begin{align}
\partial_t F_{\beta, \lambda}(\rho_t) \approx& - \int_{\reals^d} \| \nabla_\btheta (\Psi_\lambda(\btheta; \rho_*) + 1/\beta \cdot \log \rho_t(\btheta)) \|_2^2 \rho_t(\btheta) \de \btheta\, .
\end{align}
This is the same as the free energy dissipation for the Fokker-Planck equation with potential $\Psi_{\lambda}(\btheta;\rho_*)$. This suggests that, close to $\rho_*$,
convergence should be dominated by the speed of convergence in this Fokker-Plank equation, which is controlled by the log-Sobolev constant of the potential 
 $\Psi_{\lambda}(\btheta;\rho_*)$, to be denote by $c_{*}$ \cite{markowich2000trend}:
\begin{align}
F_{\beta, \lambda}(\rho_t) \lesssim F_{\beta, \lambda}(\rho_{t_0}) \, e^{-c_*(t-t_0)} \, .
\end{align}
Note that the log-Sobolev constant can be exponentially small in $D$.
We expect this heuristic to capture the rough dependence of the convergence time $T$ on $\eta$ and $D$, hence
suggesting $T= e^{O(D)}\log (1/\eta)$. 

\section{Numerical Experiments}

In this section, we discuss numerical experiments whose results were presented in the main text, as well as some additional ones. 
Some technical details of the figures in the main text are also presented here; in particular, Section \ref{subsubsec:num_isotropic_dyna} for Figure \ref{fig:SGD_Spherically}, Section \ref{subsubsec:num_isotropic_stat} for 
Figure \ref{fig:R_r_isotropic}, Section \ref{subsec:num_ani} for Figure \ref{fig:RiskAnisotropic}, and Section \ref{subsec:num_fail} for Figure \ref{fig:Failure}.

\subsection{Isotropic Gaussians} \label{subsec:num_iso}
In this section, we present details of the numerical experiments pertaining to the example of centered isotropic Gaussians:
\begin{itemize}	
	\item[] With probability $1/2$: $y=+1$, $\bx\sim\normal(0,(1+\Delta)^2\id_d)$.
	\item[] With probability $1/2$: $y=-1$, $\bx\sim\normal(0,(1+\Delta)^2\id_d)$.	
\end{itemize}
In all numerical examples in this section, we use the activation $\sigma_*(\bx;\btheta_i) = \sigma(\langle\bw_i,\bx\rangle)$, where 
$\sigma(t) = s_1$ if $t\le t_1$, $\sigma(t) = s_2$ if $t\ge t_2$, and $\sigma(t)$ interpolated linearly for $t\in (t_1,t_2)$.
In simulations we use $t_1 = 0.5$, $t_2=1.5$, $s_1=-2.5$, $s_2=7.5$. This is also used for examples with centered Gaussians in the main text, cf. Figures \ref{fig:SGD_Spherically} and \ref{fig:R_r_isotropic}, and Section \ref{sec:IsotropicGaussian} in the supplemental information. This activation is plotted in Figure \ref{fig:num_sigmoid}. 

\begin{figure}[]
	\begin{center}
		\begin{tabular}{ll}
			\includegraphics[width=0.5\textwidth]{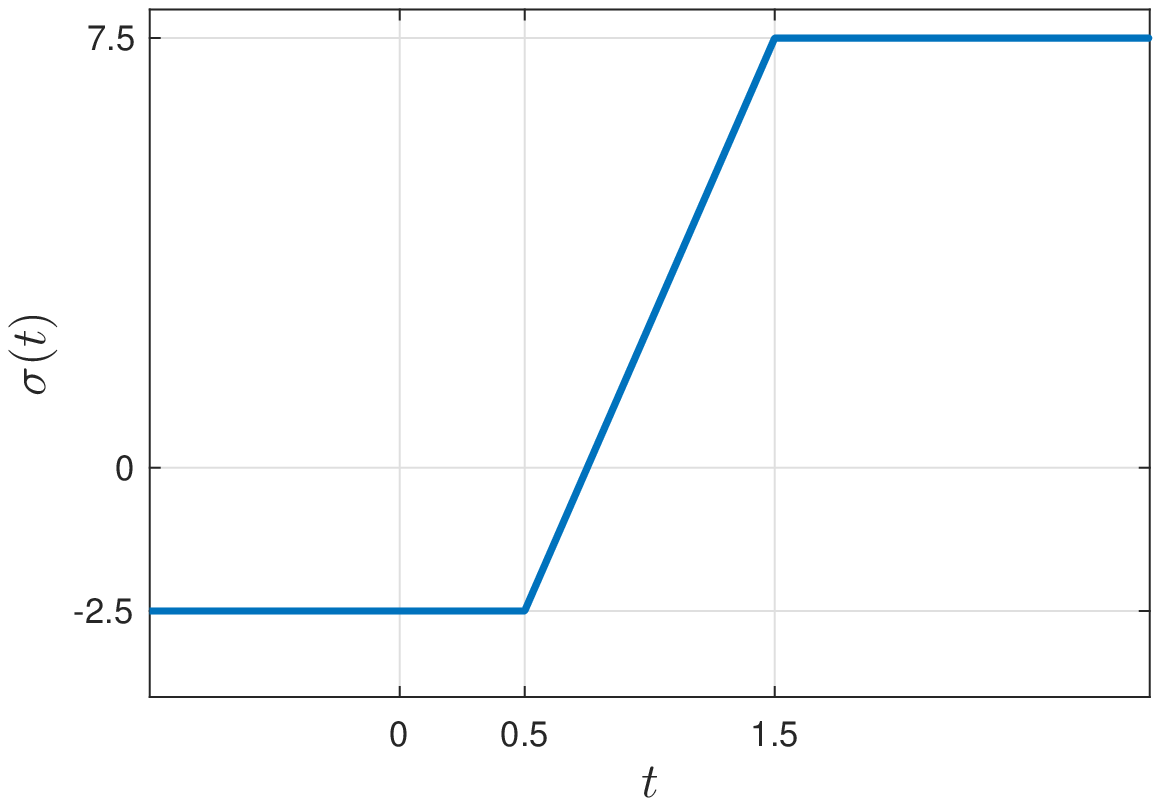}
			&
			\includegraphics[width=0.5\textwidth]{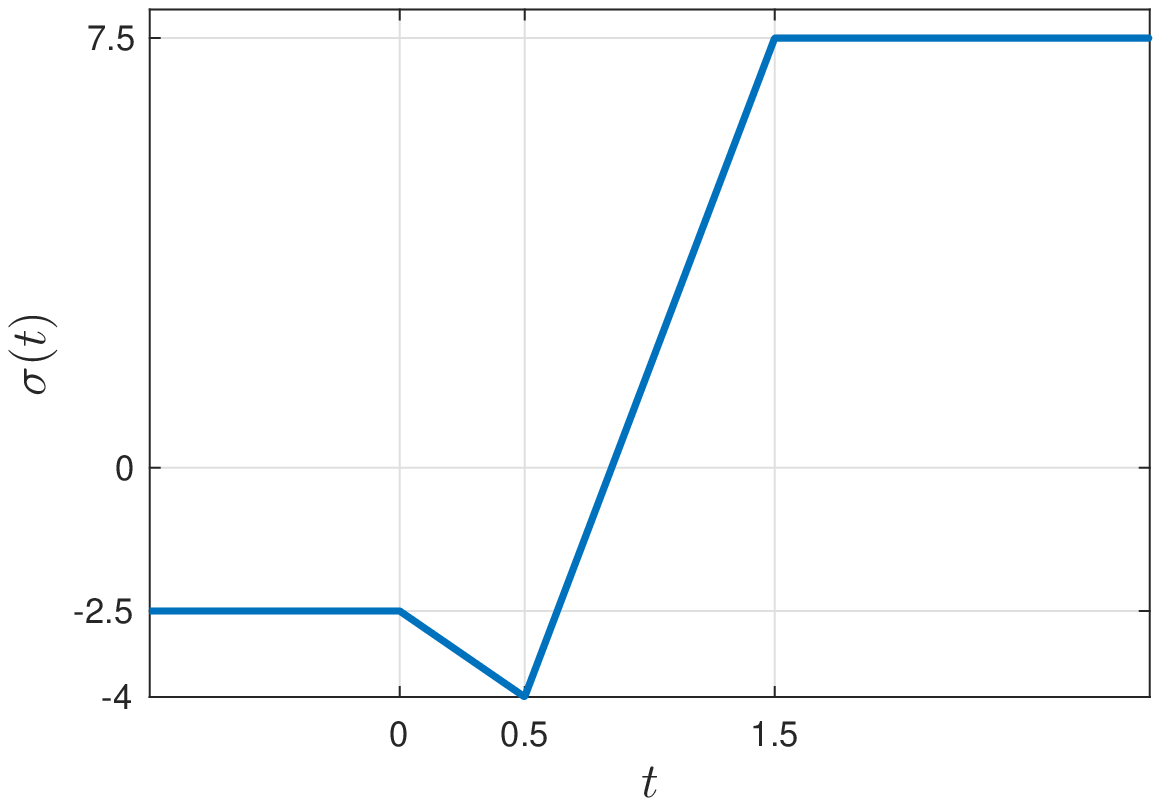}
		\end{tabular}
	\end{center}
	\caption{The activation functions $\sigma(t)$ used in Section \ref{subsec:num_iso} (left plot) and Section \ref{subsec:num_fail} (right plot).}
	\label{fig:num_sigmoid}
\end{figure}

\subsubsection{Empirical validation of distributional dynamics}\label{subsubsec:num_isotropic_dyna}
Here we discuss empirical validation for the dynamics in the isotropic Gaussian example.

\noindent{\bf PDE simulation.} Simulating the PDE (Eq. [\ref{eq:ReducedPDE}] of the main text) for general $d$ is computationally intensive. 
In order to simplify the problem, we only consider $d=\infty$. In that case, we recall that the risk is given by Eq. (\ref{eqn:Risk_infinite_isotropic}), which we copy here for ease of reference:
\begin{align}
\barR_{\infty}(\rad)= \frac{1}{2}\left(1-\int q_+(r)\, \rad(\de r)\right)^2+\frac{1}{2}\left(1+\int q_-(r)\, \rad(\de r)\right)^2\, ,
\end{align}
where $q_\pm(t) = \mathbb{E}\{\sigma((1\pm \Delta)tG)\}$, $G\sim\mathsf{N}(0,1)$. In addition, from Eq. (\ref{eq:psi_infty_isotropic_description}),
\begin{align}
\psi_\infty(r; \rad) = \frac{1}{2} [ \< q_+, \rad\> - 1 ] q_+(r) + \frac{1}{2} [ \< q_-, \rad\> + 1 ] q_-(r).
\end{align}
The PDE is then $\partial_t \rad_t = 2\xi(t) \partial_r [\rad_t \partial_r \psi_\infty(r; \rad_t)]$.

The solution to the PDE is approximated, at all time $t$, by the following multiple-deltas ansatz:
\begin{align}
\rad_t = \frac{1}{J}\sum_{i=1}^J\delta_{r_i(t)}\, ,
\end{align}
where $J\in\mathbb{N}$ is a pre-chosen parameter. Note that for a fixed $J$, if the PDE is initialized at $\rad_0$ taking the above form, then for any $t\geq 0$, $\rad_t$ remains in the above form. Then for any smooth test function $f:\,\mathbb{R}\mapsto\mathbb{R}$ with compact support,
\begin{align}
\frac{1}{J}\sum_{i=1}^J f'(r_i(t))r_i'(t) &= \partial_t\<f, \rad_t\> =  -2\xi(t) \<f', \rad_t \partial_r \psi_\infty(r; \rad_t) \> \\
&=  -2\xi(t) \frac{1}{J}\sum_{i=1}^J f'(r_i(t))\partial_r \psi_\infty(r_i(t); \rad_t).
\end{align}
Under this ansatz, let us write $\barR_\infty(\rad_t) = \barR_{\infty,J}(\mathbf{r}(t))$, where $\mathbf{r}(t)=(r_1(t),...,r_J(t))^\top$, and
\begin{align}
\barR_{\infty,J}(\mathbf{r}) = \frac{1}{2}\left(1-\frac{1}{J} \sum_{i=1}^J q_+(r_i)\right)^2+\frac{1}{2}\left(1+\frac{1}{J} \sum_{i=1}^J q_-(r_i)\right)^2.
\end{align}
Notice that $\partial_r \psi_\infty(r_i(t); \rad_t) =(J/2) (\nabla \barR_{\infty,J}(\mathbf{r}(t)))_i$. Therefore we obtain
\begin{align}
\frac{\de\phantom{t}}{\de t}\mathbf{r}(t) = -J\xi(t)\nabla \barR_{\infty,J}(\mathbf{r}(t)).
\end{align}
Hence under the multiple-deltas ansatz, one can simulate numerically the PDE via the above evolution equation of $\mathbf{r}(t)$. In particular, given $\mathbf{r}(t)$,  one approximates $\mathbf{r}(t+\delta t)$ for some small displacement $\delta t$ by
\begin{align}
\mathbf{r}(t+\delta t) \approx \mathbf{r}(t) -J\xi(t)\nabla \barR_{\infty,J}(\mathbf{r}(t))\delta t.\label{eq:num_PDE_iso_discretize}
\end{align}
In general, one would want to take a large $J$ to obtain a more accurate approximation. There are certain cases where one can take small $J$ (even $J=1$). An example of such case is given in the following.

\noindent{\bf Details of Figure \ref{fig:SGD_Spherically} of the main text.} For the data generation, we set $\Delta=0.8$. For the SGD simulation, we take $d=40$, $N=800$, with $\eps=10^{-6}$ and $\xi(t) = 1$. The weights are initialized as $(\bw_{i})_{i\leq N}\sim_{iid}\mathsf{N}(0,0.8^2/d\cdot \mathbf{I}_d)$. We take a single SGD run. At iteration $10^3, 4\times 10^6, 10^7$, we plot the histogram of $(\Vert\bw_i\Vert_2)_{i\leq N}$. This produces the results of the SGD in Figure \ref{fig:SGD_Spherically} of the main text.

To obtain results from the PDE, we take $J=400$, and generate $r_i(0)=\Vert Z_i\Vert_2$, where $(Z_i)_{i\leq J}\sim_{iid}\mathsf{N}(0,0.8^2/d\cdot \mathbf{I}_d)$. We obtain $\mathbf{r}(t)$ from $t=0$ until $t=10^7\eps$, by discretizing this interval with $10^5$ points equally spaced on the $\log_{10}$ scale and sequentially computing $\mathbf{r}(t)$ at each point using Eq. (\ref{eq:num_PDE_iso_discretize}). Note that the SGD result at iteration $k$ corresponds to $\mathbf{r}(\eps k)$. We re-simulate the PDE for 100 times, each with an independently generated initialization. The obtained histogram for the PDE, as shown in the figure, is the aggregation of these 100 runs.

\noindent{\bf Further numerical simulations.} Figure \ref{fig:num_iso_rhist_Delta2} plots the evolution of $\rad_t$ for $\Delta=0.2$. The setting is identical to the one in Figure \ref{fig:SGD_Spherically} of the main text, described in the previous paragraphs.

In Figure \ref{fig:num_iso_risk}, we plot the evolution of the population risk for the SGD and its PDE prediction counterpart, for $\Delta=0.2$ and $\Delta=0.8$. The setting for the SGD plots is the same as described in the previous paragraphs. We compute the risk attained by the SGD by Monte Carlo averaging over $10^4$ samples. The setting for the PDE plots tagged ``$J=400$'' is almost the same as in the previous paragraphs, except that we take only 1 run. For the PDE plot tagged ``$J=1$'', we take $J=1$ and $r(0)=0.8$ instead. In the inset plot, we also show the evolution of $(1/N)\sum_{i=1}^N\Vert\bw_i\Vert_2$ of the SGD, and $(1/J)\sum_{i=1}^Jr_i(t)$ of the PDE.

In Figure \ref{fig:num_iso_riskSingleDelta_Delta2}, we plot the function $\barR^{(1)}_d(r)$, for $d=40$ and $\Delta=0.2$. (Recall $\barR^{(1)}_d(r)$ from Eq. [\ref{eq:Isotropic}] of the main text, and see also Section \ref{subsubsec:num_iso_Lemma1}.) On this landscape, we also plot the evolution of the corresponding SGD and PDE, as described in the last paragraph.

\noindent{\bf Comments.} We observe in Figure \ref{fig:num_iso_risk} a good match between the SGD and the PDE, even when $J=1$, for $\Delta=0.2$. This can be explained with our theory, which predicts that at $\Delta=0.2$, the minimum risk is achieved by the uniform distribution over a sphere of radius $\Vert\bw\Vert_2=r_*$ (see also Section \ref{subsubsec:num_iso_Lemma1}). This corresponds to $\rad_t$, as $t\to\infty$, being a delta function and placing probability 1 at $r_*$. Furthermore due to the way we initialize the SGD, $\rad_0$ is well concentrated. One can then expect that $\rad_t$ is also well concentrated at all time $t$, in which case $J=1$ is sufficient. This claim is reflected in our numerical experiments, shown in Figure \ref{fig:num_iso_rhist_Delta2}.

We also observe in Figure \ref{fig:num_iso_risk} that the case $\Delta=0.2$ has a rapid transition from a high risk to a lower risk, unlike the case $\Delta=0.8$. This is also expected from our theory. As said above, $\rad_t$ is approximately a delta function at all time $t$, and the position $r(t)$ evolves by gradient flow in the landscape of $\barR^{(1)}_d(r)$. This latter claim is well supported by Figure \ref{fig:num_iso_riskSingleDelta_Delta2}. As observed in Figure \ref{fig:num_iso_riskSingleDelta_Delta2}, $\barR^{(1)}_d(r)$ is rather benign, and hence the transition of the population risk should be smooth. However the case for $\Delta=0.8$ is different: $\rad_t$ is not concentrating at large $t$, as evident in Figure \ref{fig:SGD_Spherically} of the main text, even though $\barR^{(1)}_d(r)$ is generally benign for a vast variety of values of $d$ and $\Delta$ (see Figure \ref{fig:num_iso_riskSingleDelta} and Section \ref{subsubsec:num_iso_Lemma1}).

Note that the computation of the PDE assumes $d=\infty$. Furthermore it also requires $N=\infty$ (recalling Theorem \ref{thm:GeneralPDE} of the main text). The discrepancy to the SGD is due to the fact that $d$ and $N$ are finite in the SGD simulations. Nevertheless in our numerical examples, such discrepancy is insignificant.

\begin{figure}[]
	\begin{center}
		\includegraphics[width=1.0\textwidth]{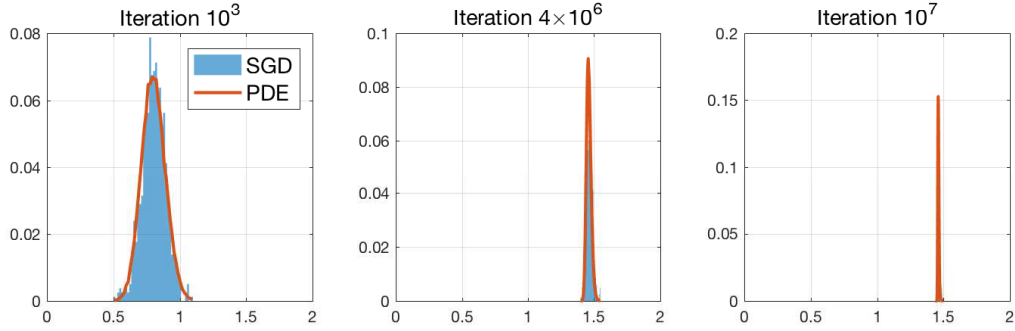}
	\end{center}
	\caption{Evolution of the reduced distribution $\rad_t$ for $\Delta=0.2$, in the isotropic Gaussians example of Section \ref{subsec:num_iso}.}
	\label{fig:num_iso_rhist_Delta2}
\end{figure}

\begin{figure}[]
	\begin{center}
		\includegraphics[width=0.9\textwidth]{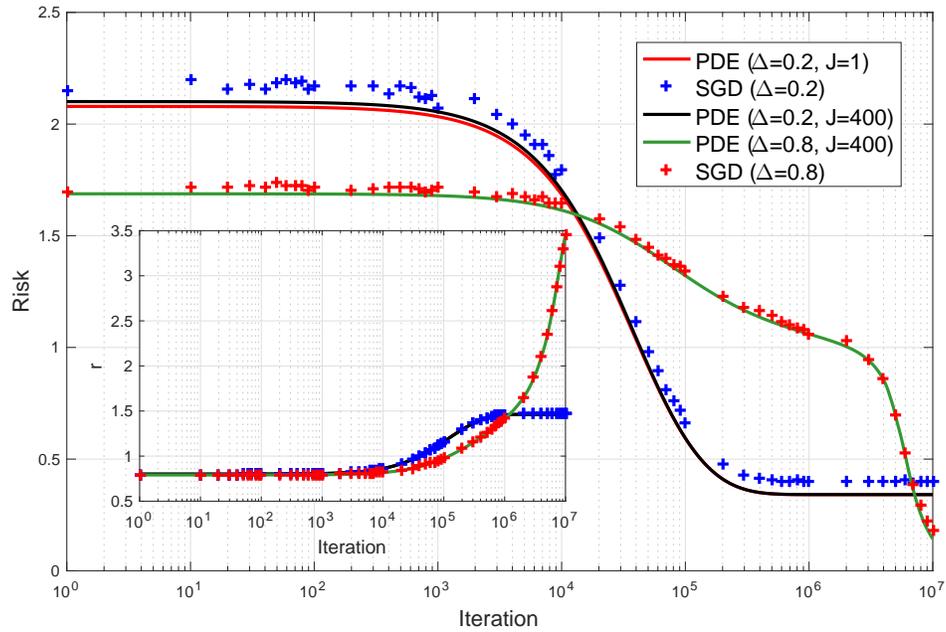}
	\end{center}
	\caption{The evolution of the population risk and the parameter $r$ of the reduced distribution $\rad_t$, in the isotropic Gaussians example of Section \ref{subsec:num_iso}.}
	\label{fig:num_iso_risk}
\end{figure}

\begin{figure}[]
	\begin{center}
		\includegraphics[width=0.7\textwidth]{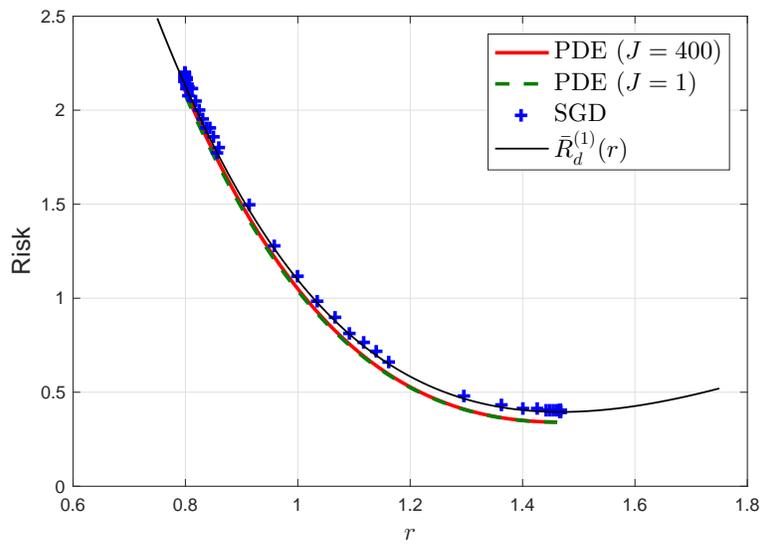}
	\end{center}
	\caption{The function $\barR_d^{(1)}(r)$ vs $r$, as well as the evolution of the SGD and the PDE on this landscape, for $\Delta=0.2$ and $d=40$, in the isotropic Gaussians example of Section \ref{subsec:num_iso}. Here the SGD and the PDE evolve from the leftmost point to the rightmost point.}
	\label{fig:num_iso_riskSingleDelta_Delta2}
\end{figure}

\subsubsection{Empirical validation  of the statics}\label{subsubsec:num_isotropic_stat}
Here we discuss numerical verification for the statics in the isotropic Gaussian example.

\noindent{\bf Optimizing $\barR_d(\rad)$.} For the chosen activation, we have from Eq. (\ref{eqn:u_d_isotropic}) that
\begin{align}
\barR_d(\rad) &= 1+2 \int v(r) \, \rad(\de r) + \int u_d(r_1,r_2) \, \rad(\de r_1)\,\rad(\de r_2)\, ,\\
v(r) &= -\frac{1}{2}g(0,(1+\Delta)r) + \frac{1}{2}g(0,(1-\Delta)r)\, ,\label{eq:num_iso_v}\\
u_d(r_1,r_2) &= \frac{\Gamma(d/2)}{\Gamma(1/2)\Gamma((d-1)/2)}\int_{\theta=0}^\pi \hat{u}(r_1,r_2,\theta)\sin^{d-2}\theta \de\theta\, ,\label{eq:num_iso_u}\\
\hat{u}(r_1,r_2,\theta) &= \frac{1}{2}f((1+\Delta)r_1,(1+\Delta)r_2,\theta) + \frac{1}{2}f((1-\Delta)r_1,(1-\Delta)r_2,\theta)\, , \\
f(r_1,r_2,\theta) &= \int_{x=-\infty}^{+\infty} \sigma(r_1x)g(r_2x\cos\theta, r_2\sin\theta)\phi(x) \de x\, ,\\
g(a,b) &= s_2 + (s_1-\sigma_{\rm itc}-\sigma_{\rm sl}a)\Phi\left(\frac{t_1-a}{b}\right) + (\sigma_{\rm sl}a+\sigma_{\rm itc}-s_2)\Phi\left(\frac{t_2-a}{b}\right) \nonumber\\
&\quad + \sigma_{\rm sl}b\left[\phi\left(\frac{t_1-a}{b}\right) - \phi\left(\frac{t_2-a}{b}\right)\right].
\end{align}
where $\sigma_{\rm sl}=(s_2-s_1)/(t_2-t_1)$, $\sigma_{\rm itc}=s_1-\sigma_{\rm sl}t_1$, $\phi(x)=\exp(-x^2/2)/\sqrt{2\pi}$, $\Phi(x)= \int_{-\infty}^x\phi(t)\de t$, and $\Gamma$ is the Gamma function. To numerically optimize $\barR_d(\rad)$, we perform the following approximation:
\begin{align}
\inf_{\rad}\barR_d(\rad) \approx \inf_{p_i\geq 0, \, \sum_{i=1}^K p_i=1} \barR_d\left( \sum_{i=1}^K p_i \delta_{o_i} \right).
\end{align}
Here $o_i\in\mathbb{R},\;i=1,...,K,$ are $K$ pre-chosen points. Let $\mathbf{v}=(v(o_1),...,v(o_K))^\top$ and $\mathbf{U}=(u_d(o_i,o_j))_{1\leq i,j\leq K}$. Then the approximation becomes
\begin{align}
\inf_{\rad}\barR_d(\rad) \approx \inf_{p_i\geq 0, \, \sum_{i=1}^K p_i=1} \left\{1 + 2\mathbf{v}^\top\mathbf{p} + \mathbf{p}^\top\mathbf{U}\mathbf{p}\right\},
\end{align}
which is a quadratic programming problem and can be solved numerically. Here $\mathbf{v}$ can be computed easily with the explicit formula, and the computation of $\mathbf{U}$ amounts to numerically evaluating double integrals. In the case $d=\infty$, the computation of $\mathbf{U}$ is much easier, since
\begin{align}
u_\infty(r_1,r_2) = \frac{1}{2}g(0,(1+\Delta)r_1)g(0,(1+\Delta)r_2) + \frac{1}{2}g(0,(1-\Delta)r_1)g(0,(1-\Delta)r_2).
\end{align}

\noindent{\bf Details of Figure \ref{fig:R_r_isotropic} of the main text.} For the SGD simulation, we take $N=800$, with $\eps=3\times 10^{-3}$ and $\xi(t) = t^{-1/4}$. The weights are initialized as $(\bw_{i})_{i\leq N}\sim_{iid}\mathsf{N}(0,0.4^2/d\cdot \mathbf{I}_d)$. We compute the risk attained by the SGD by Monte Carlo averaging over $10^4$ samples. We take a single SGD run per $\Delta$, per $d$, and report the risk at iteration $10^7$.

For the approximate optimization of $\barR_d(\rad)$, we choose $K=100$, and $o_i$, $i=1,...,K$, being equally spaced on the interval [0.01, 10].

For the optimization of $\barR^{(1)}_d(r)$ (recalling Eq.~[\ref{eq:Isotropic}] in the main text), we approximate it with $\min_{i=1,...,K}\barR^{(1)}_d(o_i)$, for the above chosen $o_i$ and $K$.

We find that in general, one needs higher $\max_{i=1,...,K}o_i$ to produce accurate results for higher $\Delta$. For the chosen set of $o_i$'s, we choose to plot up until $\Delta=0.8$.

\noindent{\bf Further numerical simulations.} In Figure \ref{fig:num_iso_risk_all}, we extend Figure \ref{fig:R_r_isotropic} of the main text to include results for additional values of $d$. The setting remains the same.

This figure provides further support to the respective discussion in the main text. For the threshold values of $\Delta$ for which the minimum risk is achieved by a uniform distribution $\rho^{\rm unif}_{r_*}$ over a sphere of radius $\Vert\bw\Vert_2=r_*$ (see the main text around Eq.~[\ref{eq:Isotropic}], and Section \ref{subsubsec:num_iso_Lemma1}).

\begin{figure}[]
	\begin{center}
		\includegraphics[width=0.8\textwidth]{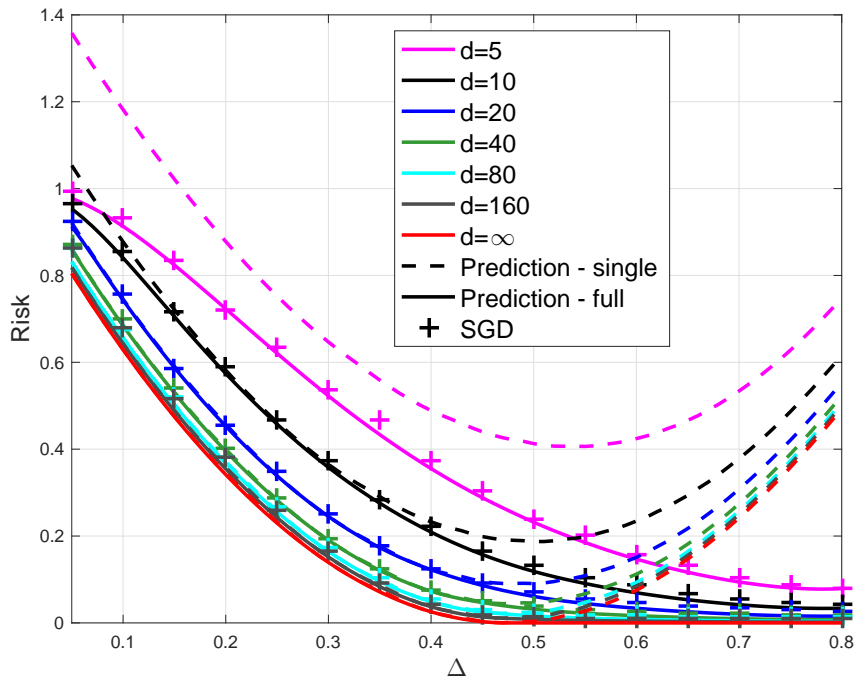}
	\end{center}
	\caption{The population risk as a function of $\Delta$ for different values of $d$, in the isotropic Gaussians example of Section \ref{subsec:num_iso}. Here ``Prediction - single'' refers to $\min_{r\geq 0}\barR_d^{(1)}(r)$, ``Prediction - full'' refers to the optimized $R(\rho)$ as described in Section \ref{subsubsec:num_isotropic_stat}, and ``SGD'' refers to the risk attained by the SGD.}
	\label{fig:num_iso_risk_all}
\end{figure}

\subsubsection{Checking the condition of Lemma \ref{lemma:OneDeltaCondition} in the main text}\label{subsubsec:num_iso_Lemma1}
We check of the condition of Lemma  \ref{lemma:OneDeltaCondition}  in the main text. This has two steps: (1) we solve for the minimizer $r_*$ of $\barR_d^{(1)}(r) = 1+2v(r)+u_d(r,r)$, where $v(r)$ and $u_d(r_1,r_2)$ are given by Eq. (\ref{eq:num_iso_v}) and (\ref{eq:num_iso_u}) respectively, and (2) we check whether $v(r)+u_d(r,r_*)\geq v(r_*)+u_d(r_*,r_*)$ for all $r\geq 0$. Figure \ref{fig:num_iso_riskSingleDelta} suggests that the behavior of $\barR_d^{(1)}(r)$ is rather benign and hence $r_*$ can be solved easily by searching for a local minimum. For the second step, we check the condition on a grid of values of $r$ from 0.1 to 10 with a spacing of 0.1, for each value of $\Delta$ on a grid from 0.01 to 0.99 with a spacing of 0.01. In general, we find that the conditioned is satisfied for $\Delta\in[\Delta_d^{\rm l}, \Delta_d^{\rm h}]$. Table \ref{tab:num_iso_Deltad} reports $\Delta_d^{\rm l}$ and $\Delta_d^{\rm h}$ for a number of values of $d$ for the isotropic Gaussians example with the given activation function.

\begin{figure}[]
	\begin{center}
		\includegraphics[width=1.0\textwidth]{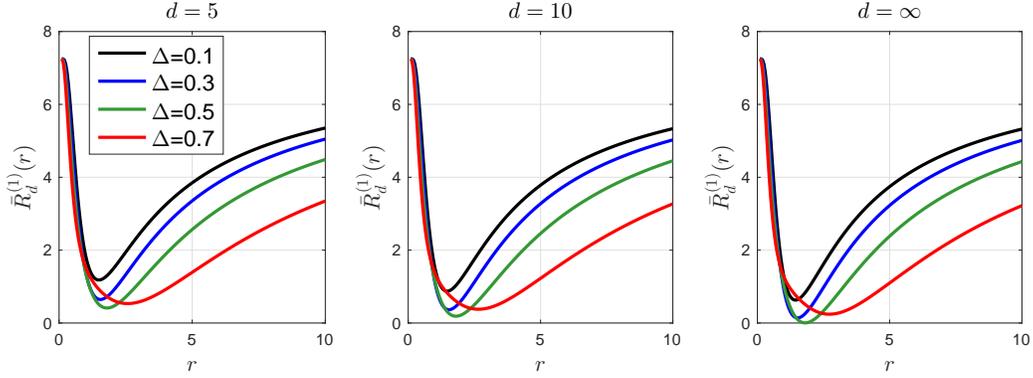}
	\end{center}
	\caption{The function $\barR_d^{(1)}(r)$ for different values of $d$ and $\Delta$, in the isotropic Gaussians example of Section \ref{subsec:num_iso}.}
	\label{fig:num_iso_riskSingleDelta}
\end{figure}

\begin{table}
	\centering	
	\begin{tabular}{|c|c|c|}	
		\hline
		$d$ & $\Delta_d^{\rm l}$ & $\Delta_d^{\rm h}$ \\
		\hline \hline
		5 & N/A & N/A \\
		10 & N/A & N/A \\
		20 & 0.08 & 0.38 \\
		40 & 0.03 &0.42 \\
		80 &  0.02 & 0.45 \\
		160 & 0.0 & 0.46 \\
		$\infty$ & 0 & 0.47 \\
		\hline
	\end{tabular}
	\caption{$\Delta_d^{\rm l}$ and $\Delta_d^{\rm h}$ for different values of $d$, in the isotropic Gaussians example of Section \ref{subsec:num_iso}. Here ``N/A'' refers to that no values of $\Delta$ are found to satisfy the condition of
 Lemma \ref{lemma:OneDeltaCondition} in the main text. Note that for $d=\infty$, the value $\Delta_\infty^{\rm l}=0$ is exact, according to Theorem \ref{thm:global_minimizer_infinite_d_isotropic}.}
	\label{tab:num_iso_Deltad}
\end{table}

\subsection{Centered anisotropic Gaussians with ReLU Activation}\label{subsec:num_ani}

In this section, we present details of the numerical experiments pertaining to the example of anisotropic Gaussians with ReLU activation. In particular, we use the activation $\sigma_*(\bx;\btheta) = a\max(\langle\bw,\bx\rangle + b, 0)$, with $\btheta = (\bw, a, b)\in\mathbb{R}^{d+2}$. We consider the centered anisotropic Gaussian case:
\begin{itemize}	
	\item[] With probability $1/2$: $y=+1$, $\bx\sim\normal(0,\mathbf{\Sigma}_+)$.
	\item[] With probability $1/2$: $y=-1$, $\bx\sim\normal(0,\mathbf{\Sigma}_-)$.	
\end{itemize}
More specifically, we opt for
\begin{align}
\mathbf{\Sigma}_+ &={\rm Diag}\big(\underbrace{(1+\Delta)^2,\dots,(1+\Delta)^2}_{s_0},\underbrace{1,\dots,1}_{d-s_0}\big)\, ,\\
\mathbf{\Sigma}_- &={\rm Diag}\big(\underbrace{(1-\Delta)^2,\dots,(1-\Delta)^2}_{s_0},\underbrace{1,\dots,1}_{d-s_0}\big)\, .
\end{align}
This setting is used in Figure \ref{fig:RiskAnisotropic} in the main text.

We consider $s_0=\gamma d$ for some $\gamma\in(0,1)$. For simplicity, we consider the limit $d\to\infty$. For $\btheta\sim\rho$, let $\rad$ be the joint distribution of four parameters $\mathbf{r} = (a,b,r_1=\Vert\bw_{1:s_0}\Vert_2,r_2=\Vert\bw_{(s_0+1):d}\Vert_2)$, where $\bw_{i:j}=(w_i,...,w_j)^\top$. Using a similar argument to Section \ref{sec:IsotropicGaussian}, we have, in the limit $d\to\infty$, the risk $R(\rho) = \barR_\infty(\rad)$ for
\begin{align}
\barR_{\infty}(\rad) &= \frac{1}{2}\left(1-\int aq_+(r_1,r_2,b) \rad(\de \mathbf{r})\right)^2+\frac{1}{2}\left(1+\int aq_-(r_1,r_2,b) \rad(\de \mathbf{r})\right)^2\, ,\\
q_\pm(r_1,r_2,b) &= b\, \Phi\left(\frac{b}{\sqrt{(1\pm\Delta)^2r_1^2+r_2^2}}\right) + \sqrt{(1\pm\Delta)^2r_1^2+r_2^2}\phi\left(\frac{b}{\sqrt{(1\pm\Delta)^2r_1^2+r_2^2}}\right)\, ,
\end{align}
where $\phi(x)=\exp(-x^2/2)/\sqrt{2\pi}$ and $\Phi(x)= \int_{-\infty}^x\phi(t)\de t$. Furthermore, letting $\rad_t$ denote the corresponding distribution at time $t$, the PDE [\ref{eq:GeneralPDE}] in the main text can be reduced to the following PDE of $\rad_t$:
\begin{align}
\partial_t\rad_t &= 2\xi(t)\nabla_\mathbf{r} \cdot \left( \rad_t\nabla_\mathbf{r}\psi_\infty(\mathbf{r};\rad_t) \right)\, ,\\
\psi_\infty(\mathbf{r}; \rad) &= \frac{1}{2} \left[ \int a'q_+(r_1',r_2',b') \de\rad(a',b',r_1',r_2') - 1 \right] aq_+(r_1,r_2,b) \nonumber\\
&\quad + \frac{1}{2} \left[ \int a'q_-(r_1',r_2',b') \de\rad(a',b',r_1',r_2') + 1 \right] aq_-(r_1,r_2,b).
\end{align}

\noindent{\bf PDE simulation.} As in Section \ref{subsubsec:num_isotropic_dyna}, we posit that the solution to the PDE can be approximated, at all time $t$, by the multiple-deltas ansatz:
\begin{align}
\rad_t = \frac{1}{J}\sum_{i=1}^J\delta_{\mathbf{r}_i(t)}\, ,
\end{align}
where $J\in\mathbb{N}$ is a pre-chosen parameter, and $\mathbf{r}_i(t)=(a_i(t),b_i(t),r_{1,i}(t),r_{2,i}(t))$. Following the same argument as in Section \ref{subsubsec:num_isotropic_dyna}, we obtain the following evolution equation:
\begin{align}
\frac{\de\phantom{t}}{\de t}\mathbf{r}_i(t) = -J\xi(t)\nabla_i \barR_{\infty,J}(\mathbf{r}_1(t),...,\mathbf{r}_J(t)),
\end{align}
for $i=1,...,J$, where $\barR_{\infty,J}(\mathbf{r}_1(t),\dots,\mathbf{r}_J(t)) = \barR_\infty(\rad_t)$ under the ansatz, and $\nabla_i$ denotes the gradient of $\barR_{\infty,J}(\mathbf{r}_1,...,\mathbf{r}_J)$ w.r.t. $\mathbf{r}_i$. More explicitly,
\begin{align}
\barR_{\infty,J}(\mathbf{r}_1,\dots ,\mathbf{r}_J) = \frac{1}{2}\left(1-\frac{1}{J}\sum_{i=1}^Ja_iq_+(r_{1,i},r_{2,i},b_i)\right)^2+\frac{1}{2}\left(1+\frac{1}{J}\sum_{i=1}^J a_iq_-(r_{1,i},r_{2,i},b_i)\right)^2.
\end{align}
Again, given $\mathbf{r}_i(t)$,  one approximates $\mathbf{r}_i(t+\delta t)$ for some small displacement $\delta t$ by
\begin{align}
\mathbf{r}_i(t+\delta t) \approx \mathbf{r}_i(t) -J\xi(t)\nabla_i \barR_{\infty,J}(\mathbf{r}_1,\dots,\mathbf{r}_J)\delta t.\label{eq:num_PDE_ani_discretize}
\end{align}

\noindent{\bf Details of Figure \ref{fig:RiskAnisotropic} of the main text.} For the SGD simulation, we take $d=320$, $s_0=60$, $N=800$, with $\eps=2\times10^{-4}$ and $\xi(t) = t^{-1/4}$. The weights are initialized as $(\bw_{i})_{i\leq N}\sim_{iid}\mathsf{N}(0,0.8^2/d\cdot \mathbf{I}_d)$, $(a_i)_{i\leq N}=1$ and $(b_i)_{i\leq N}=1$. We take a single SGD run. We compute the risk attained by the SGD by Monte Carlo averaging over $10^4$ samples.

To obtain results from the PDE, we take $J=400$. We initialize $r_{1,i}(0)=\Vert Z_{1,i}\Vert_2$ and $r_{2,i}(0)=\Vert Z_{2,i}\Vert_2$, where $(Z_{1,i})_{i\leq N}\sim_{iid}\mathsf{N}(0,0.8^2/d\cdot \mathbf{I}_{s_0})$ and $(Z_{2,i})_{i\leq N}\sim_{iid}\mathsf{N}(0,0.8^2/d\cdot \mathbf{I}_{d-s_0})$ independently, along with $a_i(0) = 1$, $b_i(0) = 1$. We obtain $\mathbf{r}_i(t)$ from $t=0$ until $t=10^7\eps$, by discretizing this interval with $10^5$ points equally spaced on the $\log_{10}$ scale and sequentially computing $\mathbf{r}_i(t)$ at each point using Eq. (\ref{eq:num_PDE_ani_discretize}). Note that the SGD result at iteration $\ell$ corresponds to $\mathbf{r}_i(\eps^{4/3} \ell)$. We take a single run of the PDE.

To produce the inset plot in Figure \ref{fig:RiskAnisotropic} of the main text, for the ``$a$ (mean)'' axis, we compute $\frac{1}{N}\sum_{i=1}^Na_i$ for the SGD and $\frac{1}{J}\sum_{i=1}^Ja_i(t)$ for the PDE. Similarly, for the ``$b$ (mean)'' axis, we compute $\frac{1}{N}\sum_{i=1}^Nb_i$ for the SGD and $\frac{1}{J}\sum_{i=1}^Jb_i(t)$ for the PDE, and for the ``$r_1$ (mean)'' axis, we compute $\frac{1}{N}\sum_{i=1}^N\Vert\bw_{i,1:s_0}\Vert_2$ for the SGD and $\frac{1}{J}\sum_{i=1}^Jr_{1,i}(t)$ for the PDE.

\noindent{\bf Further numerical simulations.} In Figure \ref{fig:num_ani_parameters}, we plot the evolution of the four parameters, for the same setting as Figure \ref{fig:RiskAnisotropic} of the main text. Here ``$a$ (mean)'', ``$b$ (mean)'' and ``$r_1$ (mean)'' hold the same meanings, and ``$r_2$ (mean)'' refers to $\frac{1}{N}\sum_{i=1}^N\Vert\bw_{i,(s_0+1):d}\Vert_2$ for the SGD and $\frac{1}{J}\sum_{i=1}^Jr_{2,i}(t)$ for the PDE.

In Figure \ref{fig:num_ani_risk_Delta6}, we plot the population risk's evolution for the same setting as Figure \ref{fig:RiskAnisotropic} of the main text, apart from that $\Delta=0.6$ and $s_0$ varies.

\noindent{\bf Comments.} We observe a good match between the SGD and the PDE in Figure \ref{fig:RiskAnisotropic} of the main text as well as Figure \ref{fig:num_ani_parameters}, up until iteration $10^6$. In general there is less discrepancy with larger $s_0$, $d$ and $N$, recalling that the PDE is computed assuming infinite $s_0$, $d$ and $N$. This is evident from Figure \ref{fig:num_ani_risk_Delta6}.

As a note, in Figure \ref{fig:num_ani_risk_Delta6}, the PDE evolves differently for different $s_0$. This is because the ratio $s_0/d$ is used to determine the initialization of the PDE.

\begin{figure}[]
	\begin{center}
		\begin{tabular}{ll}
			\includegraphics[width=0.5\textwidth]{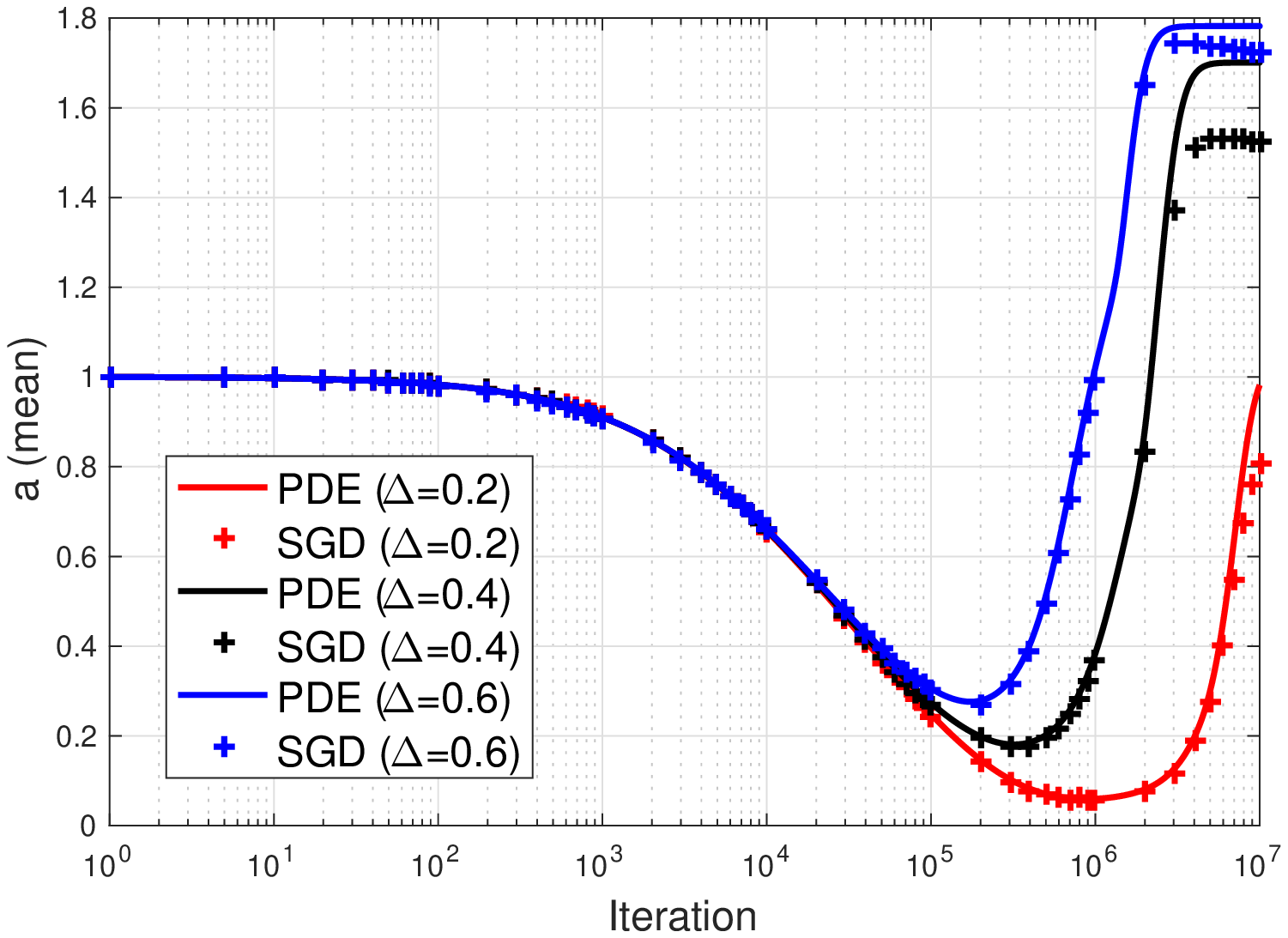}
			&
			\includegraphics[width=0.5\textwidth]{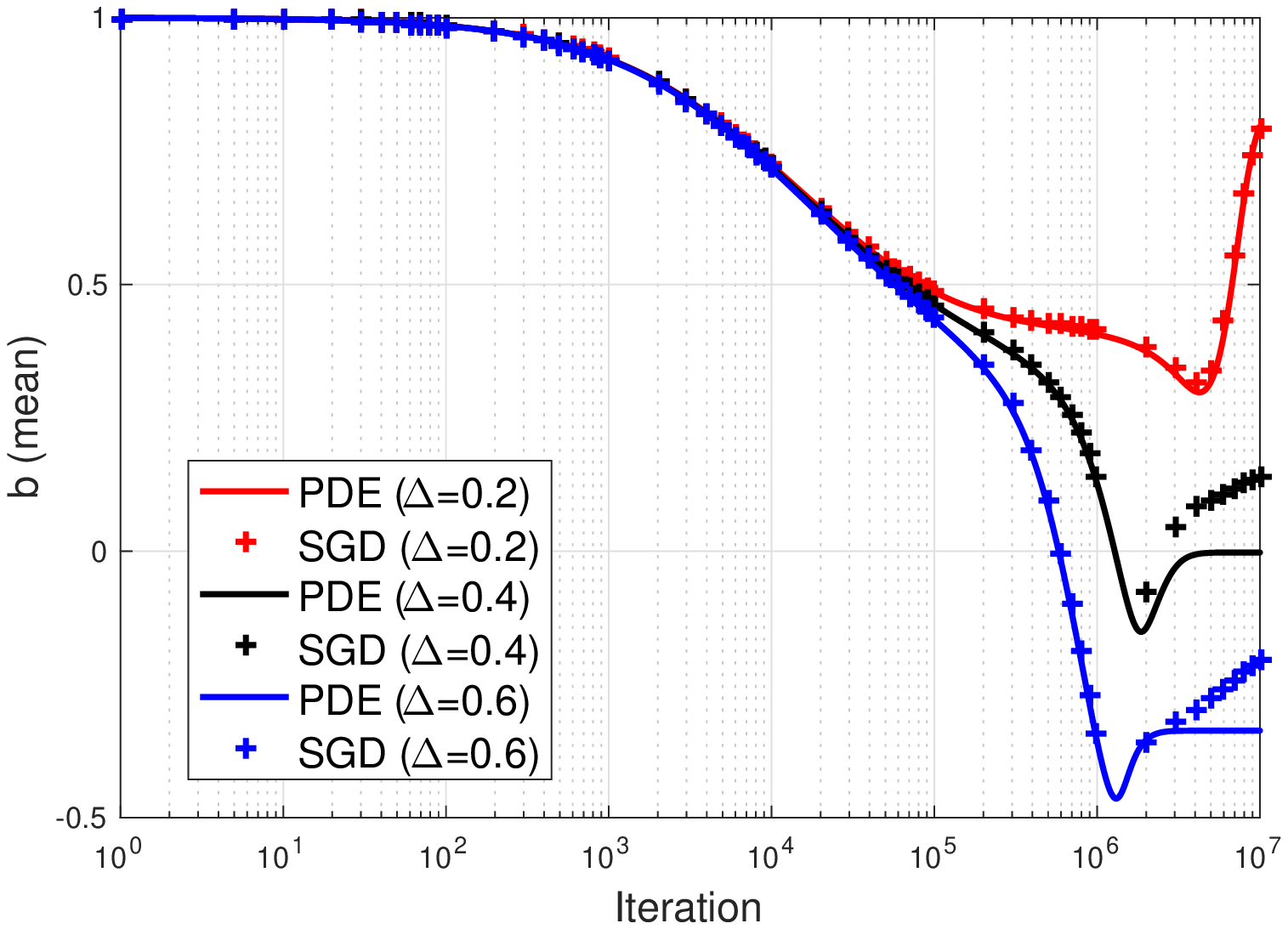}\\			
			\includegraphics[width=0.5\textwidth]{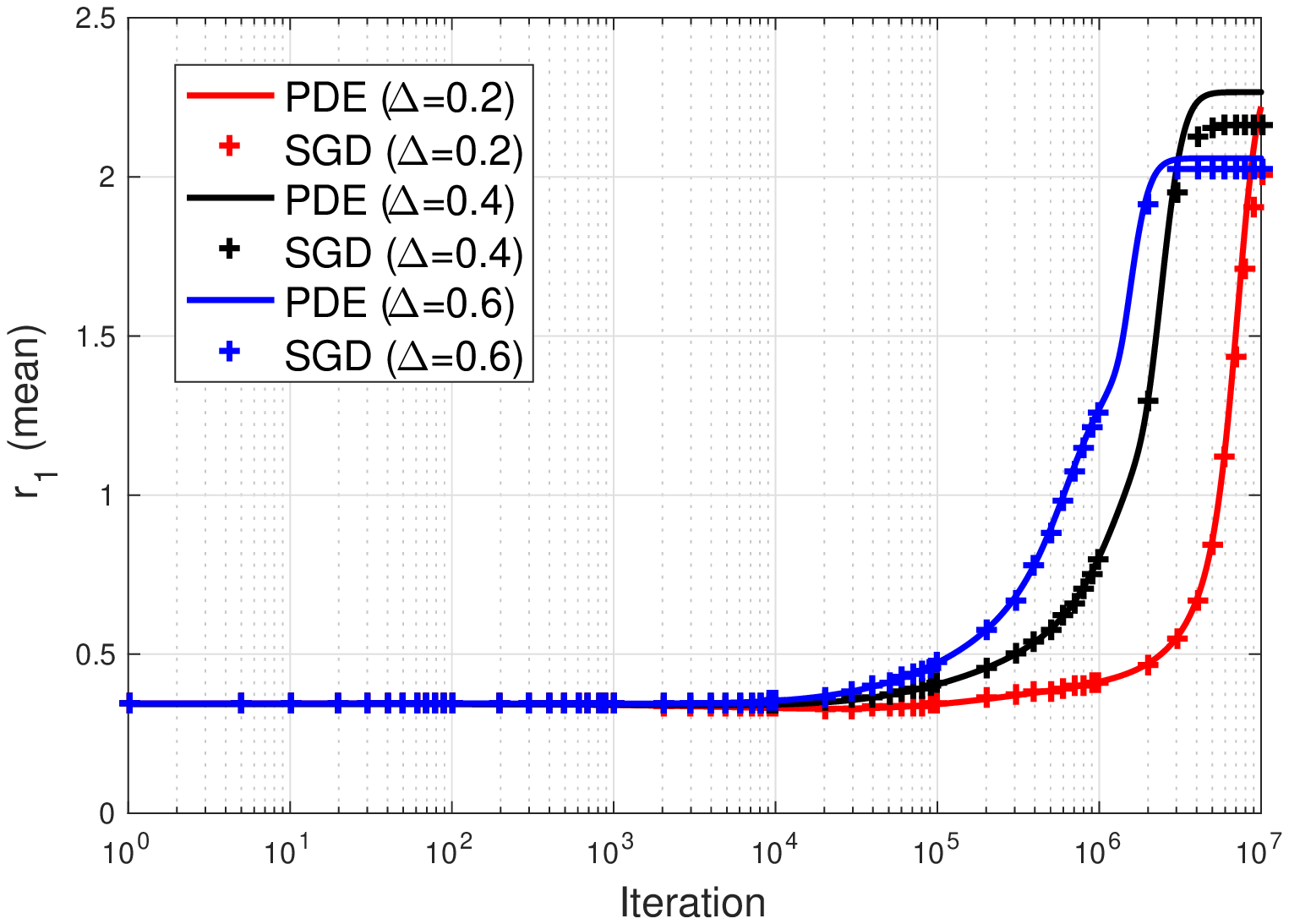}
			&
			\includegraphics[width=0.5\textwidth]{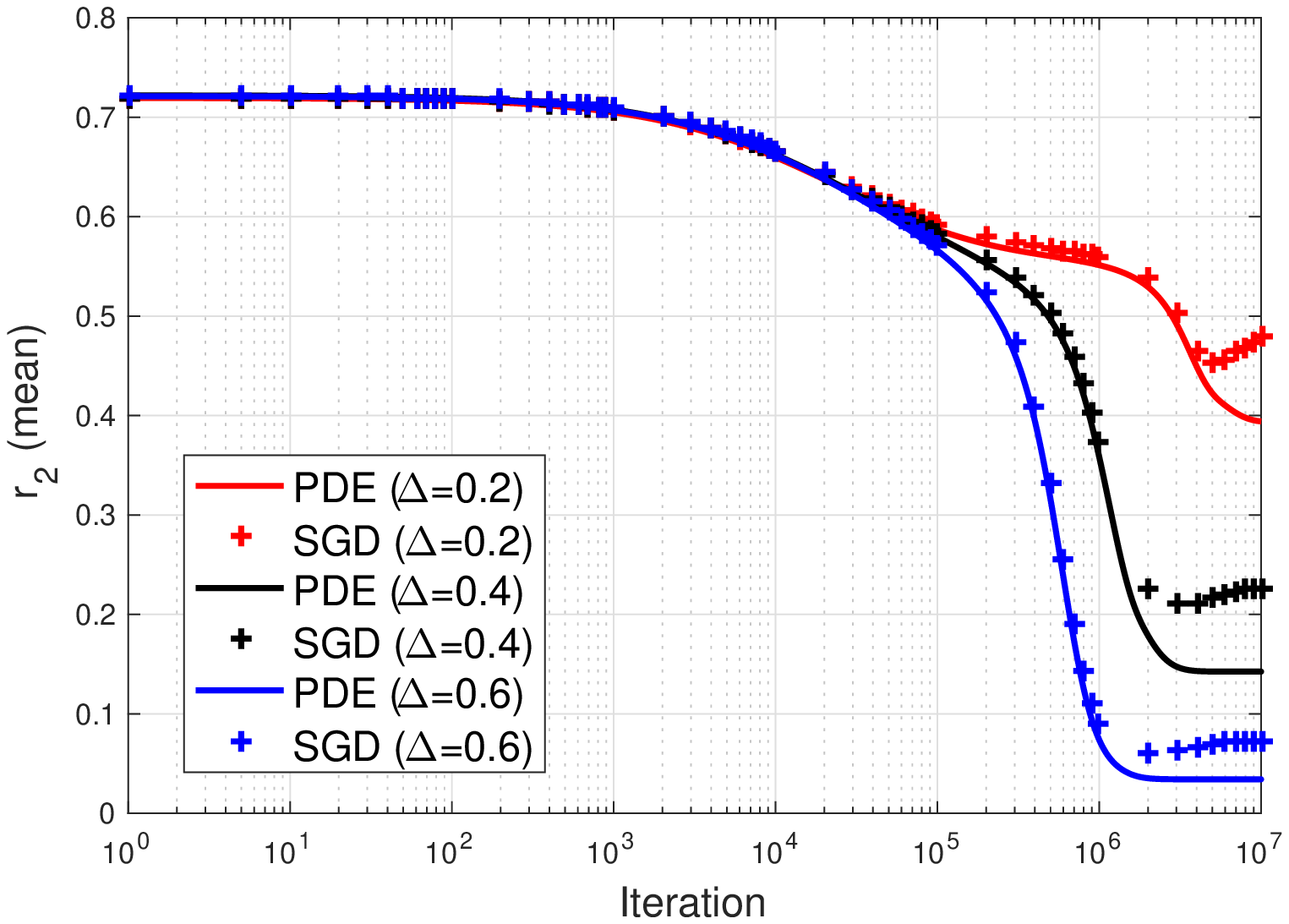}
		\end{tabular}
	\end{center}
	\caption{The evolution of the four parameters in the anisotropic Gaussians example of Section \ref{subsec:num_ani}.}
	\label{fig:num_ani_parameters}
\end{figure}

\begin{figure}[]
	\begin{center}
		\includegraphics[width=0.7\linewidth]{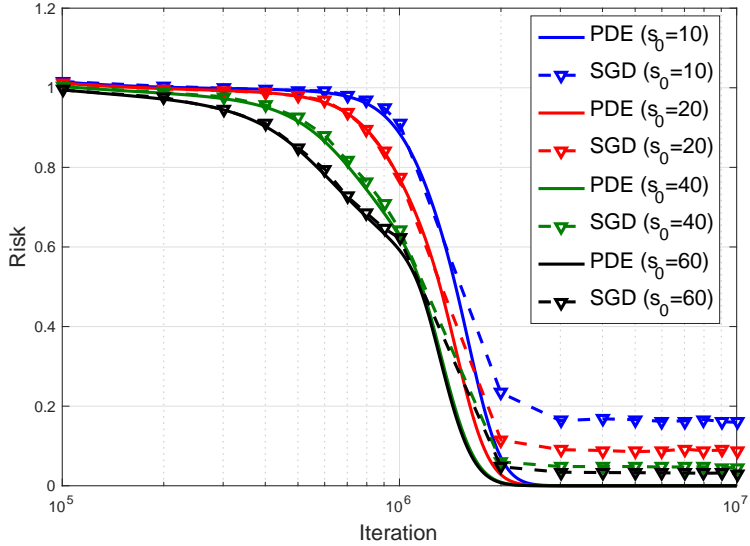}
	\end{center}
	\caption{The evolution of the population risk for $\Delta=0.6$, $d=320$, $N=800$ in the anisotropic Gaussians example of Section \ref{subsec:num_ani}.}
	\label{fig:num_ani_risk_Delta6}
\end{figure}

\subsection{Isotropic Gaussians: Predictable Failure of SGD}\label{subsec:num_fail}
In this section, we consider the isotropic Gaussians example (see Section \ref{subsec:num_iso} for the setting and notations), with the following activation function: $\sigma_*(\bx;\btheta_i) = \sigma(\langle\bw_i,\bx\rangle)$, where $\sigma(t) = -2.5$ for $t\leq 0$, $\sigma(t) = 7.5$ for $t\geq 1.5$, and $\sigma(t)$ linearly interpolates from the knot $(0,-2.5)$ to $(0.5,-4)$, and from $(0.5,-4)$ to $(1.5,7.5)$. This activation is plotted in Figure \ref{fig:num_sigmoid}. This corresponds to Section ``Predicting failure" and Figure \ref{fig:Failure} in the main text. The simulation of the PDE can be done in the same way as in Section \ref{subsubsec:num_isotropic_dyna}.

\noindent{\bf Rationale of the activation choice.} We give an explanation for the choice of the above activation based on our theory. We aim to find an activation $\sigma_*(\bx;\btheta_i) = \sigma(\langle\bw_i,\bx\rangle)$ in which there exists a local minimum that does not generalize well. To simplify the task, we wish for such minimum to be attained at $\rho_* = \delta_\bzero$. This minimum does not generalize well, since it implies all the weights are zero and the neuron outputs are constant, rendering the network unable to perform classification. Theorem \ref{thm:StabilityDelta} of the main text suggests taking $\sigma(t)$ such that
\begin{align}
\nabla^2V(\bzero) + \nabla^2_{1,1}U(\bzero,\bzero) \succ 0\, .
\end{align}
In the isotropic Gaussians case, this becomes
\begin{align}
\sigma''(0)\left\{ (1-\Delta)^2-(1+\Delta)^2 + \sigma(0)[(1-\Delta)^2+(1+\Delta)^2] \right\} > 0\, .
\end{align}
(Note that the condition $\nabla V(\bzero) + \nabla_{1}U(\bzero,\bzero) = \bzero$ in Theorem \ref{thm:StabilityDelta} of the main text is trivially satisfied.) Another requirement is that there should still be a minimum whose risk is nearly zero. Hence we do not wish for a dramatic change in the choice of the activation function, as compared to the one used in Section \ref{subsec:num_iso}. That is, we leave $\sigma(0)<0$ unchanged. Hence we would want $\sigma''(0)<0$, which is accomplished by our aforementioned choice.

Note that Theorem \ref{thm:StabilityDelta}  of the main text also suggests that if the SGD is initialized sufficiently close to this local minimum, the SGD trajectory should converge to it.

\noindent{\bf Details of Figure \ref{fig:Failure} of the main text.} For the data generation, we set $\Delta=0.5$. For the SGD simulation, we take $d=320$, $N=800$, with $\eps=10^{-5}$ and $\xi(t) = t^{-1/4}$. We take a single SGD run each for two different initializations: the weights are initialized as $(\bw_{i})_{i\leq N}\sim_{iid}\mathsf{N}(0,\kappa^2/d\cdot \mathbf{I}_d)$ for either $\kappa=0.1$ or $\kappa=0.4$. We compute the risk attained by the SGD by Monte Carlo averaging over $10^4$ samples.

To obtain results from the PDE, we take $J=400$, and generate $r_i(0)=\Vert Z_i\Vert_2$, where $(Z_i)_{i\leq N}\sim_{iid}\mathsf{N}(0,\kappa^2/d\cdot \mathbf{I}_d)$. We obtain $\br(t)$ from $t=0$ until $t=10^7\eps$, by discretizing this interval with $10^5$ points equally spaced on the $\log_{10}$ scale and sequentially computing $\mathbf{r}(t)$ at each point using Eq. (\ref{eq:num_PDE_iso_discretize}). Note that the SGD result at iteration $k$ corresponds to $\mathbf{r}(\eps^{4/3} k)$. We take a single run of the PDE.

To produce the inset plot, we compute $\frac{1}{N}\sum_{i=1}^N\Vert\bw_i\Vert_2$ for the SGD, and $\frac{1}{J}\sum_{i=1}^Jr_i(t)$ for the PDE.

As observed from Figure \ref{fig:Failure} of the main text, the SGD trajectory with $\kappa=0.1$ converges to a point where $\Vert\bw_i\Vert_2$ is nearly zero and the risk is very high, in stark contrast to the SGD trajectory with $\kappa=0.4$, as expected.

\noindent{\bf Error plot.} In Figure \ref{fig:num_fail_error}, we plot the empirical error rate attained by the SGD in the above example for the two initializations. Here the error rate is defined as the misclassification probability $\mathbb{P}\{{\rm sign}(\hat{y}(\bx;\btheta)) \neq y\}$, and is computed with Monte Carlo averaging over $10^4$ samples. This validates the claim that, in this example, there exists a local minimum which the SGD can converge to, yet has bad generalization (i.e. attains the trivial misclassification rate of 0.5), whereas there is a global minimum which the SGD can also find and yet generalizes well.

\begin{figure}[]
	\begin{center}
		\includegraphics[width=0.7\linewidth]{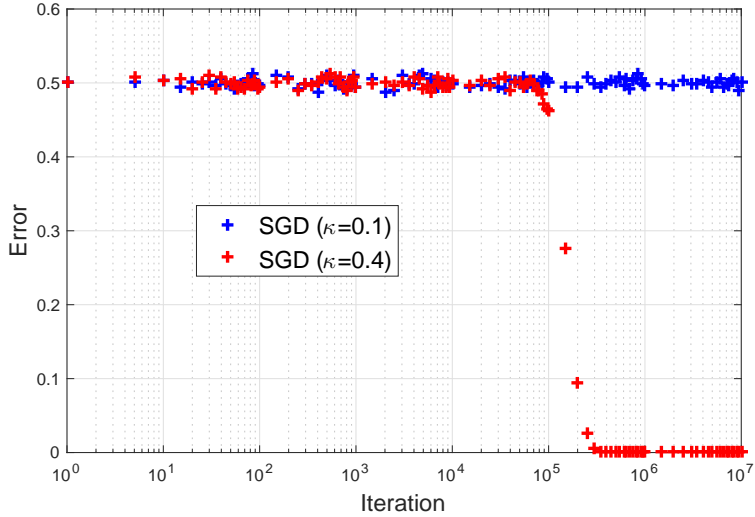}
	\end{center}
	\caption{The error rate attained by the SGD in the example of Figure \ref{fig:Failure} of the main text.}
	\label{fig:num_fail_error}
\end{figure}

\appendix

\section{Concentration inequalities}
\label{sec:Tools}

\begin{lemma}[Azuma-H\"oeffding bound]\label{lemma:AH}
Let $(\bX_k)_{k\ge 0}$, be a martingale taking values in $\reals^d$ with respect to the filtration $(\cF_k)_{k\ge 0}$, with $\bX_0= \bzero$.
Assume  that the following holds almost surely for all $k\ge 1$:
\begin{align}
\E\big\{ e^{\<\blambda,\bX_k-\bX_{k-1}\>}\big|\cF_{k-1}\big\}\le e^{L^2\|\blambda\|^2/2}\, . \label{eq:MG_diff}
\end{align}
Then we have
\begin{align}
\prob\Big(\max_{k\le n}\|\bX_k\|_2\ge 2 L\sqrt{n} (\sqrt{d} + t)\Big)\le \, e^{-t^2}\, .
\end{align}
\end{lemma}
\begin{proof}
Let $\bZ_k = \bX_k-\bX_{k-1}$ be the martingale differences. By the subgaussian condition (\ref{eq:MG_diff}), we get
\begin{align}
\E\big\{ e^{\<\blambda,\bX_{n}\>}\big\} &\le \E\Big\{\E\{e^{\<\blambda,\bZ_n\>}|\cF_{n-1}\}\, e^{\<\blambda,\bX_{n-1}\>}\Big\}\\
& \le e^{L^2\|\blambda\|^2/2}\, \E\big\{ e^{\<\blambda,\bX_{n-1}\>}\big\} \le e^{n L^2\|\blambda\|_2^2/2}\, .
\end{align}
Letting $\bG\sim\normal(0,\id_d)$ a standard Gaussian vector and $\xi\ge 0$,
\begin{align}
\E\big\{e^{\xi\|\bX_n\|_2^2/2}\big\} &=  \E_{\bG}\E\big\{e^{\sqrt{\xi}\<\bG,\bX_n\>}\big\} \le    \E_{\bG} e^{nL^2 \xi\|\bG\|_2^2/2}\\
& = \big(1-n L^2 \xi\big)^{-d/2}\, .
\end{align}
By Markov inequality, setting $\xi = 1/(2 n L^2)$, we get, for all $t \ge 0$,
\begin{align}
\prob\Big(\|\bX_n\|_2\ge 2 L\sqrt{n} (\sqrt{d} + t)\Big)&\le e^{d/2 - (\sqrt{d} + t)^2} \le e^{-t^2}. 
\end{align}

Finally, to upper bound $\max_{k\le n}\|\bX_k\|_2$, we define the stopping time $\tau\equiv \min\{k :\; \| \bX_k\|_2\ge 2L\sqrt{n} (\sqrt{d} + t) \}$,
and the martingale $\obX_k =\bX_{k\wedge \tau}$. Since $\{\max_{k\le n}\|\bX_n\|_2\ge  2L\sqrt{n} (\sqrt{d} + t) \} =\{\|\obX_n\|_2\ge  2L\sqrt{n} (\sqrt{d} + t) \} $, the claim follows by applying the previous inequality to $\obX_n$.
\end{proof}

\section{On the generalization to other loss functions}

The objective of this section is to show that the framework of this paper can be formally extended to other loss functions $\ell:\reals\times \reals\to\reals$. 
All arguments here will be heuristic, and we defer a rigorous study of this problem to future work.

First of all, we note that the population risk reads
\begin{align}
R_N(\btheta) = \E\left\{\ell\left(y,\frac{1}{N}\sum_{i=1}^N\sigma_*(\bx;\btheta_i)\right)\right\}\, ,
\end{align}
which naturally leads to the following mean field risk $R:\cuP(\reals^D)\to  \reals$:
\begin{align}
R(\rho) = \E\left\{\ell\left(y,\int\sigma_*(\bx;\btheta)\, \rho(\de\btheta)\right)\right\}\, .
\end{align}
The corresponding distributional dynamics is formally identical to the one for quadratic loss, cf. Eq.~(\ref{eq:GeneralPDE_App}).
The only change is in the definition of $\Psi(\btheta;\rho)$:
\begin{align}
\partial_t\rho_t(\btheta) =& 2 \xi(t) \nabla\cdot \big[\rho_t(\btheta)\nabla\Psi(\btheta;\rho_t)\big]\,, \label{eq:GeneralPDE_GenLoss}\\
\Psi(\btheta; \rho) =& \frac{\delta R(\rho)}{\delta\rho(\btheta)} =  \E\left\{\partial_2\ell\left(y,\int\sigma_*(\bx;\bar{\btheta})\, \rho(\de\bar{\btheta})\right) \, \sigma_*(\bx;\btheta)\right\} \,,
\end{align}
where $\partial_2\ell$ denotes the derivative of $\ell$ with respect to its second argument.
It is immediate to see that, for the quadratic loss $\ell(y,\hy) = (y-\hy)^2$, we recover the expressions used in the rest of the paper.

\end{document}